\documentclass[twoside,11pt]{article}

\usepackage{jmlr2e}

\usepackage{url}            %
\usepackage{booktabs}       %
\usepackage{amsfonts}       %
\usepackage{nicefrac}       %
\usepackage{microtype}      %

\usepackage{graphicx}
\usepackage{subfigure}
\usepackage{bm}
\usepackage{pdflscape}

\usepackage[ruled,vlined,noend]{algorithm2e}

\usepackage{amsmath}
\usepackage{amssymb}
\usepackage{xcolor}
\usepackage{mathtools}
\mathtoolsset{showonlyrefs}
\usepackage{sidecap}
\usepackage{wrapfig}
\usepackage{multirow}

\usepackage{mdframed}
\definecolor{theoremcolor}{rgb}{0.94, 0.97, 1.0}
\mdfsetup{
    backgroundcolor=theoremcolor,
    linewidth=0pt,
}

\usepackage{marvosym}

\usepackage{hyperref}
\ifdefined\nohyperref\else\ifdefined\hypersetup
  \definecolor{mydarkblue}{rgb}{0,0.08,0.45}
  \hypersetup{ %
    pdftitle={},
    pdfauthor={},
    pdfsubject={},
    pdfkeywords={},
    pdfborder=0 0 0,
    pdfpagemode=UseNone,
    colorlinks=true,
    linkcolor=mydarkblue,
    citecolor=mydarkblue,
    filecolor=mydarkblue,
    urlcolor=mydarkblue,
    pdfview=FitH}

  \ifdefined\isaccepted \else
    \hypersetup{pdfauthor={Anonymous Submission}}
  \fi
\fi\fi

\usepackage{xfrac}
\usepackage{comment}

\newcommand{\Otv}{{\operatorname{TV}}}
\newcommand{\Orof}{{\Omega}_{\gamma\text{ROF}}}
\newcommand\DP[2]{\langle #1,#2 \rangle}

\DeclareMathOperator*{\argmax}{arg\,max}
\DeclareMathOperator*{\argmin}{arg\,min}

\newcommand{\insertion}[1]{#1}

\ShortHeadings{Sparse Continuous Distributions and Fenchel-Young Losses}{Sparse Continuous Distributions and Fenchel-Young Losses}
\firstpageno{1}

\usepackage{lastpage}
\jmlrheading{23}{2022}{1-\pageref{LastPage}}{8/21; Revised
4/22}{7/22}{21-0879}{Andr\'e F.~T.~Martins, Marcos Treviso, Ant\'onio Farinhas, Pedro M.~Q.~Aguiar,  M\'ario A.~T.~Figueiredo, Mathieu Blondel, and Vlad Niculae}
\ShortHeadings{Sparse Continuous Distributions and Fenchel-Young Losses}{Martins, Treviso, Farinhas, Aguiar, Figueiredo, Blondel, and Niculae}

\begin{document}

\title{Sparse Continuous Distributions and Fenchel-Young Losses}

\author{%
  \name Andr\'e F.~T.~Martins \email andre.t.martins@tecnico.ulisboa.pt \\
  \addr Instituto de Telecomunica\c{c}\~oes, Instituto Superior T\'ecnico\\
  Lisbon ELLIS Unit (LUMLIS) \&
  Unbabel,
  Lisbon, Portugal
  \AND
  \name Marcos Treviso \email marcos.treviso@tecnico.ulisboa.pt \\
  \addr Instituto de Telecomunica\c{c}\~oes, Instituto Superior T\'ecnico,
  Lisbon, Portugal
  \AND
  \name Ant\'onio Farinhas \email antonio.farinhas@tecnico.ulisboa.pt \\
  \addr Instituto de Telecomunica\c{c}\~oes, Instituto Superior T\'ecnico,
  Lisbon, Portugal
  \AND
  \name Pedro M.~Q.~Aguiar \email aguiar@isr.ist.utl.pt \\
  \addr Instituto de Sistemas e Rob\'otica, Instituto Superior T\'ecnico\\
  Lisbon ELLIS Unit (LUMLIS),
  Lisbon, Portugal
  \AND
  \name M\'ario A.~T.~Figueiredo \email mario.figueiredo@tecnico.ulisboa.pt \\
  \addr Instituto de Telecomunica\c{c}\~oes, Instituto Superior T\'ecnico\\
  Lisbon ELLIS Unit (LUMLIS),
  Lisbon, Portugal
   \AND
  \name Mathieu Blondel \email mblondel@google.com \\
  \addr Google Research,
  Paris, France
  \AND
  \name Vlad Niculae \email v.niculae@uva.nl \\
  \addr Language Technology Lab, University of Amsterdam, The Netherlands
}

\editor{Sebastian Nowozin}

\maketitle

\begin{abstract}%
Exponential families are widely used in machine learning, including many
distributions in continuous and discrete domains (\textit{e.g.}, Gaussian,
Dirichlet, Poisson, and categorical distributions via the softmax
transformation). Distributions in each of these families have fixed support. In
contrast, for finite domains, %
recent work on sparse alternatives to softmax (\textit{e.g.}, sparsemax, $\alpha$-entmax, and fusedmax), %
has led to distributions with %
varying support. %

This paper develops sparse alternatives to continuous
distributions, based on
several technical contributions: 
First, we define
$\Omega$-regularized prediction maps and Fenchel-Young losses for arbitrary
domains (possibly countably infinite or continuous). For linearly parametrized families, we
show that minimization of Fenchel-Young losses is equivalent to moment matching
of the statistics, generalizing a fundamental property of exponential families.
When $\Omega$ is a Tsallis negentropy with parameter $\alpha$, we obtain ``deformed exponential families,'' which include $\alpha$-entmax and sparsemax ($\alpha=2$) as particular cases. %
For quadratic energy functions, %
the resulting densities are $\beta$-Gaussians, an instance of elliptical
distributions that contain as particular cases the Gaussian, biweight,
triweight, and Epanechnikov densities, and for which we derive closed-form expressions for the variance, Tsallis entropy, and Fenchel-Young loss.
When $\Omega$ is a total variation or Sobolev regularizer, we obtain a continuous version of the fusedmax.
Finally, we introduce continuous-domain attention mechanisms, deriving efficient
gradient backpropagation algorithms for $\alpha \in \{1,\sfrac{4}{3},
\sfrac{3}{2}, 2\}$.  Using these algorithms, we demonstrate our sparse continuous
distributions for attention-based audio classification and visual question
answering, showing that they allow attending to time intervals and compact
regions. %
\end{abstract}


\section{Introduction}\label{sec:intro}

Exponential families \citep{brown1986fundamentals,barndorff2014information}
are ubiquitous in statistics and machine learning.
They include many common
distributions, both in continuous (Gaussian, exponential, Dirichlet, ...) and
discrete (Poisson, Bernoulli, categorical, ...) domains.  They enjoy many useful
properties, such as the existence of conjugate priors (crucial in Bayesian
inference) and the classical Pitman-Koopman-Darmois theorem
\citep{pitman1936sufficient,darmois1935lois,koopman1936distributions}, which
states that, among families with {\bf fixed support} (independent of the
parameters), exponential families are the only having sufficient statistics of
fixed dimension for any number of i.i.d.\ samples.

There have been several efforts to further generalize exponential families.
\citet{grunwald_2004} introduced \textbf{generalized exponential families} as
maximum entropy distributions for generalized entropy functions.
Based upon these results, \citet{frongillo_2014} studied these distributions
from a convex duality perspective.
\citet{amari_2012} studied deformed exponential families, including their
entropy and canonical divergence.

More recently, there has been work with a focus on distributions with {\bf
varying and sparse support} over a finite domain. Examples include
\textit{sparsemax} \citep{Martins2016ICML}, \textit{entmax}
\citep{peters2019sparse,correia2019adaptively}, and \textit{fusedmax}
\citep{fusedmax}. They have been used for sparse differentiable dynamic programming
\citep{mensch2018differentiable} and for improving the interpretability of attention
mechanisms in neural networks \citep{bahdanau2014neural}.

A common task when it comes to probability distributions is to fit their
parameters to observed data. Unfortunately, unlike for exponential families,
maximum likelihood for generalized exponential families does
not always lead to a convex objective with respect to the parameters.
Proper scoring rules, which can be seen as primal-space Bregman divergences,
have been widely studied \citep{gneiting_2007,reid_composite_binary,vernet_2016}.
Typically, proper scoring rules are composed with a link function.
However, when the link function is non-invertible, which is the case with sparse
distributions, the resulting composite loss function can be non-convex
\citep{blondel2020learning}.
Based on convex duality arguments,
\citet{blondel2020learning} introduced \textbf{Fenchel-Young losses},
which can be seen as mixed-space Bregman divergences \citep[Theorem
1.1]{amari2016information}. Unlike with proper scoring rules, the link function,
called \textbf{regularized prediction map}, is not explicitly composed with the
loss but instead kept implicit. This leads to convex loss functions, even for
distributions with sparse support.

\paragraph{This paper.}

We extend sparse probability distributions and Fenchel-Young losses
to \textbf{infinite domains} (\textbf{continuous} or \textbf{countably
infinite}). Similarly to (and
generalizing) the free energy variational principle \citep{dayan1995helmholtz},
a convex regularizer $\Omega$, which can be regarded as a generalized
negentropy, induces a mapping from energy functions to probability
densities.
When $\Omega$ is a Tsallis negentropy \citep{Tsallis1988}, the resulting
densities are \textbf{deformed exponential families}. These families have been
studied in statistical physics and machine learning
\citep{naudts2009q,sears2010generalized,ding2010t} with most focus %
given
to heavy-tailed distributions. Our paper focuses instead on light and
zero-tailed distributions, which can be regarded as continuous counterparts of
sparsemax and entmax transformations.  We use this construction to obtain new
density families, called $\alpha$\textbf{-sparse families}, with sparse and varying
support,
including the \emph{truncated parabola/paraboloid} distributions and the wider
family of {\boldmath $\beta$}\textbf{-Gaussian} distributions (see
Figures~\ref{fig:misfit} and \ref{fig:beta_gaussians}).  In addition, we also
provide a continuous counterpart for the \insertion{discrete smoothing} fusedmax transformation
\citep{fusedmax} by designing a $\Omega$ that depends on the density derivative,
via Rudin-Osher-Fatemi and Sobolev regularization \citep{rof1992}.

We use our theoretical results above in two ways. First, we extend neural
attention mechanisms \citep{bahdanau2014neural} to continuous domains, making
them able to attend to continuous data streams and to domains that are
inherently continuous, such as visual scenes. Unlike traditional attention mechanisms,
ours are suitable for selecting compact regions, such as 1D-segments or
2D-ellipses, and we illustrate this fact on audio classification and visual
question answering tasks.  Second, we demonstrate the usefulness of
continuous-domain Fenchel-Young losses in a simple heteroscedastic regression
problem modeled with bounded noise \citep{d2013bounded}.

 To encourage reproducibility and further experimentation by the research community,
we release an %
easy-to-use Python package alongside our paper:
\url{https://github.com/deep-spin/sparse_continuous_distributions/}.

\paragraph{Previous papers.}

This paper builds upon two previously published papers: a journal paper \citep{blondel2020learning} and a shorter conference paper \citep{martins2020sparse}. The former introduced and analyzed Fenchel-Young losses for finite and combinatorial domains, with a focus on structured prediction, without considering non-finite probability spaces.
The latter focused on regularized prediction maps with Tsallis regularizers and sparse and continuous attention mechanisms,  but without considering Fenchel-Young losses.
This paper provides a comprehensive study of regularized prediction maps and Fenchel-Young losses for arbitrary measure spaces, including continuous and countably infinite domains, being a natural companion for \citet{blondel2020learning}. Besides a much more
thorough treatment of previously covered topics, this paper contributes entirely
new sections, including \S\ref{sec:fy_losses} on Fenchel-Young losses for arbitrary measure spaces and parametrized families, \S\ref{sec:elliptical} on elliptical distributions and $\beta$-Gaussians, and \S\ref{sec:fusedmax} on a continuous generalization of fusedmax.
We also provide additional properties of Tsallis regularized families in \S\ref{sec:tsallis} (Propositions~\ref{prop:entropy_normalizing} and \ref{prop:gradient_hessian_fy}) and more examples of sparse families in \S\ref{sec:sparsemax}, such as the sparse Poisson and the truncated Gaussian.
We derive closed form expressions for Fenchel-Young losses with several continuous densities (including $\beta$-Gaussians, in Proposition~\ref{prop:beta_gaussian_fy}) and demonstrate how to use our framework to fit continuous densities on data by Fenchel-Young loss minimization, not covered in the previous two papers.

\paragraph{Notation.}
Let $(S, \mathcal{A}, \nu)$ be a measure space,
where $S$ is a set, $\mathcal{A}$ is a $\sigma$-algebra, and $\nu$ is a measure.
We denote by $\mathcal{M}_+^1(S)$ the set of $\nu$-absolutely continuous probability measures. From the Radon-Nikodym theorem \citep[\S31]{halmos2013measure}, each element of $\mathcal{M}_+^1(S)$ is identified (up to equivalence within measure zero) with a probability density function $p: S \rightarrow \mathbb{R}_+$, with $\int_S p(t)\, d\nu(t) = 1$.
For convenience, we often drop $d\nu(t)$ from the integral. %
We denote the measure of $A\in  \mathcal{A}$ as
$|A| = \nu(A) = \int_{A} 1$, and
the support of a density $p \in \mathcal{M}_+^1(S)$ as $\mathrm{supp}(p) = \{t \in S \mid p(t) > 0\}$. %
Given $\phi:S \rightarrow \mathbb{R}^m$, %
we write expectations
as
$\mathbb{E}_p[\phi(t)] := \int_S p(t) \, \phi(t)$.  %
Finally, we define $[a]_+ := \max\{a, 0\}$.

\insertion{
Throughout the paper, we use the following definition of ``sparse densities'',\footnote{This should not be confused with sparsity-inducing distributions \citep{FigueiredoNIPS2001,TippingJMLR2001}.} %
which generalizes the notion of sparse vectors, recovered when $S$ is finite and $\nu$ is the counting measure. The concept is illustrated in Figure~\ref{fig:sparse_densities}. %
\begin{definition}[Sparse density.] %
\label{def:sparse_function}
Let $(S, \mathcal{A}, \nu)$ be a measure space. A density $p:S \rightarrow \mathbb{R}$ is called {\it sparse} if $\nu(S \setminus \mathrm{supp}(p)) > 0$. It is called {\it dense} otherwise.
\end{definition}

\begin{figure*}[t]
\centering
\includegraphics[width=.45\textwidth]{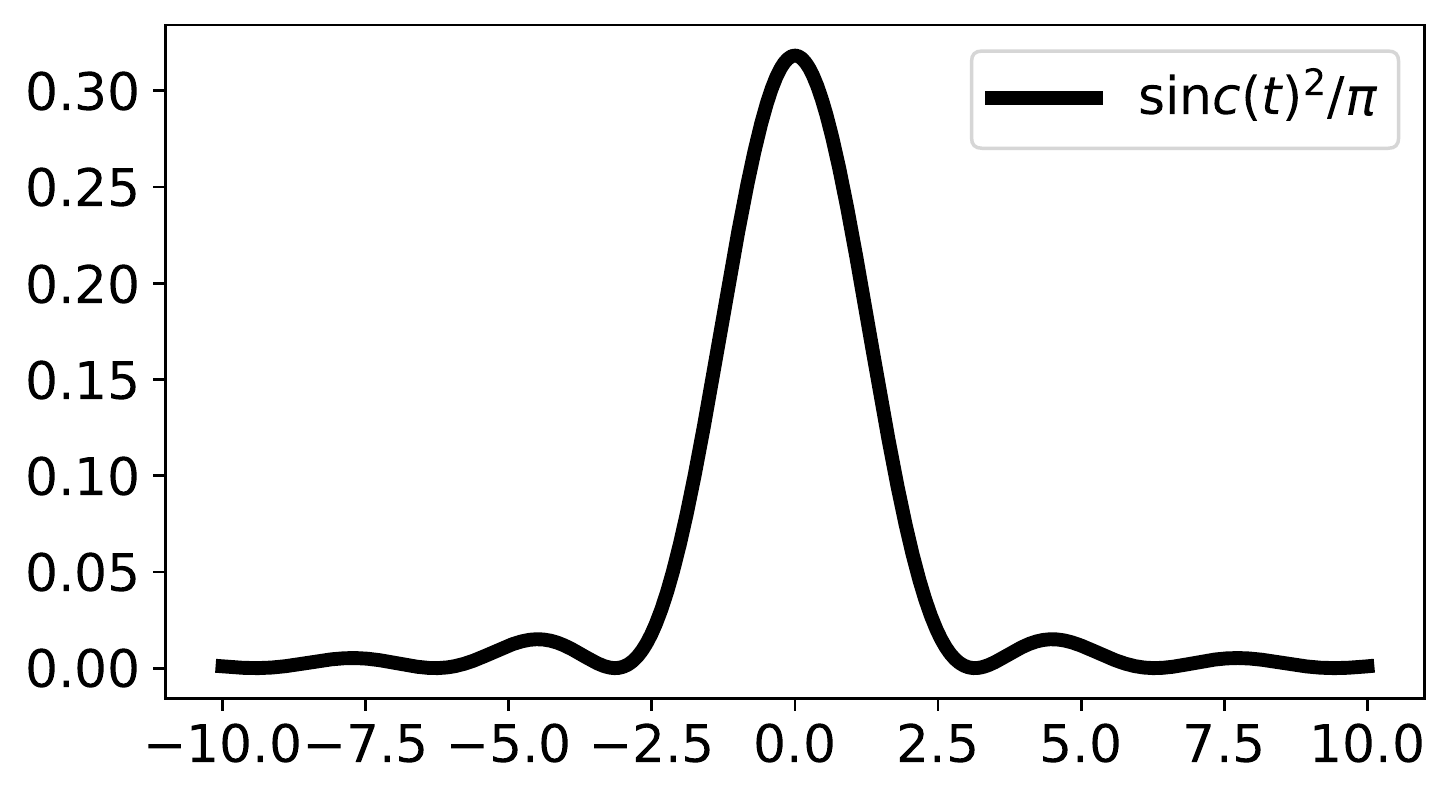}\quad
\includegraphics[width=.45\textwidth]{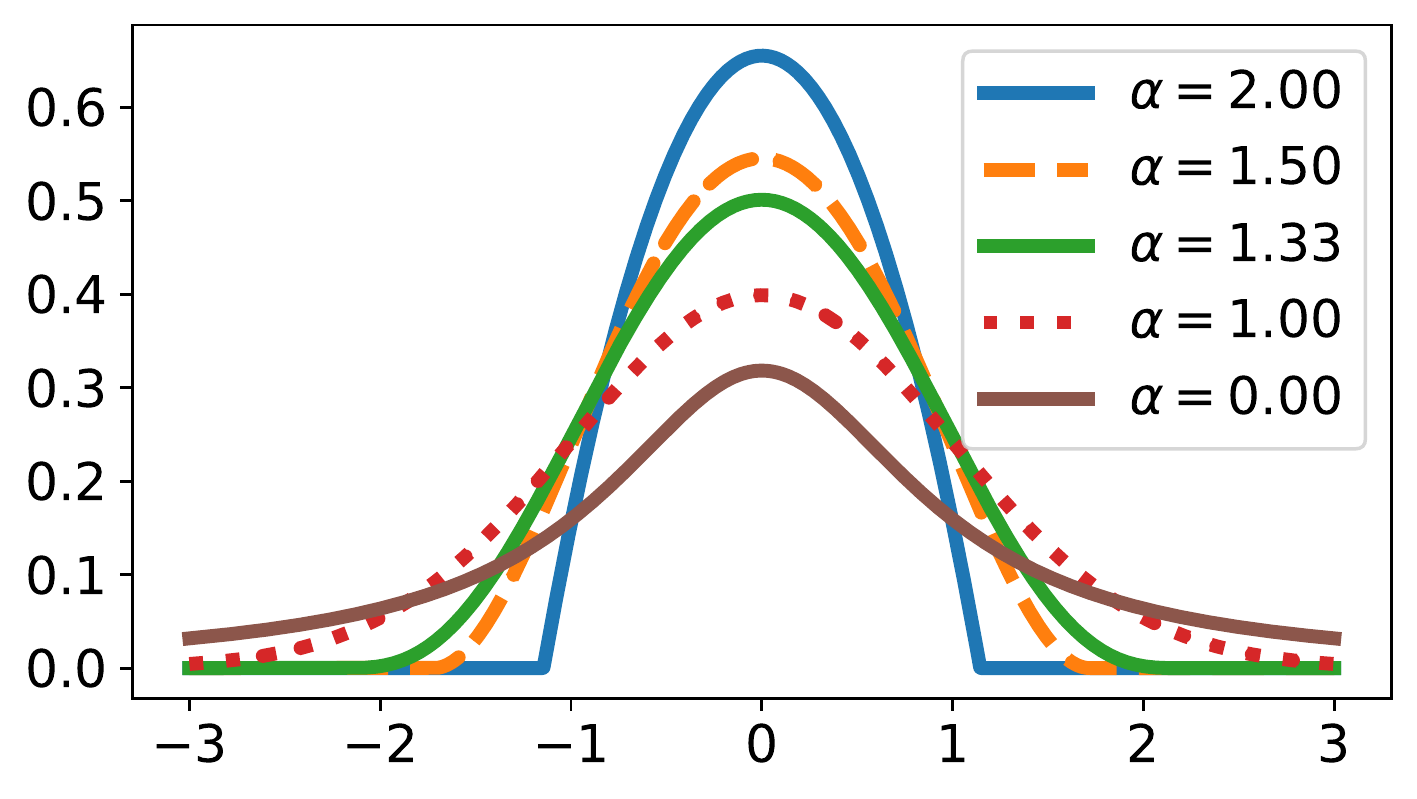}%
\caption{{\bf Non-sparse and sparse densities for $S=\mathbb{R}$.} Left: The density $p(t) = \sin^2(t)/(\pi t^2)$ is non-sparse, since it has only a countable number of zeros and therefore the set $\mathbb{R} \setminus \mathrm{supp}(p)$ has null measure. Right: Univariate $\beta$-Gaussians $\mathcal{N}_\beta(t, 0, \sigma^2)$ for several values of $\alpha=2-\beta$ (see \S\ref{sec:elliptical} for details). We used $\sigma^2=1$ except for $\alpha=0$, for which $\sigma^2=(2\pi)^{-1}$ (Cauchy distribution). $\alpha=1$ corresponds to a Gaussian, $\alpha<1$ to heavy-tail distributions ($t$-Student), and $\alpha>1$ to zero-tail distributions, recovering scaled versions of the biweight ($\alpha=\tfrac{3}{2}$), triweight ($\alpha=\tfrac{4}{3}$), and Epanechnikov kernels ($\alpha=2$, same as truncated parabola) used in density estimation. For $\alpha>1$, the case of focus in our paper, all these densities are sparse. \label{fig:sparse_densities}
}
\end{figure*}
}

\paragraph{Table of contents.} The rest of the paper is organized as follows. Figure~\ref{fig:diagram} helps navigating through the different sections.

\begin{itemize}%

\item[\S\ref{sec:rpm}] Regularized Prediction Maps

\item[\S\ref{sec:fy_losses}] Continuous Fenchel-Young Losses

\item[\S\ref{sec:tsallis}]
Tsallis Regularizers and Deformed Exponential Families

\item[\S\ref{sec:sparsemax}] Infinite Sparsemax

\item[\S\ref{sec:elliptical}]
Elliptical Distributions and $\beta$-Gaussians

\item[\S\ref{sec:fusedmax}] Continuous Fusedmax

\item[\S\ref{sec:attention}] Continuous Attention Mechanisms

\item[\S\ref{sec:experiments}] Experiments

\item[\S\ref{sec:related}] Related work.

\end{itemize}

\begin{figure}[t]
\centering
\includegraphics[width=1\columnwidth]{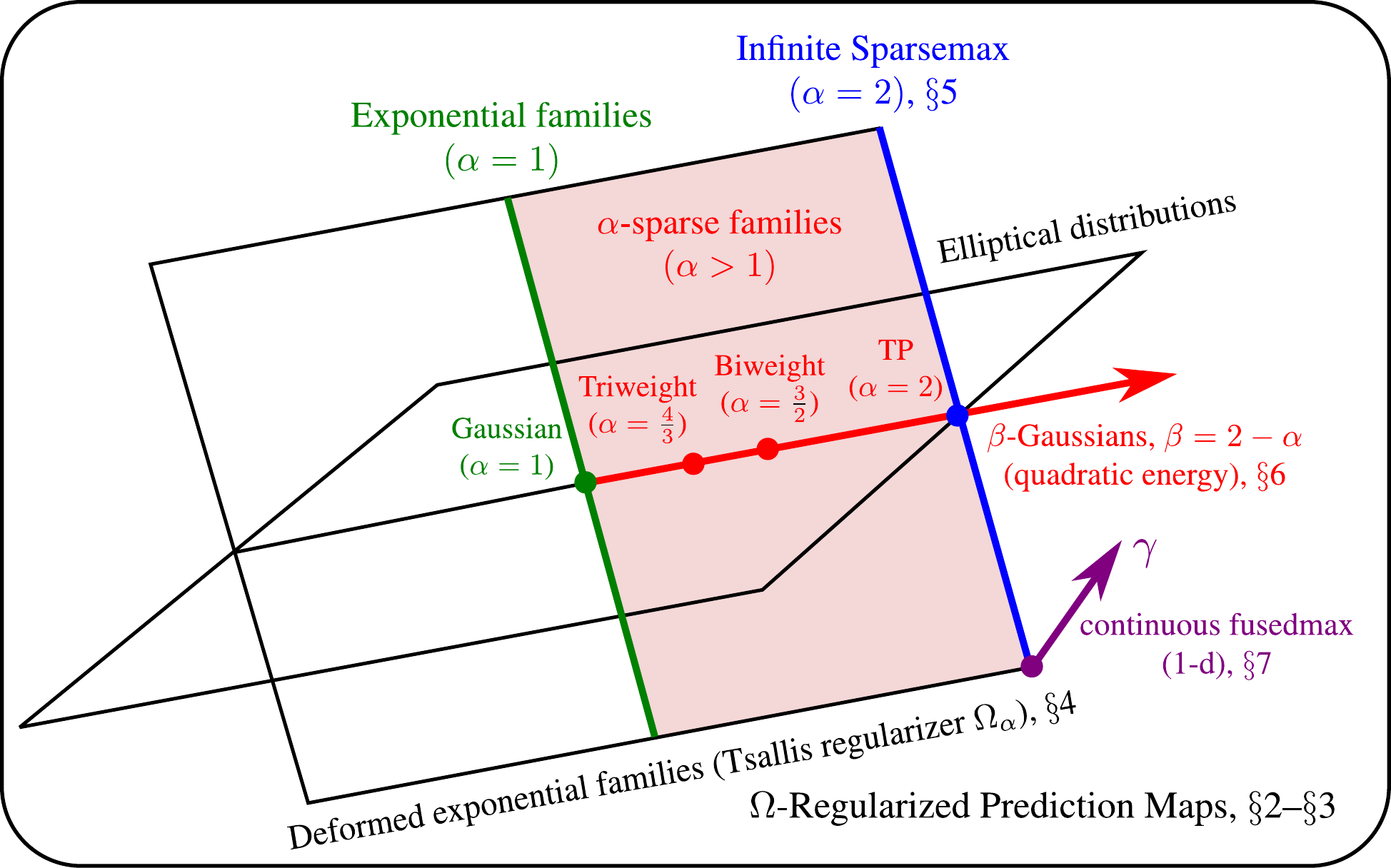}
\caption{\textbf{Diagram representing $\Omega$-regularized prediction maps (\S\ref{sec:rpm}--\S\ref{sec:fy_losses}) and some of its particular cases covered in this paper.}
$\beta$-Gaussian distributions (\S\ref{sec:elliptical}) lie at the intersection of elliptical distributions and deformed exponential families, corresponding to quadratic energies, and they include the Gaussian and Truncated Paraboloid (TP) distributions as particular cases. Exponential families and (infinite) sparsemax distributions (\S\ref{sec:sparsemax}) are a particular case of deformed exponential families (\S\ref{sec:tsallis}) for $\alpha\in \{1,2\}$; examples of such distributions for different energy functions are given in Table~\ref{tab:distributions}. Fusedmax distributions (\S\ref{sec:fusedmax}) extend 1-d sparsemax by incorporating total variation or Sobolev regularizers.}
\label{fig:diagram}
\end{figure}

\section{Regularized Prediction Maps}\label{sec:rpm}

The crux of this paper is the notion of $\Omega$-regularized prediction maps, which have been introduced by \citet{blondel2020learning} for finite domains $S$, and which we generalize here to arbitrary measure spaces. We will show in the sequel that these maps generalize the free energy variational principle \citep{dayan1995helmholtz}.

\insertion{\subsection{Warm-up: Finite domains}}

Let us start with the finite case, $S = [K] = \{1, \ldots, K\}$. We consider the following problem: given a vector of real numbers $f \in \mathbb{R}^{K}$, convert them into a probability vector $p \in \triangle^{K}$, where $\triangle^{K} := \{p \in \mathbb{R}^{K} \mid p \ge 0, \,\, p^\top 1 = 1\}$ denotes the probability simplex. For example, $f$ could be a vector of label scores (or ``logits'') computed by a neural network classifier, and $p$ the corresponding label probabilities.
The idea behind regularized prediction maps is to smooth the argmax operator with a convex regularizer $\Omega : \triangle^{|S|} \rightarrow \mathbb{R}$ which encourages uniform distributions:
\begin{equation}\label{eq:rpm_finite}
\hat{p}_\Omega[f] = \argmax_{p \in \triangle^{|S|}} p^\top f - \Omega(p).
\end{equation}
This operator can be regarded as the gradient map of the smoothed max function
$\max_\Omega(f) := \max_{p  \in \triangle^{|S|}} p^\top f - \Omega(p)$
\citep{nesterov_smooth,beck_2012,fusedmax}.
Without any regularization ($\Omega \equiv 0$), we obtain the argmax transformation, where the maximizer $p^\star$ in \eqref{eq:rpm_finite} becomes a one-hot vector.
Non-trivial choices of $\Omega$ recover well-known transformations such as \textbf{softmax} \citep{bridle1990probabilistic}, and recently proposed ones, including \textbf{sparsemax} \citep{Martins2016ICML}, \textbf{fusedmax} \citep{fusedmax}, and \textbf{entmax} \citep{peters2019sparse}.
We will cover these transformations in this paper and we will show in the
subsequent sections how they can be extended to arbitrary infinite measure spaces (countably infinite or continuous).

\insertion{\subsection{Extension to infinite domains}}

\insertion{In several practical applications, the domain $S$ is not finite: For
example, it can be a continuous space, such as $\mathbb{R}$ or $\mathbb{R}^N$, or a countably infinite set such as $\mathbb{N}$. To accommodate this in an unified manner,} we need to consider the space of probability densities $\mathcal{M}_+^1(S)$ instead of the probability simplex $\triangle^{|S|}$.
Our definition below extends regularized prediction maps to \textit{arbitrary} measure spaces $S$. Instead of a finite vector $f \in \mathbb{R}^K$, we assume now a \textbf{scoring function}  $f: S\rightarrow \mathbb{R}$.
\begin{definition}[$\Omega$-\textbf{regularized prediction map}.]
\label{def:regularized_prediction}
Let  $\Omega: \mathcal{M}_+^1(S) \rightarrow \mathbb{R}$ be a lower semi-continuous (l.s.c.),  proper, and strictly convex function.
The $\Omega$-{regularized prediction map} $\hat{p}_{\Omega}: \mathcal{F} \rightarrow \mathcal{M}_+^1(S)$ is defined as
\begin{equation}\label{eq:reg_prediction}
\hat{p}_{\Omega}[f]
= \argmax_{p \in \mathcal{M}_+^1(S)} \mathbb{E}_{p}[f(t)] - \Omega(p)
= \argmax_{p \in \mathcal{M}_+^1(S)} \int_S p(t) \, f(t)\, d\nu(t) - \Omega(p),
\end{equation}
where $\mathcal{F}$ is the set of functions for which the maximizer above exists and is unique.
\end{definition}

\begin{figure*}[t]
\centering
\includegraphics[width=0.4\textwidth]{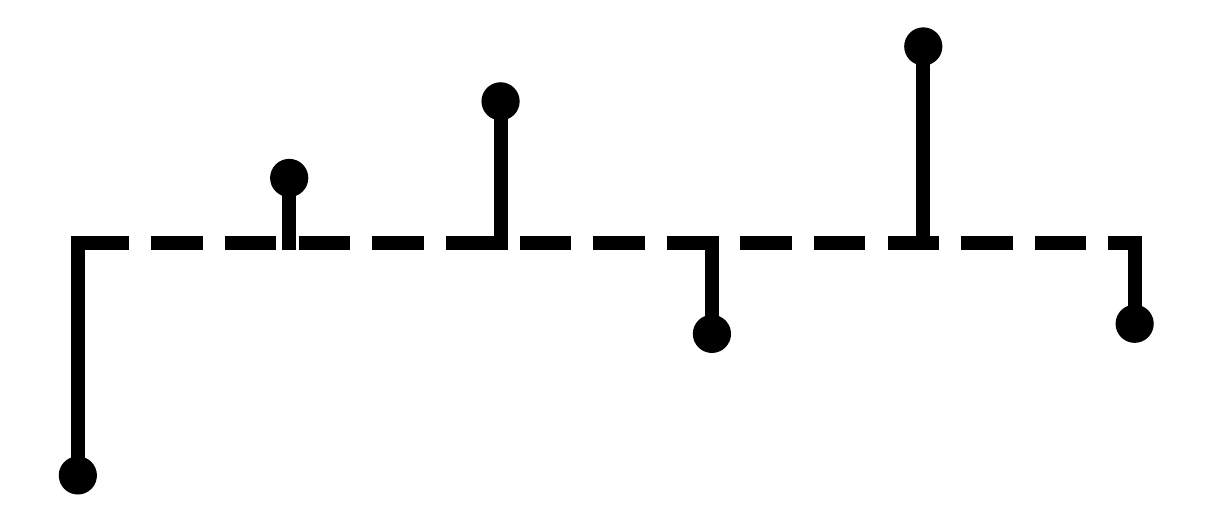}
\,\,
\includegraphics[width=0.4\textwidth]{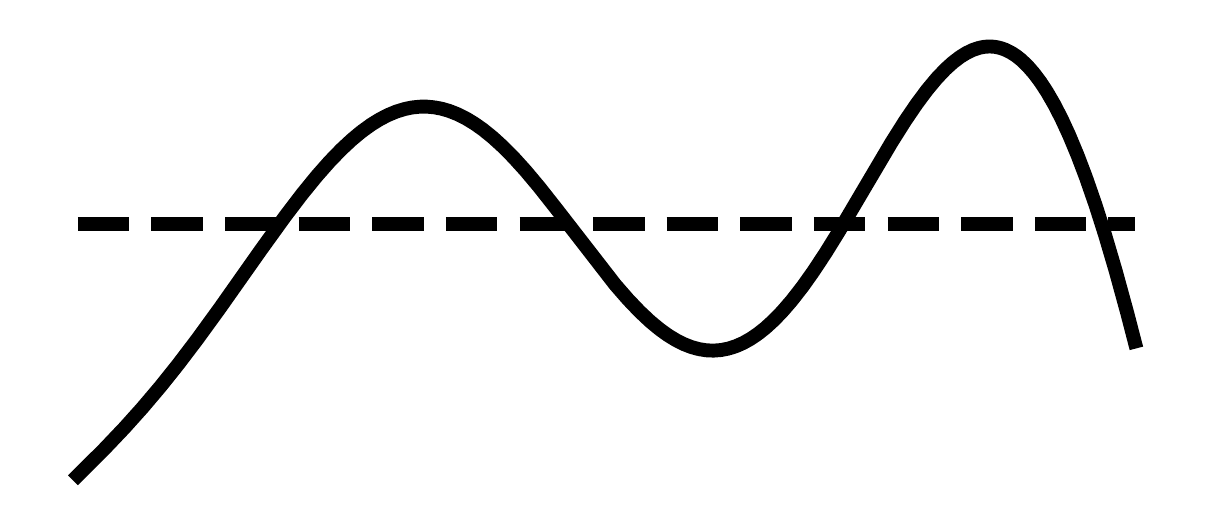}\\
\includegraphics[width=0.4\textwidth]{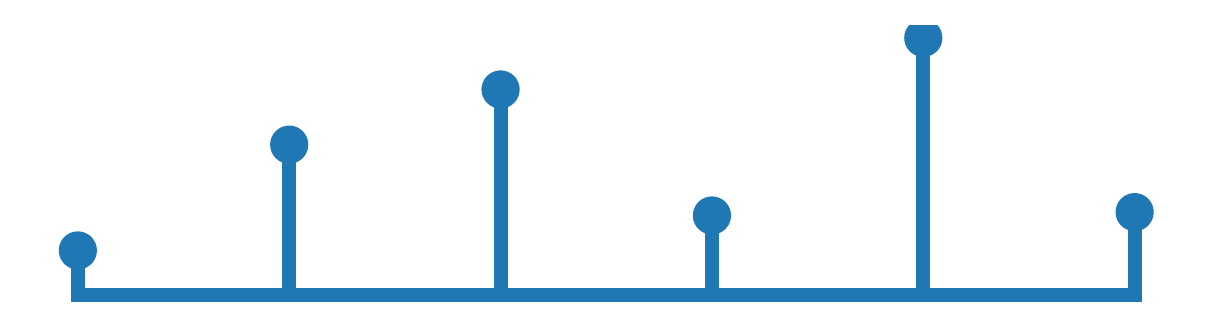}
\,\,
\includegraphics[width=0.4\textwidth]{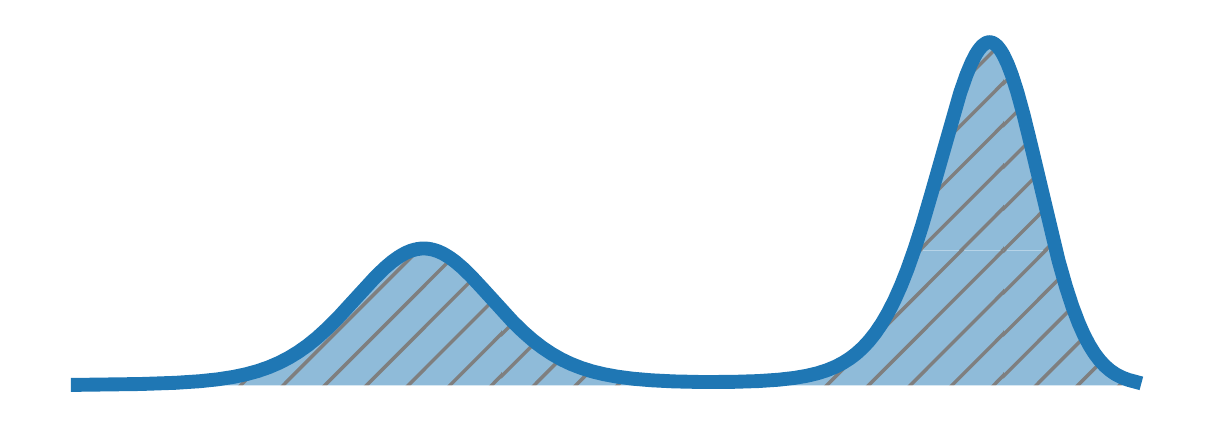}
\caption{\textbf{Discrete and continuous $\Omega$-regularized prediction maps}. For each case, we show the scoring function $f$ (top) and corresponding distribution $\hat{p}_\Omega[f]$ (bottom) when $\Omega$ is the Shannon-Boltzmann-Gibbs entropy. Left: finite $S$. Right: $S=\mathbb{R}$.\label{fig:rpm_example}}
\end{figure*}

Figure~\ref{fig:rpm_example} provides an illustration.

\paragraph{Properties.}
$\Omega$-regularized prediction maps enjoy several important properties; see \citet[Proposition~1]{blondel2020learning} for the finite $S$ case. For example, %
they are
insensitive to the addition of constants, both to the regularizer $\Omega$ and to the function $f$. That is, $\hat{p}_{\Omega} \equiv \hat{p}_{\Omega+c}$ for any $c\in \mathbb{R}$ and  $\hat{p}_{\Omega}[f] = \hat{p}_{\Omega}[g]$ if $g(t) = f(t) + c$. The former follows immediately from \eqref{eq:reg_prediction}, and the latter follows from the fact that $\mathbb{E}_{p}[f(t) + c] = \mathbb{E}_{p}[f(t)] + c$.
For the continuous case $S=\mathbb{R}^N$ and if the regularizer $\Omega$ is separable ({\it i.e.} if it can be written as $\Omega(p) = \int_S \psi(p(t))$ for some function $\psi: \mathbb{R}_+ \rightarrow \mathbb{R}$) -- which is always the case in this paper -- we also have the following equivariance property: if $\tilde{f}(t) := f(At + b)$ for a matrix $A$ with determinant $\pm 1$, then $\hat{p}_\Omega[\tilde{f}](t) = \hat{p}_\Omega[f](At + b)$. This includes equivariance with respect to translations and orthogonal transformations as particular cases. See Appendix~\ref{sec:proof_equivariance_fy} for a proof.

\paragraph{Low temperature limit.}
It is often convenient to consider a ``temperature parameter'' $\tau>0$, absorbed into $\Omega$ via $\Omega := \tau \tilde{\Omega}$. If $f$ has a unique global maximizer $t^\star$, the low-temperature limit yields ${ \lim_{\tau\rightarrow 0}}\, \hat{p}_{\tau \tilde{\Omega}}[f] = \delta_{t^\star}$,
a Dirac delta distribution at the maximizer of $f$.
For finite $S$, this is simply the {\it argmax} transformation. %
More interesting examples of regularization functionals are shown in the next sections.

\insertion{\subsection{Examples}\label{sec:rpm_examples}}

\paragraph{Shannon-Boltzmann-Gibbs entropy.}

If we interpret $-f(t)$ as an energy function and choose as regularizer the Shannon-Boltzmann-Gibbs negentropy%
\footnote{This includes as particular cases the {Shannon negentropy} when $\nu$ is the counting measure for discrete $S$, and the {differential negentropy} when $\nu$ is the Lebesgue measure for continuous $S$. Shannon-Boltzmann-Gibbs negentropies are however more general and they can work with arbitrary measures.} %
$\Omega(p) = \int_S p(t)\log p(t) d\nu(t)$, we recover the well-known {\bf free energy variational principle} \citep{dayan1995helmholtz}.
\insertion{In that case, the quantity $-U_\Omega(p; f) := -\mathbb{E}_p[f(t)] + \Omega(p)$ corresponds to the {\bf Helmholtz free energy}, and $\hat{p}_\Omega$ is its minimizer \citep{hinton1993autoencoders}.}
With this choice, the solution of the optimization problem \eqref{eq:reg_prediction} is a  Boltzmann-Gibbs distribution (\citealt{cover2012elements}; see Appendix~\ref{sec:diff_ent_exp_family} for a proof):
\begin{equation}\label{eq:boltzmann}
\hat{p}_\Omega[f](t) = \frac{\exp(f(t))}{\int_S \exp(f(t'))d\nu(t')} = \exp \bigl(f(t) - A(f)\bigr),
\end{equation}
where
$A(f) := \log \int_S \exp(f(t))$
is the log-partition function.
Some particular cases are:
\begin{itemize}
\item If $S$ is finite and $\nu$ is the counting measure, the integral in
    \eqref{eq:boltzmann} is a summation and we can write $f$ as a vector $[f_1,
    \ldots, f_{|S|}] \in \mathbb{R}^{|S|}$. In this case, the
    $\Omega$-regularized prediction map is the {\bf softmax transformation},
\begin{equation}
\hat{p}_\Omega[f] = \mathrm{softmax}(f) = \tfrac{\exp(f)}{\sum_{k=1}^{|S|} \exp(f_k)} \in \triangle^{|S|}.
\end{equation}
The vector $\hat{p}_\Omega[f]$ parameterizes a {\bf categorical} distribution in this case.

\item If $S=\mathbb{N}$, $\nu(A)$ the counting measure, and $f(t) =
t\log \lambda - \log(t!)$ for
$\lambda > 0$, we obtain a {\bf Poisson} distribution, $\mathrm{Pr}\{t = k\} =
\frac{\hat{p}_\Omega[f](k)}{k!} = \frac{\lambda^k \exp(-\lambda)}{k!}$, with
$\Omega(\hat p_\Omega[f]) = -\lambda(1-\log\lambda) - \exp(-\lambda) \sum_{t=0}^\infty \frac{\lambda^t\log(t!)}{t!}$.%
\footnote{It is also possible to obtain a Poisson distribution by letting $\nu(A) = \sum_{t \in A} \frac{1}{t!}$ for $A \subseteq \mathbb{N}$, $f(t) = t\log \lambda$ for $\lambda > 0$, and $\Omega(p) = \sum_{k=0}^\infty \frac{p(k) \log p(k)}{k!}$. The formulation above, however, is more convenient for sparse generalizations, as we shall see.} %

\item If $S=\mathbb{R}^N$, $\nu$ is the Lebesgue measure, and %
$f(t) = -\frac{1}{2} (t-\mu)^\top\Sigma^{-1}(t-\mu)$ for $\mu \in \mathbb{R}^N$ and $\Sigma \succ 0$,
we obtain a {\bf multivariate Gaussian}, $\hat{p}_{\Omega}[f](t) =
\mathcal{N}(t; \mu, \Sigma) = (2\pi)^{-\sfrac{N}{2}} |\Sigma|^{-\sfrac{1}{2}} \exp(-\frac{1}{2}(t-\mu)^\top\Sigma^{-1}(t-\mu))$,
with differential negentropy $\Omega(\hat p_\Omega[f]) = -\frac{1}{2}\log \det (2\pi e \Sigma)$.
This becomes a {\bf univariate Gaussian} $\mathcal{N}(t; \mu, \sigma^2)$ if
$N=1$.

\item For $S=\mathbb{R}$ and defining $f(t) = -|t-\mu|/b$ for  $\mu \in
    \mathbb{R}$ and $b>0$ we get a {\bf Laplace} density,
    $\hat{p}_{\Omega}[f](t) = \tfrac{1}{2b}\exp\left(-|t-\mu|/b\right)$, with
    differential negentropy $\Omega(\hat p_\Omega[f]) = -\log(2b e)$.
\end{itemize}

These distributions are summarized in Table~\ref{tab:distributions} (rows with $\alpha=1)$.

\paragraph{Sparsity-inducing regularizers.}
Other regularizers $\Omega$ have been considered for the finite case. Choosing the Gini entropy $\Omega(p) = \tfrac{1}{2}\|p\|_2^2 - \frac{1}{2}$ (equivalent to $\ell_2$-regularization) leads to the \textbf{sparsemax} transformation \citep{Martins2016ICML}, and Tsallis entropy regularizers lead to \textbf{entmax} \citep{peters2019sparse}, covered later in this paper.
These regularizers are able to promote sparse probability mass functions. However,
their development has been limited
 so far to finite domains. In this paper, we generalize sparsemax and entmax to continuous domains (\S\ref{sec:tsallis}--\S\ref{sec:sparsemax}).
For entmax, we draw a new connection with elliptical distributions \citep{fang} when $f(t)$ is a quadratic scoring function
(\S\ref{sec:elliptical}). In this case, the $\Omega$-regularized prediction map leads to a generalization of multivariate Gaussian distributions called $\beta$-Gaussians, which can have bounded support and relate to some well-known density estimation kernels \citep{epanechnikov1969non,silverman1986density}. %

\begin{table*}
\caption{\label{tab:distributions}Distributions induced by
    $\Omega_\alpha$-regularized prediction maps for several scoring functions $f$ for
    finite, countably infinite, and continuous domains. We show the cases
    $\alpha\in \{1,2\}$ \insertion{($\alpha=1$ corresponds to the Shannon-Boltzmann-Gibbs regularizer, covered in \S\ref{sec:rpm_examples}, whereas $\alpha=2$ corresponds to the Gini regularizer, covered in \S\ref{sec:sparsemax}).} We denote by $\|t-\mu\|^2_{\Sigma^{-1}} \!:=\!
    (t\!-\!\mu)^\top \Sigma^{-1} (t\!-\!\mu)$ the squared Mahalanobis distance
between $t$ and $\mu$.
The sparse Poisson and truncated paraboloid
are new distributions presented in a unified manner with this framework.}
\begin{center}
\scriptsize
\begin{tabular}{llllll}
\toprule
Name & $S$ & $f(t)$ & $\alpha$ &
$\hat{p}_{\Omega_\alpha}[f]$ & $\Omega_\alpha(\hat{p}_{\Omega_\alpha}[f])$ \\
\midrule
Categorical (softmax) &
\multirow{2}{*}{$[K]$} & \multirow{2}{*}{$f_t$} &
$1$ &
$\tfrac{\exp(f)}{\sum_{t=1}^{K} \exp(f_t)}$ & $\sum_{t=1}^K p_t \log p_t$ \\
Sparsemax &
&&
$2$ &
$[f_t - \tau]_+$ & $\frac{1}{2}\left(\sum_{t=1}^K p_t^2 - 1\right)$ \\
\midrule
Poisson &
\multirow{2}{*}{$\mathbb{N}$} & \multirow{2}{*}{$t \log \mu + \log(1/t!)$} &
$1$ &
$\mu^t \exp(-\mu) / t!$ &
$-\mu(1-\log \mu) - \exp(-\mu) \sum_{t=0}^{\infty} \frac{\mu^t \log(t!)}{t!}$ \\
Sparse Poisson &
&&
$2$ &
$[f(t) - \tau]_+$ &
$\frac{1}{2}\left(\sum_{t=0}^\infty [f(t) - \tau]_+^2 - 1\right)$
\\
\midrule
Gaussian &
\multirow{2}{*}{$\mathbb{R}$} & \multirow{2}{*}{$-\frac{(t - \mu)^2}{2\sigma^2}$} &
$1$ &
$\mathcal{N}(t; \mu, \sigma^2)$ & $-\sfrac{1}{2}\log (2\pi e \sigma^2)$ \\
Truncated Parabola &
&&
$2$ &
$\left[f(t) - \tau \right]_+$ & $-\frac{1}{2} +
\frac{1}{5}\left(\frac{3}{2\sigma}\right)^{2/3}$ \\
\midrule
Laplace &
\multirow{2}{*}{$\mathbb{R}$} & \multirow{2}{*}{$-\frac{|t - \mu|}{b}$} &
$1$ &
$\frac{1}{2b}\exp\left(-\frac{|t - \mu|}{b}\right)$ & $-\log(2b e)$ \\
Triangular &
&&
$2$ &
$\left[f(t) - \tau\right]_+$ & $-\frac{1}{2} +  \frac{1}{3\sqrt{b}}$ \\
\midrule
Multivariate Gaussian &
\multirow{2}{*}{$\mathbb{R}^N$} & \multirow{2}{*}{
$-\frac{1}{2}\|t-\mu\|^2_{\Sigma^{-1}}$
} &
$1$ &
$\mathcal{N}(t; \mu, \Sigma)$ & $-\sfrac{1}{2}\log \det (2\pi e \Sigma)$ \\
Truncated Paraboloid &
&&
$2$ &
$[f(t) - \tau]_+$ &
$-\tfrac{1}{2} + \tfrac{2}{N+4}\left( \frac{\Gamma\left(\tfrac{N}{2}+2\right)
}{(2\pi)^{\frac{N}{2}} |\Sigma|^{\frac{1}{2}}} \right)^{\tfrac{2}{2+N}}$ \\
\bottomrule
\end{tabular}
\end{center}
\end{table*}%

\begin{table*}
\caption{\label{tab:linear_param}Linear parametrization $f_\theta(t) = \theta^\top \phi(t)$ for the
scoring function $f(t)$ of common distributions.
We further require $\mu > 0$ for the (sparse) Poisson distribution.
For the Laplace and
Triangular distributions, we assume the location $\mu$ known (fixed).
As is standard, for Gaussians, $f(t)$ is only linear in $\theta$ up to a
constant. In \S\ref{sec:elliptical}, we use the quadratic form directly.
}
\begin{center}
\small
\begin{tabular}{lllll}
\toprule
Distribution & $S$ & $f(t)$ & $\theta$ & $\phi(t)$ \\
\midrule
Categorical &
\multirow{2}{*}{$[K]$} &
\multirow{2}{*}{$f_t$} &
\multirow{2}{*}{$[f_1,\dots,f_K]$} &
\multirow{2}{*}{$e_t$} \\
Sparsemax & & & & \\
\midrule
Poisson, Sparse Poisson & $\mathbb{N}$ & $t \log \mu + \log(1/t!)$
                        & $[\log \mu, 1]$ & $[t, \log(1/t!)]$ \\
\midrule
Gaussian &
$\mathbb{R}$ &
\multirow{3}{*}{$-\frac{(t-\mu)^2}{2 \sigma^2}$} &
\multirow{3}{*}{$[\frac{\mu}{\sigma^2}, -\frac{1}{2\sigma^2}]$} &
\multirow{3}{*}{$[t, t^2]$} \\
Truncated Parabola & $\mathbb{R}$ & & & \\
Sparse Integer Gaussian & $\mathbb{Z}$ & & & \\
\midrule
Laplace &
\multirow{2}{*}{$\mathbb{R}$} &
\multirow{2}{*}{$- \frac{|t-\mu|}{b}$} &
\multirow{2}{*}{$[-\frac{1}{b}]$} &
\multirow{2}{*}{$[|t-\mu|]$} \\
Triangular & & & & \\
\midrule
Multivariate Gaussian &
\multirow{2}{*}{$\mathbb{R}^N$} &
\multirow{2}{*}{$-\frac{1}{2} \|t - \mu\|^2_{\Sigma^{-1}}$} &
\multirow{2}{*}{$[\Sigma^{-1}\mu, -\frac{1}{2} \mathrm{vec}(\Sigma^{-1})]$} &
\multirow{2}{*}{$[t, \mathrm{vec}(t t^\top)]$} \\
Truncated Paraboloid & & & & \\
\bottomrule
\end{tabular}
\end{center}
\end{table*}%

\paragraph{Total variation regularizer.}
Also for the finite case,
\citet{fusedmax} proposed \textbf{fusedmax}, which corresponds to the regularizer
$\Omega(p) = \tfrac{1}{2}\|p\|_2^2 + \sum_{k=1}^{|S|-1} |p_{k+1} - p_k|$, inspired by the fused lasso \citep{tibshirani2005sparsity}.
Besides sparsity, this regularizer encourages the same probability value in contiguous elements. We generalize fusedmax to continuous domains in \S\ref{sec:fusedmax}, by replacing the finite difference $|p_{k+1} - p_k|$ by the \textbf{derivative} $|p'(t)|$, leading to Rudin-Osher-Fatemi and Sobolev regularizers.

\paragraph{Linearly parametrized families of scoring functions.}

Definition~\ref{def:regularized_prediction} is fully general concerning the class of functions $\mathcal{F}$ from which $f$ can be chosen. In practice, it is often useful to consider  finite-dimensional parametrized function classes.
The simplest way to do this is via linear functions $f_{\theta}(t) = \theta^\top
\phi(t)$, %
where $\phi(t) \in \mathbb{R}^M$ is a vector of \textbf{statistics} and $\theta \in \Theta \subseteq \mathbb{R}^M$ is a vector of \textbf{canonical parameters}.%
\footnote{More generally, we can write $f_\theta(t) = \theta^\top\phi(t) + c(\theta) + d(t)$ where $c$ and $d$ are functions. However, these extra terms can also be handled by absorbing $c(\theta)$ into the normalization constant or %
$d(t)$ into the base measure.} %
A family of the form \eqref{eq:boltzmann} parametrized by $\theta \in \Theta$ is called an \textbf{exponential family} \citep{barndorff2014information}.
All the examples above (the categorical distribution with the softmax transformation, the Poisson with parameter $\lambda$, the Gaussian with parameters $\mu$ and $\Sigma$, and the Laplace with fixed $\mu$ and parameter $b$) are instances of exponential families.
Exponential families
have many appealing properties, such as the existence of conjugate priors and sufficient statistics, and a dually flat geometric structure \citep{amari2016information}. %
A key property of exponential families is that {\bf the support is constant within the same family and dictated by the base measure $\nu$}: this follows immediately from the positiveness of the $\exp$ function in \eqref{eq:boltzmann}.
In \S\ref{sec:tsallis}, we describe a more general set of families -- \textbf{deformed exponential families} -- that relax this property.

\section{Continuous Fenchel-Young Losses}\label{sec:fy_losses}

We saw in \S\ref{sec:rpm} how to construct distributions $\hat p_\Omega[f]$ from
a scoring function $f(t)$ via the $\Omega$-regularized prediction map
\eqref{eq:reg_prediction}.
In practice, the scoring function will often be a parametrized function,
denoted $f_\theta(t)$.
\insertion{In this section, we will address the reverse problem: given a true
data distribution $p$ (or samples thereof), find an estimate $\theta$ such that $\hat{p}_\Omega[f_\theta] \approx p$. Many statistical tasks can be formulated in terms of finding a good empirical approximation to $p$, and loss functions are a flexible way of quantifying how good these approximations are.}
For finite $S$,
\citet{blondel2020learning} introduced the notion of Fenchel-Young loss. Here,
we extend that notion to arbitrary domains.

\insertion{\subsection{Definition}}

The construction hinges on the notion of Fenchel dual, denoted $\Omega^*$, of an l.s.c. proper convex function $\Omega\!:\!\mathcal{M}_+^1(S) \rightarrow \mathbb{R}$ \citep{Bauschke_Combettes2011}:%
\footnote{Fenchel duality is taken in the (potentially infinite-dimensional) set $\mathcal{F} \subseteq \mathbb{R}^S$, which endowed with the inner product $\langle f, g\rangle = \int_S f(t) g(t) d\nu(t)$ forms a Hilbert space \citep{Bauschke_Combettes2011}.}
\begin{equation*}\label{eq:Omega_star2}
\Omega^*(f) := \max_{p\in \mathcal{M}_+^1(S)} \mathbb{E}_{p}[f(t)] - \Omega({p})
  =  \mathbb{E}_{\hat{p}_\Omega[f]}[f(t)] - \Omega(\hat{p}_\Omega[f]),
\end{equation*}
where, for the equality, we used the fact that
$\hat{p}_\Omega[f]$ is the solution of \eqref{eq:reg_prediction}.
We can now define the Fenchel-Young loss for arbitrary domains.
\begin{definition}[Fenchel-Young loss.]\label{def:fyloss}
Given an l.s.c., proper, strictly convex function $\Omega:\mathcal{M}_+^1(S) \rightarrow \mathbb{R}$, the Fenchel-Young loss $L_{\Omega} : \mathcal{F} \times \mathcal{M}_+^1(S) \rightarrow \mathbb{R}$ is defined as
\begin{equation}
L_{\Omega}(f; p) := \Omega^*(f) + \Omega(p) - \mathbb{E}_{p}[f(t)].
\end{equation}
For convenience, we also define the cross-$\Omega$ loss $L^\times_{\Omega} : \mathcal{F} \times \mathcal{M}_+^1(S) \rightarrow \mathbb{R}$ as follows:
\begin{equation}
L^\times_{\Omega}(f; p) := \Omega^*(f) - \mathbb{E}_{p}[f(t)].
\end{equation}
\end{definition}
Note that, when $\Omega(p)$ is finite, the Fenchel-Young loss $L_\Omega$ differs
from $L^\times_\Omega$ only by a term which is constant w.r.t.\ $f$. An interesting example is when $p=\delta_t$ is a Dirac delta, in which case we obtain $L^\times_{\Omega}(f; \delta_t) = \Omega^*(f) - \mathbb{E}_{\delta_t}[f(t)] = \Omega^*(f) - f(t)$.

The name ``Fenchel-Young loss'' stems from the Fenchel-Young inequality \citep[Proposition 3.3.4]{borwein2010convex}, which immediately implies the following property:
\begin{proposition}[Non-negativity and condition for zero loss]
\label{prop:fy_properties1}
With $\Omega$ as in Definition~\ref{def:fyloss}, we have
(i) $L_{\Omega}(f;p) \geq 0$,  and  (ii) $L_{\Omega}(f;p) = 0 \Leftrightarrow p =
\hat{p}_{\Omega}[f]$ almost everywhere. %
\end{proposition}

\insertion{In fact, we can interpret the Fenchel-Young loss as the {\bf regret associated to the generalized Helmholtz free energy $-U_\Omega(p; f) := -\mathbb{E}_p[f(t)] + \Omega(p)$}: indeed, we have $\Omega^*(f) = \max_{p' \in \mathcal{M}_+^1(S)} U_\Omega(p'; f) = U_\Omega(\hat{p}_\Omega[f]; f)$, and therefore $L_\Omega(f; p) = -U_\Omega(p; f) + U_\Omega(\hat{p}_\Omega[f]; f)$.}

\insertion{Fenchel-Young losses are also tightly connected to Bregman divergences \citep{bregman1967relaxation}, as shown by \citet[Theorem 1.1]{amari2016information} and \citet[\S3.2]{blondel2020learning}.}
In particular, when $\Omega$ is the Shannon-Boltzmann-Gibbs negentropy, the Fenchel-Young loss $L_\Omega$ equals the
\textbf{Kullback-Leibler divergence} between $p$ and $\hat{p}_\Omega[f]$, and $L^\times_\Omega$ becomes the \textbf{cross-entropy loss}.
This is commonly used as an objective to minimize in estimation problems,
for example when $\bar p := \frac{1}{L}\sum_{\ell=1}^L \delta_{t_\ell}$ is the empirical data distribution associated to a sample $\{t_1, \ldots, t_L\}$, and the goal is to obtain an estimate $f$ so that $\hat{p}_\Omega[f]$ approximates  $\bar p$. In that case, the minimization of the cross-entropy loss corresponds to \textbf{maximum likelihood estimation}.%
\footnote{Note that, for finite $S$, $\Omega(\delta_t) = 0$ and therefore $L_\Omega^\times(f, \delta_t) = L_\Omega(f, \delta_t)$. This, however, does not happen in general -- for $S=\mathbb{R}$, the differential negentropy explodes for Dirac distributions, $\Omega(\delta_t) = +\infty$.} %

\insertion{\subsection{Properties}}

Proposition~\ref{prop:fy_properties1} shows that \textbf{Fenchel-Young losses generalize a key property of the Kullback-Leibler divergence and the cross-entropy loss}, since the loss minimizers are attained when $\hat{p}_\Omega[f] = p$. Indeed, one target use of Fenchel-Young losses is to obtain an estimate $f$, given some empirical data distribution $\bar p$, by minimizing $L_{\Omega}(f; \bar{p})$.
To make this practical, we need to assume a parametric family $\{f_\theta \mid \theta \in \Theta\} \subseteq \mathcal{F}$, where $\theta$ is a vector of parameters and $\Theta \subseteq \mathbb{R}^M$ is a convex set.
The goal of estimation is to find $\hat{\theta}$ which minimizes
$L_\Omega(f_\theta; \bar p)$.
The next proposition, proved in Appendix~\ref{sec:proof_fy_properties}, sheds light on this problem.

\begin{proposition}[Stationary points of Fenchel-Young losses]
\label{prop:fy_properties2}
Assume that $f_\theta(t)$ is differentiable with respect to $\theta \in \Theta$ for any $t \in S$. Then,
the following expression holds for the gradient of $L_{\Omega}(f_\theta; p)$ with respect to $\theta$:
\begin{equation}\label{eq:expected_gradient_matching}
\nabla_{\theta} L_{\Omega}(f_\theta; p) = \mathbb{E}_{\hat{p}_{\Omega}[f_\theta]}[\nabla_\theta f_\theta(t)] - \mathbb{E}_{p}[\nabla_\theta f_\theta(t)].
\end{equation}
Therefore, $\hat{\theta} \in \Theta$ is a stationary point of $L_\Omega(f_\theta; p)$ iff it satisfies the equation
\begin{equation}\label{eq:expected_gradient_matching_02}
\mathbb{E}_{\hat{p}_{\Omega}[f_{\hat{\theta}}]}[\nabla_\theta f_\theta(t)] = \mathbb{E}_{p}[\nabla_\theta f_\theta(t)].
\end{equation}
\end{proposition}
Eqs.~\eqref{eq:expected_gradient_matching}--\eqref{eq:expected_gradient_matching_02} resemble the familiar gradient expressions used to estimate energy-based models with maximum likelihood \citep{lecun2006tutorial}. Indeed, these expressions are recovered when $\Omega$ is the Shannon-Boltzmann-Gibbs entropy, in which case the distribution ${\hat{p}_{\Omega}[f_\theta]}$ is a Gibbs distribution, as seen in \S\ref{sec:rpm}.
Therefore, Fenchel-Young losses offer a more general objective function to fit densities in energy-based models which can serve as an alternative to maximum likelihood.

\paragraph{Convexity, moment matching, and sufficient statistics.}
If the parametric family is linear, $f_{\theta}(t) = \theta^\top \phi(t)$, then the gradient of $f$ with respect to $\theta$ becomes simply $\nabla_\theta f_\theta(t) = \phi(t)$, and we obtain the following stronger properties, also proved in Appendix~\ref{sec:proof_fy_properties}:
\begin{proposition}[Properties when $f_\theta(t)$ is linear in $\theta$]
\label{prop:fy_properties3}
If $f_{\theta}(t) = \theta^\top \phi(t)$, then the following holds:
\begin{enumerate}
    \item $\nabla_{\theta} L_{\Omega}(f_\theta; p) = \mathbb{E}_{\hat{p}_{\Omega}[f_\theta]}[\phi(t)] - \mathbb{E}_{p}[\phi(t)]$.
    \item $L_{\Omega}(f_\theta; p)$ is convex w.r.t.\ $\theta$.
    \item $\hat{\theta} \in \arg\min_{\theta} L_{\Omega}(f_\theta; p) \Leftrightarrow \mathbb{E}_{\hat{p}_{\Omega}[f_{\hat{\theta}}]}[\phi(t)] = v$,
    where $v = \mathbb{E}_{p}[\phi(t)]$.
\end{enumerate}
\end{proposition}
The third point in Proposition~\ref{prop:fy_properties3} is particularly significant: If $p=\bar{p}$ is an empirical data distribution based on a sample $\{t_1, \ldots, t_L\}$, then $v = \frac{1}{L}\sum_{\ell=1}^L \phi(t_\ell)$ is the empirical mean of the statistics -- the statement shows that estimating $\theta$ only depends on
$\bar{p}$ through $v$, which \textbf{generalizes the concept of sufficient statistics from exponential families.} The result shows that fitting a density from a linearly parametrized family
to an empirical distribution $\bar{p}$ by matching the expected statistics is optimal
in the Fenchel-Young loss sense, generalizing the well-known result from
exponential families that maximum likelihood estimation is equivalent to
\textbf{moment matching} of the sufficient statistics.

Figure~\ref{fig:misfit} illustrates the result of fitting $\beta$-Gaussian distributions (to be  introduced in \S\ref{sec:sparsemax}) to samples drawn from each of the %
distributions by minimizing the corresponding Fenchel-Young losses, confirming adequate fitting.

\begin{figure*}[t]\centering\includegraphics[width=.9\textwidth]{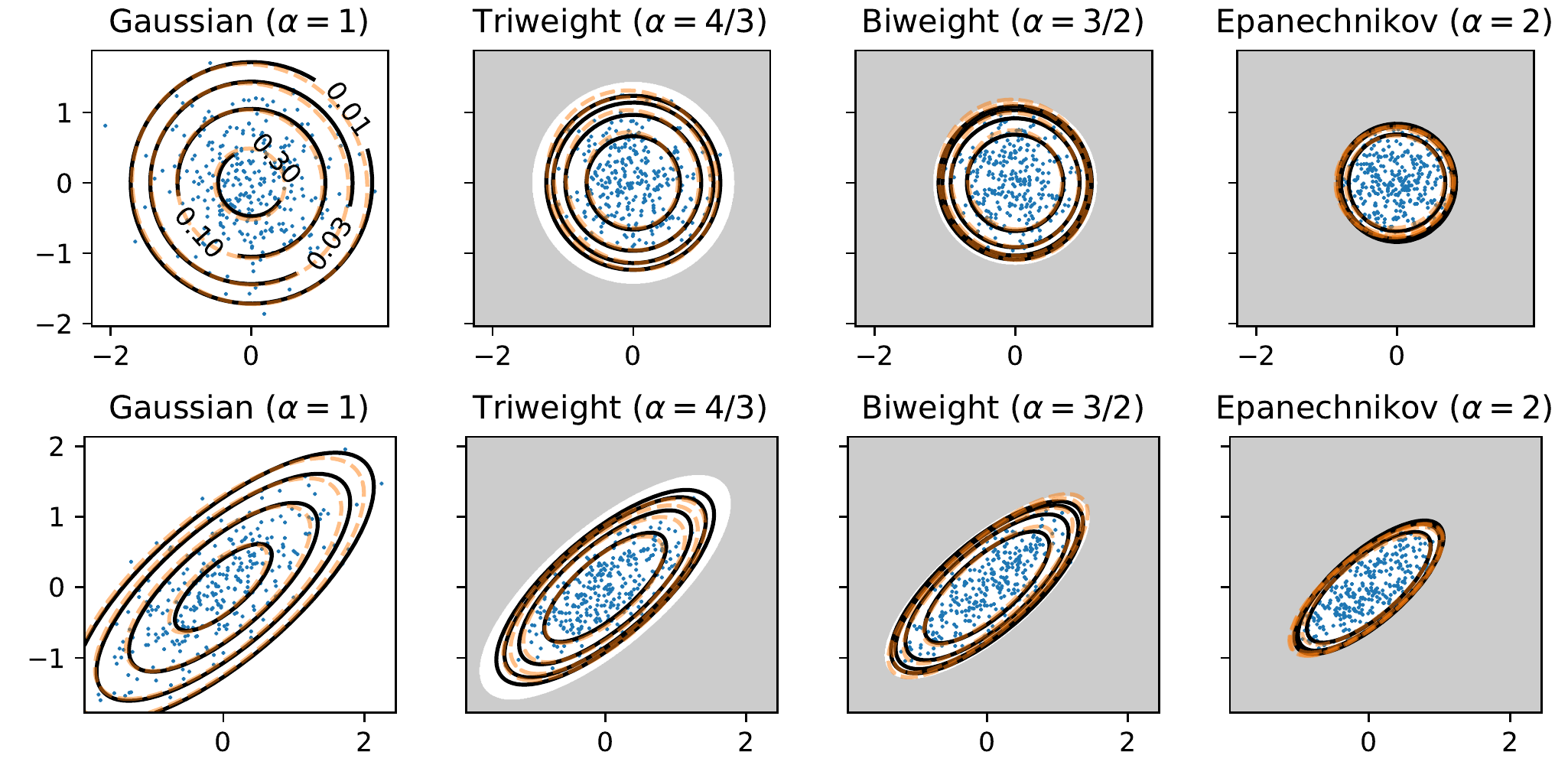}
\caption{\label{fig:misfit}Density and samples of two-dimensional $\beta$-Gaussian random
    variables, for the isotropic (top) and anisotropic cases (bottom). These distributions are obtained by applying (Tsallis) $\Omega_\alpha$-regularized prediction maps to quadratic scoring functions, with $\alpha = 2-\beta$. Shown are the original density (solid lines) and the density fit to
samples by moment matching (dashed lines). All contour lines are at the same
absolute levels, and the complement of the support is
shaded when appropriate.}
\end{figure*}

\insertion{\subsection{Examples}\label{sec:fy_examples}}

\insertion{We next provide some familiar examples, which will be generalized in the upcoming sections.}

\paragraph{Examples 1 and 2: Squared and absolute losses.}

Let $\Omega$ be the Shannon-Boltzmann-Gibbs negentropy.
Let $p=\delta_t$, \textit{i.e.}, the distribution contains a single sample $t$.
For the Gaussian distribution, by identification with \eqref{eq:boltzmann},
we get $\Omega^*(f) = A(f) = \log(\sigma \sqrt{2 \pi})$ and therefore
\begin{equation}
L^\times_\Omega(f; p) = \frac{(t - \mu)^2}{2\sigma^2}
+ \log(\sigma \sqrt{2 \pi}).
\end{equation}
Similarly, for the Laplace distribution, we get
$\Omega^*(f) = A(f) = \log(2b)$ and therefore
\begin{equation}
L^\times_\Omega(f; p) = \frac{|t - \mu|}{b} + \log(2b).
\end{equation}

\paragraph{Example 3: KL divergence between two Gaussian distributions.}
When $\Omega$ is the Shannon-Boltzmann-Gibbs entropy, the Fenchel dual is the log-partition function \eqref{eq:boltzmann}, $\Omega^*(f) = A(f)$, and the Fenchel-Young loss recovers the  Kullback-Leibler divergence. For example, if $S=\mathbb{R}^N$, $p(t) = \mathcal{N}(t; \mu, \Sigma)$, and  $f(t) = -\frac{1}{2}(t-\mu_f)^\top \Sigma_f^{-1}(t-\mu_f)$, using the expression for the entropy in \S\ref{sec:rpm}, we obtain the well-known expression for the Kullback-Leibler divergence between Gaussians:
\begin{equation}\label{eq:kl_gaussians}
L_{\Omega}(f; p) = \frac{1}{2}(\mu-\mu_f)^\top\Sigma_f^{-1}(\mu-\mu_f) + \frac{1}{2}\left(\mathrm{Tr}(\Sigma_f^{-1}\Sigma) - N + \log \frac{|\Sigma_f|}{|\Sigma|} \right).
\end{equation}
In \S\ref{sec:elliptical}, we will generalize this result for a class of elliptical distributions called $\beta$-Gaussian distributions, for which we will derive a closed-form expression for the Fenchel-Young loss.

\section{Tsallis Regularizers and Deformed Exponential Families}\label{sec:tsallis}

We introduce in this section a broader set of regularizers $\Omega$ based on Tsallis entropies \citep{Tsallis1988}, which allow generalizing the examples in \S\ref{sec:fy_examples}.
Tsallis entropies are a generalization of Shannon-Boltzmann-Gibbs entropies which are suitable to model several phenomena present in natural, artificial and social complex systems \citep[{\textit{inter alia}}]{lutz2003anomalous,burlaga2005triangle,pickup2009generalized,adare2011measurement} under the umbrella of nonextensive statistical mechanics \citep{abe2001nonextensive}, a generalization of the Boltzmann-Gibbs theory.
We will see that using these regularizers in
Definition~\ref{def:regularized_prediction}  leads to ``deformed exponential
families,'' which may correspond to sparse density functions in the sense of
Definition~\ref{def:sparse_function}. This makes a bridge between the entmax
transformation, proposed for finite domains by
\citet{blondel2020learning} and \citet{peters2019sparse},
and new transformations which we will propose in \S\ref{sec:sparsemax} and \S\ref{sec:elliptical} for the non-finite case.

\subsection{Tsallis entropies}%
\label{sec:tsallis_definitions}

A central concept in Tsallis statistics is a generalization of the standard logarithm and
exponential functions, called \textbf{$\beta$-logarithm}, $\log_{\beta}:
\mathbb{R}_{\ge 0} \rightarrow \mathbb{R}$ (not to be confused with base-$\beta$
logarithm), and \textbf{$\beta$-exponential}, $\exp_{\beta}: \mathbb{R}
\rightarrow \mathbb{R}$, defined as follows:
\begin{equation}\label{eq:beta_log_exp}
    \log_{\beta}(u) := \left\{
    \begin{array}{ll}
        \frac{u^{1-\beta} - 1}{1-\beta}, & \beta \ne 1, \\
        \log u, & \beta = 1;
    \end{array}
    \right.
    \qquad
    \exp_{\beta}(u) := \left\{
    \begin{array}{ll}
     	[1 + (1-\beta)u]_+^{1/(1-\beta)}, & \beta \ne 1, \\
        \exp u, & \beta = 1.
    \end{array}
    \right.
\end{equation}
Note that ${\lim_{\beta \rightarrow 1}}\log_\beta(u) = \log u$, ${\lim_{\beta \rightarrow 1}}\exp_\beta(u) = \exp u$, and $\log_\beta(\exp_\beta(u)) = u$ for any $\beta$ and $u \in \mathbb{R}$.
Another important concept, which we will use in the sequel, is that of \textbf{$\beta$-escort distribution}  \citep{Tsallis1988}: this is the distribution $\tilde{p}^{\beta}$ obtained by applying the following ``sharpness'' operator:
\begin{equation}\label{eq:escort}
p(t) \mapsto \tilde{p}^{\beta}(t) := \frac{p(t)^{\beta}}{\|p\|_\beta^\beta}, \quad \text{where $\|p\|_\beta^\beta = \int_{S} p(t')^{\beta} d\nu(t')$}.
\end{equation}
Note that we have $\tilde{p}^{1}(t) = p(t)$. $\beta>1$ increases sharpness, whereas $\beta<1$ decreases it, producing more uniform distributions. $\beta=0$ results in a uniform distribution.

We thus have the following definition  \citep{havrda1967quantification,Tsallis1988}:%
\footnote{This entropy is normally defined up to a constant, often presented without the $\tfrac{1}{\alpha}$ factor. We use the same definition as \citet[\S 4.3]{blondel2020learning} for convenience.} %
\begin{definition}[Tsallis negentropies.]\label{def:tsallis}
For $\alpha \ge 0$,
the {$\alpha$-Tsallis negentropy} is:
\begin{equation}\label{eq:tsallis}
\Omega_{\alpha}(p) := \tfrac{1}{\alpha} \mathbb{E}_{p}[\log_{2-\alpha}(p(t))] =
\begin{cases}
\frac{1}{\alpha(\alpha-1)}\left( \int_S  p(t)^\alpha - 1 \right), &
\alpha \ne 1,\\
\int_S p(t) \log p(t), &
\alpha = 1.
\end{cases}
\end{equation}
\end{definition}
The family of Tsallis entropies is continuous in $\alpha$, \textit{i.e.}, $\lim_{\alpha \rightarrow
1}\Omega_{\alpha}(p) = \Omega_{1}(p)$, for any $p \in \mathcal{M}_+^1(S)$, with $\Omega_{1}(p)$ recovering Shannon's negentropy (see Appendix~\ref{sec:gini_ent_sparse_family} for a proof).
Another notable case is $\Omega_2(p) = \sfrac{1}{2}\int_S p(t)^2  - \sfrac{1}{2}$,
the negative of which has several names, {\it e.g.}, Gini-Simpson index \citep{Jost2006} or Rao's quadratic entropy \citep{Rao1982}.
We will come back to the $\alpha=2$ case in \S\ref{sec:sparsemax}.

\subsection{Tsallis regularization: deformed exponential families and $\alpha$-sparse families}\label{sec:deformed_sparse}

For $\alpha>0$, the Tsallis negentropy $\Omega_\alpha$ is strictly convex, hence it can be plugged as the regularizer in  Definition~\ref{def:regularized_prediction}.
The next proposition is a reformulation of a result due to \citet{naudts2009q}
in the statistical physics literature; we include a proof in
Appendix~\ref{sec:gini_ent_sparse_family}. This result provides an expression
for the $\Omega_{\alpha}$-regularized prediction map:

\begin{proposition}[Distribution and normalizing function expressions]
\label{prop:solution_rpm_tsallis}
For $\alpha > 0$ and $f\in\mathcal{F}$, the $\Omega_{\alpha}$-regularized
prediction map has the following form:
\begin{equation}\label{eq:entmax}
\hat{p}_{\Omega_\alpha}[f](t) = \exp_{2-\alpha}(f(t) - A_\alpha(f)), %
\end{equation}
where $\exp_\beta$ is defined in \eqref{eq:beta_log_exp}
and
$A_\alpha: \mathcal{F} \rightarrow \mathbb{R}$ is a normalizing function (we write $p(t) \equiv \hat{p}_{\Omega_\alpha}[f](t)$ for simplicity):
\begin{equation}\label{eq:A_alpha}
A_\alpha(f) = \frac{\frac{1}{1-\alpha} + \int_S p(t)^{2-\alpha} f(t)}{\int_S p(t)^{2-\alpha}} - \frac{1}{1-\alpha}.
\end{equation}
\end{proposition}
It is interesting to contrast \eqref{eq:entmax} with Boltzmann-Gibbs
distributions \eqref{eq:boltzmann}, which are recovered as a limit case when $\alpha\rightarrow 1$. One key
aspect to note is that the $(2-\alpha)$-exponential, for $\alpha > 1$, can
return zero values. Therefore, {\bf the distribution
$\hat{p}_{\Omega_\alpha}[f]$ in \eqref{eq:entmax} might not have full support,
\textit{i.e.}, we may have $\mathrm{supp}(\hat{p}_{\Omega_\alpha}[f]) \subsetneq
S$.} In particular, it may be a {\bf sparse density function} in the sense of
Definition~\ref{def:sparse_function} (see Figure~\ref{fig:rpm_example}). This never happens with Boltzmann-Gibbs distributions, which always have full support.

\paragraph{Relation to sparsemax and entmax.}
\citet{blondel2020learning}  showed that, for finite $S$,  $\Omega_2$-regularized
prediction map ({\it i.e.}, picking $\alpha=2$) is the {\bf sparsemax} transformation, $\hat{p}_\Omega[f] = \mathrm{sparsemax}(f) =
\arg\min_{p \in \triangle^{|S|}} \|p - f\|^2_2$
(Euclidean projection of  $f\in\mathbb{R}^{|S|}$ onto the $|S|$-dimensional probability simplex $\triangle^{|S|}$).
Other values of $\alpha$ were studied by \citet{peters2019sparse}, under the name {\bf $\alpha$-entmax} transformation. For $\alpha>1$, these transformations have a propensity for returning sparse distributions, where several entries have zero probability.
Proposition~\ref{prop:solution_rpm_tsallis} shows that similar properties can be obtained when $S$ is non-finite (countably infinite or continuous).

\paragraph{Deformed exponential families.}
With a linear parametrization $f_\theta(t) = \theta^\top\phi(t)$,
distributions with the form \eqref{eq:entmax} are called \textit{deformed exponential families} \citep{naudts2009q,sears2010generalized},
also referred to as \textit{$t$-exponential families} \citep{ding2010t} and \textit{$q$-exponential families} \citep{matsuzoe2012geometry}.
The geometry of these families induced by Tsallis entropies was studied by \citet[\S 4.3]{amari2016information}.%
\footnote{Unfortunately, the literature is inconsistent in defining these
    coefficients. Our $\alpha$ matches that of \citet{blondel2020learning};
    Tsallis' $q$ in the context of deformed exponential families equals
    $2-\alpha$ (which we call $\beta$ in our paper). This family is also related
    to Amari's $\alpha$-divergences, but their $\alpha$ equals $2q - 1$.
    Inconsistent definitions have also been proposed for $q$-exponential
    families regarding how they are normalized; for example, the Tsallis maxent
    principle leads to a different definition. See
    Appendix~\ref{sec:tsallis_maxent} for details.
}
They include for example $t$-Student and other heavy tail distributions (heavy tails arise when $\alpha < 1$).
Unlike those prior works, in this paper we are interested in the sparse, light tail scenario ($\alpha>1$), not in heavy tails. For $\alpha>1$, we call these {\bf $\alpha$-sparse families.}
When $\alpha \rightarrow 1$, $\alpha$-sparse families become exponential families and they cease to be ``sparse'', in the sense that all distributions in the same family have the same support.
Another interesting particular case is that of $\alpha=2$, which we will see in detail in \S\ref{sec:sparsemax}. From Proposition~\ref{prop:solution_rpm_tsallis} and \eqref{eq:beta_log_exp}, we can see that a 2-sparse family takes the form (writing $p_\theta \equiv \hat{p}_{\Omega_\alpha}[f_\theta]$):
\begin{equation}\label{eq:sparse_family}
    p_\theta(t) = [\theta^\top \phi(t) - A_2(\theta) + 1]_+, \quad \text{with $A_2(\theta) := A_2(f_\theta) = 1+ \frac{-1 + {\textstyle \int_{\mathrm{supp}(p_{\theta})} \theta^\top \phi(t)}}{|\mathrm{supp}(p_{\theta})|}$}.
\end{equation}
This generalizes the result of \citet[Proposition~1]{Martins2016ICML}, who derived a similar expression for the finite case.

\paragraph{Gradient of $A_\alpha$.}
A relevant problem is that of characterizing the normalizing function
$A_\alpha(\theta) := A_\alpha(f_\theta)$. When $\alpha=1$, $A_1(\theta) =
{\lim_{\alpha\rightarrow 1}} A_\alpha(\theta) = \log \int_S \exp(\theta^\top
\phi(t))$ is the log-partition function (see \eqref{eq:boltzmann}), and  its
first and higher order derivatives are equal to the moments of the sufficient
statistics. The following proposition, stated in Theorem~5 of
\citet{amari2011geometry}, and proved in our
Appendix~\ref{sec:proof_gradient_A}, characterizes $A_\alpha(\theta)$ for
$\alpha \ne 1$ in terms of an expectation under the $\beta$-escort distribution,
defined in \eqref{eq:escort}, for $\beta=2-\alpha$.

\begin{proposition}[Gradient of normalizing function $A_\alpha$]
\label{prop:gradient_A}
Let $\beta=2-\alpha$ with $\alpha \in [0,2]$.
Let $\tilde{p}_\theta^{\beta}$ be the $\beta$-escort distribution %
\eqref{eq:escort}.
The normalizing function $A_\alpha: \Theta \rightarrow \mathbb{R}$ is convex and
its gradient coincides with the expectation under the $\beta$-escort
distribution
\begin{equation}\label{eq:derivative_of_partition}
\nabla_\theta A_\alpha(\theta) = \mathbb{E}_{\tilde{p}_\theta^{\beta}}[\phi(t)]
= \frac{\int_{S} p_\theta(t)^{\beta} \phi(t)}{\int_{S} p_\theta(t)^{\beta}}.
\end{equation}
\end{proposition}

We use this result later in this section to derive the Hessian of Fenchel-Young losses and in \S\ref{sec:attention} to obtain the Jacobian of entmax attention mechanisms.

To close the loop, we present the following result, proved in
Appendix~\ref{sec:proof_entropy_normalizing}, which relates Tsallis negentropies
$\Omega_\alpha$, their convex conjugates  $\Omega^*_\alpha$, normalizing
functions $A_\alpha(\theta)$, and provides an expression for $\Omega_\alpha$-Fenchel-Young losses for linearly parametrized families:
\begin{proposition}[Key quantities in $\alpha$-sparse families]\label{prop:entropy_normalizing}

Let $p_\theta \equiv \hat{p}_{\Omega_\alpha}[f_\theta]$, with $f_\theta(t) = \theta^\top \phi(t)$,  and define as  $\mu(\theta) := \mathbb{E}_{p_\theta}[\phi(t)]$  the ``mean parameters'' associated with $\theta \in \Theta$. %
Then the Tsallis negentropy is given by
\begin{equation}\label{eq:entropy_normalizing}
\Omega_\alpha(p_\theta) = \frac{1}{\alpha}\left(\theta^\top \mu(\theta) - A_\alpha(\theta)\right),
\end{equation}
its convex conjugate is given by
\begin{equation}\label{eq:entropy_normalizing_conjugate}
\Omega^*_\alpha(f_\theta) = (\alpha-1)\Omega_\alpha(\hat{p}_{\Omega_\alpha}[f_\theta]) + A_\alpha(\theta) = \frac{1}{\alpha}\left((\alpha-1)\theta^\top \mu(\theta) + A_\alpha(\theta)\right),
\end{equation}
and the $\Omega_\alpha$-Fenchel-Young loss between $f_\theta$ and \emph{any} $p \in \mathcal{M}_+^1(S)$ is given by
\begin{equation}\label{eq:fy_linear_families}
L_{\Omega_\alpha}(f_\theta, p) = \Omega_\alpha(p) - \Omega_\alpha(\hat{p}_{\Omega_\alpha}[f_\theta]) - \theta^\top (v - \mu(\theta)),
\end{equation}
where $v := \mathbb{E}_p[\phi(t)]$ is the empirical expected statistics (see Proposition~\ref{prop:fy_properties3}).
\end{proposition}
The expressions in Proposition~\ref{prop:entropy_normalizing} deserve some analysis.
First, note that, when $\alpha=1$, we recover the well-known duality relation between the Shannon-Boltzmann-Gibbs entropy and the log-partition function \citep{wainwright_2008}, in which case we get
from \eqref{eq:entropy_normalizing_conjugate} that $A_1(\theta) = \Omega_1^*(f_\theta)$.
Second, these expressions are \textbf{practically useful}: \eqref{eq:entropy_normalizing} offers a way of computing Tsallis negentropies for any $\alpha$, provided we have a procedure to compute the mean parameters $\mu(\theta)$ and to evaluate the normalizing function $A_\alpha(\theta)$ from $\theta$.
Finally, the expression \eqref{eq:fy_linear_families} for the Fenchel-Young loss puts in evidence its relation with Bregman divergences. This expression can be evaluated for any density $p$ (not necessarily in the family), only depending on that density through its Tsallis negentropy $\Omega_\alpha(p)$ and the expected statistics $v = \mathbb{E}_p[\phi(t)]$ (cf. Proposition~\ref{prop:fy_properties3}). In particular, it facilitates a procedure to assess how well a density from a deformed exponential family fits empirical observations.
We will use these results to obtain closed-form expressions for the Tsallis entropies and Fenchel-Young losses of several densities in \S\ref{sec:sparsemax} and \S\ref{sec:elliptical}.

\paragraph{Gradient and Hessian of Fenchel-Young losses.}
Finally, we show how to compute the first and second-order derivatives of Fenchel-Young losses.
The proof (in Appendix~\ref{sec:proof_gradient_hessian_fy}) invokes Propositions~\ref{prop:fy_properties3} and \ref{prop:gradient_A}.
To state this result, we need to define, for $\beta \ge 0$, the
\textbf{generalized $\beta$-covariance} associated to a density $p \in
\mathcal{M}_+^1$, and statistics $\phi:S\rightarrow \mathbb{R}^M$ and $\psi:S\rightarrow \mathbb{R}^N$:
\begin{equation}\label{eq:beta_covariance}
\mathrm{cov}_{p, \beta}[\phi(t), \psi(t)] \,\,=\,\, \|p\|_\beta^\beta \times \left( \mathbb{E}_{\tilde{p}_\beta}\big[\phi(t) \psi(t)^\top\big] - \mathbb{E}_{\tilde{p}_\beta}[\phi(t)]\,  \mathbb{E}_{\tilde{p}_\beta}[\psi(t)]^\top
\right),
\end{equation}
where $\tilde{p}_\beta$ is the $\beta$-escort distribution in \eqref{eq:escort}.
This can indeed be seen as a generalized covariance:
For $\beta=1$, we have the usual covariance; for $\beta=0$, we get a covariance taken w.r.t.\ a uniform density on the support of $p$,  scaled by $|\mathrm{supp}(p)|$.
\begin{proposition}[Gradient and Hessian of Fenchel-Young losses]
\label{prop:gradient_hessian_fy}
Let $p_\theta \equiv \hat{p}_{\Omega_\alpha}[f_\theta]$, with $f_\theta(t) = \theta^\top \phi(t)$,  $\mu(\theta) = \mathbb{E}_{p_\theta}[\phi(t)]$,
and $v = \mathbb{E}_{p}[\phi(t)]$.
The gradient and Hessian of $L_{\Omega_\alpha}(f_\theta, p)$ with respect to $\theta$ are given by:
\begin{equation}\label{eq:gradient_hessian_fy}
\nabla_\theta L_{\Omega_\alpha}(f_\theta, p) = \mu(\theta) - v, \qquad
\nabla\nabla_\theta L_{\Omega_\alpha}(f_\theta, p) = \mathrm{cov}_{p, 2-\alpha}[\phi(t), \phi(t)].
\end{equation}
\end{proposition}
Note that the Hessian expression \eqref{eq:gradient_hessian_fy} involves a generalized self-covariance,  which is still a covariance matrix scaled by a positive constant, hence it is positive semi-definite. This confirms the convexity of Fenchel-Young losses on linearly parametrized families stated in Proposition~\ref{prop:fy_properties3}.

\begin{figure*}[t]
\centering
\includegraphics[width=.99\textwidth]{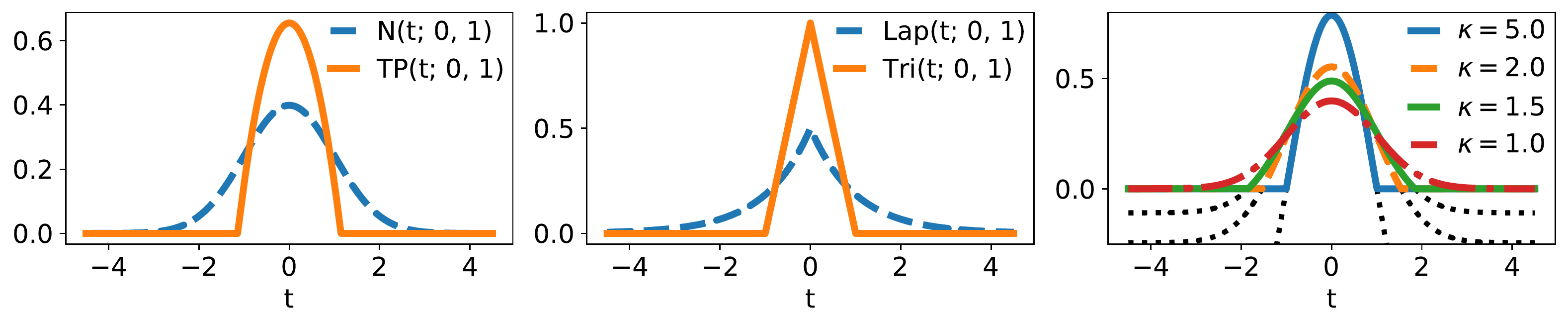}\quad
\caption{\label{fig:sparse_distributions}{\bf Some location-scale distributions
        generated by the $\Omega_\alpha$-regularized prediction map for
$\alpha\in \{1,2\}$,
$\mu=0$ and $\sigma^2=1$.}
Left: Gaussian and truncated parabola. Middle:
Laplace and triangular (bottom).
Right:
Truncated Gaussians with
$\kappa\in\{1, 1.5, 2, 5\}$.}
\end{figure*}

\section{Infinite Sparsemax}\label{sec:sparsemax}

In this section, we focus on deformed exponential families with $\alpha=2$, \textit{i.e.}, 2-sparse families.
For the same choices of $f(t)$ as in \S\ref{sec:rpm}, we will obtain sparse counterparts of the softmax, Poisson, Gaussian, and Laplace distributions, which we list in Table~\ref{tab:distributions}.

For finite $S$, the choice $\alpha=2$ corresponds to the sparsemax transfomation proposed by \citet{Martins2016ICML},
which has appealing  theoretical and computational properties.
In the general case, as seen in \eqref{eq:sparse_family},
plugging $\alpha=2$ in \eqref{eq:entmax} leads to the $\Omega_2$-regularized
prediction map,
\begin{equation}\label{eq:sparsemax}
\hat{p}_{\Omega_2}[f](t) = [f(t) - \tau]_+,
\qquad \text{where $\tau = A_2(f) - 1$,}
\end{equation}
\textit{i.e.}, $\hat{p}_{\Omega_2}[f]$ is obtained from $f$ by subtracting a constant (which may be negative) and truncating, where that constant $\tau$ must be such that $ \int_S [f(t) - \tau]_+ = 1$.

If $S$ is continuous and $\nu$ the Lebesgue measure, we call
$\Omega_2$-regularized prediction map the {\bf continuous sparsemax} transformation, and for countably infinite $S$ we call it the  {\bf discrete infinite sparsemax}.
Examples follow, %
where some correspond to  novel distributions. %

\paragraph{Truncated parabola.}
If $f(t) = -\frac{(t-\mu)^2}{2\sigma^2}$, we obtain the continuous sparsemax counterpart of a Gaussian, which we dub a ``truncated parabola'':
\begin{equation}\label{eq:truncated_parabola}
\hat{p}_{\Omega_2}[f](t)
= \left[ - \tfrac{(t-\mu)^2}{2\sigma^2} - \tau \right]_+
\eqqcolon \mathrm{TP}(t; \mu, \sigma^2),
\end{equation}
where $\tau = -\tfrac{1}{2}\bigl(3/(2\sigma)\bigr)^{2/3}$,
$\mathrm{supp}(\hat{p}_{\Omega_2}[f]) = [\mu - \tfrac{3}{-4\, \tau}, \mu + \tfrac{3}{-4\, \tau} ]$ and %
$\Omega_2(\hat{p}_{\Omega_2}[f]) = -\tfrac{1}{2} - \tfrac{2\tau}{5}$  (see Appendix~\ref{sec:proof_truncated_parabola}).
This function, depicted in Figure~\ref{fig:sparse_distributions} (left), is
widely used in density estimation \citep{silverman1986density}. For $\mu=0$ and $\sigma=\sqrt{2/3}$, it is known as the ``Epanechnikov kernel'' \citep{epanechnikov1969non}.

\paragraph{Truncated paraboloid.}
The previous example can be generalized to $S=\mathbb{R}^N$, with $f(t)=-\frac{1}{2}(t-\mu)^\top\Sigma^{-1}(t-\mu)$, where $\Sigma \succ 0$, leading to a ``multivariate truncated paraboloid,'' the sparsemax counterpart of the multivariate Gaussian,
\begin{equation}\label{eq:truncated_paraboloid}
\hat{p}_{\Omega_2}[f](t) = \bigl[-\tau - \tfrac{1}{2}(t-\mu)\Sigma^{-1}(t-\mu)\bigr]_+,
\quad \text{where $\tau = -\Bigl(\Gamma\bigl(\tfrac{N}{2} + 2\bigr) / \sqrt{\mathrm{det}(2\pi \Sigma)}\Bigr)^{\frac{2}{2+N}}$},
\end{equation}
and where $\Gamma(z) = \int_0^\infty x^{z-1}\exp(-x) dx$ is the Gamma function, which extends the factorial function to the continuous domain, $\Gamma(n) = (n-1)!$ for $n \in \mathbb{N}$.
The expression above, derived in Appendix~\ref{sec:proof_truncated_paraboloid}, reduces to \eqref{eq:truncated_parabola} for $N=1$.
Notice that (unlike in the Gaussian case) a diagonal covariance matrix $\Sigma$ does not lead to a product of independent truncated parabolas. This distribution is an instance of an elliptical distribution and will be discussed further in \S\ref{sec:elliptical}. %

\paragraph{Triangular.}
Setting $f(t) = -|t-\mu|/b$, with $b>0$,  yields the triangular distribution
\begin{equation}\label{eq:triangular}
\hat{p}_{\Omega_2}[f](t) = \left[ -\tau - \tfrac{|t-\mu|}{b}\right]_+ \eqqcolon \mathrm{Tri}(t; \mu,b),
\end{equation}
where $\tau\! =\! -1/\sqrt{b}$,
$\mathrm{supp}(\hat{p}_{\Omega_2}[f]) = [\mu - \sqrt{b}, \mu + \sqrt{b}]$, and $\Omega_2(\hat{p}_{\Omega_2}[f]) = -\frac{1}{2} +  \frac{1}{3\sqrt{b}}$  (see Appendix~\ref{sec:proof_triangular}).
Figure~\ref{fig:sparse_distributions} (middle) depicts this distribution alongside Laplace.

\paragraph{Truncated Gaussian.} For $f(t) = \kappa \, \mathcal{N}(t; \mu,
\sigma^2)$ (a scaled Gaussian), with $\kappa \ge 1$, we obtain a truncated
Gaussian distribution (Figure~\ref{fig:sparse_distributions}, right),
\begin{equation}
\hat{p}_{\Omega_2}[f](t) = \left[-\tau + \kappa \, \mathcal{N}(t; \mu, \sigma^2)\right]_+,
\end{equation}
where $\tau = \kappa\,  \mathcal{N}(a; 0, \sigma^2)$ and $a$ is the solution of the equation $\frac{1}{\kappa} + \frac{2\, a}{\sqrt{2\pi}\sigma}\exp\left(-\frac{a^2}{2\sigma^2}\right) = \mathrm{erf}\left(\frac{a}{\sqrt{2}\sigma}\right)$.

\paragraph{Location-scale families.} %
More generally, let
$f_{\mu, \sigma}(t) := -\frac{1}{\sigma}g'(|t-\mu|/\sigma)$ for a location  $\mu \in \mathbb{R}$ and a scale  $\sigma > 0$, where
$g:\mathbb{R}_+ \rightarrow \mathbb{R}$ is convex and continuously differentiable.
Then, we have
\begin{equation}
\hat{p}_{\Omega_2}[f](t) = \left[-\tau - \tfrac{1}{\sigma}g'(|t-\mu|/\sigma)\right]_+,
\end{equation}
where $\tau = -g'(a)/\sigma$
and $a$ is the solution of the equation $ag'(a) - g(a) + g(0) = \frac{1}{2}$ (a sufficient condition for such solution to exist is $g$ being strongly convex; see Appendix~\ref{sec:proof_location_scale} for a proof).
The support of this distribution is
$\mathrm{supp}(\hat{p}_{\Omega_2}[f_{\mu, \sigma}]) = [(-a+\mu)/\sigma, (a+\mu)/\sigma]$.
This example subsumes the truncated parabola ($g(t) = t^3/6$), the triangular distribution ($g(t) = t^2/2$), and the truncated Gaussian ($g(t) = -\frac{\kappa}{2} \mathrm{erf}\left(\sfrac{t}{\sqrt{2}}\right)$).

\begin{figure}[t]
\includegraphics[width=\textwidth]{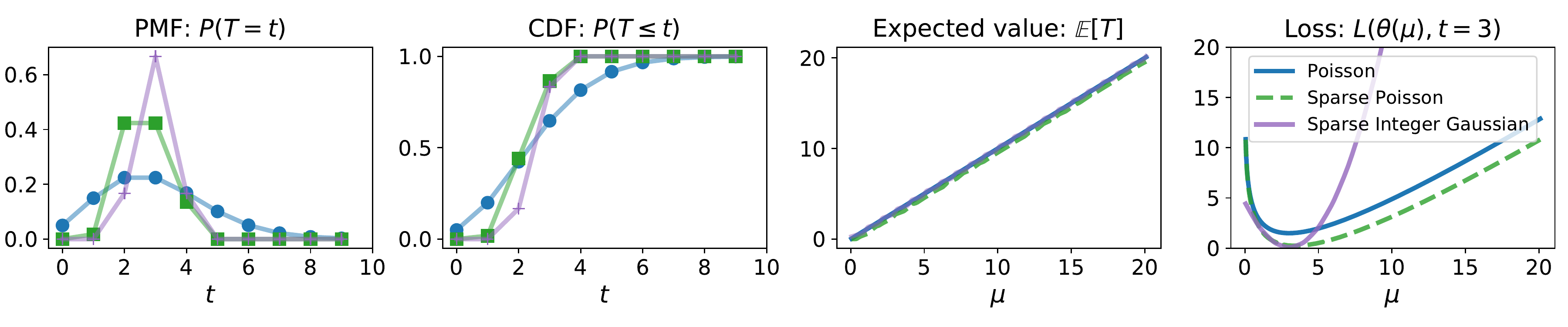}
\caption{\textbf{Sparse integer distributions}. The first two plots show
    the probability mass function (PMF) and the cumulative distribution function (CDF)
    of the distributions at $\mu=3$. Note that the lines between markers are
    shown for visual aid: these distributions do not assign probability
mass to non-integer values. The third plot shows the mean value of the
distributions when varying $\mu$. The last plot shows the Fenchel-Young loss
when the target label is fixed to $t=3$.}
\label{fig:sparse_integer_dist}
\end{figure}

\paragraph{Sparse integer distributions.}

A popular distribution for natural integers is the Poisson distribution, $p(t) = \mu^t
e^{-\mu} / t!$, where $\mu > 0$ is the mean parameter and $t \in \mathbb{N}$.
It is well-known that the Poisson distribution can be written in exponential
family form by setting $S=\mathbb{N}$, $\nu(A) = \sum_{t \in A} \frac{1}{t!}$
for $A \subseteq \mathbb{N}$, and $f(t) = t\log \mu$ in \eqref{eq:boltzmann}.
Alternatively, we can absorb the measure in $f(t)$ by setting $f(t) = t \log \mu
+ \log(1/t!)$ and letting $\nu(A)$ be the counting measure.
Choosing the latter formulation, we obtain a sparse counterpart of the Poisson
distribution
\begin{equation}
\label{eq:sparse_poisson}
\hat{p}_{\Omega_2}[f](t)
= \left[t \log \mu + \log(1/t!) -\tau\right]_+
\eqqcolon \mathrm{SparsePoisson}(t; \mu).
\end{equation}
\insertion{This corresponds to setting $\theta = [\log \mu, 1]$ and
$\phi(t) = [t, \log(1/t!)]$, see Table \ref{tab:linear_param}.
By Proposition \ref{prop:fy_properties3}, the associated Fenchel-Young loss is
convex in $\theta$.}
Since the exponential is monotonically
increasing, the loss is also convex in $\mu$.
A benefit of sparsity is that we can easily create new
distributions as long as the choice of $f(t)$ guarantees that $\tau$ exists.
For instance, inspired by the univariate Gaussian, we can choose
$f(t) = -\frac{1}{2}(t - \mu)^2$. By setting
$S=\mathbb{Z}$, we obtain a sparse integer-restricted counterpart of the
Gaussian distribution
\begin{equation}
\label{eq:sparse_integer_gaussian}
\hat{p}_{\Omega_2}[f](t)
= \left[-\frac{1}{2}(t - \mu)^2 -\tau\right]_+
\eqqcolon \mathrm{SparseIntegerGaussian}(t; \mu).
\end{equation}
These distributions are illustrated in Figure \ref{fig:sparse_integer_dist}.
They share in common that they achieve their mode near $\mu$.
Since they all have finite support, we compute $\tau$ by applying sparsemax
\citep{Martins2016ICML} at a window around the mode.

\paragraph{Exponential and sparse families.}
Whereas Poisson, Gaussian, and Laplace distributions (the latter with a fixed location) form exponential families, as seen in \S\ref{sec:rpm}, likewise
the sparse Poisson, truncated paraboloid, and triangular distributions above form 2-sparse families,
with the same statistics $\phi(t)$ and canonical parameters $\theta$.
For example, the Gaussian and truncated paraboloid cases both correspond to
the statistics $\phi(t) = [t, \mathrm{vec}(tt^\top)]$ and  canonical parameters $\theta = [\Sigma^{-1}\mu, \mathrm{vec}(-\frac{1}{2}\Sigma^{-1})]$, as shown in Table~\ref{tab:distributions}.
We will next see how these two distributions are both a particular case of $\beta$-Gaussians. %

\section{Elliptical Distributions and $\beta$-Gaussians}\label{sec:elliptical}

In this section, we extend the truncated parabola distribution to arbitrary
dimensions and $\alpha$.
This results in a family of tractable multivariate \emph{elliptical distributions}, which for $\alpha>1$ have bounded
support. We show that this family, which we call \textbf{$\beta$-Gaussians},
are a multivariate generalization of $q$-Gaussians
\citep[\S 4.1]{naudts2009q}, and correspond to a
naturally-rescaled variant of the Pearson Type-II distribution.
Throughout this section, we will always assume $\beta = 2-\alpha$.

\subsection{Definition and properties}

Our construction relies on the standard concept of spherical and elliptical
distributions, studied by \citet{cambanis1981theory}, \citet{owen1983class}, \citet{fang}, \textit{inter alia}, which we define and characterize next.
The next definition corresponds to \citet[Definition 2.1 and 2.2]{fang}. We denote by $\mathcal{O}(N)$
the orthogonal group, {\it i.e.}, the set of matrices $U \in \mathbb{R}^{N \times N}$ satisfying $U^\top U=UU^\top = \mathrm{Id}$ (called orthogonal matrices).
\begin{definition}[Spherical and elliptical distributions.]
Let $z$ be a $N$-dimensional random vector.
We say that $z$ has a \emph{spherically-contoured} (or simply spherical) distribution
if, for any $U \in \mathcal{O}(N)$, $Uz$ and $z$ are identically distributed.
We say that $t$ has an \emph{elliptically-contoured} (or elliptical)
distribution if $t = Az + \mu$ for a spherical random variable $z$,
non-singular\,{\footnotemark} matrix $A \in \mathbb{R}^{N \times N}$, and vector $\mu \in \mathbb{R}^N$.
\end{definition}
\footnotetext{
This definition can be relaxed to singular $A$, in which case $\operatorname{supp}(p) \subset
\operatorname{im}(A)$,
using the Lebesgue measure on $\operatorname{im}(A)$ instead of the one on $\mathbb{R}^n$
\citep[Theorem 2.4]{gelbrich1990formula}.
}
In other words, spherical distributions are rotationally symmetric around the
origin, and ellipticals are affine transformations thereof.
Elliptical families parametrized by $A$ and $\mu$ can be regarded as multivariate generalizations of location-scale families.
An important example of a spherical distribution is the standard Gaussian
distribution $\mathcal{N}(z; 0, \mathrm{Id})$;
anisotropic multivariate Gaussians are elliptical.
The following result characterizes spherical and elliptical densities.%
\footnote{Since not all distributions have a density function, a more general
characterization of elliptical distributions exists, based on characteristic
functions \citep[Theorem~2.1]{fang}. Since $\beta$-Gaussians have densities,
this characterization is not necessary in this section, so we omit it.}
\begin{proposition}[Characterization of spherical and elliptical densities]
Let $z$ be a spherical random variable. If $z$ has a density $p(z)$, then the
density must be of the form $p(z) = g(\|z\|^2)$ for some $g : \mathbb{R}_+ \to
\mathbb{R}_+$. By extension, for elliptically-distributed $t=Az + \mu$ with non-singular $A$,
if $z$ has a density as above, then the density of $t$ is
\begin{equation}\label{eq:elliptical_density}
p(t) = |\tilde{\Sigma}|^{-\nicefrac{1}{2}} g\big((t - \mu)^\top \tilde{\Sigma}^{-1} (t - \mu)\big).
\end{equation}
with $\tilde{\Sigma} = AA^\top$.
\end{proposition}
\begin{proof}
From \citet[Example 1.2]{fang} we have that $\|z\|^2$ is a maximal invariant under the group
$\mathcal{O}(N)$.
Therefore, by \citet[Theorem 1.1]{fang}, $p(z)$ is invariant w.r.t.\ $\mathcal{O}(N)$ iff $p(z) =
g(\|z\|^2)$. The change of density formula yields the elliptical case.
\end{proof}
This symmetry property allows us to characterize sphericals and ellipticals using a
useful stochastic representation. The following appears as a corollary in \citet{fang}. %
\begin{proposition}[Reparametrization]
\label{prop:sampling_elliptical}
Let $\mathbb{S}^N \coloneqq \{ u \in \mathbb{R}^N: u^\top u = 1\}$ denote the $(N-1)$-dimensional unit sphere.
A spherical random variable $z$ may be written as $z = ru$,
where $u \sim \text{Uniform}(\mathbb{S}^N)$
and $r \in \mathbb{R}_+$ is a non-negative scalar random variable representing the radius.
For elliptical $t$ with parameters $(\mu, A)$, we have
$t = \mu + r \cdot Au$.
\end{proposition}
As a consequence, \textbf{we may characterize the distribution of any spherical (and
thus any elliptical) in terms of the distribution of its radius $r$.}

\paragraph{Sampling.}
The stochastic representation in Proposition~\ref{prop:sampling_elliptical} can be seen as a generative story. It offers a simple procedure for sampling from any elliptical distribution in an efficient two-step
process: (1) draw a direction $u$ uniformly on the unit sphere $\mathbb{S}^N$, and (2) draw a radius $r>0$
according to the univariate radius density (which differs from case to case).
This algorithm is used for drawing $\beta$-Gaussian
distributions in Figure~\ref{fig:misfit}.

We may thus reduce generating multivariate elliptical random variates to
generating scalar variates with the distribution of the radius.
In the sequel, we introduce a particular family of elliptical distributions
dubbed ``$\beta$-Gaussians,'' we derive expressions for their essential quantities,
and characterize the distribution over the radius $r$, enabling sampling.
We show that $\beta$-Gaussians are instances of $\alpha$-sparse families (\S\ref{sec:deformed_sparse}) for $\beta = 2-\alpha$, and we derive closed-form expressions for their corresponding Fenchel-Young losses.

\subsection{$\beta$-Gaussians and $\alpha$-sparse families}

We proceed to the main result of this section, which shows
that the $\alpha$-sparse family induced by a quadratic
scoring function $f(t)$ is elliptical, related to the Pearson Type-II distribution,
and the distribution of its radius is related to the Beta distribution.
We start by defining $\beta$-Gaussian distributions.
This family generalizes the univariate $q$-Gaussians of \citet[\S 4.1]{naudts2009q},
for $q = \beta = 2-\alpha$. We use $\beta$ in this text for consistency.
\begin{definition}[$\beta$-Gaussian.]\label{def:beta_gaussian}
Let $\Sigma \succ 0$ and $\beta = 2-\alpha$.
The multivariate $\beta$-Gaussian distribution $\mathcal{N}_\beta(t; \mu, \Sigma)$ is the distribution
$\hat{p}_{\Omega_\alpha}[f]$ induced by the quadratic scoring function
\begin{equation}
f(t) = -\frac{1}{2} (t - \mu)^\top \Sigma^{-1} (t - \mu).
\end{equation}
From Proposition~\ref{prop:solution_rpm_tsallis}, the resulting density can be written as
\begin{equation}
\hat{p}_{\Omega_\alpha}[f](t) \,\,=\,\, \exp_{\beta}(f(t) - A_\alpha(f)) \,\,=\,\, [(\alpha-1)(-\tau + f(t))]_+^{\frac{1}{\alpha-1}},
\end{equation}
where $\tau = A_\alpha(f) - \frac{1}{\alpha-1}$ is a normalizing constant.

\end{definition}
Figure~\ref{fig:beta_gaussians}  shows examples of $\beta$-Gaussians in
1-d and 2-d; this includes
several kernels frequently used in density estimation \citep{silverman1986density}.
The next result, proved in  Appendix~\ref{sec:proof_beta_gauss}, shows that $\beta$-Gaussians are also elliptical distributions.

\begin{figure}[t]
\begin{center}
\includegraphics[width=0.439\columnwidth]{figures/beta_gaussian}
\includegraphics[width=0.55\columnwidth]{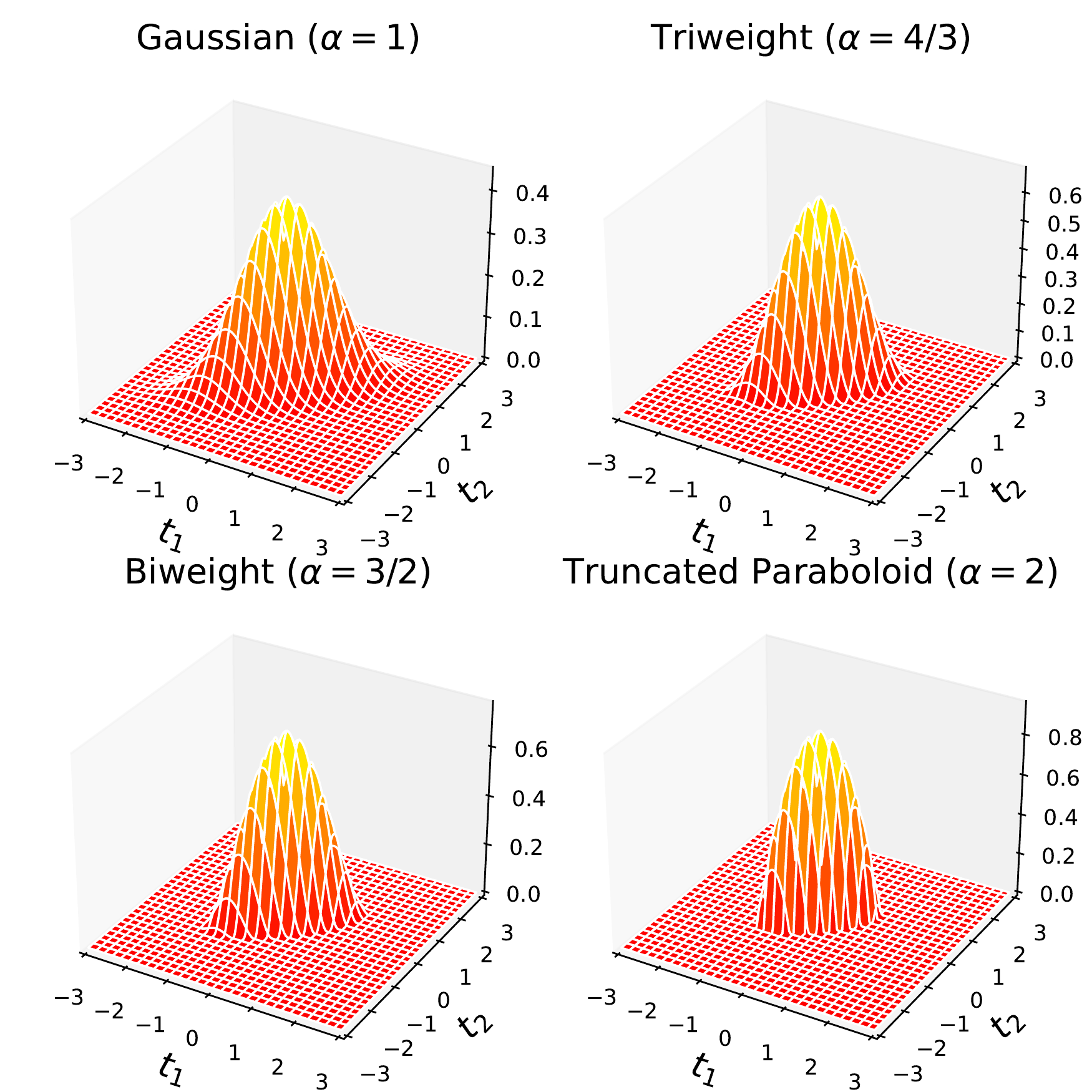}
\caption{$\beta$-Gaussians $\mathcal{N}_\beta$ for
several values of $\alpha=2-\beta$, in the univariate (left) and bivariate
(right) cases. In the univariate case, $\sigma^2=1$ except for $\alpha=0$,
where $\sigma^2=(2\pi)^{-1}$ (Cauchy distribution).
In the bivariate case, $\Sigma_{11} = .6$, $\Sigma_{22} = .48$, $\Sigma_{12}=\Sigma_{21}=.4$.  %
The case $\alpha=1$ corresponds
to a Gaussian, $\alpha<1$ to heavy-tail distributions ($t$-Student), and
$\alpha>1$ to zero-tail distributions, recovering scaled versions of the
biweight ($\alpha=\tfrac{3}{2}$), triweight ($\alpha=\tfrac{4}{3}$), and
Epanechnikov kernels ($\alpha=2$, same as truncated parabola) used in density
estimation. %
}
\label{fig:beta_gaussians}
\end{center}
\end{figure}

\begin{proposition}[$\beta$-Gaussians are elliptical]\label{prop:beta_gauss}
The multivariate $\beta$-Gaussians form a family of elliptical distributions
induced by the spherical base corresponding to $\mu=0, \Sigma=\mathrm{Id}$, \textit{i.e.},
the distribution $\hat{p}_{\Omega_\alpha}[f_0]$ induced by $f_0(z) = -\frac{1}{2} \|z\|^2$.
Moreover, $t \sim \mathcal{N}_\beta(t; \mu, \Sigma)$
admits the stochastic
representation $t = \mu + r \cdot Au$,
where $A = |\Sigma|^{-\frac{1}{2N+\frac{4}{\alpha-1}}} \Sigma^{\sfrac{1}{2}}$,
$u \sim \text{Uniform}(\mathbb{S}^N)$ and
$r$ is a random variable distributed as
\begin{equation}
\frac{r^2}{R^2} \sim \mathrm{Beta}\left(\frac{N}{2}, \frac{\alpha}{\alpha -
1}\right)\,,
\end{equation}
where $R$  is the radius of the supporting sphere of the standard $\beta$-Gaussian
$\mathcal{N}_{\beta}(z; 0, I)$, with value depending only on $N$ and $\alpha$,
\begin{equation}\label{eq:radius_beta_gaussian}
R = \left(
\frac{\Gamma(\nicefrac{N}{2}+\nicefrac{\alpha}{\alpha-1})}
{\Gamma(\nicefrac{\alpha}{\alpha-1}) \pi^{\nicefrac{N}{2}}}
\cdot
{\left(\frac{2}{\alpha-1}\right)}^{\nicefrac{1}{\alpha-1}}
\right)^{\frac{\alpha-1}{2+(\alpha-1)N}}
\end{equation}
Moreover,
defining $\tilde{\Sigma} = |\Sigma|^{-\frac{1}{N+\frac{2}{\alpha-1}}} \Sigma$,
the support of $\mathcal{N}_{\beta}(t; \mu, \Sigma)$ is the ellipsoid
\begin{equation}
\operatorname{supp}(\mathcal{N}_{\beta}(t; \mu, \Sigma)) = \{t: (t - \mu)^\top
\tilde{\Sigma}^{-1}(t-\mu) < R^2\}\,,
\end{equation}
and $R$ relates to the normalizing constant $\tau$ in Definition~\ref{def:beta_gaussian} by
\begin{equation}
\tau = -\frac{R^2}{2}|\Sigma|^{-\frac{1}{N+\frac{2}{\alpha-1}}}\,.
\end{equation}
\end{proposition}

It is worth noting from \eqref{eq:radius_beta_gaussian} that the radius $R$ does not depend on the $\beta$-Gaussian parameters $\mu$ or $\Sigma$, being  only a function of $N$ and $\alpha = 2-\beta$.
As $\alpha \to 1_+$, $R \rightarrow \infty$, and $z$ tends toward the Gaussian
distribution. For $\alpha=2$, we get the multivariate truncated paraboloid described in \S\ref{sec:sparsemax}.
The $\beta$-Gaussian family is related to the Pearson Type-II distribution \citep[Section 3.4]{fang}, in which the base radius variable $r^2$
is supported on $[0, 1]$ rather than $[0, R^2]$.
Our construction from the angle of regularized prediction maps therefore reveals a novel
connection and is particularly natural when learning the
support is part of the modeling task, as demonstrated in the experiments in \S\ref{sec:experiments}.

\subsection{Properties of $\beta$-Gaussians}

Now that we have shown that $\beta$-Gaussians are instances of both elliptical and $\alpha$-sparse family distributions, we state several important properties of these distributions, linking to the concepts introduced in the previous sections. We use several of these properties in our code implementations for \S\ref{sec:experiments}.

\begin{proposition}[Mean, variance and $\alpha$-negentropy.]\label{prop:mean_var_entropy_beta_gaussians}
Let $t \sim \mathcal{N}_\beta(t; \mu, \Sigma)$.
Then, $\mathbb{E}[t] = \mu$ and $\mathrm{Var}[t] = \frac{R^2}{N + \tfrac{2\alpha}{\alpha-1}} \tilde{\Sigma}$,
with $\tilde{\Sigma} = |\Sigma|^{-\frac{1}{N+\frac{2}{\alpha-1}}} \Sigma$, and
$R$ defined as in \eqref{eq:radius_beta_gaussian}.

Its Tsallis $\alpha$-negentropy is
\begin{equation}
\Omega_{\alpha}(p) = -\frac{1}{\alpha(\alpha-1)} + \frac{R^2 |\Sigma|^{-\frac{1}{N+\frac{2}{\alpha-1}}}}{2\alpha + N(\alpha-1)}.
\end{equation}
Therefore, $\mathrm{Var}[t]$, $\Omega_\alpha(p)$, and $\Sigma$ are related through the elegant formula
\begin{equation}
\mathrm{Var}[t] = \left(\frac{1}{\alpha} + (\alpha-1) \Omega_\alpha(p)\right) \Sigma,
\end{equation}
recovering $\mathrm{Var}[t] = \Sigma$ when $\alpha=1$ (Gaussian distribution).

\end{proposition}
\begin{proof}
The variance result follows from the Beta distribution moments, combined with \citet[Theorem~2.7]{fang}. The negentropy follows from Proposition~\ref{prop:entropy_normalizing}.
\end{proof}

The variance expression allows us to further compute the
2-Wasserstein distance between two $\beta$-Gaussians
(\citealt{gelbrich1990formula}, see also \citealt[Remark 2.32]{cot}), as

\begin{equation}
W_2^2\left(%
\mathcal{N}_{\beta}(\cdot; \mu_1, \Sigma_1),
\mathcal{N}_{\beta}(\cdot; \mu_2, \Sigma_2)\right)
= \|\mu_1 - \mu_2\|^2 +
\frac{R^2}{N + \tfrac{2\alpha}{\alpha-1}} \mathfrak{B}^2(\tilde{\Sigma}_1,
\tilde{\Sigma}_2),
\end{equation}
where $\mathfrak{B}^2(A, B) \coloneqq \operatorname{Tr}\left(A + B
- 2\left(A^{\nicefrac{1}{2}}BA^{\nicefrac{1}{2}}\right)^{\nicefrac{1}{2}}\right)$
is the squared Bures distance, and $\tilde\Sigma_{\{1,2\}}$ is as in
Proposition~\ref{prop:mean_var_entropy_beta_gaussians}.
As $\alpha \rightarrow 1_+$, the coefficient goes to 1 recovering the
Fr\'{e}chet distance; as $\alpha \rightarrow \infty$, the coefficient goes
toward $\frac{1}{2+N}$, recovering the Wasserstein distance between uniform distributions
on ellipsoids.

Next, we provide a closed-form expression for the Fenchel-Young loss between a quadratic scoring function and a $\beta$-Gaussian. This expression generalizes the KL divergence between multivariate Gaussians (cf. \eqref{eq:kl_gaussians}), recovered as a limit case when $\beta=\alpha=1$. We also provide an expression for the cross-$\Omega$ loss (see Definition~\ref{def:fyloss}), which we will use in the heteroscedastic regression experiments in \S\ref{sec:experiments}.
The full derivation is included in Appendix~\ref{sec:proof_beta_gaussian_fy} and makes use of Proposition~\ref{prop:entropy_normalizing}.
\begin{proposition}[Fenchel-Young loss for $\beta$-Gaussians.]\label{prop:beta_gaussian_fy}
Let $\beta=2-\alpha$. The Fenchel-Young loss induced by $\Omega_\alpha$ associated with a quadratic score function $f(t) = -\frac{1}{2}(t-\mu_f)^\top \Sigma_f^{-1} (t-\mu_f)$ and a $\beta$-Gaussian distribution $p(t) = \mathcal{N}_\beta(t; \mu, \Sigma)$ is:
\begin{eqnarray}
L_{\Omega_\alpha}(f, p) &=& \frac{1}{2}(\mu-\mu_f)^\top \Sigma_f^{-1} (\mu-\mu_f) +
\frac{R^2}{2\alpha + N(\alpha-1)} \cdot \\
&& \cdot \left( |\Sigma|^{-\frac{1}{N + \frac{2}{\alpha-1}}} \left(1 + \frac{\alpha-1}{2} \mathrm{Tr}(\Sigma_f^{-1} \Sigma) \right) - |\Sigma_f|^{-\frac{1}{N + \frac{2}{\alpha-1}}} \left( 1 + \frac{N(\alpha-1)}{2}\right)\right).\nonumber
\end{eqnarray}
The corresponding cross-$\Omega$ loss is
\begin{eqnarray}
L_{\Omega_\alpha}^{\times}(f_\theta, p) &=& \frac{1}{2}(\mu-\mu_f)^\top \Sigma_f^{-1} (\mu-\mu_f) + \frac{1}{\alpha(\alpha-1)}
+
\frac{R^2}{2\alpha + N(\alpha-1)} \cdot \\
&& \cdot \left( |\Sigma|^{-\frac{1}{N + \frac{2}{\alpha-1}}} \left(\frac{\alpha-1}{2} \mathrm{Tr}(\Sigma_f^{-1} \Sigma) \right) - |\Sigma_f|^{-\frac{1}{N + \frac{2}{\alpha-1}}} \left( 1 + \frac{N(\alpha-1)}{2}\right)\right).
\end{eqnarray}
\end{proposition}

Much like the cross-entropy between 1-d Gaussians induces a hyperbolic geometry in
the $[\mu, \sigma]$ half-plane, the Fenchel-Young loss induces a similar
curvature, discussed briefly in \S\ref{sec:proof_beta_gaussian_fy}.
As maximum likelihood is not suitable for learning with distributions that may
assign zero probability to data, Fenchel-Young losses are of great value
for learning and modelling data with with $\beta$-Gaussian distributions,
as we demonstrate in the experiments.

\section{Continuous Fusedmax}\label{sec:fusedmax}

We now switch gears to a different usage of regularized prediction maps,
designed to induce smoothness, focusing for simplicity on $S=\mathbb{R}$.
The reader that is interested in the proceeding with applications of $\beta$-Gaussians may skip this section and jump straight to \S\ref{sec:attention}.

In discrete attention mechanisms, the regularizer $\Omega$ has been used to
encode further prior assumptions. In particular,
\citet{fusedmax} introduce \emph{fusedmax}, a variant of sparsemax that
encourages adjacent items in a sequence to get assigned the same probability:
\begin{equation}\label{eq:discrete_fusedmax}
\operatorname{fusedmax} : \mathbb{R}^n \rightarrow \triangle^n,\quad
\operatorname{fusedmax}(\tilde{f})\coloneqq
\argmin_{\tilde{p} \in \triangle^n} \frac{1}{2} \|\tilde{p} - \tilde{f} \|^2 +
\gamma \sum_{i=2}^n |\tilde{p}_i - \tilde{p}_{i-1}|\,,
\end{equation}
where we use $\tilde{\cdot}$ to denote vectors, which we may interpret as
discretized functions.
For example, if the sequence corresponds to English words, it makes sense to
cluster the probabilities of adjacent words, since they are more likely to form
meaningful phrases.
In this section, \textbf{we extend fusedmax to the continuous case},
highlight connections with total variation denoising, and provide closed-form
expressions for some common cases.
As we shall see, continuous generalizations involve penalizing the
\textbf{derivative} of $p$.

The regularizers used so far in this paper, \textit{e.g.},
$\Omega_2(p) = \sfrac{1}{2} (-1 + \int_S p(t)^2)$, are integral functionals that
only depend on $p(t)$.
We make use of functional $L_p$ norms, defined as
\[ \|f\|_p \coloneqq \left(\int_S |f(t)|^p \mathrm{d}t\right)^{\frac{1}{p}}\,, \]
over a space of functions for which the integral of interest is finite.
With this notation, $\Omega_2(p) = -1+\frac{1}{2}\|p\|_2^2$.
To induce smoothing, we must
additionally regularize $p'$. In this section, we derive two appropriate
regularizers using the $L_1$ and squared $L_2$ norms of $p'$.
The resulting problems are closely related to operators from signal processing.
We give expressions for the regularized prediction map obtained in some tractable cases.

\subsection{$L_1$ gradient penalty and total variation}
We first consider regularizing the \textbf{total variation} of $p$,
a strategy motivated by classic research in continuous signal denoising, commonly known as
the Rudin-Osher-Fatemi (ROF) denoising model \citep{rof1992}.
If $p$ is
differentiable and its derivative is Riemann integrable,\footnote{Importantly,
Definition~\ref{def:total_variation} lets us assess
total variation for non-differentiable functions. The
ROF model allows -- and in fact often yields -- non-differentiable solutions.}
the total variation of
$p$ takes the value $\Otv(p) = \int_S |p'(t)| = \|p'\|_1$.
Define
\begin{equation}
\Orof(p) \coloneqq \Omega_2(p) + \gamma\Otv(p)\,.
\end{equation}
The induced regularized prediction map is, up to a constant,
\begin{equation}
\hat{p}_{\Orof}[f] := \argmin_{p \in \mathcal{M}_+^1(S)}
\frac{1}{2} \int_S (p(t) - f(t))^2 + \gamma\Otv(p)\,.
\end{equation}
Without the $\mathcal{M}_+^1(S)$ constraint, this optimization problem would be
equivalent to the standard ROF signal denoising model; adding the constraint
ensures the solution is a smoothed density.
Since we are optimizing over a space of functions, general solutions are not
available for arbitrary $f$.
We first show that Euler's finite difference method, often used to
discretize calculus of variations problems, recovers exactly the discrete fusedmax
problem. Then, we derive exact solutions for a useful class of functions $f$.
\begin{proposition}[Discretized ROF yields fusedmax.]
Denote the $n$-dimensional $h$-simplex as
$\triangle_h^n = \{\tilde{p} \in \mathbb{R}^n : \tilde{p} \geq 0, \sum_i
\tilde{p}_i =
\nicefrac{1}{h}\}.$
Applying Euler's finite difference method with width $h$ on $\hat{p}_{\Orof}$
gives the discretized version of the ROF regularized prediction map
\begin{equation}
\hat{p}^{(h)}_{\Omega_{\gamma\text{ROF}}}[\tilde{f}^{(h)}] =
\argmin_{\tilde{p} \in \triangle_h^n} \frac{h}{2}\|\tilde{p} -
\tilde{f}^{(h)}\|^2 + \gamma\sum_{i=1}^n |\tilde{p}_i -
\tilde{p}_{i-1}|\,,
\end{equation}
where $\tilde{f}^{(h)} \in \mathbb{R}^n$ is a discretized (sampled) function.
\end{proposition}
In particular, with $h=1$, this yields the discrete fusedmax from
\eqref{eq:discrete_fusedmax}, and for other choices of $h>0$ the problem can be
transformed into fusedmax via scaling.

Using the discretized case as motivation, we now turn to exactly solving the
continuous case. For symmetric unimodal $f$, we obtain a
direct, intuitive expression for  $\hat{p}_{\Orof}[f]$.
\begin{proposition}[Form of ROF-smoothed solutions for unimodal scores.]\label{prop:fusedmax_unimodal_symmetric}
Let $f : \mathbb{R} \to \mathbb{R}$ be even and unimodal, \textit{i.e.},
$f(-t) = f(t)$, strictly increasing on $(-\infty, 0)$ and strictly decreasing on
$(0, \infty)$. We have
\begin{equation}\label{eq:rpm_rof}
\hat{p}_{\Orof}[f](t) = [f_a(t) - \tau]_+,
\quad \text{where} \quad f_a(t) \coloneqq
\begin{cases}
f(a), & t \in (-a, a), \\
f(t), & \text{otherwise}.
\end{cases}
\end{equation}
The support is $(-b, b)$ where $\tau=f(b)$ and $a, b$ can be found by solving
\begin{equation}
-a f(a) + \int_0^a f = \gamma\,,\qquad
-b f(b) + \int_0^b f = \frac{1}{2}+\gamma\,.
\end{equation}
\end{proposition}
The proof, which we include in
Appendix~\ref{sec:proof_fusedmax_unimodal_symmetric}, invokes the \emph{taut
string} algorithm for solving the ROF optimization \citep{Grasmair2006,overgaard2019taut}.

\paragraph{Example: capped triangular and capped truncated parabola distributions.}
For the negative absolute value function $f(t) = -\nicefrac{|t|}{\sigma}$, we get
$a=\sqrt{2\sigma\gamma}$,
$b=\sqrt{\sigma(1+2\gamma)}$,
and $\tau=-\sqrt{\frac{1+2\gamma}{\sigma}}$.
(Figure~\ref{fig:sobolev}, left). For the parabola $f(t) = -\nicefrac{t^2}{2\sigma^2}$
we get
$a=\sqrt[3]{3\sigma^2\gamma}$,
$b=\sqrt[3]{3\sigma^2(1+2\gamma)/2}$, and
$\tau=-\frac{1}{2}\left(\frac{3}{2} \frac{1+2\gamma}{\sigma}\right)^\frac{2}{3}$
(Figure~\ref{fig:sobolev}, right).

\begin{figure}\centering
\includegraphics[width=.49\textwidth]{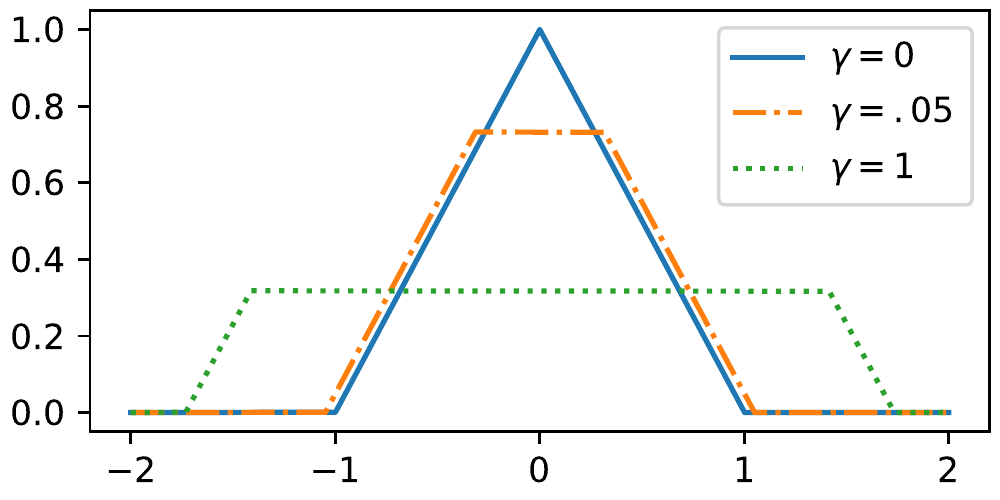}
\includegraphics[width=.49\textwidth]{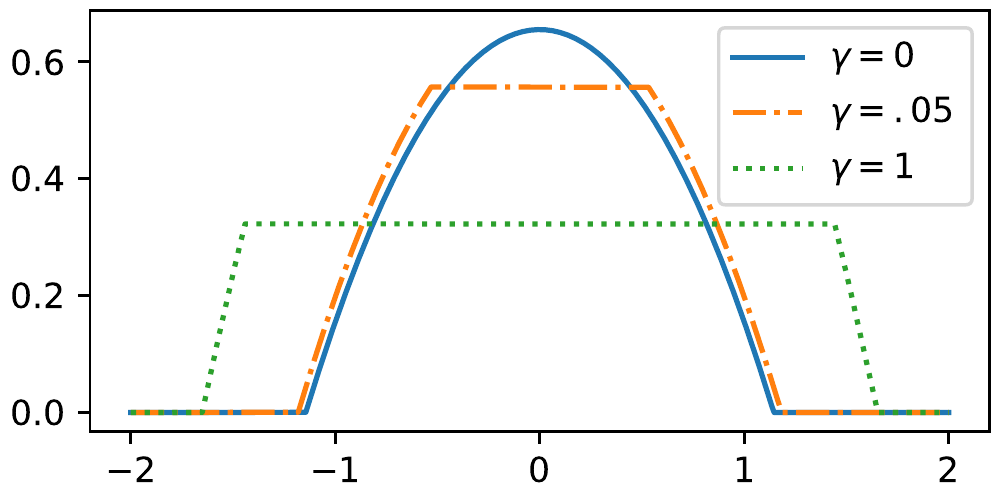}
\includegraphics[width=.49\textwidth]{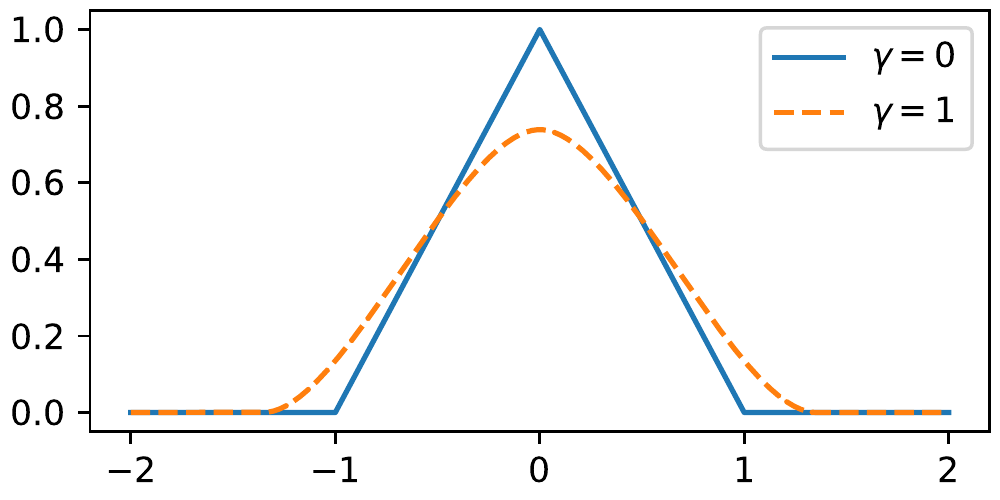}
\includegraphics[width=.49\textwidth]{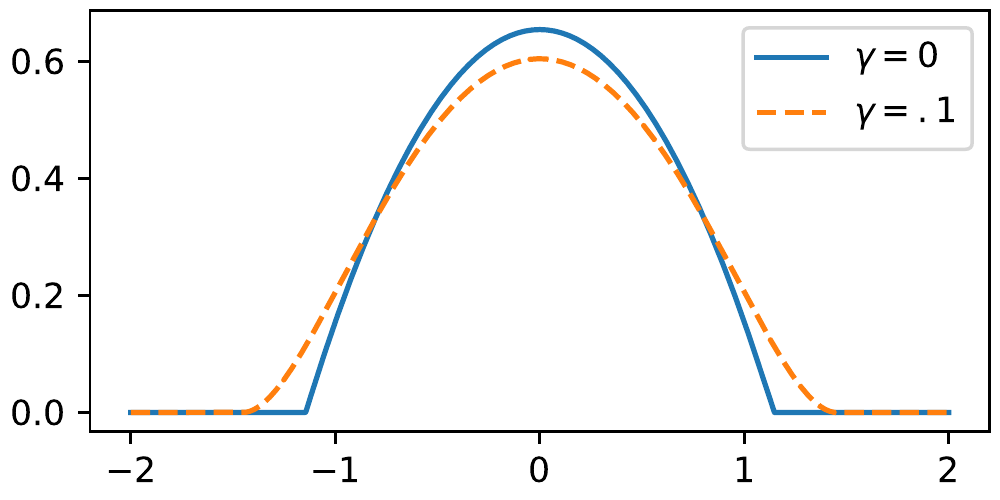}
\caption{\label{fig:sobolev}Distributions induced by regularization of the
derivative. Top: ROF regularization $\gamma\|p'\|_1$, bottom: squared
$L_2$ regularization $\gamma\|p'\|_2^2$.}
\end{figure}

\subsection{$L_2$ gradient penalty and smooth sparsemax}

In contrast to the previous section, we now consider a quadratic penalty on the
derivative,
\begin{equation}
\Omega_{2,2}(p) \coloneqq \frac{1}{2} \int_S |p(t)|^2 + \frac{\gamma}{2} \int_S (p'(t))^2 \,.
\end{equation}
The corresponding regularized prediction map is
\[
\hat{p}_{\Omega_{2,2}}[f] = \argmin_{p \in \mathcal{M}_+^1(S)}
\frac{1}{2}\int_S \left(p(t) - f(t)\right)^2 + \frac{\gamma}{2} \int_S
\left(p'(t)\right)^2\,.
\]
The quadratic regularization on the derivative of $p$ ensures the solution is
smooth. The following result shows how to derive the regularized prediction map.
\begin{proposition}[Form of $L_2$-smoothed solutions for unimodal scores.]
Assume $f$ is an even function, strictly decreasing and with continuous first
derivative on $(0, \infty)$.
The $\Omega_{2,2}$-regularized prediction map is a continuously differentiable function
\[
\hat{p}_{\Omega_{2,2}}[f](t) = \begin{cases}\bar{p}(t)
\coloneqq C \cosh\left(\frac{t}{\sqrt{\gamma}}\right) - \left(F(t) + F(-t)\right) - \tau,
& t \in [-b, b], \\
                    0, &      t \not\in [-b, b]\,,\end{cases}
\]
where
\[
F(t) \coloneqq \frac{\exp\left(\frac{t}{\sqrt{\gamma}}\right)}{2\sqrt{\gamma}}
\int f(t) \exp\left(-\frac{t}{\sqrt{\gamma}}\right) \mathrm{d}t\,,
\]
and $\tau, b,$ and $C$ are uniquely determined by continuity
at $b$ and the constraint $\int_\mathbb{S}p = 1$.
\end{proposition}
The proof is given in Appendix~\ref{sec:sobolev_smooth_proof}.
Below, we demonstrate a few examples.
\paragraph{Smooth truncated parabola.}
Let $\beta=\gamma^{-\nicefrac{1}{2}}$.
For $f(t) = - \nicefrac{t^2}{2\sigma^2}$, computation yields
\[
\bar{p}(t) = C\cosh(\beta t) - \frac{t^2}{2\sigma^2} - \frac{1}{\beta^2
\sigma^2} - \tau\,.
\]
To solve for the unknown constants $\tau$ and $C$, we use the first and second
order conditions $\bar{p}(b)=0$ and $\bar{p}'(b)=0$, yielding, respectively,
\begin{equation}
\tau = C\cosh(\beta b) - \frac{b^2}{2\sigma^2} - \frac{1}{\beta^2 \sigma^2}
\qquad\text{and}\qquad
C = \frac{b}{\beta\sigma^2\sinh(\beta b)}\,.
\end{equation}
Finally, to find $b$ we use the condition $I(b) \coloneqq \int_{-b}^b \bar{p}(t) = 1$.
The integral has the closed-form expression
$I(b)
= \frac{2b}{\sigma^2\beta^2}
- \frac{2b^2 \coth(\beta b)}{\sigma^2\beta}
+ \frac{2b^3}{3\sigma^2}$,
and we can solve $I(b)=1$ using numerical root finding methods.
For example, the standard smooth parabola ($\sigma=1$, $\gamma=1$) yields the
equation $2b(1-b \coth(b) + b^2) = 1$, with root $b \approx 1.98$.
This density is illustrated in Figure~\ref{fig:sobolev}.%
\paragraph{Smooth triangular.}
For the triangular function $f(t) = -\nicefrac{|t|}{\sigma}$,
following the same steps as for the parabola, we obtain
\[
\bar{p}(t) = C\cosh(\beta t) - \frac{|t|}{\sigma} - \frac{e^{-\beta |t|}}{\beta
\sigma} - \tau\,,
\qquad \text{where} \qquad
\tau = C\cosh(\beta b) - \frac{b}{\sigma} - \frac{e^{-\beta b}}{\beta
\sigma}\,,
\]
and the integral equation to solve for $b$ is
\[
I(b) =
\frac{b^{2}}{\sigma} + \frac{2 b e^{- B \beta}}{\beta \sigma} - \frac{4 b
\cosh{\left(b \beta \right)}}{\beta \sigma \left(e^{b \beta} + 1\right)} -
\frac{2}{\beta^{2} \sigma} + \frac{2 e^{- b \beta}}{\beta^{2} \sigma} + \frac{4
\sinh{\left(b \beta \right)}}{\beta^{2} \sigma \left(e^{b \beta} + 1\right)}.
\]
The standard smooth triangular is illustrated in
Figure~\ref{fig:sobolev}.%

\section{Continuous Attention Mechanisms}\label{sec:attention}

We now use some of the results obtained in the previous sections to develop attention mechanisms on continuous spaces. We assume in this section $S = \mathbb{R}^N$.

Attention mechanisms have become a key component of  neural networks %
\citep{bahdanau2014neural,sukhbaatar2015end,vaswani2017attention}. They dynamically detect and extract relevant input features (such as words in a text or regions of an image). So far, attention has only been applied to discrete domains; we use our framework to generalize it to {\it continuous} spaces.

\paragraph{Discrete attention.}
Assume an input object split in $L=|S|$ pieces, \textit{e.g.}, a sequence with $L$ elements  or an image with $L$ regions.
A vanilla attention mechanism works as follows: each piece has as a $D$-dimensional representation (\textit{e.g.}, coming from an RNN or a CNN), yielding a matrix $V \in \mathbb{R}^{D\times L}$. %
These representations are compared against a query vector (\textit{e.g.}, by using an additive model, \citealt{bahdanau2014neural}), leading to a score vector $f = [f_1, \ldots, f_L] \in \mathbb{R}^L$.
Intuitively, the relevant pieces that need attention should be assigned high scores. Then, a transformation $\rho : \mathbb{R}^L \rightarrow \triangle^{L}$ (\textit{e.g.}, softmax or sparsemax) is applied to the score vector to produce a probability vector $p = \rho(f)$.
We may see this as an $\Omega$-regularized prediction map, as shown in \S\ref{sec:rpm}. The probability vector $p$ is then used to compute a weighted average of the input representations, via $c = Vp \in \mathbb{R}^D$. This context vector $c$ is finally used to produce the network's decision.

\subsection{The continuous case: scoring and value functions}
\label{subsec:continuous_attention}

The extension of $\Omega$-regularized prediction maps to arbitrary domains in Definition~\ref{def:regularized_prediction} opens the door for constructing {\bf continuous attention mechanisms}. The idea is simple: instead of splitting the input object into a finite set of pieces, we assume an underlying continuous domain: \textit{e.g.}, text or a speech signal may be represented as a function $V:S \rightarrow \mathbb{R}^{D}$ that maps points in the real line ($S \subseteq \mathbb{R}$, continuous time) onto a $D$-dimensional vector representation, representing how the signal evolves over time; images (visual scenes) may be regarded as a smooth function in 2D ($S \subseteq \mathbb{R}^2$), instead of being split into regions in a grid.

Instead of scores $[f_1, \ldots, f_L]$, we now have a {\bf scoring function} $f: S  \rightarrow
\mathbb{R}$, which we map to a probability density $p \in \mathcal{M}_+^1(S)$.
This density is used in tandem with the  value mapping $V: S \rightarrow \mathbb{R}^D$ to obtain a context vector $c = \mathbb{E}_{p} [V(t)]  \in \mathbb{R}^D$.
This is illustrated in Figure~\ref{fig:attention}.
Since $\mathcal{M}_+^1(S)$ may be infinite dimensional, we need to parametrize  $f$, $p$, and $V$ to be able to compute in a finite-dimensional parametric space.

\begin{figure}[t]
\begin{center}
\includegraphics[width=0.4\textwidth]{figures/score_vector}
\,\,
\includegraphics[width=0.4\textwidth]{figures/score_function}\\
\includegraphics[width=0.4\textwidth]{figures/prob_vector}
\,\,
\includegraphics[width=0.4\textwidth]{figures/prob_density}\\
\includegraphics[width=0.4\textwidth]{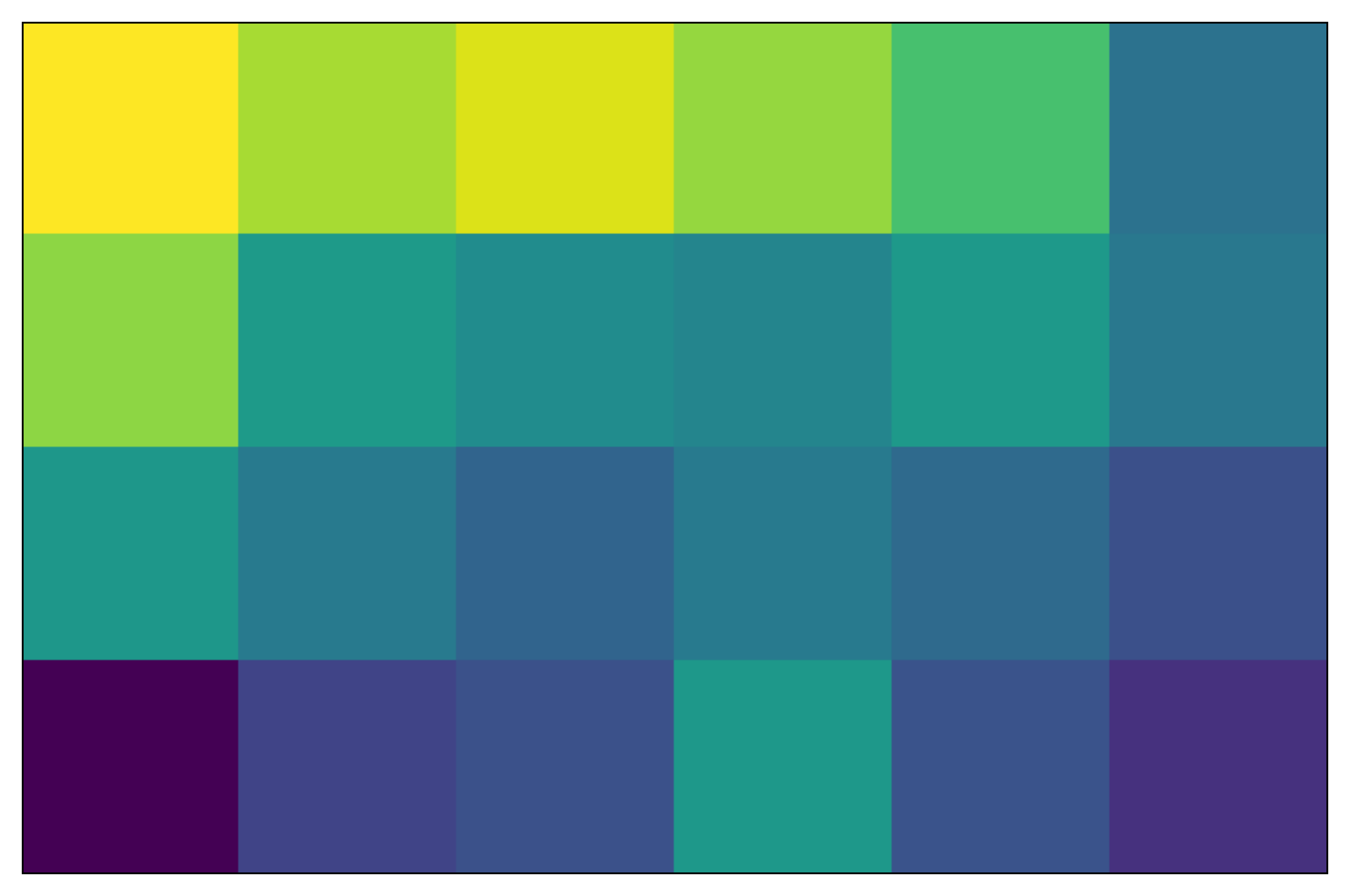}
\,\,
\includegraphics[width=0.4\textwidth]{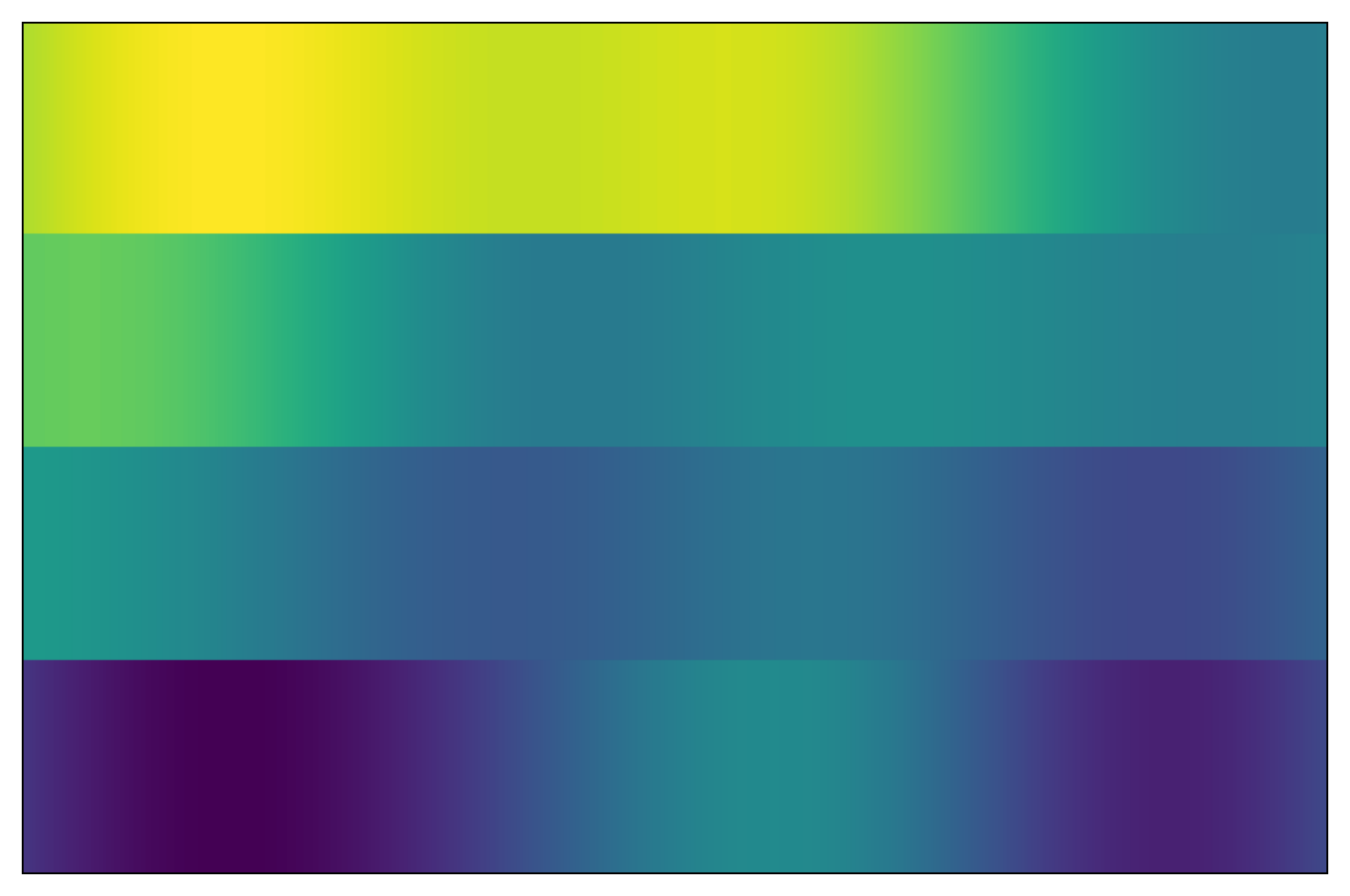}
\caption{\textbf{From discrete to continuous attention.} Left: discrete
attention maps a score vector into a probability mass function ({\it e.g.} via a
softmax transformation) and returns a weighted average of the columns of a value
matrix. Right: continuous attention (i) replaces the score vector by a scoring
function (shown in the top right; an arbitrary scoring function is used in this figure for illustration), (ii) uses a continuous $\Omega$-regularized prediction map to map it to a probability density (middle right), and (iii)
returns an expectation over a value function (the value function is shown in the bottom right). In both cases, the number of rows
in the matrix corresponds to the dimensionality $D$. \label{fig:attention}}
\end{center}
\end{figure}

\paragraph{Building attention mechanisms.}
We represent $f$ and $V$ using basis functions,
$\phi: S \rightarrow \mathbb{R}^M$ and $\psi: S \rightarrow \mathbb{R}^N$,
defining $f_\theta(t) = \theta^\top \phi(t)$
and $V_B(t) = B\psi(t)$,
where $\theta \in \mathbb{R}^M$ and $B \in \mathbb{R}^{D \times N}$. %
The scoring function $f_\theta$ is mapped into a probability density $p := \hat{p}_{\Omega}[f_\theta]$,
from which we compute the context vector as $c=\mathbb{E}_{p}[V_B(t)]$.
From the definition of $V_B(t)$, this is equivalent to writing $c = Br$,
where $r =\mathbb{E}_{p}[\psi(t)]$.
Summing up, we define general attention mechanisms as follows.
\begin{definition}[Attention mechanism.]\label{def:attention_mechanism}
Let $\langle S, \Omega, \phi, \psi \rangle$ be a tuple with  $\Omega: \mathcal{M}_+^1(S) \rightarrow \mathbb{R}$, $\phi:S\rightarrow \mathbb{R}^M$, and $\psi:S\rightarrow \mathbb{R}^N$.
An {attention mechanism} on $\langle S, \Omega, \phi, \psi \rangle$ is a mapping $\rho: \Theta \subseteq \mathbb{R}^M \rightarrow \mathbb{R}^N$, defined as:
\begin{equation}\label{eq:attention_expectation}
    \rho(\theta) = \mathbb{E}_{p}[\psi(t)],
\end{equation}
with $p = \hat{p}_\Omega[f_\theta]$ and $f_\theta(t) = \theta^\top \phi(t)$.
If $\Omega = \Omega_\alpha$, we  call this \textbf{entmax} attention, denoted as $\rho_\alpha$. The values $\alpha=1$ and $\alpha=2$ lead to \textbf{softmax} and \textbf{sparsemax} attention, respectively.
\end{definition}

\paragraph{Example: Finite attention.}
By plugging $S = \{1, ..., L\}$ and $\phi(k) = \psi(k) = e_k$ (Euclidean canonical basis) in our Definition~\ref{def:attention_mechanism}, we recover the discrete attention of  \citet{bahdanau2014neural}.
Still in the finite case, if $\phi(k)$ and $\psi(k)$ are key and value vectors and $\theta$ is a query vector, this recovers the key-value attention of \citet{vaswani2017attention}.

\paragraph{Example: Continuous attention with quadratic scoring function.}
On the other hand, for $S = \mathbb{R}^D$ and $\phi(t) = [t, \mathrm{vec}(tt^\top)]$ -- which leads to a quadratic scoring function $f_\theta(t) = \theta^\top \phi(t)$ -- we obtain new attention mechanisms (assessed experimentally for the 1-d and 2-d cases
in \S\ref{sec:experiments}): for $\alpha=1$, the underlying density $p$ is a \textbf{Gaussian}, and for $\alpha=2$, it is a \textbf{truncated paraboloid} (see Table~\ref{tab:distributions} and \S\ref{sec:sparsemax}). Intermediate cases encompass the \textbf{biweight} ($\alpha=1.5$) and \textbf{triweight} ($\alpha=4/3$) cases, part of the elliptical family described in \S\ref{sec:elliptical}.
In all these cases, we show (Appendix~\ref{sec:gaussian_basis}) that the expectation \eqref{eq:attention_expectation} is tractable (1-d) or simple to approximate numerically (2-d) if
$\psi$ are Gaussian RBFs, and we use this fact in \S\ref{sec:experiments}.
Algorithm~\ref{algo:forward_backward_gaussian} shows pseudo-code for the case $\alpha=1$.

\begin{algorithm}[t]
\small
\SetAlgoLined
\SetKwInput{KwInput}{Parameters}
\SetKwFunction{FRegression}{Regression}
\SetKwFunction{FForward}{Forward}
\SetKwFunction{FBackward}{Backward}
\def\algspace{.5\baselineskip}
\SetKwProg{Fn}{Function}{:}{}
\KwInput{Gaussian RBFs $\psi(t) = [\mathcal{N}(t; \mu_j, \Sigma_j)]_{j=1}^N$, basis functions $\phi(t) = [t, \mathrm{vec}(tt^\top)]$, value function $V_B(t) = B\psi(t)$ with $B \in \mathbb{R}^{D \times N}$, scoring function $f_\theta(t) = \theta^\top \phi(t)$ with $\theta \in \mathbb{R}^M$}
\vspace{\algspace}
\Fn{\FForward{$\theta := [\Sigma^{-1}\mu, -\frac{1}{2}\Sigma^{-1}]$}}{
    $r_j \leftarrow \mathbb{E}_{\hat{p}_\Omega[f_\theta]}[\psi_j(t)] = \mathcal{N}(\mu, \mu_j, \Sigma+\Sigma_j), \quad \forall j \in [N]$\hfill%
    \tcp*[h]{Eqs.\,\eqref{eq:attention_expectation}, \eqref{eq:continuous_softmax_forward_pass}}\\
    \KwRet{$c \leftarrow Br$ (context vector)}%
}
\vspace{\algspace}
\vspace{\algspace}
\Fn{\FBackward{$\frac{\partial \mathcal{L}}{\partial c}, \theta := [\Sigma^{-1}\mu, -\frac{1}{2}\Sigma^{-1}]$}}{
    \For{$j\gets1$ \KwTo $N$}{
        $\tilde{s} \leftarrow \mathcal{N}(\mu, \mu_j, \Sigma+\Sigma_j), \,\,
        \tilde{\Sigma} \leftarrow (\Sigma^{-1}+\Sigma_j^{-1})^{-1}, \,\,
        \tilde{\mu} \leftarrow \tilde{\Sigma}(\Sigma^{-1}\mu + \Sigma_j^{-1}\mu_j)$\\
        $\frac{\partial r_j}{\partial \theta} \leftarrow \mathrm{cov}_{\hat{p}_\Omega[f_\theta]}(\phi(t), \psi_j(t)) = [\tilde{s}(\tilde{\mu} - \mu); \tilde{s}(\tilde{\Sigma} + \tilde{\mu}\tilde{\mu}^\top - \Sigma - \mu\mu^\top)]$\hfill \tcp*[h]{\eqref{eq:jacob},~\eqref{eq:continuous_softmax_backward_pass_01}-\eqref{eq:continuous_softmax_backward_pass_02}}\\
    }
    \KwRet{$\frac{\partial \mathcal{L}}{\partial \theta} \leftarrow \left(\frac{\partial r}{\partial \theta}\right)^\top B^\top \frac{\partial \mathcal{L}}{\partial c}$}
}

\caption{Continuous softmax attention: $S=\mathbb{R}^D$, $\Omega=\Omega_1$,  Gaussian RBFs.\label{algo:forward_backward_gaussian}}
\end{algorithm}

\paragraph{Defining the value function $V_B(t)$.}
In many problems, the input is a discrete sequence of observations
(\textit{e.g.}, audio samples or text) or it was discretized (\textit{e.g.},
visual scenes), at locations $\{t_\ell\}_{\ell=1}^L$. To turn such an input into a continuous signal, we need to smooth and interpolate these observations. If we start with a  discrete encoder representing the input as a matrix $H \in \mathbb{R}^{D \times L}$, one way of obtaining a value mapping $V_B: S \rightarrow \mathbb{R}^D$ is by ``approximating'' $H$ with {\it multivariate ridge regression}. With  $V_B(t) = B \psi(t)$,
where $B \in \mathbb{R}^{D \times N}$,
and packing the basis vectors $\psi(t_\ell)$ as columns of matrix $F \in \mathbb{R}^{N\times L}$, we obtain: \begin{equation}\label{eq:B_regression}
    B^\star \,\,=\,\, \arg\min_B \|BF - H\|_F^2 + \lambda \|B\|_F^2 \,\,=\,\,
HF^\top (FF^\top + \lambda \mathrm{Id}_N)^{-1} \,\,=\,\, HG,
\end{equation}
where $\|\cdot\|_F$ is the Frobenius norm, and
the $L\times N$ matrix $G = F^\top (FF^\top + \lambda \mathrm{Id}_N)^{-1}$ depends only on the values of the basis functions at discrete time steps and can be obtained off-line for different input lenghts $L$.
The result is an expression for $V_B$ with $ND$ coefficients, cheaper than $H$ if $N \ll L$. %

\subsection{Gradient backpropagation with continuous attention}\label{sec:jacobian}

The next proposition, based on Proposition~\ref{prop:gradient_A} and proved in Appendix~\ref{sec:proof_jacobian_entmax}, allows backpropagating over continuous entmax attention mechanisms.
It uses the definition of {\it generalized $\beta$-covariance} presented in \eqref{eq:beta_covariance} and the proof is similar to that of Proposition~\ref{prop:gradient_hessian_fy}.
\begin{proposition}[Jacobian expression]
\label{prop:jacobian_entmax}
Let $p = \hat{p}_{\Omega_\alpha}[f_\theta]$ with $f_\theta(t) = \theta^\top \phi(t)$.
The Jacobian of the $\alpha$-entmax transformation $\rho_\alpha$ \eqref{eq:attention_expectation} is:
    \begin{equation}\label{eq:jacob}
        J_{\rho_\alpha}(\theta) = \frac{\partial \rho_\alpha(\theta)}{\partial \theta} = \mathrm{cov}_{p, 2-\alpha}(\phi(t), \psi(t)).
    \end{equation}
\end{proposition}
In the finite case, \eqref{eq:jacob} reduces to the expressions for the Jacobian of softmax and sparsemax derived by \citet{Martins2016ICML}:
\begin{equation}\label{eq:jacobians_discrete}
J_{\mathrm{softmax}}(f) = \mathrm{Diag}(p) - pp^\top, \qquad
J_{\mathrm{sparsemax}}(f) = \mathrm{Diag}(s) - ss^\top/(1^\top s),
\end{equation}
where $p=\mathrm{softmax}(f)$, and $s$ is a binary vector whose $\ell$\textsuperscript{th} entry is 1 iff $\ell \in \mathrm{supp}(\mathrm{sparsemax}(f))$. %

\paragraph{Example: Gaussian RBFs.}
As before, let $S = \mathbb{R}^D$, $\phi(t) = [t, \mathrm{vec}(tt^\top)]$, and $\psi_j(t) = \mathcal{N}(t; \mu_j, \Sigma_j)$.
For $\alpha=1$, we obtain closed-form expressions for the expectation \eqref{eq:attention_expectation} and the Jacobian \eqref{eq:jacob}, for any $D \in \mathbb{N}$:
$\hat{p}_\Omega[f_\theta]$ is a Gaussian, the expectation
\eqref{eq:attention_expectation} is the integral of a product of Gaussians, and the covariance \eqref{eq:jacob} involves first- and second-order Gaussian moments. Pseudo-code for the case $\alpha=1$ is shown as Algorithm~\ref{algo:forward_backward_gaussian}.
For $\alpha=2$, $\hat{p}_\Omega[f_\theta]$ is a truncated paraboloid.
In the 1-d case, both \eqref{eq:attention_expectation} and \eqref{eq:jacob} can be expressed in closed form in terms of the $\mathrm{erf}$ function.
The same holds more generally if $\alpha$ is of the form $\alpha = \frac{n+1}{n}$ for $n
\in \mathbb{N}$, which includes the biweight and triweight attention cases.%
\footnote{This is shown in Appendix~\ref{sec:gaussian_basis} by making use of closed-form expressions for $\int t^{n} \mathcal{N}(t; 0, 1) dt$ for $n \in \mathbb{N}$.} %
In the 2-d case, we can reduce the problem to 1-d integration by using the change of variables formula and working with polar coordinates. Appendix~\ref{sec:gaussian_basis} derives concrete expressions.

We use the facts above in the experimental section (\S\ref{sec:experiments}), where we experiment with $\beta$-Gaussian attention in audio classification and vision applications.

\section{Experiments}\label{sec:experiments}

We illustrate the usefulness of the theoretical results developed in the previous sections by running experiments with continuous attention mechanisms with several choices of $\beta$-Gaussian densities (\S\ref{sec:exp_continuous_attention}), and on heteroscedastic regression with continuous Fenchel-Young losses (\S\ref{sec:exp_fy_losses}).

\subsection{Continuous attention mechanisms}\label{sec:exp_continuous_attention}

We test our continuous attention mechanisms on two tasks: audio classification (1-d) and visual question answering (2-d).%
\footnote{All dataset statistics, architecture details, and hyperparameters are described in Appendix~\ref{sec:model_hyperparams}.}

\paragraph{1-d: Audio classification.} We use the UrbanSound8k dataset,\footnote{\url{https://urbansounddataset.weebly.com/}} whose inputs are short urban sound excerpts ($\leq 4s$) from 10 classes: \texttt{air conditioner}, \texttt{car horn}, \texttt{children playing}, \texttt{dog bark}, \texttt{drilling}, \texttt{engine idling}, \texttt{gun shot}, \texttt{jackhammer}, \texttt{siren}, and \texttt{street music}.
We use a 16kHz sampling rate for all audios.
We transform the input signal into a sequence of vectors using short-time Fourier transform with 400 points, a window size of 25ms, and a hop size of 10ms. After this transformation, we extract 80 Mel-frequency filter banks by applying equally-spaced triangular filters.
Our baseline is a model with a single convolutional 1-d layer followed by a discrete attention mechanism and an output layer.
For our continuous attention models, we normalize the input signal length $L$ into the unit interval $[0,1]$, and use $f(t)\! =\! \sfrac{-(t-\mu)^2}{2\sigma^2}$ as the score function. Continuous attention models obtain $p \in \triangle^L$ from discrete attention, compute $\mu = \mathbb{E}_{p}[\ell/L]$ and $\sigma^2 = \mathbb{E}_{p}[(\ell/L)^2] - \mu^2$, apply the continuous attention transformation, and sum the two context vectors (this model has the same number of parameters as the discrete attention baseline).

Since the dataset is officially split into 10 folds, we perform 10-fold cross-validation to evaluate our models.
Table~\ref{table:results_audio_classification}
shows accuracies for different values of $\alpha$ and the standard deviation across folds.
The models with continuous attention perform better than the baselines, suggesting that adding a continuous mechanism improves its discrete counterpart without increasing the number of parameters.
There is no clear winner among the different choices of $\alpha$, with all models performing similar. However, we notice that sparser choices ($\alpha > 1$) lead to more interpretable predictions, as shown in Figure~\ref{fig:examples_audio_classification}.

\begin{table}[t]
    \caption{Results on UrbanSound8k in terms of accuracy.
    For continuous attention, we used $128$ Gaussian RBFs
$\mathcal{N}(t, \tilde{\mu}, \tilde{\sigma}^2)$, with $\tilde{\mu}$ linearly spaced in $[0,1]$ and $\tilde{\sigma} \in \{.1, .5\}$.
    }
    \label{table:results_audio_classification}
    \vspace{-0.1cm}
    \begin{small}
    \begin{center}
    \begin{tabular}{lcccc}
        \toprule
        \sc Attention & $\alpha=1.0$ & $\alpha=4/3$ & $\alpha=1.5$ & $\alpha=2.0$ \\
        \midrule
        Discrete        & 0.5967 \textcolor{gray}{± 0.06}   & 0.5946 \textcolor{gray}{± 0.07}   & 0.6032 \textcolor{gray}{± 0.05}    & 0.5903 \textcolor{gray}{± 0.05} \\
        Continuous      & 0.6229 \textcolor{gray}{± 0.06}   & \bf 0.6280 \textcolor{gray}{± 0.06}   & 0.6171 \textcolor{gray}{± 0.05}    & 0.6247 \textcolor{gray}{± 0.06} \\
        \bottomrule
    \end{tabular}
    \end{center}
    \end{small}
\end{table}

\begin{figure*}[t]
\centering

\includegraphics[width=0.49\textwidth]{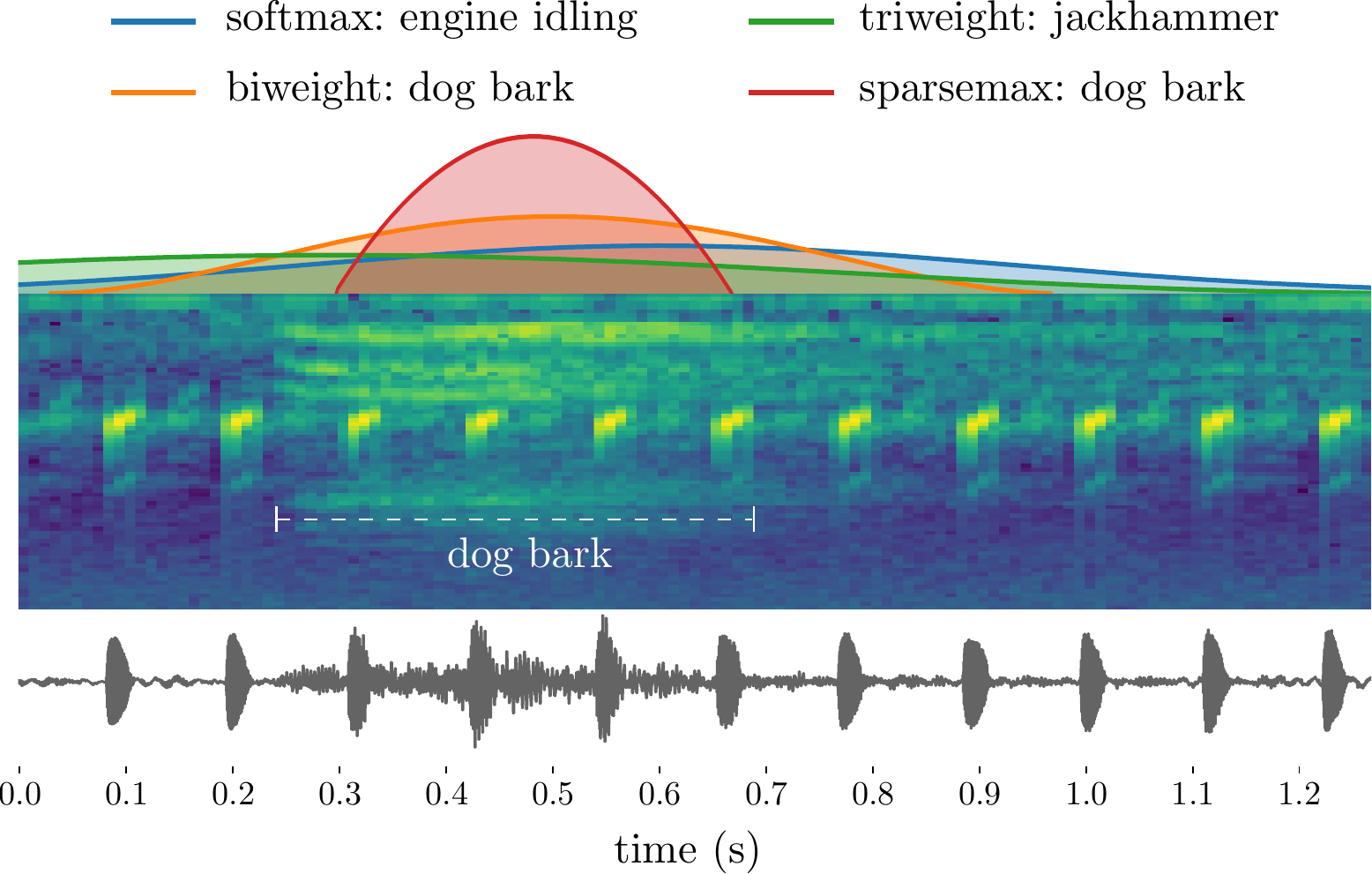}
\includegraphics[width=0.49\textwidth]{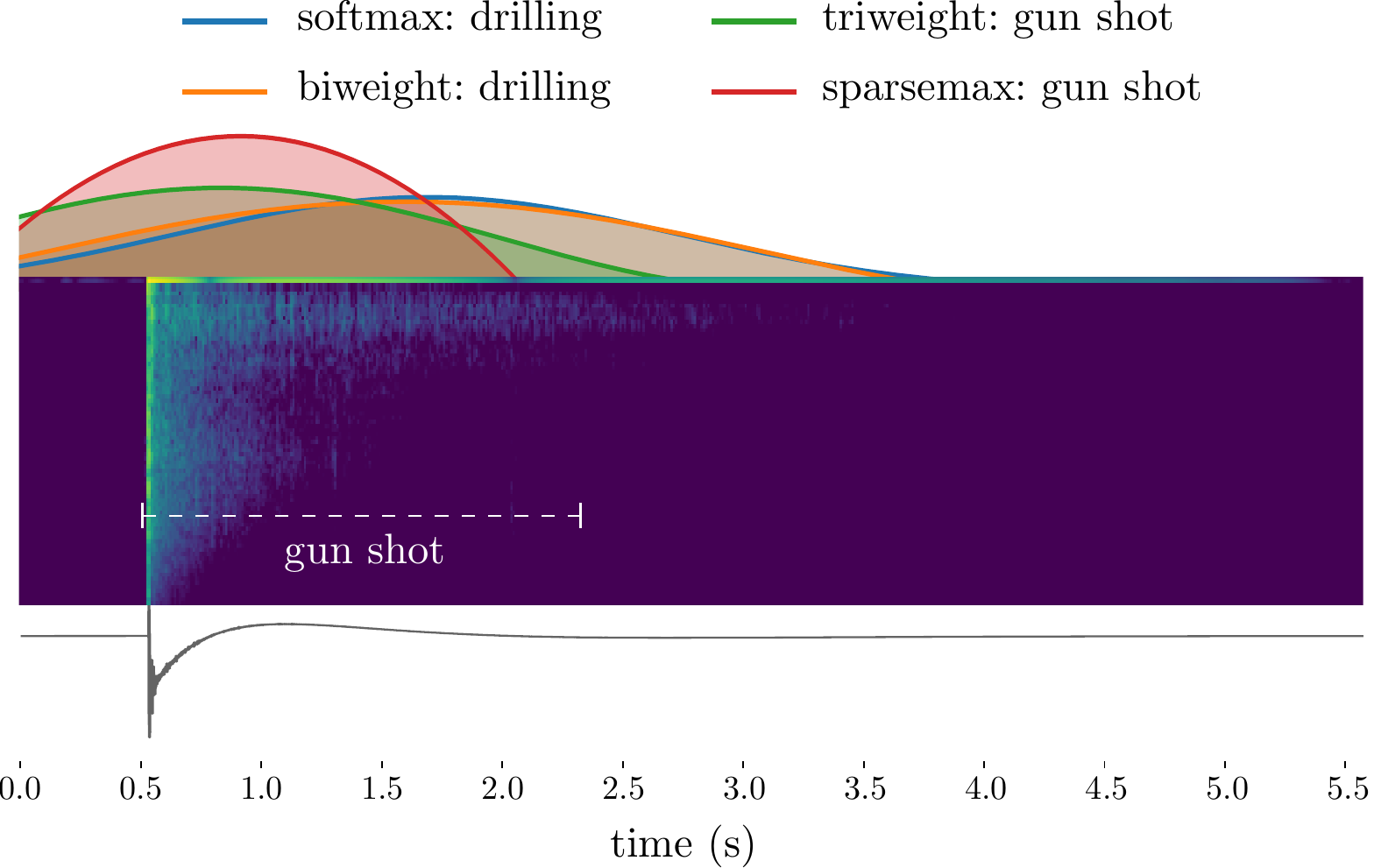}

\caption{\label{fig:examples_audio_classification}Attention densities and predictions made by models with different values of $\alpha$ on two examples from UrbanSound8k. The spectrogram and the waveform on the left represent an audio of a dog barking (around 0.2-0.7s) along with a constant background noise made by a buzzer (every $\sim$0.1s). On the right we have an example of a gun being fired (around  0.5-2.4s), showing a clear energy distinction with the silent background.}
\end{figure*}

\paragraph{2-d: Visual QA.}
We report experiments with 2-d continuous attention on visual question answering, using the VQA-v2 dataset \citep{Goyal2019} and a modular co-attention network as a baseline \citep{Yu2019}.
The discrete attention model attends over a  14$\times$14 grid.
For continuous attention, we normalize the image size into the unit square
$[0,1]^2$. We fit a 2-d Gaussian ($\alpha=1$) or truncated paraboloid
($\alpha=2$) as the attention density; both correspond  to
$f(t)=-\frac{1}{2}(t-\mu)^\top\Sigma^{-1}(t-\mu)$, with $\Sigma \succ 0$. We use
the mean and variance according to the discrete attention probabilities and
obtain $\mu$ and $\Sigma$ with moment matching (using the variance formula from
Proposition~\ref{prop:mean_var_entropy_beta_gaussians}). We use $N = 100 \ll
14^2$ Gaussian RBFs, with $\tilde{\mu}$ linearly spaced in  $[0,1]^2$ and
$\tilde{\Sigma}=0.001\cdot \mathrm{Id}$. Overall, the number of neural network parameters is the same as in discrete attention.

The results in Table~\ref{table:results_vqa} show similar accuracies for all attention models, with a slight advantage for continuous softmax. Figure~\ref{fig:examples_vqa} shows two examples (see Appendix~\ref{sec:model_hyperparams} for more examples and some failure cases): in both examples, the discrete attention is too scattered, possibly mistaking the lamp with a TV screen in the first example.  The continuous attention models focus on the right region and answer the questions correctly, with continuous sparsemax enclosing all the relevant information in its supporting ellipse.

\begin{figure*}[t]
\centering
\includegraphics[width=0.24\textwidth]{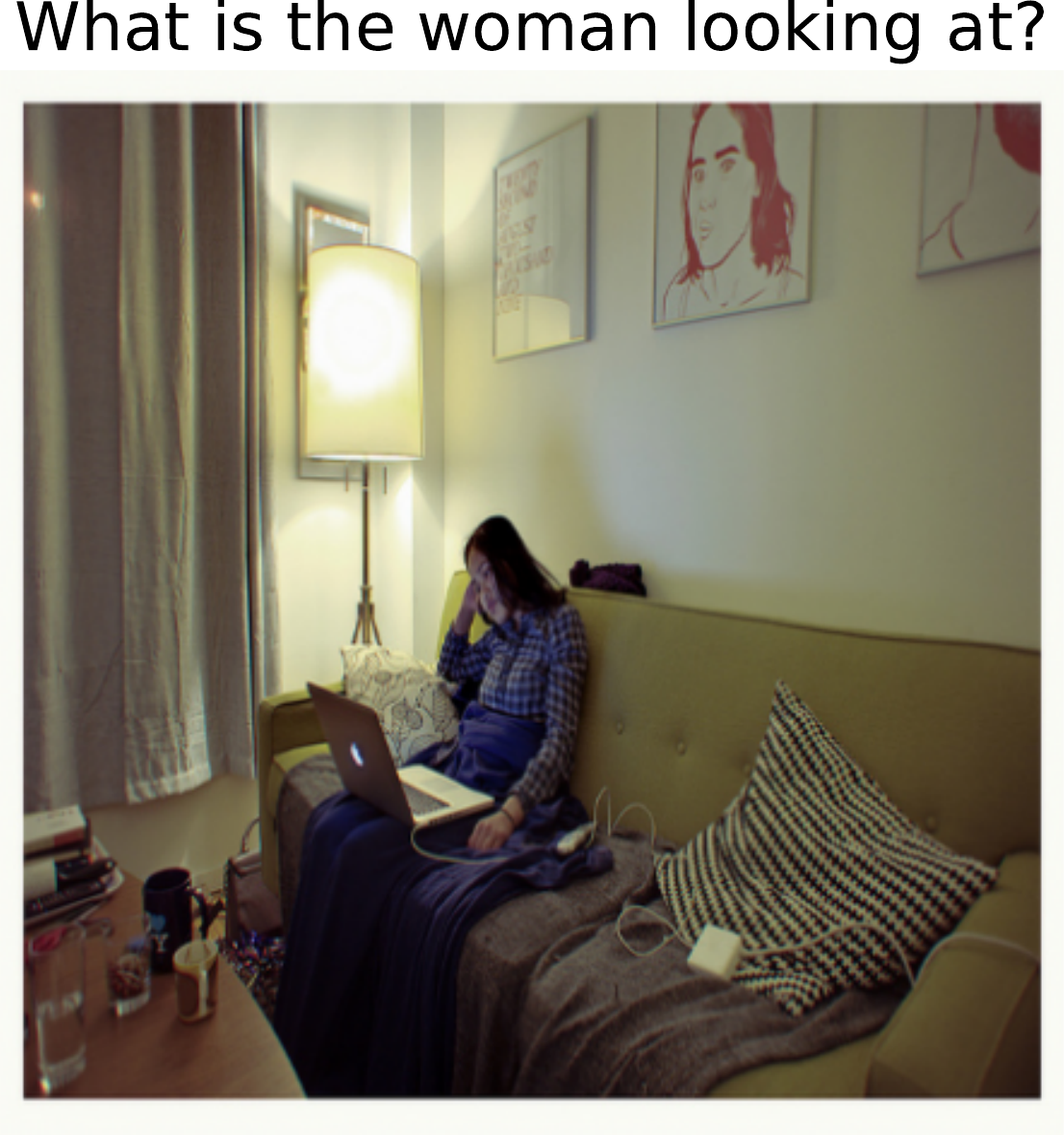}
\includegraphics[width=0.24\textwidth]{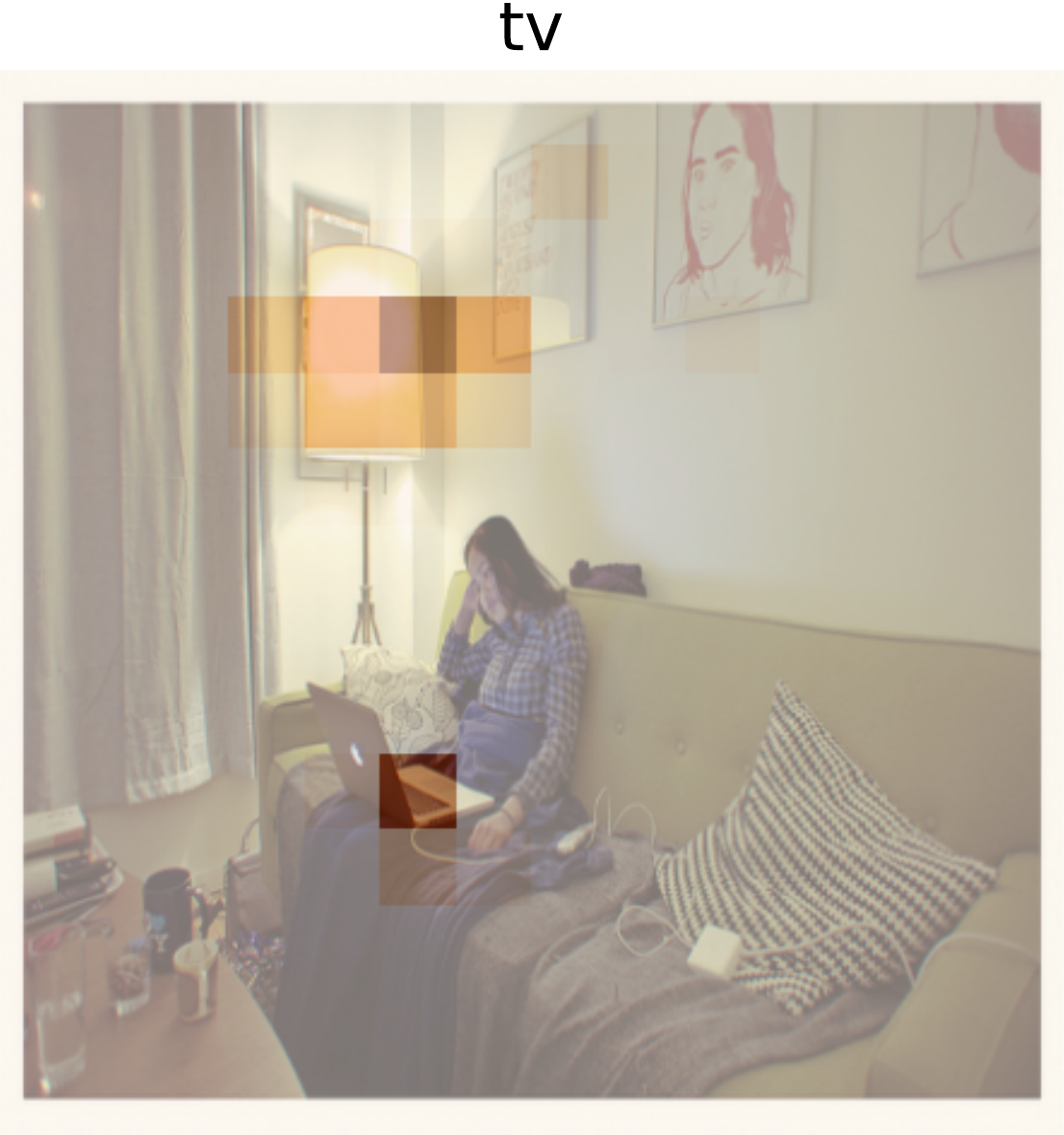}
\includegraphics[width=0.24\textwidth]{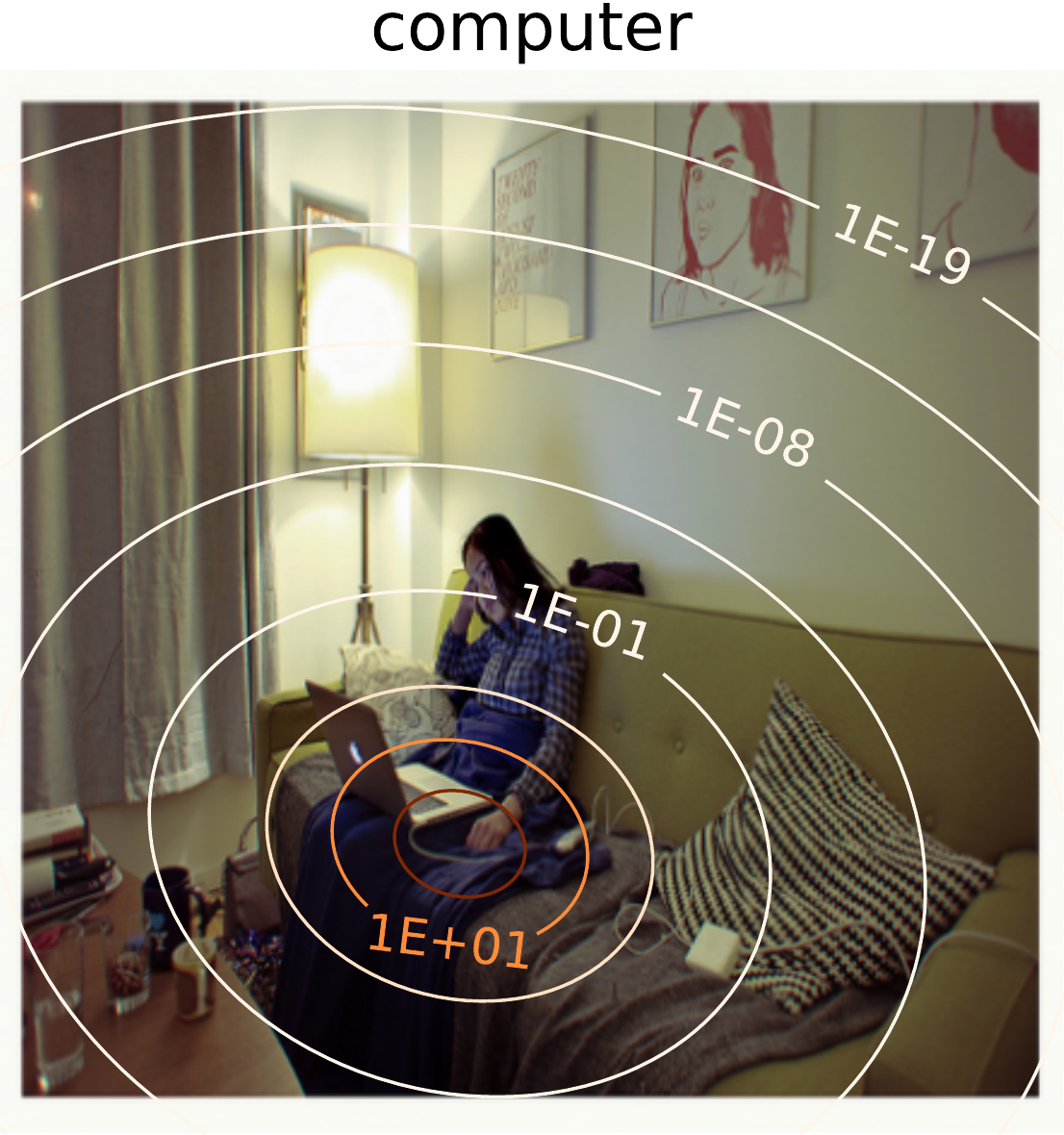}
\includegraphics[width=0.24\textwidth]{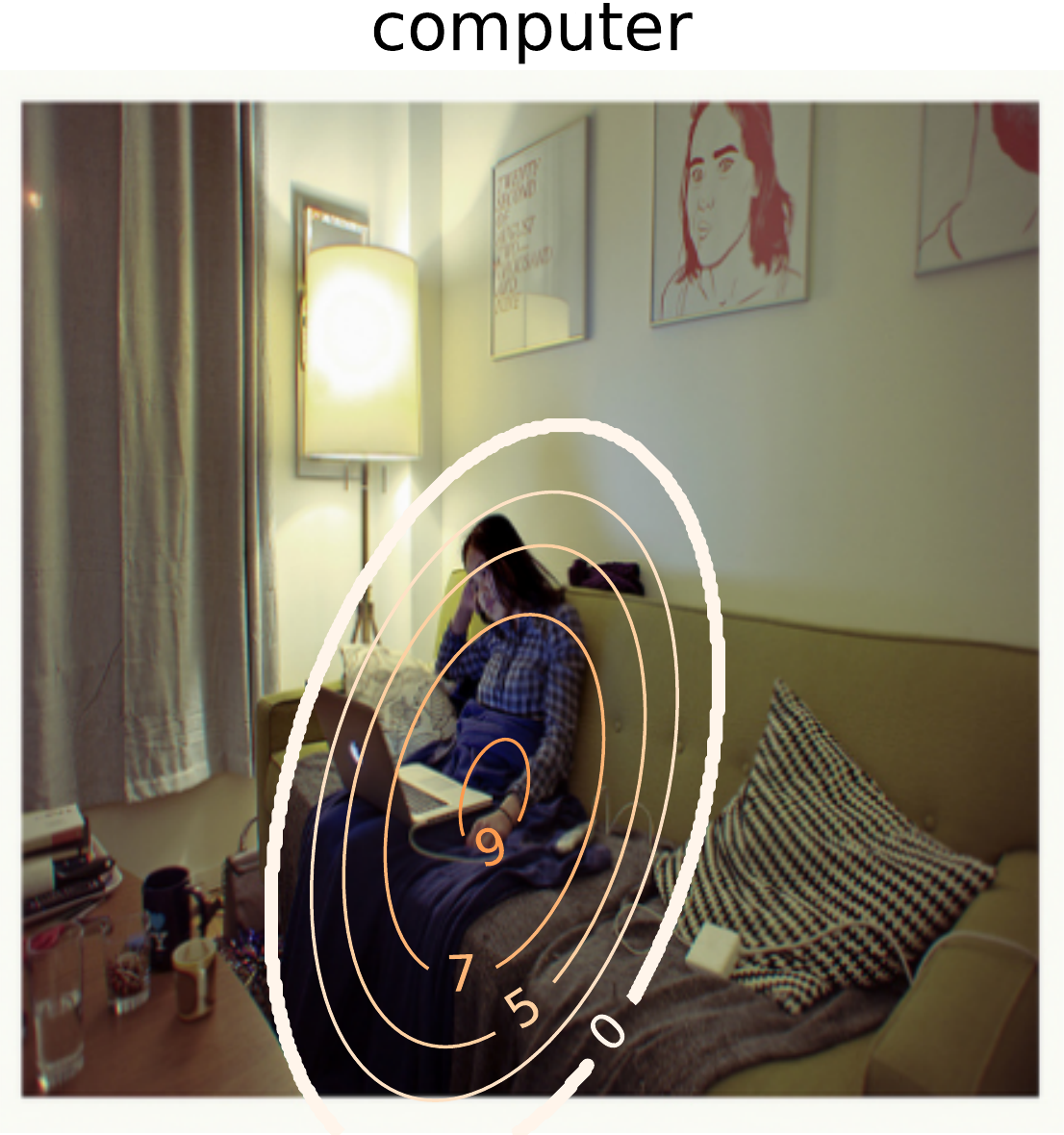}
\\
\includegraphics[width=0.24\textwidth]{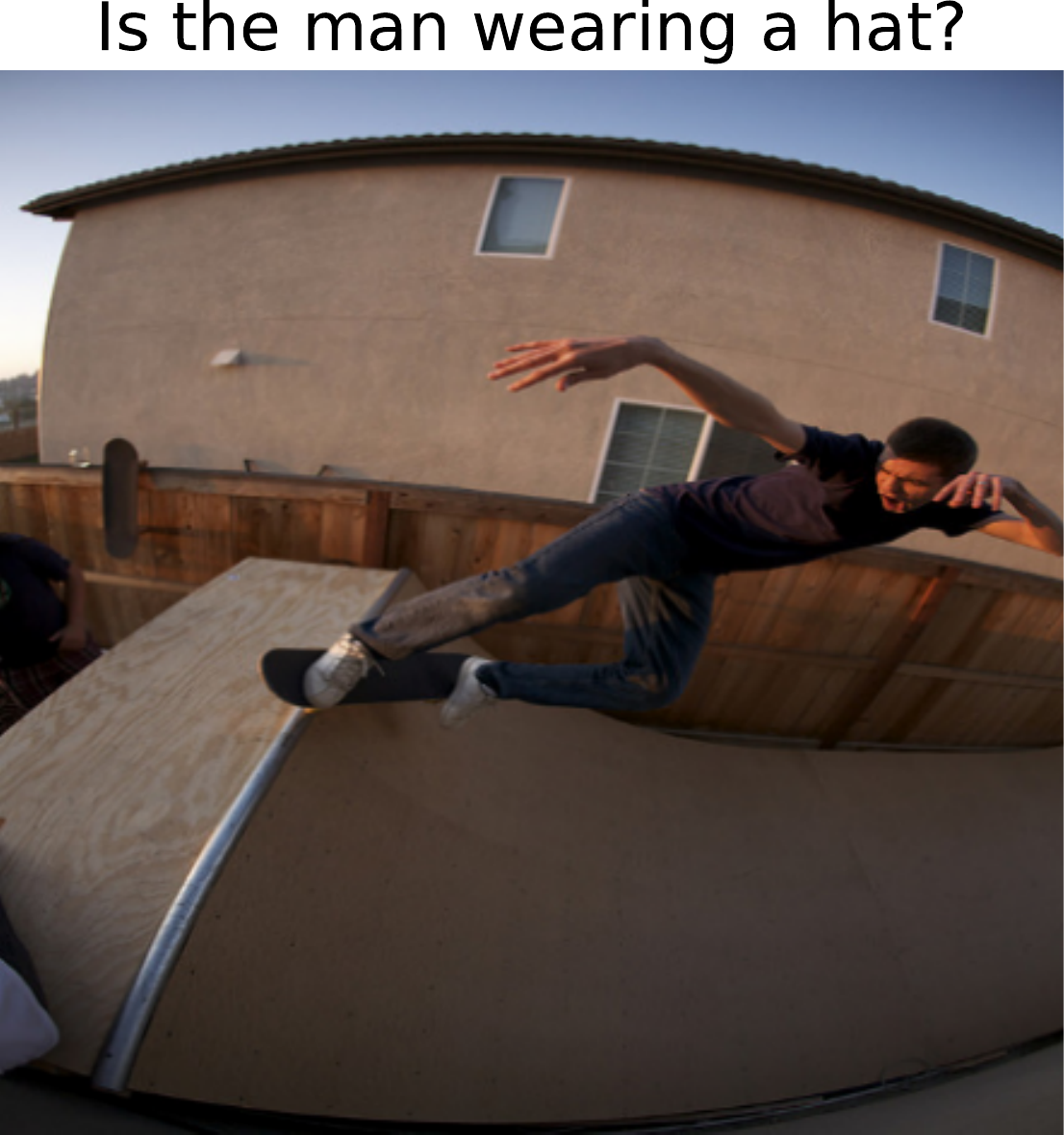}
\includegraphics[width=0.24\textwidth]{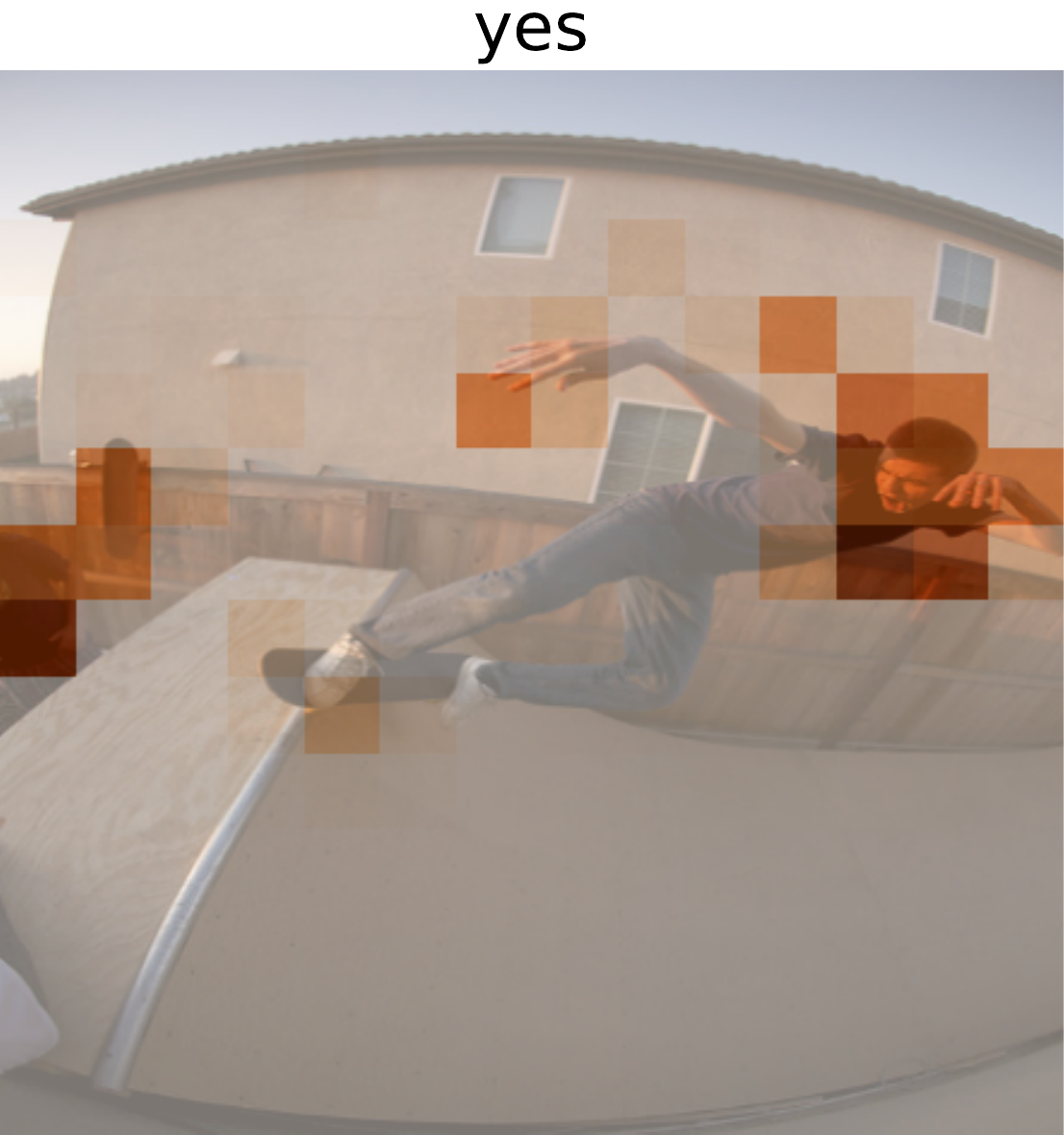}
\includegraphics[width=0.24\textwidth]{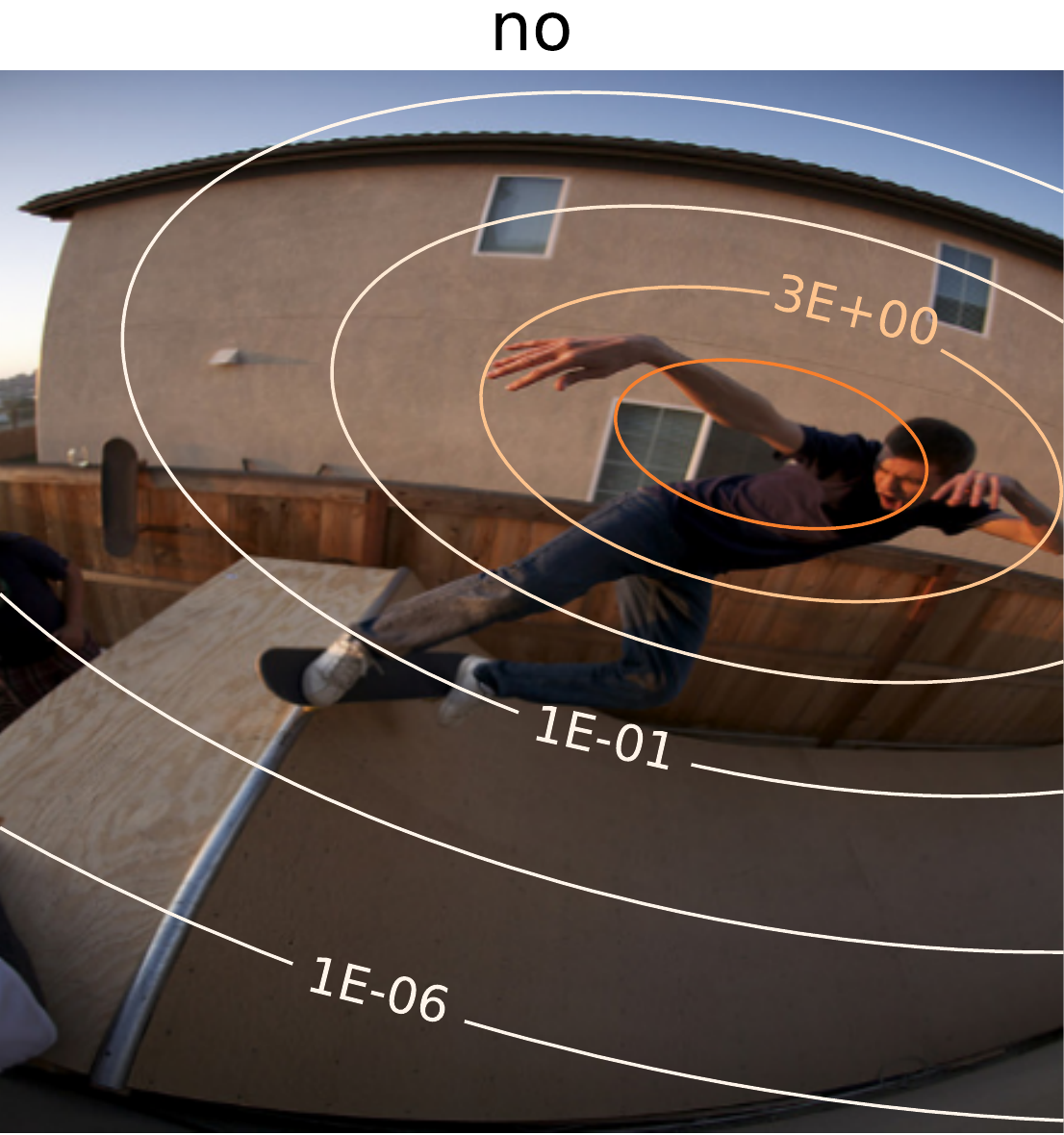}
\includegraphics[width=0.24\textwidth]{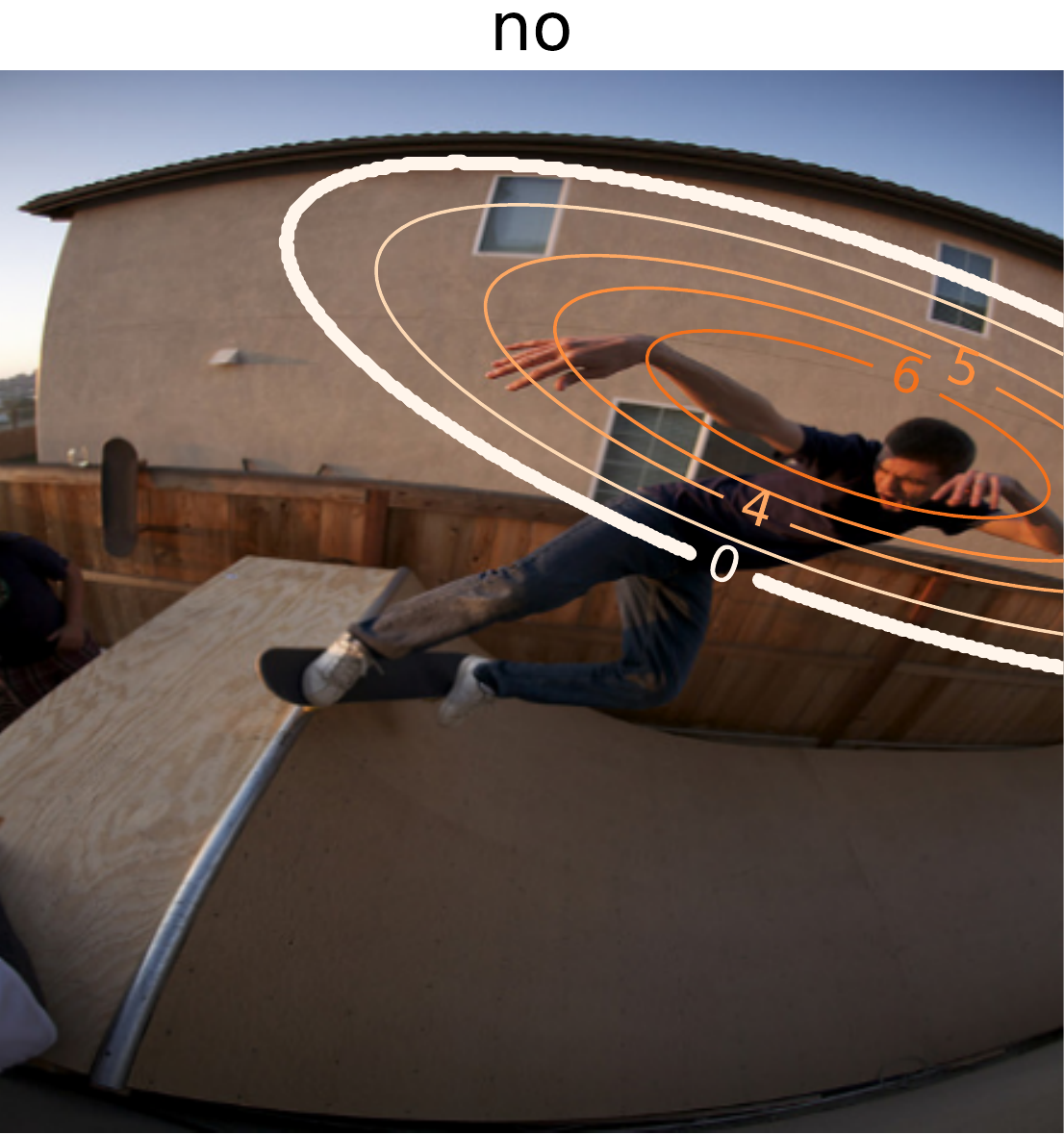}
\caption{\label{fig:examples_vqa}Attention maps for two examples in VQA-v2: the columns show  the original image, discrete attention, continuous softmax, and continuous sparsemax. The latter encloses all probability mass within the outer ellipse.}
\end{figure*}

\begin{table}[t]
    \caption{Accuracies of different models on the \textit{test-dev} and \textit{test-standard} splits of VQA-v2.}
    \label{table:results_vqa}
    \vspace{-.1cm}
    \begin{small}
    \begin{center}
\begin{tabular}{l@{\hspace{10pt}}c@{\hspace{10pt}}c@{\hspace{10pt}}c@{\hspace{10pt}}c@{\hspace{10pt}}c@{\hspace{10pt}}c@{\hspace{10pt}}c@{\hspace{10pt}}c}
        \toprule
        \sc Attention & \multicolumn{4}{c}{Test-Dev}  & \multicolumn{4}{c}{Test-Standard} \\
        {} & Yes/No & Number & Other & Overall & Yes/No & Number & Other & Overall \\
        \midrule
        Discrete softmax    		& 83.40 & 43.59	& 55.91 & 65.83 & 83.47 & 42.99 & 56.33 & 66.13    	\\
        \midrule
        2-d continuous softmax			& 83.40	& 44.80	& 55.88 & \textbf{65.96} & 83.79 & 44.33 & 56.04 & \textbf{66.27}    	\\
        2-d continuous sparsemax	 	& 83.10 & 44.12 & 55.95 & 65.79 & 83.38 & 43.91 & 56.14 & 66.10  	\\
        \bottomrule
    \end{tabular}
\end{center}
    \end{small}
\end{table}

\subsection{Heteroscedastic regression with Fenchel-Young losses}\label{sec:exp_fy_losses}

Regression is often tackled using a squared loss, which is equivalent to
assuming a one-dimensional normal distribution for the target variable.
In this experiment, we explore replacing this normal distribution
with a $\beta$-Gaussian, where not only the mean but also the variance
of the residuals is also allowed to depend on the features.
We analyze the \emph{Breast Cancer Mortality and Population} dataset from
\citet[][Problem 57]{rice}, accessed via \texttt{statsmodels} \citep{statsmodels}.
The data covers 301 counties in southern US.
The single input variable $x$ is the population of the county, and the target
variable $y$ is the breast cancer mortality rate.
As more populous counties display more variability,
a standard linear model fit on the full dataset shows strong signs of
\textbf{heteroscedasticity}
according to a Breusch-Pagan test ($LM=537.4, p<10^{-118}$).

\paragraph{Experimental setup.}

We leave out the $10\%$ most populous counties as a test set, and
fit a linear model with $\beta$-Gaussian data-dependent noise,
\begin{equation}
y \sim \mu_f(x) + \mathcal{N}_\beta(0, \sigma_f^2(x)), \qquad
\text{where}\qquad
\begin{array}{rl}
\mu_f(x) &\coloneqq w_{\mu} \cdot x + b_\mu, \\
\sigma^2_f(x) &\coloneqq (w_{\sigma} \cdot x + b_\sigma)^2\,.
\end{array}
\end{equation}
(Note that $\sigma_f^2$ is not linear in $x$.)
We first fit a baseline standard linear regression, \textit{i.e.},
$ y \sim \mu_f \cdot x + \mathcal{N} (0, 1)$,
and initialize $w_\mu$ and $b_\mu$ in all subsequent models with the baseline
values. We apply 1000 iterations of L-BFGS with a step size of $.01$ to
minimize the average cross-$\Omega$ loss $L^\times_{\Omega}$
against a target Dirac limit case $p=\delta_y$
(Definition~\ref{def:fyloss}, Proposition~\ref{prop:beta_gaussian_fy}), which in
the 1-d case simplifies to:
\begin{equation}
L^\times_{\Omega_\alpha}(\mu_f, \sigma_f^2, y) =
\frac{(\mu_f - y)^2}{2\sigma_f^2}
- \frac{R^2}{2(\sigma_f^2)^{\frac{\alpha-1}{\alpha+1}}}
\cdot \frac{\alpha-1}{3\alpha-1}
+ \frac{1}{\alpha(\alpha-1)}\,,
\end{equation}

\paragraph{Results.}

We report explained variance ($r^2$) in Table~\ref{tab:hetreg}.
Modeling $\sigma^2$ improves over the baseline, especially with $\alpha=2$.
The fit is illustrated in Figure~\ref{fig:hetreg} alongside the
Gaussian ($\alpha=1$) case. The results demonstrate that the $\beta$-Gaussian
family is useful in modeling, and that $L^\times_{\Omega}$ is an appropriate
generalization of cross-entropy.

\begin{table}
\caption{\label{tab:hetreg}Heteroscedastic regression test $r^2$: proportion of
variance explained by $\beta$-Gaussian regression models with learned variance.}
\centering
\small
\begin{tabular}{ccccc}
\toprule
\sc Baseline & $\alpha=1.0$ & $\alpha=4/3$ & $\alpha=1.5$ & $\alpha=2.0$ \\
\midrule
 0.56  & 0.67 & 0.68 & 0.69 & \textbf{0.72} \\
\bottomrule
\end{tabular}
\end{table}

\begin{figure*}[t]
\centering
\includegraphics[width=.99\textwidth]{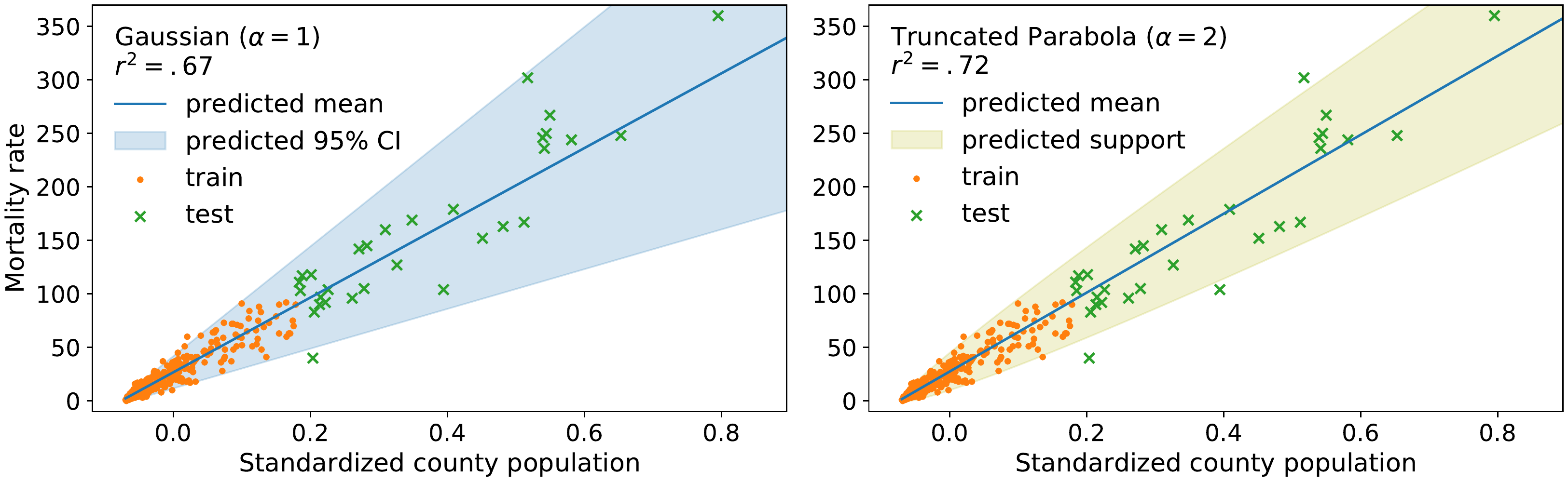}
\caption{\label{fig:hetreg}Heteroscedastic regression models with a $\beta$-Gaussian model. The
truncated parabola model achieves the best generalization out of the considered
models in terms of $r^2$. %
For $\alpha>1$, the bounded support can be computed using
Proposition~\ref{prop:beta_gauss}.
Note that, for $\alpha=2$, even though some points lie outside of the modeled support, the likelihood of the model is zero but the cross-$\Omega$ Fenchel-Young loss is finite and learns good regressors.
}
\end{figure*}

\section{Related Work}\label{sec:related}

\paragraph{Generalized exponential families and loss functions.}

\citet{grunwald_2004} introduced generalized exponential families as maximum
entropy distributions for generalized entropy functions.
Based upon these results, \citet{frongillo_2014} study generalized exponential
families (possibly in infinite spaces) from a convex duality perspective. Their
main result is a generalization of the well-known bijection between Bregman
divergences and regular exponential families.
\citet{amari_2012} study deformed exponential families, including their entropy
and canonical divergence. %
Fenchel-Young losses are closely related to proper scoring rules
\citep{gneiting_2007,reid_composite_binary,vernet_2016}.
Proper scoring rules can be seen as primal-space Bregman divergences, while
Fenchel-Young losses can be seen as mixed-space Bregman divergences
\citep{blondel2020learning}. \citet{mensch_2019} propose a Fenchel-Young loss in
the continuous setting. Their focus, however, is on a geometric notion of
entropy called ``Sinkhorn entropy''. \citet{blondel_2020_consistency} studies
the consistency of a subset of Fenchel-Young losses.
In this paper, we provide a throrough study of generalized continuous
distributions and losses from a convex duality perspective with a particular
focus on distributions with sparse support. In doing so, we unify many
continuous distributions and create new ones seamlessly. We also discuss their
Jacobian computation, enabling their use in a neural network trained by
backpropagation.

\cite{nock_2009} proposed a binary classification loss construction based on the
Legendre transformation but their construction precludes non-invertible
mappings.
\citet[Proposition 3]{duchi_2016} derived a multi-class loss which is
a special case of Fenchel-Young loss over the
probability simplex. \citet{nowak2020consistent} use
Fenchel-Young losses to construct a new loss with a max-min margin property.
This loss corresponds to choosing $\alpha \to \infty$ in the Tsallis negentropy.
Finally, \citet{bao2021fenchel} used Fenchel-Young losses to derive new losses  for class-posterior probability estimation with unbalanced classes.
All these works are limited to the finite output domains.

\insertion{
The regularized prediction map presented in \S\ref{sec:rpm} is connected to
proximal operators \citep{moreau1965proximite}.
Indeed, if $\Omega = \frac{1}{2} \|\cdot\|^2_2 + \Phi$,
then $\hat{p}_{\Omega}[f] = \text{prox}_\Phi$ \citep{blondel2020learning}.
At a high level, the Fenchel-Young loss in \S\ref{sec:fy_losses} resembles
formulations for structured prediction \citep{taskar2005structured} which
formulate learning as the problem of finding a saddle point involving a game
between the model parameters (corresponding to our $f_\theta$) and marginal
probabilities for structured outputs (corresponding to our $p$). The focus of
our paper, however, is on continuous domains rather than structured prediction.
}

\paragraph{Relation to the Tsallis maxent principle.}
Our paper unifies two lines of work: deformed exponential families from statistical physics \citep{Tsallis1988,naudts2009q,amari2011geometry}, and sparse alternatives to softmax recently proposed in the machine learning literature \citep{Martins2016ICML,peters2019sparse,blondel2020learning}, herein extended to continuous domains. This link may be fruitful for future research in both fields.
While most prior work is focused on heavy-tailed distributions ($\alpha < 1$), we focus instead on light-tailed, sparse distributions, the other side of the spectrum ($\alpha > 1$). See Appendix~\ref{sec:tsallis_maxent} for the relation to the Tsallis maxent principle.

\paragraph{Continuity in other architectures and dimensions.}
In our paper, we consider attention networks exhibiting temporal/spatial continuity in the input data, be it audio signals (1-d) or visual scenes (2-d). Recent works propose continuous-domain CNNs for 3-d structures like point clouds and molecules \citep{wang2018deep,schutt2017schnet}. The dynamics of continuous-time RNNs have been studied in \citep{funahashi1993approximation}, and similar ideas have been applied to irregularly sampled time series \citep{rubanova2019latent}.
Other recently proposed frameworks produce continuous variants in other dimensions, such as network depth \citep{chen2018neural}, or in the target domain for machine translation tasks \citep{kumar2018mises}.
Our continuous attention networks can be used in tandem with these frameworks.

\paragraph{Gaussian attention probabilities.}
\citet{cordonnier2020relationship} analyze the relationship between (discrete) attention and convolutional layers, and consider spherical Gaussian attention probabilities as relative positional encodings. By contrast, our approach removes the need for positional encodings: by converting the input to a function on a predefined continuous space, positions are encoded {\it implicitly}, not requiring explicit positional encoding.
Gaussian attention has also been hard-coded as input-agnostic self-attention layers in transformers for machine translation tasks by
\citet{you-etal-2020-hard}.
Finally, in their DRAW architecture for image generation, \citet[\S 3.1]{gregor2015draw} propose a selective attention component which is parametrized by a spherical Gaussian distribution.

\insertion{
\paragraph{Sparse latent variables.}

Related to the sparse attention models developed in \S\ref{sec:attention}, several works have presented models with sparse latent variables, mostly for the discrete (possibly structured) case, both in deterministic \citep{correia2020efficient,guerreiro2021spectra} and stochastic settings \citep{bastings2019interpretable,farinhas2022sparse}. Generalizing these constructions to continuous latent variables with sparse support (in the sense of Definition~\ref{def:sparse_function}) is an interesting direction for future work.
}

\section{Conclusions and Future Work}\label{sec:conclusion}

We extended $\Omega$-regularized prediction maps and Fenchel-Young losses to arbitrary measure spaces (\S\ref{sec:rpm}, \S\ref{sec:fy_losses}). A key result is that, for linearly parametrized families, Fenchel-Young loss minimization is equivalent to moment matching of the statistics, generalizing the concept of sufficient statistics from exponential families (Proposition~\ref{prop:fy_properties3}).
With Tsallis $\alpha$-entropies for $\alpha>1$, we obtain sparse families, whose members can have zero tails, such as triangular or truncated parabola/paraboloid distributions on continuous domains or sparse integer distributions in discrete but infinite domains, for $\alpha=2$ (\S\ref{sec:tsallis}, \S\ref{sec:sparsemax}). We provided a general characterization of the normalizing function $A_\alpha(f)$, its gradient and Hessian, and expressions for the Fenchel-Young loss for arbitrary $\alpha$ (Propositions~\ref{prop:solution_rpm_tsallis}, \ref{prop:gradient_A}, \ref{prop:entropy_normalizing}, and \ref{prop:gradient_hessian_fy}).
We then studied the particular case of $\beta$-Gaussian distributions, induced by Tsallis $\alpha$-entropies (with $\beta=2-\alpha$) and quadratic scoring functions, and we have shown that they are instances of elliptical distributions (\S\ref{sec:elliptical}), containing as particular cases the Gaussian and truncated paraboloid, as providing multivariate generalizations of distributions commonly used in kernel density estimation (Epanechnikov, biweight, triweight). We have shown that these distributions can be reparametrized by two independent random variables, a Beta distribution for the radius, and a uniform spherical distribution (Proposition~\ref{prop:beta_gauss}), and used this result to build an efficient sampler. We also characterized key properties of these distributions: their mean, variance, entropy, and a closed-form for the Fenchel-Young loss (Propositions~\ref{prop:mean_var_entropy_beta_gaussians}--\ref{prop:beta_gaussian_fy}). The combination of the sampler and estimation with Fenchel-Young loss minimization using these results is illustrated in Figure~\ref{fig:misfit}.
Finally, we have shown that by considering total variance and Sobolev regularizers $\Omega$, regularized prediction maps allow building continuous counterparts of the fusedmax transformation previously proposed in the discrete case (\S\ref{sec:fusedmax}).

In a nutshell, the theoretical contributions of our paper unify two lines of work: deformed exponential families from statistical physics \citep{Tsallis1988,naudts2009q,amari2011geometry}, and sparse alternatives to softmax recently proposed in the machine learning literature \citep{Martins2016ICML,fusedmax,peters2019sparse,blondel2020learning}.
We frame this unification in the scope of regularized prediction maps (a generalization of the variational free energy principle) and Fenchel-Young losses (a generalization of Kullback-Leibler divergences and Bregman divergences).
We believe this link may be fruitful for future research in both fields. While most prior work is focused on heavy-tailed distributions ($\alpha < 1$), we focus instead on light-tailed, sparse distributions, the other side of the spectrum ($\alpha > 1$).

We have also shown how $\Omega$-regularized predictions maps can be used in neural network models to construct \emph{continuous} attention mechanisms (\S\ref{sec:attention}),  generalizing finite attention \citep{bahdanau2014neural} to continuous input data, such as 1-d spatial or temporal signals or 2-d images (visual scenes). We derived their Jacobians in terms of generalized covariances (Proposition~\ref{prop:jacobian_entmax}), allowing for efficient forward and backward propagation.  Experiments for 1-d and 2-d cases were shown on attention-based audio classification and visual question answering (\S\ref{sec:experiments}). %

There are many avenues for future work. %
The sparse integer distributions presented in \S\ref{sec:sparsemax} open up interesting questions, such as the efficient computation of key quantities (mean, entropies, Fenchel-Young loss) and its applicability to problems that could benefit from distributions with finite but varying support (ranges).
Likewise, the $\beta$-Gaussian distributions presented in \S\ref{sec:elliptical} might be useful in embedding spaces, where objects (\textit{e.g.}, words) could be modeled as compact sets. %
Our results concerning $\Omega$-regularized prediction maps and Fenchel-Young losses provided in \S\ref{sec:rpm}--\ref{sec:fy_losses} are very general, and it is plausible that regularizers $\Omega$ other than Tsallis entropies, total variation or Sobolev regularizers might be useful.
While our paper focused on linearly parametrized energy functions, the non-linear case (\textit{e.g.}, where $f(t)$ is obtained from a neural network) deserves further study---in fact, several of our theoretical results can be easily extended to this case by replacing $\phi(t)$ by $\nabla_{\theta} f_\theta(t)$.
Regarding sparse continuous attention mechanisms, while our paper focused on unimodal distributions, there are applications in which  multiple attention modes are desirable. This can be done by considering mixtures of distributions,  multiple attention heads, or sequential attention steps. Initial work in that direction includes \citet{farinhas2021multimodal}.
Another direction concerns combining our continuous attention models with
other spatial/temporal continuous architectures for CNNs and RNNs \citep{wang2018deep,schutt2017schnet,funahashi1993approximation} or
with continuity in other dimensions, such as depth \citep{chen2018neural} or output space \citep{kumar2018mises}. 
Recent work using continuous attention mechanisms to model long-term ``sticky'' memories in transformer architectures has been done by \citet{martins2022infty}; some of the ideas above are applicable there, too.

\acks{We thank the action editor and the anonymous reviewers for their insightful comments and suggestions. This work was supported by the European Research Council  (ERC StG
DeepSPIN 758969) and by the Fundação para a Ciência e Tecnologia through project PTDC/CCI-INF/4703/2021 (PRELUNA) and contract UIDB/50008/2020 and in part by the Hybrid Intelligence Centre, a
10-year program funded by the Dutch Ministry of Education, Culture, and Science
through the Netherlands Organisation for Scientific Research
(\url{https://hybrid-intelligence-centre.nl}).}

\newpage

\appendix

\bigskip

\begin{center}
\LARGE{\bf Appendix}
\end{center}

\section{Proofs for regularized prediction maps}

\subsection{Equivariance of distributions}\label{sec:proof_equivariance_fy}

Let $S=\mathbb{R}^N$ and $\nu$ be the Lebesgue measure. We show that, if the regularizer $\Omega$ is separable ({\it i.e.} if it can be written as $\Omega(p) = \int_S \psi(p(t))$ for some function $\psi: \mathbb{R}_+ \rightarrow \mathbb{R}$), the following equivariance property holds:
\begin{equation}
\hat{p}_\Omega[\tilde{f}](t) = \hat{p}_\Omega[f](At + b),
\end{equation}
where $\tilde{f}(t) := f(At + b)$, for any matrix $A$ with determinant $\pm 1$ and any vector $b \in \mathbb{R}^N$.

By definition, we have
\begin{equation}
\hat{p}_{\Omega}[\tilde{f}]
= \argmax_{p \in \mathcal{M}_+^1(\mathbb{R}^N)} \mathbb{E}_{p}[\tilde{f}(t)] - \Omega(p)
= \argmax_{p \in \mathcal{M}_+^1(\mathbb{R}^N)} \int_S (p(t) \, f(At + b) - \psi(p(t))\, dt.
\end{equation}
Making a change of variables $s = At+b$, using the change of variables' formula (noting that $|\det(A)| = 1$), and defining $q(s) = p(A^{-1}(s-b))$ -- noting that $p \in \mathcal{M}_+^1(\mathbb{R}^N)$ iff $q \in \mathcal{M}_+^1(\mathbb{R}^N)$ -- we obtain:
\begin{eqnarray}
\hat{p}_{\Omega}[\tilde{f}](t)
&=& \left(\argmax_{q \in \mathcal{M}_+^1(\mathbb{R}^N)} \int_S (q(s) \, f(s) - \psi(q(s))\, ds \right) (s) %
= \left(\argmax_{q \in \mathcal{M}_+^1(\mathbb{R}^N)} \mathbb{E}_{q}[f(s)] - \Omega(q) \right) (s) \nonumber\\
&=& \hat{p}_{\Omega}[f](s),
\end{eqnarray}
which leads to the desired result.

\subsection{Differential Negentropy and Boltzmann-Gibbs distributions}\label{sec:diff_ent_exp_family}

We adapt a proof from \citet{cover2012elements}.
Let $\Omega$ be the Shannon negentropy, which is proper, lower semi-continuous, and strictly convex \citep[example~9.41]{Bauschke_Combettes2011}, and let  $$\mathrm{KL}(p \| q) := \int_S p(t) \log \frac{p(t)}{q(t)}$$ be the Kullback-Leibler divergence between distributions $p$ and $q$ (which is always non-negative and equals $0$ iff $p=q$).
Take $q(t) = \frac{\exp(f(t))}{\int_S \exp(f(t'))d\nu(t')} = \exp(f(t) - A(f))$ as in \eqref{eq:boltzmann}, where $A(f)$ is the log-partition function.

We have, for any $p \in \mathcal{M}_+^1(S)$:
\begin{eqnarray}
0 &\le& \mathrm{KL}(p \| q)
= \int_S p(t) \log \frac{p(t)}{q(t)}
= \Omega(p) - \int_S p(t) \log q(t)
= \Omega(p) - \int_S p(t) (f(t) - A(f)) \nonumber\\
&=& \Omega(p) - \mathbb{E}_p[f(t)]  + A(f).
\end{eqnarray}
Therefore, we have, for any  $p \in \mathcal{M}_+^1(S)$, that
\begin{equation}
\mathbb{E}_p[f(t)]  - \Omega(p) \le A(f),
\end{equation}
with equality if and only if $p=q$. Since the right hand side is constant with respect to $p$, we have that the posited $q$ must be the maximizer of \eqref{eq:reg_prediction}.

\section{Proofs for continuous Fenchel-Young losses}
\label{sec:proof_fy_properties}

\paragraph{Proof of Propositions~\ref{prop:fy_properties1}--\ref{prop:fy_properties3}}

The proof of Proposition~\ref{prop:fy_properties1} adapts that of \citet{blondel2020learning}
when  Fenchel duality is now taken in the infinite-dimensional set $\mathcal{F} \subseteq \mathbb{R}^S$, which endowed with the inner product $\langle f, g\rangle = \int_S f(t) g(t) d\nu(t)$ forms a Hilbert space \citep{Bauschke_Combettes2011}.
The non-negativity of $L_\Omega$ stems from the Fenchel-Young inequality in Hilbert spaces.
The loss is zero iff $(f_\theta, p)$ is a dual pair, \textit{i.e.}, if $p = \hat{p}_\Omega[f_\theta] = \nabla \Omega^*(f_\theta)$.

To prove Proposition~\ref{prop:fy_properties2}, note that the gradient of $L_\Omega$ is
\begin{equation*}
\nabla_\theta L_\Omega(f_\theta; p) = \int_S \frac{\partial L_\Omega(f_\theta; p)}{\partial f_\theta(t)}  \nabla_\theta f_\theta(t) d\nu(t) = \int_S (\hat{p}_\Omega[f_\theta](t) - p(t)) \nabla_\theta f_\theta(t),
\end{equation*}
where we used the fact that
$\frac{\partial L_\Omega(f_\theta; p)}{\partial f_\theta(t)} = [\nabla \Omega^*(f_\theta) - p](t)$.
This leads to the expression in \eqref{eq:expected_gradient_matching}.

The first point in Proposition~\ref{prop:fy_properties3} is a direct consequence of the last result.
The convexity of $L_\Omega$ with respect to $\theta$ stems from the fact that $L_\Omega(f_\theta, p)= \Omega(p) + \Omega^*(f_\theta) - \mathbb{E}_p[f_\theta(t)] $ is convex with respect to $f_\theta$ (since it is the sum of an affine function with $\Omega^*(f_\theta)$, which, being a Fenchel dual, is convex) and that $L_\Omega$, as a function of $\theta$, is a composition of the linear mapping $\theta \mapsto f_\theta(\cdot) = \theta^\top \phi(\cdot)$ with the said convex function, hence it is convex.
Finally, the last statement is an immediate consequence of the two previous claims: Since
$L_{\Omega}$ is convex, any stationary point is a global minimum,
and, from the first claim, any stationary point $\hat\theta$ must satisfy
$\mathbb{E}_{\hat{p}_{\Omega}[f_{\hat{\theta}}]}[\phi(t)] =
\mathbb{E}_{p}[\phi(t)]$.

\section{Proofs for Tsallis regularization and Deformed Exponential Families}
\label{sec:gini_ent_sparse_family}

\subsection{Shannon as a limit case of Tsallis when $\alpha\rightarrow 1$}

We show that $\lim_{\alpha \rightarrow 1} \Omega_{\alpha}(p) = \Omega_1(p)$ for any $p\in \mathcal{M}_+^1(S)$.
From \eqref{eq:tsallis}, it suffices to show that $\lim_{\beta\rightarrow 1} \log_\beta (u) = \log(u)$ for any $u \ge 0$.
Let $g(\beta) \coloneqq u^{1-\beta} - 1$, and
$h(\beta)\coloneqq 1 - \beta$. Observe that $$\lim_{\beta\rightarrow 1}\log_\beta (u) = \lim_{\beta\rightarrow 1} \frac{g(\beta)}{h(\beta)} =
\frac{g(1)}{h(1)}
= \frac{0}{0},$$
so we are in an indeterminate case.
We take the derivatives of $g$ and $h$:
\begin{equation}
g'(\beta) = \left(\exp (\log u^{1-\beta}) \right)'
= \exp (\log u^{1-\beta}) \cdot ((1-\beta) \log u)'
= -u^{1-\beta} \log u,
\end{equation}
and $h'(\beta) = -1$.
From l'H\^{o}pital's rule,
\begin{equation}
\lim_{\beta \rightarrow 1} \frac{g(\beta)}{h(\beta)}
=\lim_{\beta \rightarrow 1} \frac{g'(\beta)}{h'(\beta)}
= \log u.
\end{equation}

\subsection{Proof of Proposition~\ref{prop:solution_rpm_tsallis}}\label{proof_solution_rpm_tsallis}

The proof of Proposition~\ref{prop:solution_rpm_tsallis} is  similar to the one in \S\ref{sec:diff_ent_exp_family}, replacing the KL divergence by the Bregman divergence induced by $\Omega_\alpha$, and using an additional bound.
Let $$B_{\Omega_\alpha}(p, q) := \Omega_\alpha(p) - \Omega_\alpha(q) - \langle \nabla \Omega_\alpha(q), p-q \rangle$$ be the (functional) Bregman divergence between distributions $p$ and $q$ induced by $\Omega_\alpha$, and let $$q(t) = \exp_{2-\alpha}(f(t) - A_\alpha(f)) = [1 + (\alpha-1)(f(t) - A_\alpha(f))]_+^{\frac{1}{\alpha-1}}.$$
Note that, from \eqref{eq:tsallis}, $$\left(\nabla_q \Omega_{\alpha}(q)\right)(t) = \frac{q(t)^{\alpha-1}}{\alpha-1}.$$
From the non-negativity of the Bregman divergence \cite{bregman1967relaxation}, we have, for any $p \in \mathcal{M}_+^1(S)$:
\begin{eqnarray}
0 &\le^{(a)}& B_{\Omega_\alpha}(p, q) \nonumber\\
&=& \Omega_\alpha(p) - \Omega_\alpha(q) - \langle \nabla \Omega_\alpha(q), p-q \rangle \nonumber\\
&=& \Omega_\alpha(p) - \Omega_\alpha(q) - \int_S \frac{q(t)^{\alpha-1}}{\alpha-1} (p(t) - q(t)) \nonumber\\
&=& \Omega_\alpha(p) - \Omega_\alpha(q) - \underbrace{\mathbb{E}_p[[f(t) - A_\alpha(f) + (\alpha-1)^{-1}]_+]}_{\ge \mathbb{E}_p[f(t) - A_\alpha(f) + (\alpha-1)^{-1}]} + \frac{1}{\alpha-1}\int_S q(t)^{\alpha} \nonumber\\
&\le^{(b)}& \Omega_\alpha(p) - \Omega_\alpha(q) - \mathbb{E}_p[f(t) - A_\alpha(f) + (\alpha-1)^{-1}] + \frac{1}{\alpha-1}\int_S q(t)^{\alpha} \nonumber\\
&=& \Omega_\alpha(p) - \mathbb{E}_p[f(t)]  - \Omega_\alpha(q) + \underbrace{\frac{1}{\alpha-1}\left(\int_S q(t)^{\alpha} - 1\right)}_{=\alpha \Omega_\alpha(q)} +A_\alpha(f) \nonumber\\
&=& \Omega_\alpha(p) - \mathbb{E}_p[f(t)]  + (\alpha-1)\Omega_\alpha(q) + A_\alpha(f).
\end{eqnarray}

Therefore, we have, for any $p \in \mathcal{M}_+^1(S)$,
\begin{equation}\label{eq:variational_proof_tsallis}
\mathbb{E}_p[f(t)] - \Omega_\alpha(p) \le (\alpha-1)\Omega_\alpha(q) +A_\alpha(f),
\end{equation}
with equality iff $p = q$, which leads to zero Bregman divergence ({\it i.e.}, a tight inequality $(a)$) and to
$\mathbb{E}_p[[f(t) - A_\alpha(f) + (\alpha-1)^{-1}]_+] = \mathbb{E}_p[f(t) - A_\alpha(f) + (\alpha-1)^{-1}]$ ({\it i.e.}, a tight inequality $(b)$).

We can use the equality above to obtain an expression for the Fenchel conjugate $\Omega_\alpha^*(f) = \mathbb{E}_q[f(t)] - \Omega_\alpha(q)$ ({\it i.e.}, the value of the maximum in \eqref{eq:reg_prediction} and the right hand side in \eqref{eq:variational_proof_tsallis}):
\begin{equation}\label{eq:conjugate_expression}
\Omega_\alpha^*(f) = (\alpha-1)\Omega_{\alpha}(q) + A_\alpha(f).
\end{equation}

\subsection{Normalizing function $A_\alpha(f)$}\label{sec:A_alpha}

Let $p = \hat{p}_{\Omega_\alpha}[f]$.
The expression for $A_\alpha$ in Prop.~\ref{prop:solution_rpm_tsallis} is
obtained by inverting \eqref{eq:entmax}, yielding
$A_\alpha(f) = f(t) - \log_{2-\alpha}(p(t))$,  and integrating with respect to  $p(t)^{2-\alpha} d\nu(t)$, leading to:
\begin{eqnarray}
\int_S p(t)^{2-\alpha} A_\alpha(f) &=& \int_S p(t)^{2-\alpha} f(t) - \int_S p(t)^{2-\alpha} \log_{2-\alpha}(p(t))\nonumber\\
&=& \int_S p(t)^{2-\alpha} f(t) - \frac{\int_S  (p(t) - p(t)^{2-\alpha})}{\alpha-1} \nonumber\\
&=& \int_S p(t)^{2-\alpha} f(t) - \frac{1}{\alpha-1} + \frac{\int_S  p(t)^{2-\alpha}}{\alpha-1},
\end{eqnarray}
from which the desired expression follows.

\subsection{Relation to the Tsallis Maxent Principle}\label{sec:tsallis_maxent}

We discuss here the relation between the $(2-\alpha)$-exponential family of distributions as presented in Prop.~\ref{prop:solution_rpm_tsallis} and the distributions arising from the Tsallis maxent principle \citep{Tsallis1988}. We put in perspective the related work in statistical physics \citep{abe2003geometry,naudts2009q}, information geometry \citep{amari2011geometry,amari2016information}, and the discrete case presented in the machine learning literature \citep{blondel2020learning,peters2019sparse}.

We start by noting that our $\alpha$ parameter matches the $\alpha$ used in prior machine learning literature related to the ``$\alpha$-entmax transformation'' \citep{blondel2020learning,peters2019sparse}. In the definition of Tsallis entropies  \eqref{eq:tsallis}, our $\alpha$ corresponds to the entropic index $q$ defined by \citet{Tsallis1988}.
However, our $(2-\alpha)$-exponential families correspond to the $q$-exponential families as defined by \citet{naudts2009q}, and to the $t$-exponential families described by \citet{ding2010t} (which include the $t$-Student distribution). The family of Amari's $\alpha$-divergences relates to this $q$ as $\alpha=2q-1$ \citep[\S4.3]{amari2016information}.

These differences in notation have historical reasons, and they are explained by the different ways in which Tsallis entropies relate to $q$-exponential families.
In fact, the physics literature has defined $q$-exponential distributions in two distinct ways, as we next describe.

Note first that the $\Omega$-regularized prediction map in our  Def.~\ref{def:regularized_prediction} is a generalization of the free energy variational principle, if we see $-f_\theta(t) = -\theta^\top \phi(t)$ as an energy function and $\Omega$ the entropy scaled by a temperature.
Let $\Omega = \Omega_\alpha$ be the Tsallis $\alpha$-entropy.
An equivalent constrained version of this problem is the maximum entropy
\emph{(maxent)} principle \citep{jaynes1957information}:
\begin{equation}\label{eq:constrained_maxent}
\max_{p \in \mathcal{M}_+^1(S)}  -\Omega_\alpha(p), \hspace{0.5cm} \mathrm{s.t.}  \hspace{0.35cm} \mathbb{E}_{p}[\phi(t)] = b.
\end{equation}
The solution of this problem corresponds to a distribution in the $(2-\alpha)$-exponential family \eqref{eq:entmax}:
\begin{equation}\label{eq:sol_maxent}
    p^\star(t) = \exp_{2-\alpha}(\theta^\top \phi(t) - A_\alpha(\theta)),
\end{equation}
for some Lagrange multiplier $\theta$.

However, this construction differs from the one by \citet{Tsallis1988} and others, who use {\it escort distributions} \eqref{eq:escort} in the expectation constraints. Namely, instead of \eqref{eq:constrained_maxent}, they consider the problem:
\begin{equation}\label{eq:constrained_maxent_tsallis}
\max_{p \in \mathcal{M}_+^1(S)}  -\Omega_\alpha(p), \hspace{0.5cm} \mathrm{s.t.}  \hspace{0.35cm} \mathbb{E}_{\textcolor{blue}{\tilde{p}^{\alpha}}}[\phi(t)] = b.
\end{equation}
The solution of \eqref{eq:constrained_maxent_tsallis} is of the form
\begin{equation}\label{eq:sol_maxent_tsallis}
    p^\star(t) = B_{\alpha}(\theta)\exp_{\alpha}(\theta^\top (\phi(t) - b)),
\end{equation}
where $\theta$ is again a Lagrange multiplier. This is derived, for example, in \citep[Eq.~15]{abe2003geometry}.
There are two main differences between \eqref{eq:sol_maxent} and \eqref{eq:sol_maxent_tsallis}:
\begin{itemize}
    \item While \eqref{eq:sol_maxent} involves the $(2-\alpha)$-exponential, \eqref{eq:sol_maxent_tsallis} involves the $\alpha$-exponential.
    \item In \eqref{eq:sol_maxent}, the normalizing term $A_\alpha(\theta)$ is {\it inside} the $(2-\alpha)$-exponential. In \eqref{eq:sol_maxent_tsallis}, there is an normalizing factor $B_{\alpha}(\theta)$ {\it outside} the $\alpha$-exponential.
\end{itemize}
Naturally, when $\alpha=1$, these two problems become equivalent, since an additive term inside the exponential is equivalent to a multiplicative term outside. However, this does {\it not} happen with $\beta$-exponentials ($\exp_\beta(u+v) \ne \exp_\beta(u) \exp_\beta(v)$ in general, for $\beta \ne 1$), and therefore these two alternative paths lead to two different definitions of $q$-exponential families. Unfortunately, both have been considered in the physics literature, under the same name, and
this has been subject of debate. Quoting \citet[\S 1]{naudts2009q}:

\begin{quote}
{\it ``An important question is then whether in the modification the normalization should stand in front of the deformed exponential function, or whether it should be included as $\ln Z(\beta)$ inside. From the general formalism mentioned above it follows
that the latter is the right way to go.''}
\end{quote}

Throughout our paper, we use the definition of \citep{naudts2009q,amari2011geometry}, equivalent to the maxent problem \eqref{eq:constrained_maxent}.

\subsection{Proof of Proposition~\ref{prop:gradient_A}}\label{sec:proof_gradient_A}

We adapt the proof from \citet[Theorem 5]{amari2011geometry}.
Note first that, for $t \in \mathrm{supp}(p_\theta)$,
\begin{eqnarray}
\nabla_\theta p_\theta(t) &=& \nabla_\theta [(\alpha-1)(\theta^\top \phi(t) - A_\alpha(\theta))+1]^{1/(\alpha-1)} \nonumber\\
&=&  [(\alpha-1)(\theta^\top \phi(t) - A_\alpha(\theta))+1]^{(2-\alpha)/(\alpha-1)} (\phi(t) - \nabla_\theta A_\alpha(\theta)) \nonumber\\
&=& p_\theta(t)^{2-\alpha} (\phi(t) - \nabla_\theta A_\alpha(\theta)),
\end{eqnarray}
and
\begin{eqnarray}
\nabla^2_\theta p_\theta(t) &=&
\nabla_\theta p_\theta^{2-\alpha}(t) (\phi(t) - \nabla_\theta A_\alpha(\theta))^\top - p_\theta^{2-\alpha}(t) \nabla^2_\theta A_\alpha(\theta) \nonumber\\
&=& (2-\alpha) p_\theta^{1-\alpha}(t) \nabla_\theta  p_\theta(t) (\phi(t) - \nabla_\theta A_\alpha(\theta))^\top - p_\theta^{2-\alpha}(t) \nabla^2_\theta A_\alpha(\theta) \nonumber\\
&=& (2-\alpha)p_\theta(t)^{3-2\alpha}\bigl(\phi(t) - \nabla_\theta A_\alpha(\theta)\bigr) \bigl(\phi(t) - \nabla_\theta A_\alpha(\theta)\bigr)^\top \nonumber\\
&& - p_\theta(t)^{2-\alpha} \nabla^2_\theta A_\alpha(\theta).
\end{eqnarray}

Therefore we have:
\begin{equation}
0 = \nabla_\theta \underbrace{\int_S  p_\theta(t)}_{=1} = \int_S \nabla_\theta p_\theta(t) = \int_S p_\theta(t)^{2-\alpha} (\phi(t) - \nabla_\theta A_\alpha(\theta)),
\end{equation}
from which we obtain
\begin{equation}
    \nabla_\theta A_\alpha(\theta) = \frac{\int_S p_\theta(t)^{2-\alpha} \phi(t)}{\int_S p_\theta(t)^{2-\alpha}}.
\end{equation}

To prove that $A_\alpha(\theta)$ is convex, we will show that its Hessian is positive semidefinite. Note that
\begin{eqnarray}
0 &=& \nabla^2_\theta \underbrace{\int_S  p_\theta(t)}_{=1} = \int_S \nabla^2_\theta p_\theta(t) \nonumber\\
&=& \int_S (2-\alpha)p_\theta(t)^{3-2\alpha}\bigl(\phi(t) - \nabla_\theta A_\alpha(\theta)\bigr) \bigl(\phi(t) - \nabla_\theta A_\alpha(\theta)\bigr)^\top  - p_\theta(t)^{2-\alpha} \nabla^2_\theta A_\alpha(\theta)\nonumber\\
&=& (2-\alpha) \int_S p_\theta(t)^{3-2\alpha}\bigl(\phi(t) - \nabla_\theta A_\alpha(\theta)\bigr) \bigl(\phi(t) - \nabla_\theta A_\alpha(\theta)\bigr)^\top  \nonumber\\
&& - \nabla^2_\theta A_\alpha(\theta) \int_S p_\theta(t)^{2-\alpha},
\end{eqnarray}
hence, for $\alpha \le 2$,
\begin{equation}
    \nabla^2_\theta A_\alpha(\theta) = \frac{(2-\alpha) \int_S p_\theta(t)^{3-2\alpha}\overbrace{\bigl(\phi(t) - \nabla_\theta A_\alpha(\theta)\bigr) \bigl(\phi(t) - \nabla_\theta A_\alpha(\theta)\bigr)^\top}^{\succeq 0}}{\int_S p_\theta(t)^{2-\alpha}} \succeq 0,
\end{equation}
where we used the fact that $p_\theta(t) \ge 0$ for $t \in S$ and that integrals of positive semidefinite functions are positive semidefinite.

\subsection{Proof of Proposition~\ref{prop:entropy_normalizing}\label{sec:proof_entropy_normalizing}}

From Definition~\ref{def:tsallis}, we have
\begin{equation}
\Omega_\alpha(p_\theta) = \frac{1}{\alpha}\mathbb{E}_{p_\theta}[\log_{2-\alpha}(p(t))] %
= \frac{1}{\alpha}\mathbb{E}_{p_\theta}[\theta^\top \phi(t) - A_\alpha(\theta)]\nonumber\\
= \frac{1}{\alpha}\left(\theta^\top \mathbb{E}_{p_\theta}[\phi(t)] - A_\alpha(\theta)\right),
\end{equation}
from which \eqref{eq:entropy_normalizing} follows.
The expression \eqref{eq:entropy_normalizing_conjugate} was obtained in Appendix~\ref{proof_solution_rpm_tsallis} (see \eqref{eq:conjugate_expression}); the second equality is a simple consequence of \eqref{eq:entropy_normalizing}.
Finally, using the two former results, we have
\begin{eqnarray}
L_{\Omega_\alpha}(f_\theta, p) &=& \Omega_\alpha(p) + \Omega_\alpha^*(f_\theta) - \mathbb{E}_p[f_\theta(t)]\nonumber\\
&=& \Omega_\alpha(p) + (\alpha-1)\Omega_\alpha(\hat{p}_{\Omega_\alpha}[f_\theta]) + A_\alpha(\theta) - \mathbb{E}_p[\theta^\top \phi(t)]\nonumber\\
&=& \Omega_\alpha(p) - \Omega_\alpha(\hat{p}_{\Omega_\alpha}[f_\theta]) + \theta^\top \mu(\theta) - A_\alpha(\theta) + A_\alpha(\theta)  - \theta^\top \mathbb{E}_p[\phi(t)],
\end{eqnarray}
which leads to \eqref{eq:fy_linear_families}.

\subsection{Proof of Proposition~\ref{prop:gradient_hessian_fy}}\label{sec:proof_gradient_hessian_fy}

The expression for the gradient comes directly from Proposition~\ref{prop:fy_properties3}.
As for the Hessian:
\begin{eqnarray}
\nabla \nabla_\theta L_{\Omega_\alpha}(f_\theta, p) &=& \nabla_\theta \mu(\theta)^\top
= \nabla_\theta \mathbb{E}_{p_\theta}[\phi(t)]^\top
= \int_S \nabla_\theta p_\theta(t) \phi(t)^\top \nonumber\\
&=& \int_S p_\theta^{2-\alpha}(t) \nabla_\theta \log_{2-\alpha}(p_\theta(t)) \phi(t)^\top
= \int_S p_\theta^{2-\alpha}(t) \nabla_\theta (\theta^\top \phi(t) - A_\alpha(\theta)) \phi(t)^\top \nonumber\\
&=& \int_S p_\theta^{2-\alpha}(t) (\phi(t) - \nabla_\theta A_\alpha(\theta)) \phi(t)^\top.
\end{eqnarray}
Using the expression for $\nabla_\theta A_\alpha(\theta)$ from Proposition~\ref{prop:gradient_A} yields the desired result.

\section{Proofs for infinite sparsemax}\label{sec:normalization_constants}

\subsection{Truncated parabola}\label{sec:proof_truncated_parabola}

Let $p(t) = \left[-\tau - \frac{(t-\mu)^2}{2\sigma^2}\right]_+$ as in \eqref{eq:truncated_parabola}.
Let us determine the constant $\tau$ that ensures this distribution normalizes to 1. Note that $\tau$ does not depend on the location parameter $\mu$, hence we can assume $\mu=0$ without loss of generality. We must have
$\tau = -\frac{a^2}{2\sigma^2}$ and $1 = \int_{-a}^{a} \left(-\tau - \frac{x^2}{2\sigma^2}\right) = -2\tau a - \frac{a^3}{3\sigma^2} = \frac{2a^3}{3\sigma^2}$, hence
$a = \left(\frac{3}{2}\sigma^2\right)^{1/3}$, which finally gives:
\begin{equation}\label{eq:lambda_gaussian_proof}
\tau = -\frac{1}{2}\left(\frac{3}{2\sigma}\right)^{2/3}.
\end{equation}

The Gini negentropy of this distribution is
\begin{eqnarray}
\Omega_2(\hat{p}_{\Omega_2}[f]) &=& -\frac{1}{2} + \frac{1}{2}\int \hat{p}^2_{\Omega_2}[f](x)%
= -\frac{1}{2} + \frac{1}{2}\int_{-a}^a \left(-\lambda - \frac{x^2}{2\sigma^2}\right)^2 %
= -\frac{1}{2} - \lambda^2 a + \frac{\lambda a^3}{3\sigma^2} + \frac{a^5}{20\sigma^4}\nonumber\\
&=& -\frac{1}{2} + \frac{a^5}{4\sigma^4} - \frac{a^5}{6\sigma^4} + \frac{a^5}{20\sigma^4} %
= -\frac{1}{2} + \frac{2a^5}{15\sigma^4} %
= -\frac{1}{2} + \frac{1}{5}\left(\frac{3}{2\sigma}\right)^{2/3}.
\end{eqnarray}

\subsection{Multivariate truncated paraboloid}\label{sec:proof_truncated_paraboloid}

Let $p(t) = \left[-\tau - \frac{1}{2}(t-\mu)\Sigma^{-1}(t-\mu)\right]_+$ as in \eqref{eq:truncated_paraboloid}.
Let us determine the constant $\tau$ that ensures this distribution normalizes to 1, where we assume again $\mu=0$ without loss of generality. To obtain $\tau$, we start by invoking the formula for computing the volume of an ellipsoid defined
by the equation $x^\top \Sigma^{-1} x \le 1$:
\begin{equation}
V_{\mathrm{ell}}(\Sigma) = \frac{\pi^{n/2}}{\Gamma(n/2 + 1)} \mathrm{det}(\Sigma)^{1/2},
\end{equation}
where $\Gamma(t)$ is the Gamma function.
Since each slice of a paraboloid is an ellipsoid, we can apply Cavalieri's principle to obtain the volume of a paraboloid $y=\frac{1}{2} x^\top \Sigma^{-1} x$ of height $h = -\tau$ as follows:
\begin{eqnarray}
V_{\mathrm{par}}(h) &=& \int_{0}^{h} V_{\mathrm{ell}}(2 \Sigma y)dy
\,\,=\,\, \frac{(2\pi)^{n/2}\mathrm{det}(\Sigma)^{1/2}}{\Gamma(\frac{n}{2} + 1)}  \int_{0}^{h} y^{\frac{n}{2}}dy %
= \frac{(2\pi)^{n/2}\mathrm{det}(\Sigma)^{1/2}}{(\frac{n}{2} + 1)\Gamma(\frac{n}{2} + 1)}  h^{\frac{n}{2}+1} \nonumber\\
&=&
\frac{\sqrt{(2\pi)^{n}\mathrm{det}(\Sigma)}}{\Gamma(\frac{n}{2} + 2)}  h^{\frac{n}{2}+1}.
\end{eqnarray}
Equating the volume to 1, we obtain $\tau = -h$ as
$\tau = -\left(\frac{\Gamma(\frac{n}{2} + 2)}{\sqrt{(2\pi)^{n}\mathrm{det}(\Sigma)}}\right)^{\frac{2}{2+n}}$.

\subsection{Triangular}\label{sec:proof_triangular}

Let $p(t) = \left[-\tau - \frac{|t-\mu|}{b}\right]_+$ as in \eqref{eq:triangular}.
Let us determine the constant $\tau$ that ensures this distribution normalizes to 1.
Assuming again $\mu=0$ without loss of generality, we must have
$\tau = -\frac{a}{b}$ and $1 = \int_{-a}^{a} \left(-\tau - \frac{|x|}{b}\right) = -2\tau a - \frac{a^2}{b} = \frac{a^2}{b}$, hence
$a = \sqrt{b}$, which finally gives
$\tau = -b^{-1/2}$.

The negentropy of this distribution is
\begin{eqnarray}
\Omega_2(\hat{p}_{\Omega_2}[f]) &=& -\frac{1}{2} + \frac{1}{2}\int \hat{p}^2_{\Omega_2}[f](x)%
= -\frac{1}{2} + \frac{1}{2}\int_{-a}^a \left(-\lambda - \frac{|x|}{b}\right)^2 %
= -\frac{1}{2} + \frac{1}{2}\int_{-a}^a \left(\lambda^2 +\frac{2\lambda|x|}{b} + \frac{x^2}{b^2} \right)\nonumber\\
&=& -\frac{1}{2} + \lambda^2 a + \frac{\lambda a^2}{b} + \frac{\lambda a^3}{3b^2} %
= -\frac{1}{2} +  \frac{a^3}{b^2} - \frac{a^3}{b^2} + \frac{a^3}{3b^2} %
= -\frac{1}{2} +  \frac{1}{3\sqrt{b}}.
\end{eqnarray}

\subsection{Location-scale families}\label{sec:proof_location_scale}

We first show that $a$ is the solution of the equation $ag'(a) - g(a) + g(0) = \frac{1}{2}$.
From symmetry around $\mu$, we must have
\begin{eqnarray}
\frac{1}{2} = \int_{\mu}^{\mu + a\sigma} \left(\frac{1}{\sigma}g'(a) - \frac{1}{\sigma}g'\left( \frac{t-\mu}{\sigma}\right)\right) dt = \int_{0}^{a} \left(g'(a) - g'(s)\right)  ds = ag'(a) - g(a) + g(0),
\end{eqnarray}
where we made a variable substitution $s = (t-\mu)/\sigma$,
which proves the desired result.
Now we show that a solution always exists if $g$ is strongly convex, {\it i.e.}, if there is some $\gamma>0$ such that
$g(0) \ge g(s) - sg'(s) + \frac{\gamma}{2}s^2$ for any $s \ge 0$.
Let $F(s) := sg'(s) - g(s) + g(0)$. We want to show that the equation $F(a) = \frac{1}{2}$ has a solution. Since $g$ is continuously differentiable, $F$ is continuous.
From the strong convexity of $g$, we have that $F(s) \ge \frac{\gamma}{2}s^2$ for any $s \ge 0$, which implies that $\lim_{s\rightarrow +\infty} F(s) = +\infty$.
Therefore, since $F(0) = 0$, we have by the intermediate value theorem that there must be some $a$ such that $F(a) = \frac{1}{2}$.

\section{Proofs for $\beta$-Gaussian distributions}

\subsection{Proof of Proposition~\ref{prop:beta_gauss}}\label{sec:proof_beta_gauss}

First, we note that the standard parabola $f_0(z) = -\frac{1}{2}\|z\|^2$ indeed induces a spherical
distribution, since it has density
\begin{equation}\label{eq:standard_beta_gauss}
p_0(z) = \hat{p}_{\Omega_\alpha}[f_0](z) = {\left[(\alpha-1)\left(-\tau -
\frac{1}{2} \|z\|^2\right) \right]}^{\nicefrac{1}{\alpha-1}}_{+} = g(\|z\|^2) \,
\end{equation}
where $g(r^2) = [(\alpha-1)(-\tau - \nicefrac{r^2}{2}]_{+}^{\nicefrac{1}{\alpha-1}}$.
The density of $t = \mu + Az$, where $AA^\top=\tilde{\Sigma}$, is
\begin{eqnarray}
    p(t) &=& \left[(\alpha-1)\left(-\tau - \frac{1}{2}(t-\mu)^\top {\tilde{\Sigma}}^{-1} (t - \mu)\right)
\right]^{\tfrac{1}{\alpha-1}}_+ |\tilde{\Sigma}|^{-\nicefrac{1}{2}}\nonumber\\
    &=& \left[(\alpha-1)\left(-\tau |\tilde{\Sigma}|^{-\frac{\alpha-1}{2}} - \frac{1}{2}|\tilde{\Sigma}|^{-\frac{\alpha-1}{2}}(t-\mu)^\top \tilde{\Sigma}^{-1} (t - \mu)\right)
\right]^{\tfrac{1}{\alpha-1}}_+ \nonumber\\
    &=& \hat{p}_{\Omega_{\alpha}}[f](t),
\end{eqnarray}
with $f(t) = -\frac{1}{2}(t-\mu)^\top \Sigma^{-1}(t-\mu)$ and $\Sigma = |\tilde{\Sigma}|^{\frac{\alpha-1}{2}} \tilde{\Sigma}$.
The expression for $A$ is obtained by solving $\tilde{\Sigma} = AA^\top$ and $\Sigma = |\tilde{\Sigma}|^{\frac{\alpha-1}{2}}\tilde{\Sigma}$, which leads to
$\tilde\Sigma = |\Sigma|^{-\frac{1}{N + \frac{2}{\alpha-1}}} \Sigma$ and
$A = \tilde{\Sigma}^{\sfrac{1}{2}} = |\Sigma|^{-\frac{1}{2N+\frac{4}{\alpha-1}}} \Sigma^{\sfrac{1}{2}}$.
This allows us to focus our study on
the standard $\beta$-Gaussian from Equation~\eqref{eq:standard_beta_gauss}.
This is a spherical distribution, and thus has stochastic characterization $z =
ru$, for some radius random variable $r$.

First, we establish the support and normalizing constants. From
Equation~\eqref{eq:standard_beta_gauss}, $p_0(z) > 0$ iff $1/2 \|z\|^2 > \tau$.
The support is therefore the open sphere $\|z\| < R$, with radius
$R=(-2\tau)^\frac{1}{2}$.

Next, we characterize the density of the random variable $r$.
By \citep[Theorem~2.9]{fang}, the density of $r$ is
\begin{equation}
q(r) \,\,=\,\, \frac{2 \pi^{\nicefrac{N}{2}}}{\Gamma(\nicefrac{N}{2})} r^{N-1} g(r^2)
\,\,=\,\, \frac{2 \pi^{\nicefrac{N}{2}}}{\Gamma(\nicefrac{N}{2})} r^{N-1}
\left[(\alpha-1)\left( -\tau - \frac{r^2}{2}\right)\right]_{+}^{\nicefrac{1}{\alpha-1}}\,.
\end{equation}
Substituting $R$ for $\tau$ and rearranging, we have
\begin{equation}
\begin{aligned}
q(r) &= \frac{2 \pi^{\nicefrac{N}{2}}}{\Gamma(\nicefrac{N}{2})} r^{N-1}
\left(\frac{\alpha-1}{2}\right)^{\tfrac{1}{\alpha-1}}
\left[R^2 - r^2\right]_{+}^{\nicefrac{1}{\alpha-1}} \\
&= \frac{2 \pi^{\nicefrac{N}{2}}}{\Gamma(\nicefrac{N}{2})} r^{N-1}
\left(\frac{\alpha-1}{2}\right)^{\tfrac{1}{\alpha-1}}
R^{\tfrac{2}{\alpha-1}}\left[1 - (\nicefrac{r}{R})^2\right]_{+}^{\nicefrac{1}{\alpha-1}}\,,
\end{aligned}
\end{equation}
and notice that $q(r)>0$ iff $r \in [0, R)$,
thus, the radius has bounded support.
The CDF is
\begin{equation}
Q(\gamma) = \int_0^\gamma q(r) \mathrm{d}r = \frac{2 \pi^{\nicefrac{N}{2}}}{\Gamma(\nicefrac{N}{2})}
\left(\frac{\alpha-1}{2}\right)^{\tfrac{1}{\alpha-1}}
R^{\tfrac{2}{\alpha-1}}
\int_0^\gamma
r^{N-1}\left(1 - (\nicefrac{r}{R})^2\right)^{\nicefrac{1}{\alpha-1}}\mathrm{d}r\,.
\end{equation}
The integral satisfies
\begin{equation}
\begin{aligned}
\int_0^\gamma r^{N-1}\left(1 - (\nicefrac{r}{R})^2\right)^{\nicefrac{1}{\alpha-1}}
\mathrm{d}r
&=
\frac{R^2}{2}
\int_0^\gamma r^{N-2}\left(1 - (\nicefrac{r}{R})^2\right)^{\nicefrac{1}{\alpha-1}}
\frac{2r}{R^2}\mathrm{d}r \\
&=
\frac{R^N}{2}
\int_0^\gamma (\nicefrac{r}{R})^{N-2}\left(1 - (\nicefrac{r}{R})^2\right)^{\nicefrac{1}{\alpha-1}} \frac{2r}{R^2}
\mathrm{d}r \\
&=
\frac{R^N}{2}
\int_0^{\nicefrac{\gamma^2}{R^2}} u^{\tfrac{N}{2}-1}\left(1 -
u\right)^{\nicefrac{1}{\alpha-1}} \mathrm{d}u \\
&=
\frac{R^N}{2}
B\left(\tfrac{N}{2}, \tfrac{\alpha}{\alpha-1}\right)
I_{\tfrac{\gamma^2}{R^2}}\left(\tfrac{N}{2}, \tfrac{\alpha}{\alpha-1}\right)\,,
\end{aligned}
\end{equation}
where $B$ is the Beta function and $I_z$ is the incomplete regularized Beta
function, satisfying $I_1(\cdot, \cdot) = 1$.
In other words, we have $Q(\gamma) = c I_{\gamma^2/R^2}(N/2,
\alpha/(\alpha-1))$,
with $c$ not depending on $\gamma$.
All the mass must be contained within
radius $R$, \textit{i.e.}, $Q(R)=1$, thus $c=1$ and
\begin{equation}
Q(\gamma) = I_{\tfrac{\gamma^2}{R^2}}\left(\tfrac{N}{2},
\tfrac{\alpha}{\alpha-1}\right)\,.
\end{equation}
Since $I_{z}$ is the CDF of the Beta
distribution, we have
$\frac{r^2}{R^2} \sim \mathrm{Beta}\left(\frac{N}{2}, \frac{\alpha}{\alpha -
1}\right)$.
Solving for $R$ in $c=1$ gives the desired value.

To establish the relationship between $R$ and $\tau$ for a general
$\beta$-Gaussian $\mathcal{N}_\beta(t, \mu, \Sigma)$, write
\begin{equation}
f(t) = -\frac{1}{2} (t-\mu)^\top \Sigma^{-1} (t-\mu)
=  -\frac{1}{2} \|\Sigma\|^{-\frac{1}{N + \frac{2}{\alpha-1}}} (t-\mu)^\top \tilde{\Sigma}^{-1} (t-\mu)
=  -\frac{1}{2} \|\Sigma\|^{-\frac{1}{N+\frac{2}{\alpha-1}}} \|z\|^2 \,,
\end{equation}
therefore $\|z\| < R$ is equivalent to $f(t) > \tau=-\frac{R^2}{2}
\|\Sigma\|^{-\frac{1}{N+\frac{2}{\alpha-1}}}$.

\subsection{Fenchel-Young Loss for $\beta$-Gaussian Distributions: Proof of
Proposition~\ref{prop:beta_gaussian_fy}}\label{sec:proof_beta_gaussian_fy}

First, note that, up to a constant term which does not affect the Fenchel-Young loss, we can write $f_\theta(t) = -\frac{1}{2}(t-\mu_f)^\top \Sigma_f^{-1}(t-\mu_f) + \frac{1}{2}\mu_f^\top \Sigma_f^{-1} \mu_f = \theta^\top \phi(t)$, with
$\phi(t) = [t, \mathrm{vec}(tt^\top)]$ and $\theta = [\Sigma_f^{-1}\mu_f, -\frac{1}{2}\mathrm{vec}(\Sigma_f^{-1})]$.
Let $p_\theta \equiv \hat{p}_{\Omega_\alpha}[f_\theta]$.
From Prop.~\ref{prop:entropy_normalizing} we have
\begin{equation}\label{eq:fyloss_beta_gaussians_proof}
L_{\Omega_\alpha}(f_\theta, p) = \Omega_\alpha(p)  - \Omega_\alpha(p_\theta) - \theta^\top (\mathbb{E}_{p}[\phi(t)] - \mathbb{E}_{p_\theta}[\phi(t)]).
\end{equation}
From Prop.~\ref{prop:mean_var_entropy_beta_gaussians} we have
\begin{equation}
\mathbb{E}_{p}[\phi(t)] = [\mu, \mathrm{vec}(\mathrm{Var}(t) + \mu\mu^\top)]
= \left[\mu, \mathrm{vec}\left(\left(\frac{1}{\alpha} + (\alpha-1)\Omega_{\alpha}(p)\right)\Sigma + \mu\mu^\top\right)\right]
\end{equation}
and
\begin{equation}
\mathbb{E}_{p_\theta}[\phi(t)] = [\mu_f, \mathrm{vec}(\mathrm{Var}(t) + \mu_f\mu_f^\top)]
= \left[\mu_f, \mathrm{vec}\left(\left(\frac{1}{\alpha} + (\alpha-1)\Omega_{\alpha}(p_\theta)\right)\Sigma_f + \mu_f\mu_f^\top\right)\right].
\end{equation}
Plugging in Equation~\eqref{eq:fyloss_beta_gaussians_proof}, we get
\begin{eqnarray}\label{eq:fyloss_beta_gaussians_proof_02}
L_{\Omega_\alpha}(f_\theta, p) &=& \Omega_\alpha(p)  - \Omega_\alpha(p_\theta) - \mu_f^\top \Sigma_f^{-1}(\mu-\mu_f) + \frac{1}{2}\mathrm{vec}(\Sigma_f^{-1})^\top \cdot \nonumber\\
&&
\mathrm{vec}\left(\left(\frac{1}{\alpha} + (\alpha-1)\Omega_{\alpha}(p)\right)\Sigma - \left(\frac{1}{\alpha} + (\alpha-1)\Omega_{\alpha}(p_\theta)\right)\Sigma_f + \mu\mu^\top - \mu_f\mu_f^\top\right)\nonumber\\
&=& \Omega_\alpha(p)  - \Omega_\alpha(p_\theta) - \mu_f^\top \Sigma_f^{-1}(\mu-\mu_f) + \frac{1}{2}(\mu^\top \Sigma_f^{-1}\mu - \mu_f^\top \Sigma_f^{-1}\mu_f) +
\nonumber\\
&&
\frac{1}{2}\mathrm{vec}(\Sigma_f^{-1})^\top
\mathrm{vec}\left(\left(\frac{1}{\alpha} + (\alpha-1)\Omega_{\alpha}(p)\right)\Sigma - \left(\frac{1}{\alpha} + (\alpha-1)\Omega_{\alpha}(p_\theta)\right)\Sigma_f\right)\nonumber\\
&=& \Omega_\alpha(p)  - \Omega_\alpha(p_\theta) + \frac{1}{2}(\mu-\mu_f)^\top \Sigma_f^{-1}(\mu - \mu_f) +
\nonumber\\
&&
\frac{1}{2} \left(\frac{1}{\alpha} + (\alpha-1)\Omega_{\alpha}(p)\right) \mathrm{Tr}(\Sigma_f^{-1}\Sigma)
- \frac{N}{2} \left(\frac{1}{\alpha} + (\alpha-1)\Omega_{\alpha}(p_\theta)\right).
\end{eqnarray}
Using the expression for the entropy in Prop.~\ref{prop:mean_var_entropy_beta_gaussians}, we get
\begin{equation}
\Omega_\alpha(p)  - \Omega_\alpha(p_\theta) = \frac{R^2}{2\alpha + N(\alpha-1)}\left(|\Sigma|^{-\frac{1}{N + \frac{2}{\alpha-1}}} - |\Sigma_f|^{-\frac{1}{N + \frac{2}{\alpha-1}}}\right)
\end{equation}
and
\begin{equation}
\frac{1}{\alpha} + (\alpha-1)\Omega_\alpha(p) =  \frac{(\alpha-1)R^2}{2\alpha + N(\alpha-1)} |\Sigma|^{-\frac{1}{N + \frac{2}{\alpha-1}}}.
\end{equation}
Plugging this in \eqref{eq:fyloss_beta_gaussians_proof_02} leads to the  expression in Prop.~\ref{prop:beta_gaussian_fy}.

As for the cross-$\Omega$ loss $L_{\Omega_\alpha}^{\times}(f_\theta, p)$, we have from Definition~\ref{def:fyloss} and Prop.~\ref{prop:mean_var_entropy_beta_gaussians}:
\begin{eqnarray}
L_{\Omega_\alpha}^{\times}(f_\theta, p) &=& L_{\Omega_\alpha}(f_\theta, p) - \Omega_{\alpha}(p)\nonumber\\
&=& L_{\Omega_\alpha}(f_\theta, p) +\frac{1}{\alpha(\alpha-1)} - \frac{R^2 |\Sigma|^{-\frac{1}{N+\frac{2}{\alpha-1}}}}{2\alpha + N(\alpha-1)}\nonumber\\
&=& \frac{1}{2}(\mu-\mu_f)^\top \Sigma_f^{-1} (\mu-\mu_f) + \frac{1}{\alpha(\alpha-1)}
+
\frac{R^2}{2\alpha + N(\alpha-1)} \cdot \\
&& \cdot \left( |\Sigma|^{-\frac{1}{N + \frac{2}{\alpha-1}}} \left(\frac{\alpha-1}{2} \mathrm{Tr}(\Sigma_f^{-1} \Sigma) \right) - |\Sigma_f|^{-\frac{1}{N + \frac{2}{\alpha-1}}} \left( 1 + \frac{N(\alpha-1)}{2}\right)\right).
\end{eqnarray}

In the univariate case ($N=1$) this becomes:
\begin{eqnarray}
L_{\Omega_\alpha}^{\times}(f_\theta, p) &=& \frac{(\mu-\mu_f)^2}{2\sigma_f^2} + \frac{1}{\alpha(\alpha-1)}
+
\frac{R^2}{3\alpha - 1} \cdot \left( \frac{\alpha-1}{2} \sigma^{\frac{2(1-\alpha)}{1+\alpha}} \frac{\sigma^2}{\sigma_f^2}  - \frac{\alpha+1}{2} \sigma_f^{\frac{2(1-\alpha)}{1+\alpha}} \right)\nonumber\\
&=& \frac{(\mu-\mu_f)^2}{2\sigma_f^2} + \frac{1}{\alpha(\alpha-1)}
+
\frac{R^2}{3\alpha - 1} \cdot \left( \frac{\alpha-1}{2} \frac{\sigma^{\frac{2}{1+\alpha}}}{\sigma_f^2}  - \frac{\alpha+1}{2} \sigma_f^{\frac{2(1-\alpha)}{1+\alpha}} \right).
\end{eqnarray}

\paragraph{Geometry.}
In the case of 1-d Gaussian distributions, the KL divergence induces a
hyperbolic geometry on the $[\mu, \sigma]$ half-space, isomorphic to the
Poincar\'e half-space model: geodesics and interpolating points in this space
correspond to half-circles, \textit{e.g.}, the midpoint between two 1-d Gaussians has
larger standard deviation than each \cite[Remark 8.2]{cot}. The Fenchel-Young
loss between a $\beta$-Gaussian and a parabola $f$ reduces to the KL divergence
for $\alpha=1$, suggesting a similarly interesting induced geometry.
Considering the $[\mu, \sigma]$ space as a manifold,
its geometry is captured by the metric tensor, which in the Gaussian case
is $F_1 = \operatorname{diag}([\sigma^{-2}, 2\sigma^{-2}])$.
Taking a second-order Taylor expansion of the Fenchel-Young loss%
\footnote{Despite the asymmetry, the result is the same regardless which pair of
parameters are varied.}
yields the Riemannian metric tensor
$F_\alpha = \operatorname{diag}([\sigma^{-2}, \frac{4R^2(\alpha -
1)}{(\alpha+1)(3\alpha-1)} \sigma^{-\frac{4\alpha}{\alpha+1}}])$.
Figure \ref{fig:fy_geodesics} shows geodesics in this space for different values
of $\alpha$.

\begin{figure}
\centering
\includegraphics[width=.49\textwidth]{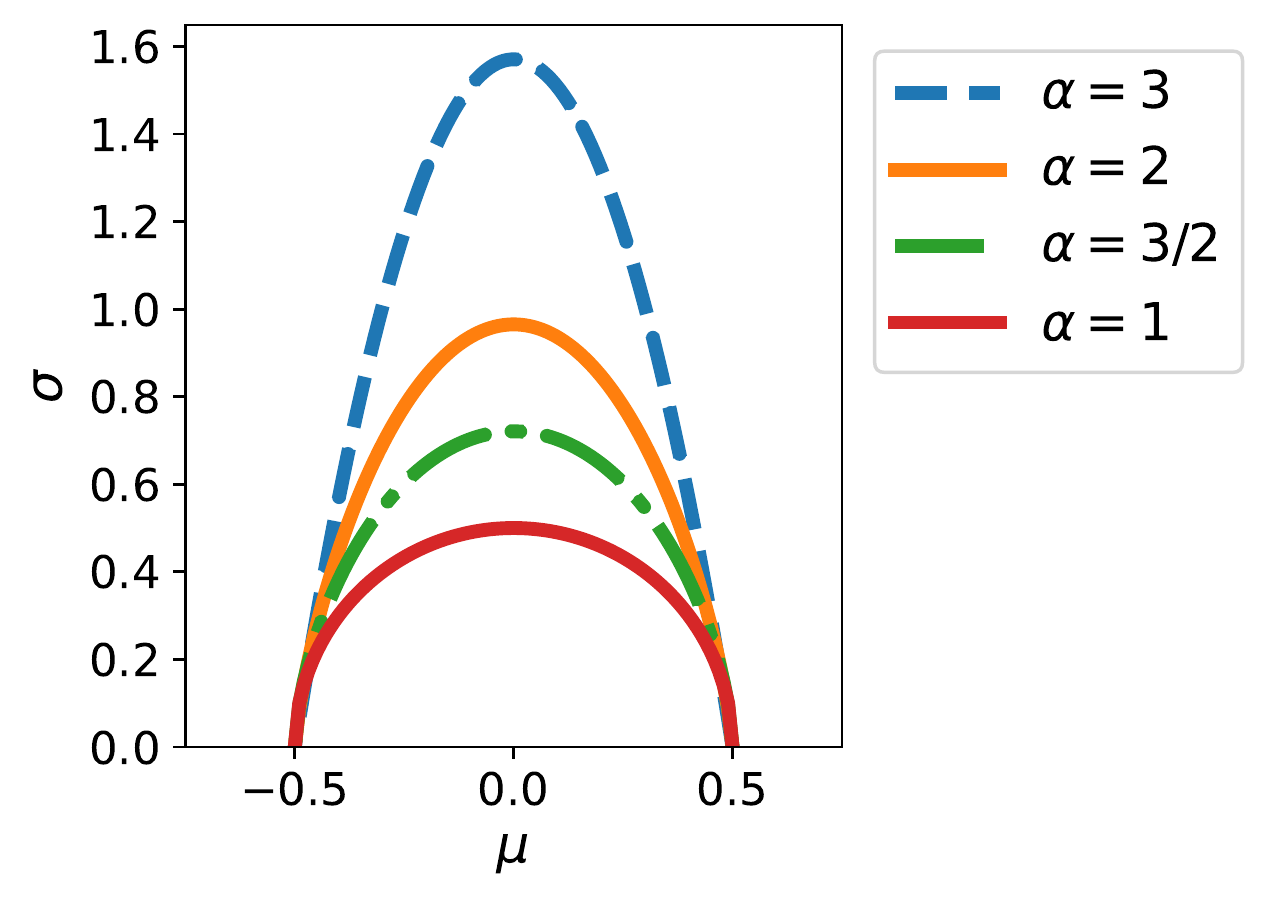}
\caption{\label{fig:fy_geodesics}Geodesics of the $\beta$-Gaussian Fenchel-Young loss
in the $[\mu, \sigma]$ half-plane between two Dirac limit
cases with means $\pm 0.5$. For $\alpha=1$, the FY loss (equivalent to the Kullback-Leibler
divergence) induces the Poincar\'e half-plane
geometry, and geodesics are half-circles.}
\end{figure}

\subsection{Proof of Proposition~\ref{prop:jacobian_entmax}}\label{sec:proof_jacobian_entmax}

We have
\begin{eqnarray}
\nabla_\theta \mathbb{E}_{p}[\psi_i(t)] &=& \nabla_\theta \int_S p_\theta(t) \psi_i(t)
= \int_S \nabla_\theta p_\theta(t) \psi_i(t) \nonumber\\
&=& \int_S p_\theta^{2-\alpha}(t) \nabla_\theta \log_{2-\alpha}(p_\theta(t)) \psi_i(t) \nonumber\\
&=& \int_S p_\theta^{2-\alpha}(t) \nabla_\theta (\theta^\top \phi(t) - A_\alpha(\theta)) \psi_i(t) \nonumber\\
&=& \int_S p_\theta^{2-\alpha}(t) (\phi(t) - \nabla_\theta A_\alpha(\theta)) \psi_i(t).
\end{eqnarray}
Using the expression for $\nabla_\theta A_\alpha(\theta)$ from Proposition~\ref{prop:gradient_A} yields the desired result.

\section{Proofs for continuous fusedmax}

\subsection{Proof of Proposition~\ref{prop:fusedmax_unimodal_symmetric}}\label{sec:proof_fusedmax_unimodal_symmetric}

We split this proof into two parts. First, we show that
$\hat{p}_{\Orof}(t) = \left[\hat{u}[f](t) - \tau\right]_+$ where $\hat{u}[f]$ is the solution of the
unconstrained ROF optimization:
\begin{equation}
\argmin_{u \in H^1} \int_S (f - u)^2 + \gamma \operatorname{TV}(u)\,.
\end{equation}
Then, we invoke the \emph{taut string} algorithm to solve the ROF optimization
for signals of the given form, yielding the desired result.

\begin{definition}[Total variation.]\label{def:total_variation}
The total variation of a function $f \in L^1(S)$ is defined as
\begin{equation}
\operatorname{TV}(u) = \sup \left\{
\int_S u(t) \xi'(t) : \xi \in C_0^1(S), \| \xi \| \leq 1
\right\}\,,
\end{equation}
where $C_0^1$ denotes the set of continuously differentiable functions with
compact support over $S$.
\end{definition}
Indeed, when $u'$ exists, this definition leads to $\operatorname{TV}(u) = \int_S |u'|$.

\paragraph{Decomposition of constrained ROF optimization.}

Let $L^2(S)$ denote the standard Hilbert space of Lebesgue-measurable,
square-integrable functions over an interval $S$,
and $L^2_+(S)$ the cone of non-negative functions.
We can identify densities in $\mathcal{M}_+^1(S)$ with probability density
functions in $L^2_+(S) \cap \{ p : \int_S p = 1 \}.$
We shall use $\langle \cdot, \cdot \rangle$ to denote the inner product in
$L^2(S)$, and $\| \cdot \|$ the corresponding norm.
\begin{proposition}\label{prop:fusedmax_decomposition}
Assume that $f$ is chosen such that the ROF objective is bounded, \textit{i.e.},
\begin{equation}
\inf_{u \in L^2}
\frac{1}{2} \| f - u \|^2 +
\gamma \operatorname{TV}(u)
< \infty\,,
\end{equation}
and let $\hat{u}[f]$ denote the maximizer above. Then,
\begin{equation}
\hat{p}_{\Orof}[f] =
\argmin_{p \in \mathcal{M}_+^1}
\frac{1}{2} \|f - p\|^2 + \gamma \operatorname{TV}(p)
\end{equation}
exists and is given by
\begin{equation}
\hat{p}_{\Orof}[f](t) = \left[\hat{u}[f] - \tau\right]_+,
\end{equation}
for some $\tau$ that can be found by solving
$\int_S \hat{p}_{\Orof}[f] = 1$.
\end{proposition}

\begin{proof}
The proof proceeds in two parts. First, we eliminate the normalization
constraint by showing it can be absorbed into the function $f$. Then, we show
that the non-negativity constraint can be obtained via clipping.
We remark that in the discrete case this result is well-known \citep{yudecomp},
but the proof therein does not readily apply in the continuous case.

Using the method of Lagrange multipliers, we move the normalization constraint
into the objective,
yielding
\begin{align*}
\hat{p}_{\Orof}[f] =&
\argmin_{p \in L^2_+}
\frac{1}{2} \|f - p\|^2 + \gamma \operatorname{TV}(p) + \tau \int_S p \\
=&
\argmin_{p \in L^2_+}
\frac{1}{2} \| p \|^2 + \frac{1}{2} \| f \|^2 - \langle p, f - \tau \rangle + \gamma
\operatorname{TV}(p). \\
\intertext{Assuming $\tau$ fixed at its (unknown) optimal value, this is equivalent to}
=&\argmin_{p \in L^2_+}
\frac{1}{2} \| p - (f - \tau) \|^2 + \gamma \operatorname{TV}(p). \\
\intertext{Using the invariance of $TV$ to a constant, we then get}
=&\argmin_{p \in L^2_+}
\frac{1}{2} \| (p + \tau) - f \|^2 + \gamma \operatorname{TV}(p + \tau). \\
\end{align*}
Choosing $p = \hat{u}[f] - \tau$ would minimize the above objective, but
might not satisfy the non-negativity constraints. We next show that
$[\hat{u}[f] - \tau]_+$
is optimal for the constrained problem.
Without loss of generality, we may assume $\tau=0$, so it suffices to show that:
\begin{equation}
\argmin_{u \in L^2_+}
\frac{1}{2} \| u - f \|^2 + \gamma \operatorname{TV}(u) =
\left[\hat{u}[f]\right]_+\,.
\end{equation}
The rest of the proof closely follows \citep[Theorem 5]{overgaard2019taut},
replacing $\|u\|_1$ with $\iota_{L^2_+}(u)$, and thus replacing
the soft threshold map with the clipping map at zero, and the dual
set $B$, instead of the $L^\infty$ unit ball, is the polar cone
$(L^2_+)^\circ = \{ f \in L^2 : f(t) \leq 0 \} = L^2_-.$
Specifically, since
\[\iota_{L^2_+}(f) =
(\sigma_{L^2_+}^*)(f) =
\sigma_{(L^2_+)^\circ}(f) = \sup_{\eta \in L^2_-} \langle u, \eta \rangle,\]
we have
\[
E_\gamma(u) = \sup_{\xi \in K, \eta \in L^2_-}
\frac{1}{2} \| f - u \|^2 + \gamma \langle u, \xi' \rangle + \langle u, \eta
\rangle
= \sup_{\zeta \in C}
\frac{1}{2} \| f - u \|^2 + \langle u, \zeta \rangle
\,
\]
where
$L^2_-$ is a polar cone in $L^2(S)$ and thus closed and convex
\citep[Proposition 6.24]{Bauschke_Combettes2011},
$K$ is a set of test functions, closed and convex in $H^1(S)$ per \citep[Lemma 2]{overgaard2019taut}
implying $K' = \{\xi': \xi \in K\}$ is convex and closed in $L^2(S)$.
We define
$C = \gamma K' + L^2_-$, which has the same structure as in the proof of
\citep[Theorem 5]{overgaard2019taut}, so it is also a closed convex set.
Following \citep[Theorem 3]{overgaard2019taut} we have that
\begin{equation}
\min E_\gamma(u) = \max_{\zeta \in C} \|f \|^2 - \| f - \zeta \|^2\,
\end{equation}
with an optimal primal-dual pair satisfying
\begin{equation}
u^\star = f - \zeta^\star
\end{equation}
alongside the necessary and sufficient optimality condition
\begin{equation}
\langle f - \zeta^\star, \zeta - \zeta^\star \rangle \leq 0
~\text{for all}~\zeta \in C\,.
\end{equation}
Setting $\zeta = \gamma {\xi^\star}' - \eta$ we get the condition
\begin{equation}
\langle
f - \gamma {\xi^\star}' - \eta^\star, \eta - \eta^\star
\rangle \leq 0
~\text{for all}~\eta \in L^2_-\,,
\end{equation}
which implies by the projection theorem that
\begin{equation}
\eta^\star = [f - \gamma {\xi^\star}']_-
\end{equation}
and thus
\begin{equation}
u^\star
= f - \gamma {\xi^\star}' - \eta^\star
= f - \gamma {\xi^\star}' - [f - \gamma{\xi^\star}']_-
= [f - \gamma {\xi^\star}']_+\,.
\end{equation}
It remains to show that $\xi^\star$ can be taken to be the optimal dual variable
from the unconstrained model, \textit{i.e.}, that $\gamma{\xi^\star}' = f - \hat{u}[f]$,
then, it follows that $u^\star = \left[ \hat{u}[f] \right]_+$.
To show this, we set $\zeta = \gamma \xi - \eta^\star$, giving
\[
\langle f - \gamma{\xi^\star}' - \eta^\star, \xi' - {\xi^\star}' \rangle \leq 0
~\text{for all}~\xi \in K\,,
\]
and since $\eta^\star = [f - \gamma {\xi^\star}']_-$, $\xi^\star$ must satisfy
\begin{equation}\label{eq:rof_xi_condition}
\langle [f - \gamma {\xi^\star}']_+, \xi' - {\xi^\star}' \rangle \leq 0
~\text{for all}~\xi \in K\,,
\end{equation}
We define $H(t) = \frac{1}{2}[t]_+^2$, chosen such that $H'(t) = [t]_+$.
By \citep[Lemma 1]{overgaard2019taut},
we have that the ROF taut-string solution $\hat{\xi}[f]$ is also a solution to
\[ \inf_{\xi \in K} L_H(f - \gamma \xi')
\quad\text{where}\quad
L_H(W) = \int_S H(W')\,.
\]
But this problem has optimality condition
\[
\langle H'(f - \gamma {\xi^\star}'), \xi' - {\xi^\star}' \rangle \leq 0
\]
which is exactly Equation~\eqref{eq:rof_xi_condition}.
This shows that the choice $\xi^\star = \hat{\xi}$ and $\eta^\star = [f - \gamma
\hat{\xi}']_-$ satisfies the optimality conditions, so $[\hat{u}]_+$ is optimal
for the constrained problem.
\end{proof}

\paragraph{Application of the taut string algorithm.}

In this section, we prove Proposition~\ref{prop:fusedmax_unimodal_symmetric},
using Proposition~\ref{prop:fusedmax_decomposition} and the taut string
algorithm \citep{overgaard2019taut}.

Specifically, we assume $f$ is even and unimodal, strictly decreasing on $(0,
\infty)$, and show that
\begin{equation}
\hat{p}_{\Orof}[f](t) = [f_a(t) - \tau]_+,
\quad \text{where} \quad f_a(t) \coloneqq
\begin{cases}
f(a), & t \in (-a, a), \\
f(t), & \text{otherwise}.
\end{cases}
\end{equation}

We make the technical assumption that $S=[-B, B]$, to ensure that all
subproblems are computable and bounded. We shall see that the end result does
not depend on $B$ as long as $B$ is large enough, and therefore holds for
$S=(-\infty, \infty).$

We begin by computing the cumulative signal
\begin{equation}
F(x) = \int_{-B}^x f,
\end{equation}
which, from the monotonicity of $f$, is concave on $[-B, 0]$ and convex on $[0,
B]$.
We must compute the trajectory of a
\emph{taut string} between the ends of $F$ through a tube of radius $\gamma$,
\textit{i.e.},
\begin{equation}
\min_{W \in T_\gamma} J[W] \coloneqq \frac{1}{2} \int_{-B}^{B} (W'(x))^2
\mathrm{d}x.
\end{equation}
where $T_\gamma \coloneqq \big\{W \in H^1(S): W(-B)=F(-B), W(B)=F(B), F-\gamma
\leq W \leq F+\gamma\big\}.$
Then, by \citep[Theorem 1]{overgaard2019taut}, we have $\hat{u}[f] = W'$.

The functional $J[W]$ is equivalent to the arc length functional, so this
corresponds to finding the shortest path between the end points. The problem is
illustrated in Figure~\ref{fig:fused_cumulative}.

\begin{figure}\centering
\includegraphics[width=.48\textwidth]{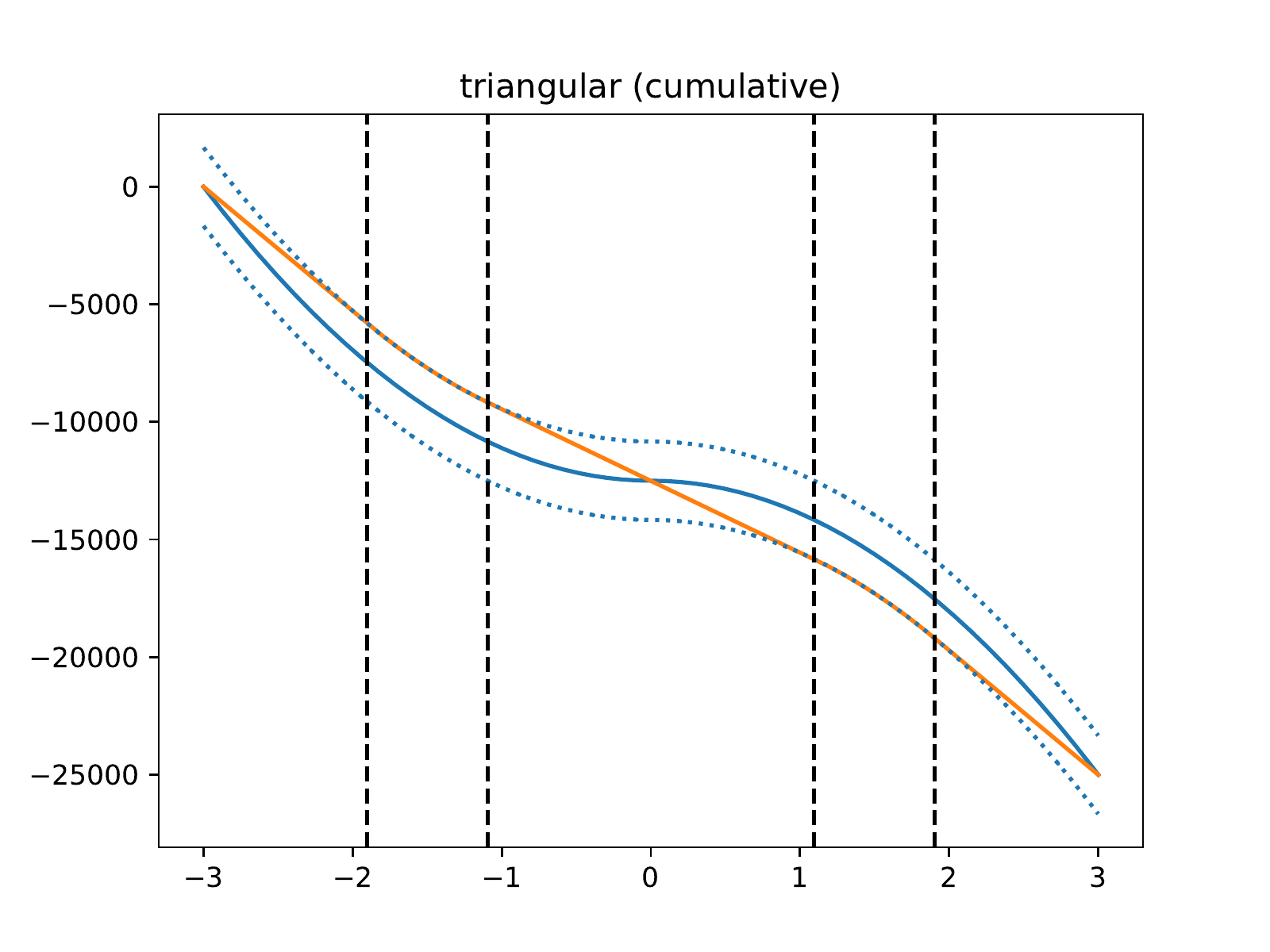}
\includegraphics[width=.48\textwidth]{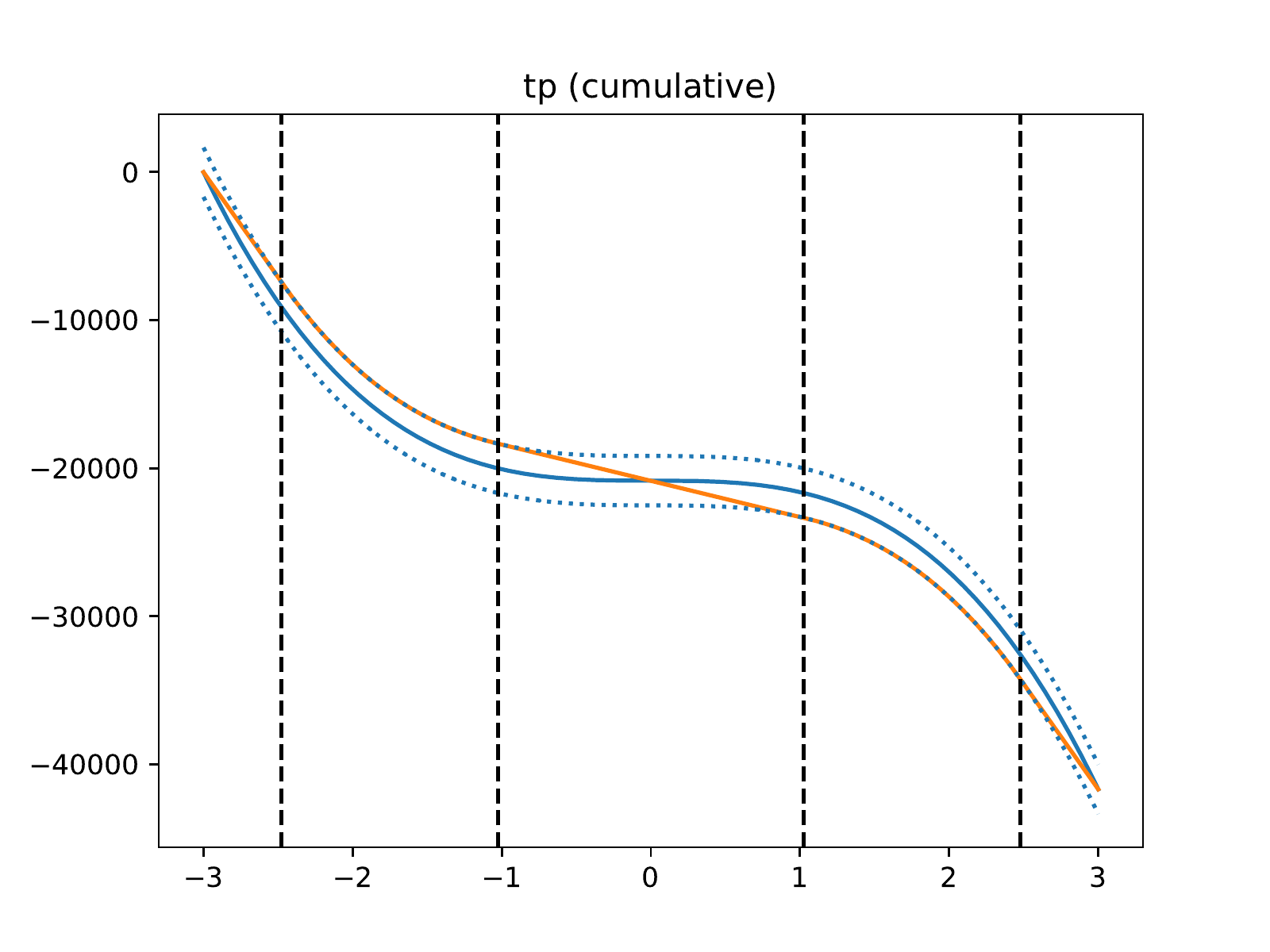}
\caption{\label{fig:fused_cumulative}Taut string interpretation. For unimodal
even potential, the solution is symmetric and the contact set has the form $(-B_0, -a) \cup (a, B_0)$. Left: $f(t) = -\nicefrac{|t|}{\sigma}$, right: $f(t) =
-\nicefrac{t^2}{2\sigma^2}$.}
\end{figure}

First, since $f$ is symmetric around the origin,
it is sufficient to consider the interval $[0, B]$. This will greatly
simplify the derivation.
Then, observe that on $[0, B]$, the ``top'' part of the tube is never an active
constraint. To show this, note that $F$ itself is feasible. It suffices to show
that any solution above $F$ has higher objective value.
Consider a perturbation $\xi \in H^1(S)$ such that $\xi(0)=\xi(B)=0$ and $\xi
\geq 0$. Calculate
\begin{equation}
\begin{aligned}
J[F+\xi] - J[F] &= \|f + \xi'\|^2 - \|f\|^2 \\
&= \|f\|^2 + \|\xi'\|^2 + 2 \DP{f}{\xi'} - \|f\|^2
&= \|\xi'\|^2 + 2\DP{f}{\xi'}.
\end{aligned}
\end{equation}
Using integration by parts, we have
\begin{equation}
\DP{f}{\xi'} = f\xi|_0^B - \DP{f'}{\xi} = \DP{-f'}{\xi}.
\end{equation}
Since $f$ is decreasing on $[0, B]$, $-f'\geq0$, therefore
\begin{equation}
J[F+\xi] - J[F] = \|\xi'\|^2 + 2 \DP{-f'}{\xi} \geq 0.
\end{equation}
We have thus shown we may ignore the top part of the tube, leaving
the simpler variational problem
\begin{equation}
\min_{W} J[W] \quad \text{s.t.}\quad W(0)=F(0), W(B)=F(B), ~\text{and}~ W \geq F-\gamma.
\end{equation}
To handle the inequality constraint, we introduce a slack function $Z$,
\begin{equation}
W = F - \gamma + \nicefrac{1}{2}Z^2.
\end{equation}
such that $W' = f + ZZ'$, and the Lagrangian can be written in terms of $Z$ and
$Z'$ as
\begin{equation}
\mathcal{L}(x, W, W') = \nicefrac{1}{2}(f + ZZ')^2.
\end{equation}
The solution must satisfy the Euler-Lagrange equations,
\begin{equation}
\frac{\mathrm{d}}{\mathrm{d}x}
\frac{\partial \mathcal{L}}{\partial \mathcal{Z'}}
-
\frac{\partial \mathcal{L}}{\partial \mathcal{Z}}
= 0.
\end{equation}
By the chain rule,
\begin{equation}
\begin{aligned}
\frac{\partial \mathcal{L}}{\partial \mathcal{Z'}}
&= Z (f + ZZ') = ZW',\\
\frac{\partial \mathcal{L}}{\partial \mathcal{Z}}
&= Z' (f + ZZ') = Z'W'.
\end{aligned}
\end{equation}
Then, using the product rule, we have
\begin{equation}
\frac{\mathrm{d}}{\mathrm{d}x} (ZW') - W'Z'
= ZW''
+ W'Z'  - W'Z'
= ZW'' \stackrel{!}{=}0.
\end{equation}
This means that for any $x \in [0, B]$, either $Z(x)=0$ (in which case $W = F -
\gamma$, so the path follows the path of the tube), or $Z(x)>0$ in which case
$W''(t)=0$ so the solution must be locally linear.

We may safely assume $\gamma>0$, otherwise, there is no ROF regularization and
the solution is $W=F$.
Therefore, in a small enough ball around the end points $F(0)$ and $F(B)$, the solution must be
locally linear. It remains to show that the set of points on which $Z=0$ is an
interval $(a, B_0)$. Assume there exist $c<d$ such that $Z(c)=Z(d)=0$,
but $Z(x) > 0$ for $c<x<d$. We must have $W(c)=F(c)-\gamma$ and
$W(d)=F(d)-\gamma$, but,
since $W''=0$ on $(c, d)$, $W$ must be a straight line in between, therefore
$W\big((1-\alpha)c + \alpha d)= (1-\alpha)F(c) + \alpha F(d) - \gamma$ for $\alpha \in
[0, 1]$. But since $f$ is decreasing, $F$ is concave thus
\begin{equation}
(1-\alpha)F(c) + \alpha F(d) -\gamma \leq
F\big((1-\alpha)c + \alpha d) - \gamma,
\end{equation}
therefore the choice of $W$ violates the tube constraints and is infeasible,
so the optimal $W$ must be stuck to the tube for a contiguous interval of the
form $(a,B_0)$. Taking $\hat{u}[f]=W'$ and extending by symmetry to $[-B, B]$
leads to the general form of the ROF transform of a denoised unimodal potential:
\begin{equation}\label{eq:rof_u_hat}
\hat{u}[f](t) =
\begin{cases}
f(B_0), & t \in (-B,  -B_0), \\
f(t), & t \in (-B_0, -a), \\
f(a), & t \in (-a, a), \\
f(t), & t \in (a, B_0), \\
f(B_0), & t \in (B_0,  B), \\
\end{cases}
\end{equation}
for some $a$ and $B_0$. To find these values, we turn to the ROF objective, which evaluates to
\begin{equation}
\begin{aligned}
V[\hat{u}] &= .5 \int_0^B (\hat{u} - f)^2 + \gamma \int_0^B |\hat{u}'| \\
     &= .5\Big(
\int_0^a (f(a) - f(t))^2
+\int_{B_0}^B (f(B_0) - f(t))^2
\Big)
- \gamma \big(f(B_0) - f(a)\big).
\end{aligned}
\end{equation}
Note that $V[\hat{u}]$ separates into two independent terms. To solve for
$a$, we evaluate
\begin{equation}
\frac{\partial}{\partial a} V[\hat{u}] =
f'(a) \big( a f(a) - \int_0^a f + \gamma \big) \stackrel{!}{=}0\,.
\end{equation}
Since $f$ is strictly decreasing, $f'(a) \neq 0$, leaving the identity
\begin{equation}\label{eq:implicit_a}
a f(a) - \int_0^a f + \gamma =0\,.
\end{equation}

\paragraph{Sparse projection.}
For the purposes of computing $\hat{p}_{\Orof}$ on $S=\mathbb{R}$, the specific
value of $B_0$ is not important. We next show that $B_0$ is increasing as a
function of $B$, therefore we may always set $B$ to a large enough finite value
to yield a sufficiently large $B_0$.
\begin{lemma}
As a function of $B$, $B_0$ is strictly increasing.\end{lemma}
\begin{proof}
As $f$ is strictly decreasing (non-constant), $f'<0$, thus the relationship
between $B_0$ and $B$ is given by $\frac{\partial V[\hat{u}]}{\partial B_0} =0$
as
\begin{equation}
M(B, B_0) = (B - B_0) f(B_0) - \int_{B_0}^B f - \gamma = 0.
\end{equation}
The partial derivatives with respect to each variable are
\begin{equation}
\begin{aligned}
\frac{\partial}{\partial B} M &= f(B_0) - f(B), \\
\frac{\partial}{\partial B_0} M &= Bf'(B_0) - f(B_0) - B_0f'(B_0) + f(B_0) =
(B-B_0)f'(B_0). \\
\end{aligned}
\end{equation}
The implicit function theorem applies, yielding
\begin{equation}
\frac{\partial B_0}{\partial B} =
-\left(\frac{\partial}{\partial B_0} M\right)^{-1}
\left(\frac{\partial}{\partial B} M\right) = \frac{f(B)-f(B_0)}{(B-B_0) f'(B_0)} > 0,
\end{equation}
where we used the monotonicity of $f$ and the fact that $B>B_0$.
\end{proof}

\paragraph{Putting things together: form of the fusedmax solution.}
The form of $\hat{u}$ was established in Equation~\eqref{eq:rof_u_hat}: it
matches the form of $f$ on $(a, B_0) \cup (-B_0, -a)$ and is constant everywhere
else. From Proposition~\ref{prop:fusedmax_decomposition}, we have that
\begin{equation}
\hat{p}_{\Omega_2}[f](t) = [\hat{u}[f](t) - \tau]_{+}\,,
\end{equation}
so $\hat{p}_{\Omega_2}$ corresponds to a shift of $\hat{u}$ followed by
a clipping to zero.
The lemma we just proved allows us to ignore $B$ and $B_0$ and solve fusedmax
directly for $S=\mathbb{R}$ for unimodal potentials, by choosing a large enough
(but still finite) $B$ for the inner ROF problem, so that $B_0$ lie outside of
the support.  For arbitrarily large $B$ and thus $B_0$,
$\hat{p}$ has support on an interval $[-b, b]$, where $b$ satisfies
$\hat{u}(b)=\tau$. We must now find $b$ such that $\int_S \hat{p} = 1$.
First, we see that we must have $b>a$, because otherwise $\hat{p} \equiv 0$
contradicting $\int_S \hat{p} =1$.
Thus, $a < b < B_0$, giving
\begin{equation}
\begin{aligned}
1 &\stackrel{!}{=} 2 \int_0^a f(a) + 2 \int_a^b f(t) - 2 \int_0^b \tau \\
  &= 2\left(af(a) - bf(b) + \int_a^b f(t)\right).
\end{aligned}
\end{equation}
From Equation~\eqref{eq:implicit_a}, we have $af(a) - \int_0^a f(t) = -\gamma$.
Subtracting from the above gives
\begin{equation}\label{eq:implicit_b}
\int_0^b f(t) - bf(b) = \nicefrac{1}{2} + \gamma.
\end{equation}
If we have access to $f$ and its antiderivative, we can therefore compute both
$a$ and $b$ from Equations~\eqref{eq:implicit_a} and \eqref{eq:implicit_b}
respectively. This completes the proof of the proposition.

\subsection{Sobolev regularization: smooth sparsemax.}\label{sec:sobolev_smooth_proof}

We recall the definition of the optimization problem to be solved,
\begin{equation}
\hat{p}_{\Omega_{2,2}}[f] \coloneqq \argmin_{p \in \mathcal{M}_+^1}
\frac{1}{2}\int_S \left(p(t) - f(t)\right)^2 + \frac{\gamma}{2} \int_S
\left(p'(t)\right)^2\,.
\end{equation}
This problem falls within the framework of calculus of variations.
We first remark that since $f$ is even, so is $p$: to see this, consider
$q(t) = p(-t)$ and observe that $J[p] = J[q]$. Since the solution is unique we
must have evenness in the optimum.
We can therefore restrict the optimization to
$(0, \infty)$, where $f$ is strictly decreasing and continuously
differentiable.

Rewriting the problem in more standard notation, we have
\[
\argmin_{p \in H^1(0, \infty)} \int_S F(t, p, p')
~\text{subject to}~ \int_S G(t, p, p') = 1,
~g(t, p, p') \geq 0\,,
\]
where $F(t, p, q) = \nicefrac{1}{2} (f - p)^2 + \nicefrac{\gamma}{2}~q^2,
G(t,p,q) = p,$ and $g(t, p, q) = p.$
To handle the equality constraint, we introduce the dual scalar
$\lambda$ for the equality constraint, leading to the lagrangian
\[ \mathcal{L}[p] = \int_S F + \tau G\,. \]
To handle the inequality constraint, we make the change of variable $p(t) = \frac{1}{2}z(t)^2.$
We have
\[F(t, p, p') = \nicefrac{1}{2}\left(p-f\right)^2 +
\nicefrac{\gamma}{2}\left(p'\right)^2
= \nicefrac{1}{2}\left(\frac{z^2}{2} - f\right)^2
+ \nicefrac{\gamma}{2}\left(zz'\right)^2
=\bar{F}(t, z, z')\,,
\]
where $\bar{F}(t, z, z') \coloneqq
\nicefrac{1}{2}\left(\frac{z^2}{2} - f\right)^2
+\nicefrac{\gamma}{2}\left(zr\right)^2,$ and similarly
\[G(t, p, p') = p = \frac{z^2}{2} = \bar{G}(t, z, z')\,\]
where $\bar{G}(t,z,r) = \frac{1}{2}z^2$.
Now, consider the functional in terms of $z$,
\[ \bar{\mathcal{L}}[z] = \int_S \bar{F} + \tau \bar{G}\,. \]
The associated Euler-Lagrange equation is
\[
\bar{F}_z - \frac{\mathrm{d}}{\mathrm{d}t} \bar{F}_r + \tau\bar{G}_z = 0\,.
\]
The partial derivatives of the functionals above are
\[
\bar{F}_z(t, z, z') = z(p-f) + \gamma z(z')^2,\qquad
\bar{F}_r(t, z, z') = \gamma z^2z',\qquad
\bar{G}_z(t, z, z') = z\,.
\]
Taking the total derivative of $\bar{F}_r$ we get
\[
\frac{\mathrm{d}}{\mathrm{d}t} (z^2z') = 2\gamma z(z')^2 + \gamma z^2z''\,.
\]
Substituting everything into the Euler-Lagrange equation, we get
\[
z\left(p - f - \gamma\big((z')^2 + zz''\big) + \tau\right) = 0\,.
\]
Remarking that $p'' = (z')^2 + zz''$, we rewrite in terms of $p$:
\[
z(p - \gamma p'' - f + \tau) = 0\,.
\]
Note that $z(t)=0$ implies $p(t)=0$.
Let $\bar{p}$ denote a solution of the differential equation
\begin{equation}
p - \gamma p'' = f - \tau\,.
\end{equation}
Then, our regularized prediction map is
\[
p(t) = \begin{cases}\bar{p}(t), & t \in \bar{S}, \\
                    0, &      t \in S \setminus \bar{S}.\end{cases}
\]
It remains to figure out $\bar{S}$ and a suitable $\bar{p}$.

\paragraph{Form of the support.}
We show that $\bar{S}$ takes the form $[0, b]$.
Surely we cannot have $b=0$, due to the constraint that $p$ must integrate to $1$.
We then show that for any $0<c_1<c_2$ with $p(c_1)=p(c_2)=0$,
we must have $p(t)=0$ for all $t \in (c_1, c_2)$.
To show this, we first argue that the optimal $p$ must be non-increasing on $(0,
\infty).$
Let $(d_1, d_2)$ be some interval on which $p$ is non-decreasing.
According to \citep[lemma 2]{monotone} (after flipping the constraint),
the minimizer of $\min \int_{d_1}^{d_2} (f - q)^2$ over the set of non-decreasing functions is the
(left-)derivative of the greatest convex minorant of $F(x) \coloneqq \int_{d_1}^t f(t)$.
But since $f$ is strictly decreasing, $F$ is concave, so its greatest convex
minorant is linear. Therefore, in terms of the L2 norm, no non-decreasing
function is a better approximator of a decreasing $f$ than a constant function.
Moreover, the constant function is also optimal in terms of $\Omega_{2,2}$.
Therefore, $p$ must be constant on any interval on which it is non-decreasing;
Since $p$ is continuous, it is non-increasing.
But the only non-increasing function on $(c_1, c_2)$ with $p(c_1)=p(c_2)=0$ must be
equal to $0$ on the entire interval.
Therefore, the support takes the form $[0, b]$.

\paragraph{Form of the function.}
The corresponding homogeneous differential equation,
$ p - \gamma p'' = 0, $
has characteristic polynomyal
$ 1 - \gamma r^2 = (1-r)(1+r)$, with roots $\pm \gamma^{-\nicefrac{1}{2}}$.
For brevity of notation let $\beta = \gamma^{-\nicefrac{1}{2}}$.
This
are $p_1 = e^{-\beta t}, p_2=e^{\beta t}$.
To find a particular solution for any $f$,
we apply the method of variation of parameters.
Rewrite the equation as $p'' - \beta^2p = g$, where $g = -\beta^2(f - \tau).$
The Wronskian is $W = p_1p_2' - p_1'p2 = 2\beta$. A particular solution is
\begin{equation}
\begin{aligned}
P &=
-p_1 \int \frac{gp_2}{2\beta}
&&
+p_2 \int \frac{gp_1}{2\beta}
\\
&=
\frac{\beta e^{-\beta t}}{2}\int (f - \tau)e^{\beta t}
&&
-\frac{\beta e^{\beta t}}{2} \int (f - \tau) e^{-\beta t}\\
&=
\frac{\beta e^{-\beta t}}{2}
\left(\int fe^{\beta t} - \frac{\tau}{\beta}e^{\beta t}\right)
&&
-\frac{\beta e^{\beta t}}{2}
\left(\int fe^{-\beta t} + \frac{\tau}{\beta}e^{-\beta t}\right) \\
&=
\frac{\beta e^{-\beta t}}{2} \int fe^{\beta t}
&&-\frac{\beta e^{\beta t}}{2} \int fe^{-\beta t}
-\tau\,.
\end{aligned}
\end{equation}
Solutions take the form $C_1 p_1 + C_2p_2 + P$, giving the general form
\[
\bar{p} =
e^{\beta t}\left(C_2 - \frac{\beta}{2}\int fe^{-\beta t}\right)
+
e^{-\beta t}\left(C_1 + \frac{\beta}{2}\int fe^{\beta t}\right)
-\tau\,.
\]
We now make use of the assumption that $f(-t)=f(t)$. Letting $F(t) = \frac{\beta \exp(\beta t)}{2} \int f(t) \exp(-\beta t)
\mathrm{d}t$, a change of variable yields
\[
\bar{p}(t) = C_2 \exp(\beta t) + C_1 \exp(-\beta t) - (F(t) + F(-t)) - \tau\,.
\]
Since by symmetry $\bar{p}(t)=\bar{p}(-t)$, we must have $C_2=C_1=C$ and thus
\[
\bar{p}(t) = C \cosh(\beta t) - (F(t) + F(-t)) - \tau\,.
\]

\section{Proofs for continuous attention with Gaussian RBFs}\label{sec:gaussian_basis}

We derive expressions for the evaluation and gradient computation of  continuous attention mechanisms where $\psi(t)$ are Gaussian radial basis functions and $f(t)$ is a quadratic function, both for the softmax ($\alpha=1$) and sparsemax ($\alpha=2$) cases.
For softmax, we show closed-form expressions for any number of dimensions (including the 1-d and 2-d cases).
For sparsemax, we derive closed-form expressions for the 1-d case, and we reduce the 2-d case to a univariate integral on an interval, easy to compute numerically.
More generally, we show how closed-form expressions can be obtained for the 1-d case when $\alpha$ is of the form $\alpha = \frac{n+1}{n}$ with $n \in \mathbb{N}$ (including $\alpha \in \{\sfrac{4}{3}, \sfrac{3}{2}, 2\}$ as particular cases, corresponding to triweight, biweight, and sparsemax).

This makes it possible to plug both continuous attention mechanisms in neural networks and learn them end-to-end with the gradient backpropagation algorithm.

\subsection{Continuous softmax ($\alpha=1$)}

We derive expressions for continuous softmax for multivariate Gaussians in $\mathbb{R}^D$.
This includes the 1-d and 2-d cases, where $D \in \{1,2\}$.

If $S=\mathbb{R}^D$, for $\phi(t)=[t,tt^\top]$, the distribution $p=\hat{p}_{\Omega_1}[f_\theta]$, with $f_\theta(t)=\theta^\top \phi(t)$, is a multivariate Gaussian where the mean $\mu$ and the covariance matrix $\Sigma$ are related to the canonical parameters as $\theta=[\Sigma^{-1}\mu,-\frac{1}{2}\Sigma^{-1}]$.

We derive closed form expressions for the attention mechanism output
$\rho_1(\theta)=\mathbb{E}_p[\psi(t)]$ in \eqref{eq:attention_expectation} and for its Jacobian $J_{\rho_1}(\theta)= \mathrm{cov}_{p,1}(\phi(t), \psi(t))$ in \eqref{eq:jacob},  when $\psi(t)$ are Gaussian RBFs, \textit{i.e.}, each $\psi_j$ is of the form $\psi_j(t)=\mathcal{N}(t;\mu_j, \Sigma_j)$.

\paragraph{Forward pass.}

Each coordinate of the attention mechanism output becomes the integral of a product of Gaussians,
\begin{equation}
    \mathbb{E}_p[\psi_j(t)]=\int_{\mathbb{R}^D}\mathcal{N}(t;\mu, \Sigma)\mathcal{N}(t;\mu_j, \Sigma_j).
\end{equation}
We use the fact that the product of two Gaussians is a scaled Gaussian,
$\mathcal{N}(t;\mu, \Sigma)\mathcal{N}(t;\mu_j, \Sigma_j)=\Tilde{s}\mathcal{N}(t;\Tilde{\mu}, \Tilde{\Sigma})$,
with
\begin{equation}
    \Tilde{s}=\mathcal{N}(\mu;\mu_j, \Sigma + \Sigma_j), \qquad \Tilde{\Sigma}=(\Sigma^{-1}+\Sigma_j^{-1})^{-1}, \qquad \Tilde{\mu}=\Tilde{\Sigma}(\Sigma^{-1}\mu+\Sigma_j^{-1}\mu_j).
\end{equation}
Therefore, the forward pass can be computed as:
\begin{equation}\label{eq:continuous_softmax_forward_pass}
\mathbb{E}_p[\psi_j(t)] %
=\Tilde{s}\int_{\mathbb{R}^D}\mathcal{N}(t;\Tilde{\mu}, \Tilde{\Sigma})=\Tilde{s} %
=\mathcal{N}(\mu;\mu_j, \Sigma + \Sigma_j).
\end{equation}

\paragraph{Backward pass.}
To compute the backward pass, we have that each row of the Jacobian $J_{\rho_1}(\theta)$ becomes a first or second moment under the resulting Gaussian:
\begin{equation}\label{eq:continuous_softmax_backward_pass_01}
\begin{split}
\mathrm{cov}_{p,1}(t, \psi_j(t)) & = \mathbb{E}_p[t\psi_j(t)]-\mathbb{E}_p[t]\mathbb{E}_p[\psi_j(t)] %
= \int_{\mathbb{R}^D}t\mathcal{N}(t;\mu, \Sigma)\mathcal{N}(t;\mu_j, \Sigma_j)-\Tilde{s}\mu\\
 & =\Tilde{s}\int_{\mathbb{R}^D}t\mathcal{N}(t;\Tilde{\mu}, \Tilde{\Sigma})-\Tilde{s}\mu %
\,\, =\,\, \Tilde{s}(\Tilde{\mu}-\mu),
\end{split}
\end{equation}
and, noting that $\Sigma=\mathbb{E}[(t-\mu)(t-\mu)^\top]=\mathbb{E}[tt^\top]-\mu\mu^\top$,
\begin{equation}\label{eq:continuous_softmax_backward_pass_02}
\begin{split}
\mathrm{cov}_{p,1}(tt^\top, \psi_j(t)) & = \mathbb{E}_p[tt^\top\psi_j(t)]-\mathbb{E}_p[tt^\top]\mathbb{E}_p[\psi_j(t)] \\
&=\int_{\mathbb{R}^D}tt^\top\mathcal{N}(t;\mu, \Sigma)\mathcal{N}(t;\mu_j, \Sigma_j)-\Tilde{s}(\Sigma+\mu\mu^\top)\\
 &=\Tilde{s}\int_{\mathbb{R}^D}tt^\top\mathcal{N}(t;\Tilde{\mu}, \Tilde{\Sigma})-\Tilde{s}(\Sigma+\mu\mu^\top) %
 =\Tilde{s}(\Tilde{\Sigma}+\Tilde{\mu}\Tilde{\mu}^\top)-\Tilde{s}(\Sigma+\mu\mu^\top)\\
 &=\Tilde{s}(\Tilde{\Sigma}+\Tilde{\mu}\Tilde{\mu}^\top-\Sigma-\mu\mu^\top).
\end{split}
\end{equation}

\subsection{Continuous sparsemax in 1-d ($\alpha=2$, $D=1$)}

With $\phi(t) = [t, t^2]$, the distribution $p = \hat{p}_{\Omega_2}[f_\theta]$, with $f_\theta(t) = \theta^\top \phi(t)$, becomes a truncated parabola where  $\mu$ and  $\sigma^2$ are related to the canonical parameters as above, \textit{i.e.}, $\theta = [\frac{\mu}{\sigma^2}, -\frac{1}{2\sigma^2}]$.
We  derive closed form expressions for the attention mechanism output $\rho_2(\theta) = \mathbb{E}_{p}[\psi(t)]$ in \eqref{eq:attention_expectation} and its Jacobian $J_{\rho_2}(\theta) = \frac{\partial \rho_2(\theta)}{\partial \theta} = \mathrm{cov}_{p, 2}(\phi(t), \psi(t))$ in \eqref{eq:jacob} when $\psi(t)$ and Gaussian RBFs, {\it i.e.}, each $\psi_j$ is of the form $\psi_j(t) = \mathcal{N}(t; \mu_j, \sigma_j^2)$.

\paragraph{Forward pass.}
Each coordinate of the attention mechanism output becomes:
\begin{eqnarray}\label{eq:expectation_continuous_sparsemax_rbf}
\mathbb{E}_{p}[\psi_j(t)] &=& \int_{\mu-a}^{\mu+a} \left(-\tau - \frac{(t-\mu)^2}{2\sigma^2}\right)  \mathcal{N}(t; \mu_j, \sigma_j^2)\nonumber\\
&=& \int_{\frac{\mu-\mu_j-a}{\sigma_j}}^{\frac{\mu-\mu_j+a}{\sigma_j}} \frac{1}{\sigma_j} \left(-\tau - \frac{(\sigma_j s + \mu_j - \mu)^2}{2\sigma^2}\right)  \mathcal{N}(s; 0, 1) \sigma_j ds,
\end{eqnarray}
where $a=(\frac{3}{2}\sigma^2)^{1/3}$ and $\tau=-\frac{a^2}{2\sigma^2} = -\frac{1}{2}(\frac{3}{2\sigma})^{2/3}$, as stated in \eqref{eq:lambda_gaussian_proof}, and we made the substitution
$s = \frac{t-\mu_j}{\sigma_j}$.
We use the fact that, for any $u, v \in \mathbb{R}$ such that $u \le v$:
\begin{eqnarray}\label{eq:erf_expr123}
\int_{u}^{v} \mathcal{N}(t; 0, 1) &=&  \frac{1}{2}\left( \mathrm{erf}\left(\frac{v}{\sqrt{2}}\right) - \mathrm{erf}\left(\frac{u}{\sqrt{2}}\right) \right),\nonumber\\
\int_{u}^{v} t\mathcal{N}(t; 0, 1) &=&  -\mathcal{N}(v; 0, 1) + \mathcal{N}(u; 0, 1),\nonumber\\
\int_{u}^{v} t^2\mathcal{N}(t; 0, 1) &=&
\frac{1}{2}\left( \mathrm{erf}\left(\frac{v}{\sqrt{2}}\right) - \mathrm{erf}\left(\frac{u}{\sqrt{2}}\right) \right) - v\mathcal{N}(v; 0, 1) + u\mathcal{N}(u; 0, 1),
\end{eqnarray}
from which the expectation \eqref{eq:expectation_continuous_sparsemax_rbf} can be computed directly.

\paragraph{Backward pass.}
Since $|\mathrm{supp}(p)| = 2a$, we have from \eqref{eq:beta_covariance} and \eqref{eq:erf_expr123} that each row of the Jacobian $J_{\rho_2}(\theta)$ becomes:
\begin{eqnarray}
\lefteqn{\mathrm{cov}_{p, 2}(t, \psi_j(t)) =} \nonumber\\
&& \int_{\mu-a}^{\mu+a} t\mathcal{N}(t; \mu_j, \sigma_j^2)
- \frac{1}{2a}\left(\int_{\mu-a}^{\mu+a} t\right)\left(\int_{\mu-a}^{\mu+a} \mathcal{N}(t; \mu_j, \sigma_j^2)\right)\nonumber\\
&=&
\int_{\frac{\mu-\mu_j-a}{\sigma_j}}^{\frac{\mu-\mu_j+a}{\sigma_j}} (\mu_j + \sigma_j s)\mathcal{N}(s; 0, 1)
- \underbrace{\frac{1}{2a}\left( \frac{(\mu+a)^2}{2} - \frac{(\mu-a)^2}{2} \right)}_{=\mu}
\left( \int_{\frac{\mu-\mu_j-a}{\sigma_j}}^{\frac{\mu-\mu_j+a}{\sigma_j}} \mathcal{N}(s; 0, 1) \right)\nonumber\\
&=&
(\mu_j-\mu)\int_{\frac{\mu-\mu_j-a}{\sigma_j}}^{\frac{\mu-\mu_j+a}{\sigma_j}} \mathcal{N}(s; 0, 1)
+ \sigma_j \int_{\frac{\mu-\mu_j-a}{\sigma_j}}^{\frac{\mu-\mu_j+a}{\sigma_j}} s\mathcal{N}(s; 0, 1)
\nonumber\\
&=&
\frac{\mu_j-\mu}{2}\left( \mathrm{erf}\left( \frac{\mu-\mu_j+a}{\sqrt{2}\sigma_j} \right) - \mathrm{erf}\left( \frac{\mu-\mu_j-a}{\sqrt{2}\sigma_j} \right)\right)\nonumber\\
&& - \sigma_j \left( \mathcal{N}\left(\frac{\mu-\mu_j+a}{\sigma_j}; 0, 1\right) - \mathcal{N}\left(\frac{\mu-\mu_j-a}{\sigma_j}; 0, 1\right) \right),
\end{eqnarray}
and
\begin{eqnarray}
\lefteqn{\mathrm{cov}_{p, 2}(t^2, \psi_j(t)) =} \nonumber\\
&& \int_{\mu-a}^{\mu+a} t^2\mathcal{N}(t; \mu_j, \sigma_j^2)
- \frac{1}{2a}\left(\int_{\mu-a}^{\mu+a} t^2\right)\left(\int_{\mu-a}^{\mu+a} \mathcal{N}(t; \mu_j, \sigma_j^2)\right)
\nonumber\\
&=&
\int_{\frac{\mu-\mu_j-a}{\sigma_j}}^{\frac{\mu-\mu_j+a}{\sigma_j}} (\mu_j + \sigma_j s)^2\mathcal{N}(s; 0, 1)
- \underbrace{\frac{1}{2a}\left( \frac{(\mu+a)^3}{3} - \frac{(\mu-a)^3}{3} \right)}_{=\frac{a^2}{3} + \mu^2}
\left( \int_{\frac{\mu-\mu_j-a}{\sigma_j}}^{\frac{\mu-\mu_j+a}{\sigma_j}} \mathcal{N}(s; 0, 1) \right)\nonumber\\
&=&
\left(\mu_j^2-\mu^2 -\frac{a^2}{3}\right)\int_{\frac{\mu-\mu_j-a}{\sigma_j}}^{\frac{\mu-\mu_j+a}{\sigma_j}} \mathcal{N}(s; 0, 1)
+ 2\mu_j\sigma_j \int_{\frac{\mu-\mu_j-a}{\sigma_j}}^{\frac{\mu-\mu_j+a}{\sigma_j}} s\mathcal{N}(s; 0, 1)
+ \sigma_j^2 \int_{\frac{\mu-\mu_j-a}{\sigma_j}}^{\frac{\mu-\mu_j+a}{\sigma_j}} s^2\mathcal{N}(s; 0, 1)
\nonumber\\
&=&
\left(\mu_j^2-\mu^2 +\sigma_j^2 -\frac{a^2}{3}\right)\left( \mathrm{erf}\left( \frac{\mu-\mu_j+a}{\sqrt{2}\sigma_j} \right) - \mathrm{erf}\left( \frac{\mu-\mu_j-a}{\sqrt{2}\sigma_j} \right)\right)\nonumber\\
&& -\sigma_j(\mu+\mu_j+a) \mathcal{N}\left(\frac{\mu-\mu_j+a}{\sigma_j}; 0, 1\right) + \sigma_j(\mu+\mu_j-a) \mathcal{N}\left(\frac{\mu-\mu_j-a}{\sigma_j}; 0, 1\right).
\end{eqnarray}

\subsection{Continuous entmax in 1-d ($\alpha=\frac{n+1}{n}$, $D=1$)}

The above procedure can be extended to the case where $\alpha=\frac{n+1}{n}$ with $n \in \mathbb{N}$, which includes the biweight ($\alpha=\sfrac{3}{2}$) and triweight ($\alpha=\sfrac{4}{3}$) as particular cases.

\paragraph{Forward pass.}
Each coordinate of the attention mechanism output becomes:
\begin{eqnarray}\label{eq:expectation_continuous_entmax_rbf}
\mathbb{E}_{p}[\psi_j(t)] &=& \int_{\mu-a}^{\mu+a} \left((\alpha-1)\left(-\tau - \frac{(t-\mu)^2}{2\sigma^2}\right)\right)^{\frac{1}{\alpha-1}}  \mathcal{N}(t; \mu_j, \sigma_j^2)\nonumber\\
&=& \int_{\mu-a}^{\mu+a} \left(\frac{1}{n}\left(-\tau - \frac{(t-\mu)^2}{2\sigma^2}\right)\right)^{n}  \mathcal{N}(t; \mu_j, \sigma_j^2),
\end{eqnarray}
where $\tau$ and $a$ can be computed via Proposition~\ref{prop:beta_gauss}.
With $n\in \mathbb{N}$, the integrand in \eqref{eq:expectation_continuous_entmax_rbf} becomes the product of a polynomial function of $t$ and a Gaussian, and the integral admits a closed form expression obtainable through the following formulas:
\begin{eqnarray}
\int t^{2k+1} \mathcal{N}(t; 0, 1) dt &=& -\mathcal{N}(t; 0, 1) \sum_{j=0}^{k} \frac{(2k)!!}{(2j)!!} t^{2j} + \mathrm{const.}\nonumber\\
\int t^{2k+2} \mathcal{N}(t; 0, 1) dt &=& -\mathcal{N}(t; 0, 1) \sum_{j=0}^{k} \frac{(2k+1)!!}{(2j+1)!!} t^{2j+1} + (2k+1)!! \Phi(t) + \mathrm{const.},
\end{eqnarray}
where $\Phi(t) = \frac{1}{2}\left(1+ \mathrm{erf}\left(\frac{t}{\sqrt{2}}\right) \right)$ is the cumulative standard normal distribution, and $n!!$ denotes the double factorial.

\paragraph{Backward pass.}
From \eqref{eq:beta_covariance} and the fact that, with $\beta = 2-\alpha = \frac{n-1}{n}$, we have
\begin{eqnarray}
\|p\|_{\beta}^\beta &=& \int_{\mu-a}^{\mu+a} \left((\alpha-1)\left(-\tau - \frac{(t-\mu)^2}{2\sigma^2}\right)\right)^{\frac{2-\alpha}{\alpha-1}}\nonumber\\
&=& \int_{\mu-a}^{\mu+a} \left((\alpha-1)\left(-\tau - \frac{(t-\mu)^2}{2\sigma^2}\right)\right)^{n-1},
\end{eqnarray}
and all the integrands necessary for the computation of $\mathrm{cov}_{p,\alpha}(t, \psi_j(t))$ and $\mathrm{cov}_{p,\alpha}(t^2, \psi_j(t))$ become either polynomial functions of $t$ (up to degree $2(n-1) + 2 = 2n$) or products of polynomial functions of $t$ and a Gaussian, hence admit closed-form expressions as above.
For the biweight density ($n=2$), we need polynomials up to degree $4$, and for the triweight ($n=3$), we need polynomials up to degree $6$.

\subsection{Continuous sparsemax in 2-d ($\alpha=2$, $D=2$)}

Let us now consider the case where $D=2$.
For $\phi(t)=[t,tt^\top]$, the distribution $p=\hat{p}_{\Omega_2}[f_\theta]$, with $f_\theta(t)=\theta^\top \phi(t)$, becomes a bivariate truncated paraboloid where $\mu$ and  $\Sigma$ are related to the canonical parameters as before, $\theta=[\Sigma^{-1}\mu,-\frac{1}{2}\Sigma^{-1}]$. We obtain expressions for the attention mechanism output $\rho_2(\theta)=\mathbb{E}_p[\psi(t)]$ and its Jacobian $J_{\rho_2}(\theta)= \mathrm{cov}_{p,2}(\phi(t), \psi(t))$ that include 1-d integrals (simple to integrate numerically), when $\psi(t)$ are Gaussian RBFs, {\it i.e.}, when each $\psi_j$ is of the form $\psi_j(t)=\mathcal{N}(t;\mu_j, \Sigma_j)$.

We start with the following lemma:

\smallskip

\begin{lemma}\label{lemma:affine_transform_gaussian}
Let $\mathcal{N}(t, \mu, \Sigma)$ be a $D$-dimensional multivariate Gaussian,
Let $A \in \mathbb{R}^{D \times R}$ be a full column rank matrix (with $R \le D$), and $b \in \mathrm{R}^D$.
Then we have $\mathcal{N}(Au + b; \mu, \Sigma) = \tilde{s} \mathcal{N}(u; \tilde{\mu}, \tilde{\Sigma})$ with:
\begin{eqnarray*}
\tilde{\Sigma} &=& (A^\top \Sigma^{-1} A)^{-1}, \quad
\tilde{\mu} = \tilde{\Sigma}A^\top \Sigma^{-1} (\mu - b)\\
\tilde{s} &=& (2\pi)^{\frac{R-D}{2}} \frac{|\tilde{\Sigma}|^{1/2}}{|\Sigma|^{1/2}} \exp\left(-\frac{1}{2}(\mu-b)^\top P(\mu-b)\right), \quad P = \Sigma^{-1} - \Sigma^{-1} A \tilde{\Sigma} A^\top\Sigma^{-1}.
\end{eqnarray*}
If $R=D$, then $A$ is invertible and the expressions above can be simplified to:
\begin{eqnarray*}
\tilde{\Sigma} = A^{-1} \Sigma A^{-\top}, \quad
\tilde{\mu} = A^{-1} (\mu - b), \quad
\tilde{s} = |A|^{-1}.
\end{eqnarray*}
\end{lemma}

\begin{proof}
The result can be derived by writing $\mathcal{N}(Au+b; \mu, \Sigma) = (2\pi)^{-\frac{R}{2}} |\Sigma|^{-\frac{1}{2}} \exp(-\tfrac{1}{2}(Au+b-\mu)^\top\Sigma^{-1}(Au+b-\mu))$ and splitting the exponential of the sum as a product of exponentials.
\end{proof}

\paragraph{Forward pass.}
For the forward pass, we need to compute
\begin{equation}
\mathbb{E}_p[\psi_j(t)] = \iint_{\mathbb{R}^2} \left[-\tau - \frac{1}{2}(t-\mu)^\top\Sigma^{-1}(t-\mu)\right]_+ \mathcal{N}(t; \mu_j, \Sigma_j) dt,
\end{equation}
with
(from \eqref{eq:truncated_paraboloid})
$\tau=-\left(\frac{1}{\pi\sqrt{\det(\Sigma)}}\right)^{\frac{1}{2}}$.
Using Lemma~\ref{lemma:affine_transform_gaussian} and  the change of variable formula (which makes the determinants cancel), we can reparametrize $u = (-2\tau)^{-\frac{1}{2}} \Sigma^{-\frac{1}{2}} (t - \mu)$ and write this as an integral over the unit circle:
\begin{equation}
\mathbb{E}_p[\psi_j(t)] = \iint_{\|u\|\le 1}
-\tau (1-\|u\|^2) \mathcal{N}(u; \tilde{\mu}, \tilde{\Sigma}) du,
\end{equation}
with $\tilde{\mu} = (-2\tau)^{-\frac{1}{2}} \Sigma^{-\frac{1}{2}}(\mu_j - \mu)$,
$\tilde{\Sigma} = (-2\tau)^{-1} \Sigma^{-\frac{1}{2}} \Sigma_j \Sigma^{-\frac{1}{2}}$.
We now do a change to polar coordinates, $u = (r\cos\theta, r\sin\theta) = ar$, where $a = [\cos \theta, \sin \theta]^\top \in \mathbb{R}^{2\times 1}$. The integral becomes:
\begin{eqnarray}
\mathbb{E}_p[\psi_j(t)] &=& \int_{0}^{2\pi} \int_{0}^1
-\tau (1-r^2) \mathcal{N}(ar; \tilde{\mu}, \tilde{\Sigma}) r \, dr \, d\theta\nonumber\\
&=& \int_{0}^{2\pi} \int_{0}^1
-\tau r(1-r^2) \tilde{s} \mathcal{N}(r; r_0, \sigma^2) \, dr \, d\theta,
\end{eqnarray}
where in the second line we applied again Lemma~\ref{lemma:affine_transform_gaussian}, resulting in
\begin{eqnarray*}
\sigma^2(\theta) \equiv \sigma^2 &=& (a^\top \tilde{\Sigma}^{-1} a)^{-1}\\
r_0(\theta) \equiv r_0 &=& \sigma^2 a^\top \tilde{\Sigma}^{-1} \tilde{\mu}\\
\tilde{s}(\theta) \equiv \tilde{s} &=& \frac{1}{\sqrt{2\pi}}  \frac{\sigma}{|\tilde{\Sigma}|^{1/2}} \exp\left(-\frac{1}{2}\tilde{\mu}^\top P\tilde{\mu}\right), \quad P = \tilde{\Sigma}^{-1} - \sigma^2 \tilde{\Sigma}^{-1} a a^\top\tilde{\Sigma}^{-1}.
\end{eqnarray*}
Applying Fubini's theorem,
we fix $\theta$ and integrate with respect to $r$. We use the formulas \eqref{eq:erf_expr123} and the fact that, for any $u, v \in \mathbb{R}$ such that $u \le v$:
\begin{equation}\label{eq:erf_expr4}
\int_{u}^v t^3 \mathcal{N}(t; 0, 1) = -\mathcal{N}(v; 0, 1)(2+v^2) + \mathcal{N}(u; 0, 1)(2+u^2).
\end{equation}
We obtain a closed from expression for the inner integral:
\begin{eqnarray}
 F(\theta) &=& \int_{0}^1
r(1-r^2)  \mathcal{N}(r; r_0, \sigma^2) \, dr\nonumber\\
&=& (2\sigma^3 + r_0^2 \sigma + r_0\sigma) \mathcal{N}\left(\frac{1-r_0}{\sigma}; 0, 1\right)
- (2\sigma^3 + r_0^2 \sigma -\sigma) \mathcal{N}\left(-\frac{r_0}{\sigma}; 0, 1\right)\nonumber\\
&&
-\frac{r_0^3 + (3\sigma^2 - 1)r_0}{2} \left[ \mathrm{erf}\left(\frac{1-r_0}{\sqrt{2}\sigma}\right)- \mathrm{erf}\left(-\frac{r_0}{\sqrt{2}\sigma}\right)\right].
\end{eqnarray}
The desired integral can then be expressed in a single dimension as
\begin{eqnarray}
\mathbb{E}_p[\psi_j(t)] &=& -\tau \int_{0}^{2\pi} \tilde{s}(\theta) F(\theta),
\end{eqnarray}
which may be integrated numerically.

\paragraph{Backward pass.}

For the backward pass we need to solve
\begin{equation}
\label{eq:1-2_row_J}
\mathrm{cov}_{p, 2}(t, \psi_j(t)) =
\iint_{E} t\mathcal{N}(t; \mu_j, \Sigma_j)
- \frac{1}{|E|}\left(\iint_{E} t\right)\left(\iint_{E} \mathcal{N}(t; \mu_j, \Sigma_j)\right)
\end{equation}and
\begin{equation}
\label{eq:3-6_row_J}
\mathrm{cov}_{p, 2}(tt^\top, \psi_j(t)) =
\iint_{E} tt^\top\mathcal{N}(t; \mu_j, \Sigma_j)
- \frac{1}{|E|}\left(\iint_{E} tt^\top\right)\left(\iint_{E} \mathcal{N}(t; \mu_j, \Sigma_j)\right)
\end{equation}
where $E = \mathrm{supp}(p) = \{t \in \mathbb{R}^2 \mid \frac{1}{2}(t-\mu)^\top \Sigma^{-1} (t-\mu) \le -\tau\}$ denotes the support of the density $p$, a region bounded by an ellipse.
Note that these expressions include integrals of vector-valued functions and that \eqref{eq:1-2_row_J} and \eqref{eq:3-6_row_J} correspond to the first to second and the third to sixth row of the Jacobian, respectively. The integrals that do not include Gaussians have closed form expressions and can be computed as
\begin{equation}
    \frac{1}{|E|}\left(\iint_{E} t\right)=\mu \qquad \text{and} \qquad
    \frac{1}{|E|}\left(\iint_{E} tt^\top\right)=\mu\mu^\top+\frac{\Sigma}{|E|},
\end{equation}
where $|E|$ is the area of the region $E$ given by
    $|E|=\frac{\pi}{\sqrt{ \det \left( \frac{1}{-2\tau}\,\Sigma^{-1}\right)}}$.

All the other integrals are solved using the same affine transformation and change to polar coordinates as in the forward pass. Given this, $\tilde{\mu}$, $\tilde{\Sigma}$, $a$, $\sigma^2, r_0$ and $\tilde{s}$ are the same as before.
To solve \eqref{eq:1-2_row_J} we write
\begin{equation}
    \iint_{E} t\mathcal{N}(t; \mu_j, \Sigma_j) = \iint_{\|u\|\le 1} \left((-2\tau)^{\frac{1}{2}} \Sigma^{\frac{1}{2}}u+\mu\right) \mathcal{N}(u; \tilde{\mu}, \tilde{\Sigma}) du
\end{equation}
in polar coordinates,
\begin{equation}
    \int_{0}^{2\pi} \int_{0}^1
 r\left((-2\tau)^{\frac{1}{2}} \Sigma^{\frac{1}{2}}ar+\mu\right) \tilde{s} \, \mathcal{N}(r; r_0, \sigma^2) dr \, d\theta,
\end{equation}
which can be then expressed in a single dimension as
\begin{eqnarray}
\iint_{E} t\mathcal{N}(t; \mu_j, \Sigma_j) &=& \int_{0}^{2\pi} \tilde{s}(\theta) G(\theta)d\theta,
\end{eqnarray}
with
\begin{eqnarray}
 G(\theta) &=& \int_{0}^1
r\left((-2\tau)^{\frac{1}{2}} \Sigma^{\frac{1}{2}}ar+\mu\right)  \mathcal{N}(r; r_0, \sigma^2) \, dr \nonumber\\
&=& \int_{-\frac{r_0}{\sigma}}^{\frac{1-r_0}{\sigma}}
(s\sigma+r_0)\left((-2\tau)^{\frac{1}{2}} \Sigma^{\frac{1}{2}}a(s\sigma+r_0)+\mu\right)  \mathcal{N}(r; r_0, \sigma^2) \, ds \nonumber\\
&=&
\left((-2\tau)^{\frac{1}{2}} \Sigma^{\frac{1}{2}}a\sigma(r_0)+\mu\sigma\right)\mathcal{N}\left(-\frac{r_0}{\sigma};0,1\right) \nonumber\\
&&-\left((-2\tau)^{\frac{1}{2}} \Sigma^{\frac{1}{2}}a\sigma(1+r_0)+\mu\sigma\right) \mathcal{N}\left(\frac{1-r_0}{\sigma};0,1\right) \nonumber\\
&&+\frac{1}{2}\left((-2\tau)^{\frac{1}{2}} \Sigma^{\frac{1}{2}}a(\sigma^2+r_0^2)+\mu r_0\right)\left[ \mathrm{erf}\left(\frac{1-r_0}{\sqrt{2}\sigma}\right)- \mathrm{erf}\left(-\frac{r_0}{\sqrt{2}\sigma}\right)\right].
\end{eqnarray}
We do the same for
\begin{equation}
    \iint_{E} \mathcal{N}(t; \mu_j, \Sigma_j)=\iint_{\|u\|\le 1} \mathcal{N}(u; \tilde{\mu}, \tilde{\Sigma}) du =     \int_{0}^{2\pi} \int_{0}^1
 r\tilde{s} \, \mathcal{N}(r; r_0, \sigma^2) dr \, d\theta,
\end{equation}
which can then be expressed in a single dimension as
\begin{eqnarray}
\iint_{E} \mathcal{N}(t; \mu_j, \Sigma_j) &=& \int_{0}^{2\pi} \tilde{s}(\theta) H(\theta)d\theta,
\end{eqnarray}with
\begin{eqnarray}
 H(\theta) &=& \int_{0}^1
r\mathcal{N}(r; r_0, \sigma^2) \, dr \nonumber = \int_{-\frac{r_0}{\sigma}}^{\frac{1-r_0}{\sigma}}
(s\sigma+r_0)  \mathcal{N}(r; r_0, \sigma^2) \, ds \\
&=&
\sigma\left[\mathcal{N}\left(-\frac{r_0}{\sigma};0,1\right)
-\mathcal{N}\left(\frac{1-r_0}{\sigma};0,1\right)\right] \nonumber +\frac{r_0}{2}\left[ \mathrm{erf}\left(\frac{1-r_0}{\sqrt{2}\sigma}\right)- \mathrm{erf}\left(-\frac{r_0}{\sqrt{2}\sigma}\right)\right].
\end{eqnarray}
Finally, to solve \eqref{eq:3-6_row_J} we simplify the integral
\begin{eqnarray}
    \iint_{E} tt^\top\mathcal{N}(t; \mu_j, \Sigma_j)&=&\iint_{\|u\|\le 1} \left((-2\tau)^{\frac{1}{2}} \Sigma^{\frac{1}{2}}u+\mu\right)\left((-2\tau)^{\frac{1}{2}} \Sigma^{\frac{1}{2}}u+\mu\right)^\top \mathcal{N}(u; \tilde{\mu}, \tilde{\Sigma}) du \nonumber \\
    &=&\int_{0}^{2\pi} \int_{0}^1
 r(r^2A+rB+C)\tilde{s} \, \mathcal{N}(r; r_0, \sigma^2) dr \, d\theta
\end{eqnarray}
with
\begin{equation}
    A=(-2\tau)\Sigma^{\frac{1}{2}}aa^\top(\Sigma^{\frac{1}{2}})^\top, \qquad
    B=(-2\tau)^{\frac{1}{2}}\left( \Sigma^{\frac{1}{2}}a\mu^\top+\mu a^\top (\Sigma^{\frac{1}{2}})^\top \right), \qquad
    C=\mu\mu^\top.
\end{equation}
The integral can then be expressed in a single dimension as
\begin{eqnarray}
\iint_{E} tt^\top \mathcal{N}(t; \mu_j, \Sigma_j) &=& \int_{0}^{2\pi} \tilde{s}(\theta) M(\theta)d\theta,
\end{eqnarray}
with
\begin{eqnarray}
M(\theta) &=& \int_{0}^1
(r^3A+r^2B+rC)\, \mathcal{N}(r; r_0, \sigma^2) dr \nonumber\\
&=& \int_{-\frac{r_0}{\sigma}}^{\frac{1-r_0}{\sigma}}
(s^3\tilde{A}+s^2\tilde{B}+s\,\tilde{C}+\tilde{D})  \mathcal{N}(s; 0, 1) \, ds \nonumber\\
&=& \left[\left(2+\left(-\frac{r_0}{\sigma}\right)^2\right)\tilde{A}-\frac{r_0}{\sigma}\tilde{B}+\tilde{C}\right]\mathcal{N}\left(-\frac{r_0}{\sigma};0,1\right)\nonumber\\
&&-\left[\left(2+\left(\frac{1-r_0}{\sigma}\right)^2\right)\tilde{A}+\frac{1-r_0}{\sigma}\tilde{B}+\tilde{C}\right]\mathcal{N}\left(\frac{1-r_0}{\sigma};0,1\right) \nonumber\\
&&+\frac{1}{2}\left(\tilde{B}+\tilde{D}\right)\left[ \mathrm{erf}\left(\frac{1-r_0}{\sqrt{2}\sigma}\right)- \mathrm{erf}\left(-\frac{r_0}{\sqrt{2}\sigma}\right)\right]
\end{eqnarray}
where
\begin{equation}
    \tilde{A}=\sigma^3A, \qquad
    \tilde{B}=\sigma^2(3r_0\,A+B), \qquad
    \tilde{C}=\sigma(3r_0^2\,A+2r_0\,B+C), \qquad
    \tilde{D}=r_0^3\,A+r_0^2\,B+r_0\,C.
\end{equation}

\section{Experimental details}\label{sec:model_hyperparams}

\subsection{Audio classification}

We used the UrbanSound8k dataset \citep{Salamon:UrbanSound:ACMMM:14}, which contains 8732 labeled sound excerpts ($\leq 4s$) from 10 urban classes.
We set the sampling rate to 16kHz for all audios.
The audios were transformed into a sequence of vectors by using short-time Fourier transform (STFT) with 400 points, a window size of 25ms, and a hop size of 10ms.
After this transformation, we extract 80 Mel-frequency filter banks.
We used SpeechBrain \citep{speechbrain} to implement the input pipeline and the model, following the standard recipe for UrbanSound8k.\footnote{\url{https://github.com/speechbrain/speechbrain/tree/develop/recipes/UrbanSound8k}}
Our model consists of a convolutional 1-d layer followed by an attention mechanism and an output layer.
Table~\ref{tab:table_all_hyperparams_audio} shows the hyperparameters used for all audio classification experiments.

\begin{table}[t]
    \caption{Hyperparmeters for audio classification.}
    \label{tab:table_all_hyperparams_audio}
    \begin{small}
    \begin{center}
    \begin{tabular}{llllll}
        \toprule
        \sc Hyperparameter & \sc Value  \\
        \midrule
        Batch size                  & 16    \\
        Number of epochs            & 20     \\
        Optimizer                   & Adam      \\
        $\ell_2$ regularization     & 0.000002     \\
        Learning rate               & 0.001     \\
        Conv. filters                   & 128   \\
        Conv. kernel size               & 5     \\
        Conv. activation                & ReLU     \\
        Conv. dropout                   & 0.15     \\
        Max-pooling size                & 3     \\
        Gaussian RBFs (\S\ref{subsec:continuous_attention})       & $128 \ll L$ with $\mu$ linearly spaced in $[0,1]$ and $\Sigma = [0.1, 0.5]$ \\
        Ridge penalty $\lambda$         & 0.1     \\
        Discrete attention              & \citep{bahdanau2014neural}     \\
        \bottomrule
    \end{tabular}
    \end{center}
    \end{small}
    \vskip -0.1in
\end{table}

\subsection{Visual question answering}

We used the VQA-v2 dataset \citep{Goyal2019} with the standard splits (443K, 214K, and 453K question-image pairs for train/dev/test, the latter subdivided into  test-dev, test-standard, test-challenge and test-reserve). We adapted the implementation of \cite{Yu2019},%
\footnote{\url{https://github.com/MILVLG/mcan-vqa}} %
consisting of a Modular Co-Attention Network (MCAN). Our architecture is the same as \cite{Yu2019} except that we represent the image input with grid features generated by a ResNet \citep{He2016} pretrained on ImageNet \citep{Russakovsky2015}, instead of bounding-box features \citep{Anderson2018}.
The images are resized to $448 \times 448$ before going through the ResNet that outputs a feature map of size $14 \times 14 \times 2048$. To represent the input question words we use 300-dimensional GloVe word embeddings \citep{pennington2014glove}, yielding a question feature matrix representation. Table~\ref{tab:table_hyperparams_VQA} shows the hyperparameters used for all the VQA experiments presented.

All the models we experimented with use the same features and were trained only on the train set without data augmentation.

\paragraph{Examples.}
Figure~\ref{fig:examples_vqa_skate} illustrates the difficulties that continuous attention models may face when trying to focus on objects that are too far from each other or that seem to have different relative importance to answer the question. Intuitively, in VQA, this becomes a problem when counting objects in those conditions. On the other side, in counting questions that require the understanding of a contiguous region of the image only, continuous attention may perform better (see Figure~\ref{fig:examples_vqa_2birds}).
Figure~\ref{fig:examples_vqa_soccer} shows another example where continuous attention focus on the right region of the image and answers the question correctly. For this case, discrete attention is more diffuse than its continuous counterpart: %
it attends to two different regions in the image, leading to incorrect answers.

\begin{table}[t]
    \caption{Hyperparmeters for VQA.}
    \label{tab:table_hyperparams_VQA}
    \begin{small}
    \begin{center}
    \begin{tabular}{llllll}
        \toprule
        \sc Hyperparameter & \sc Value  \\
        \midrule
        Batch size                  & 64    \\
        Word embeddings size        & 300     \\
        Input image features size   & 2048 \\
        Input question features size & 512 \\
        Fused multimodal features size & 1024 \\
        Multi-head attention hidden size        & 512 \\
        Number of MCA layers        & 6 \\
        Number of attention heads   & 8 \\
        Dropout rate                & 0.1 \\
        MLP size in flatten layers  & 512 \\
        Optimizer                   & Adam \\
        Base learning rate at epoch $t$ starting from 1 & $\mathrm{min}(2.5 t \cdot 10^{-5}, 1\cdot 10^{-4})$\\
        Learning rate decay ratio at epoch $t\in \{10,12\}$  & 0.2 \\
        Number of epochs            & 13 \\

        \bottomrule
    \end{tabular}
    \end{center}
    \end{small}
\end{table}

\begin{figure*}[t]
\centering
\includegraphics[width=0.24\textwidth]{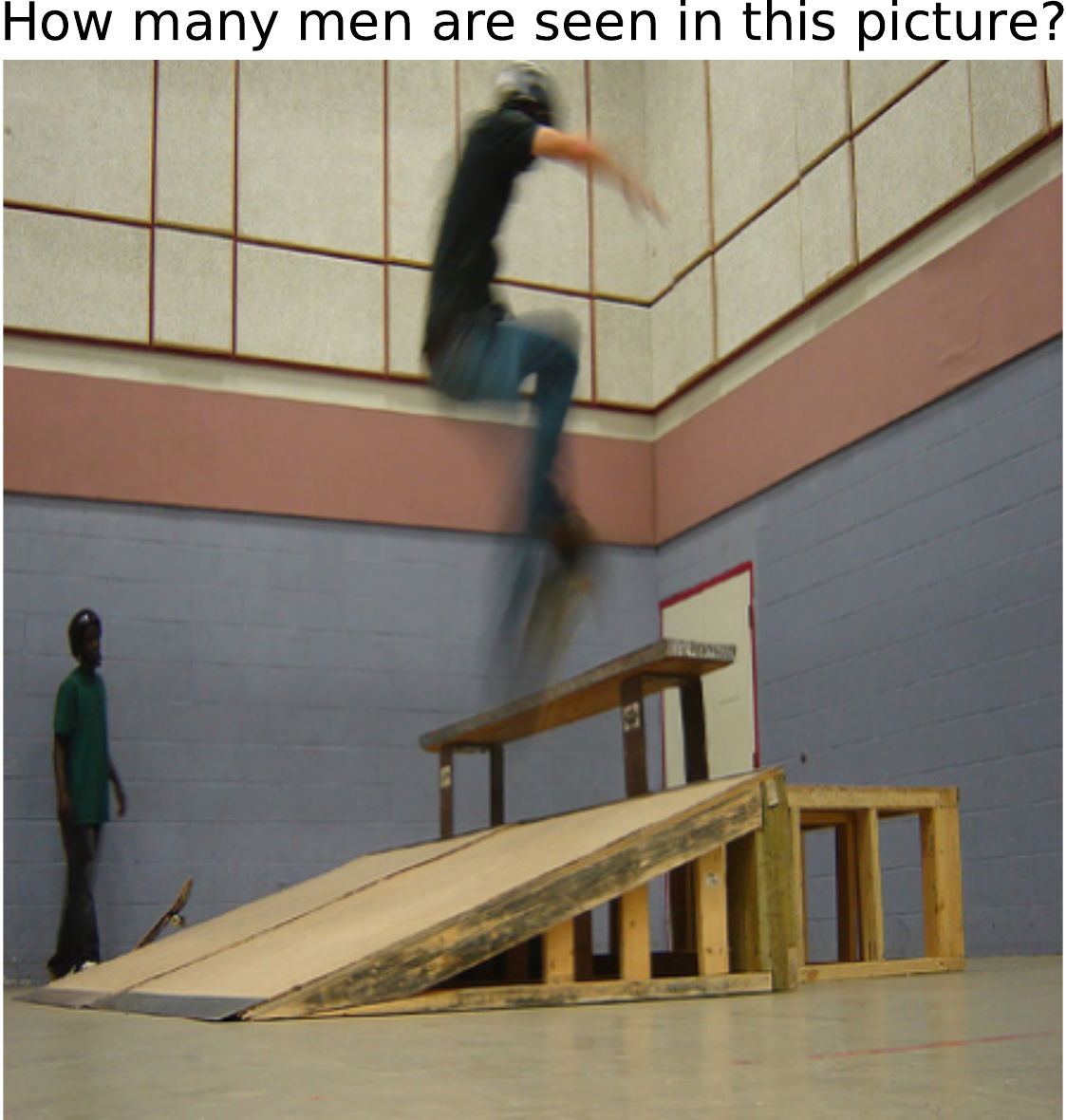}
\includegraphics[width=0.24\textwidth]{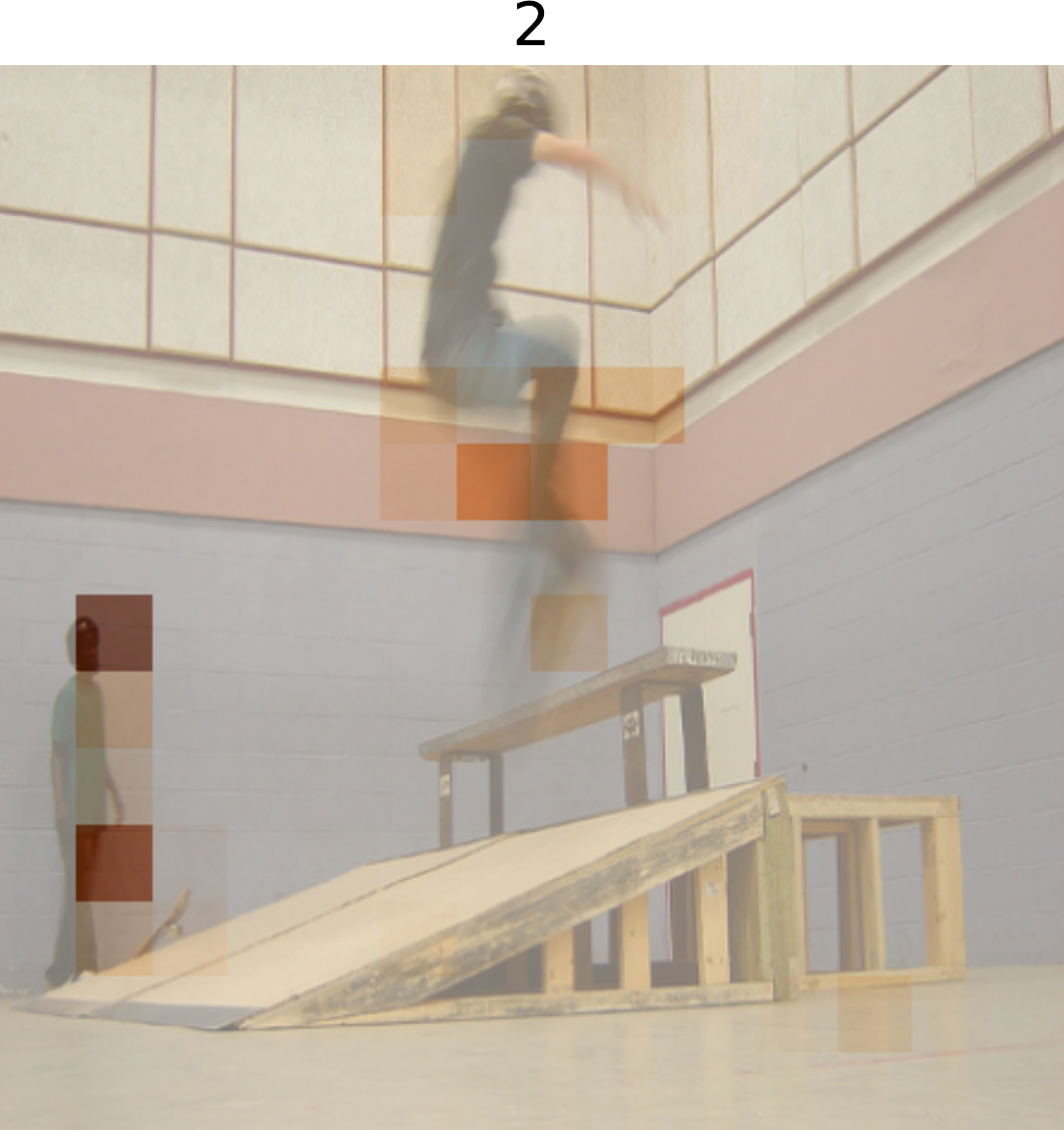}
\includegraphics[width=0.24\textwidth]{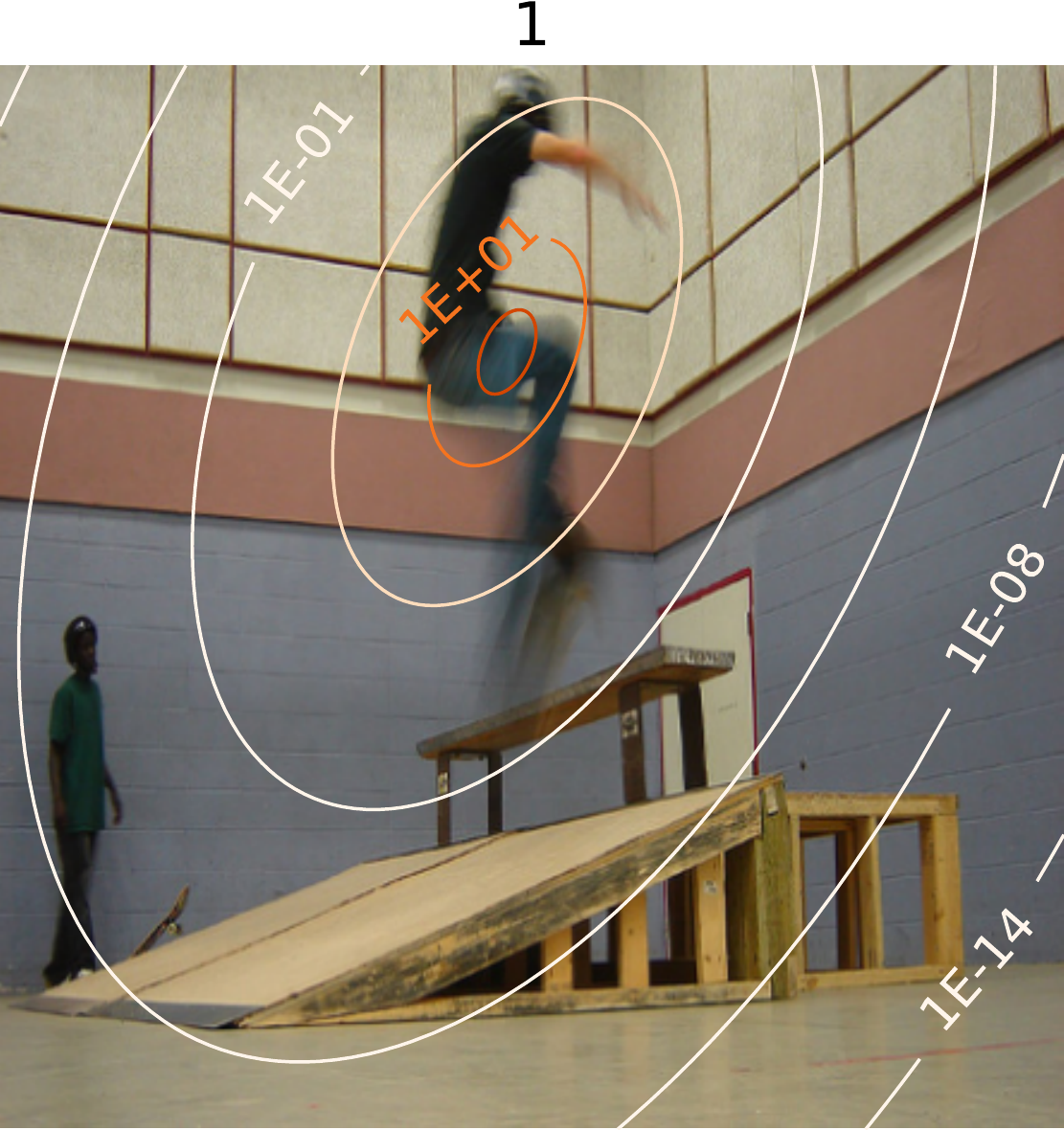}
\includegraphics[width=0.24\textwidth]{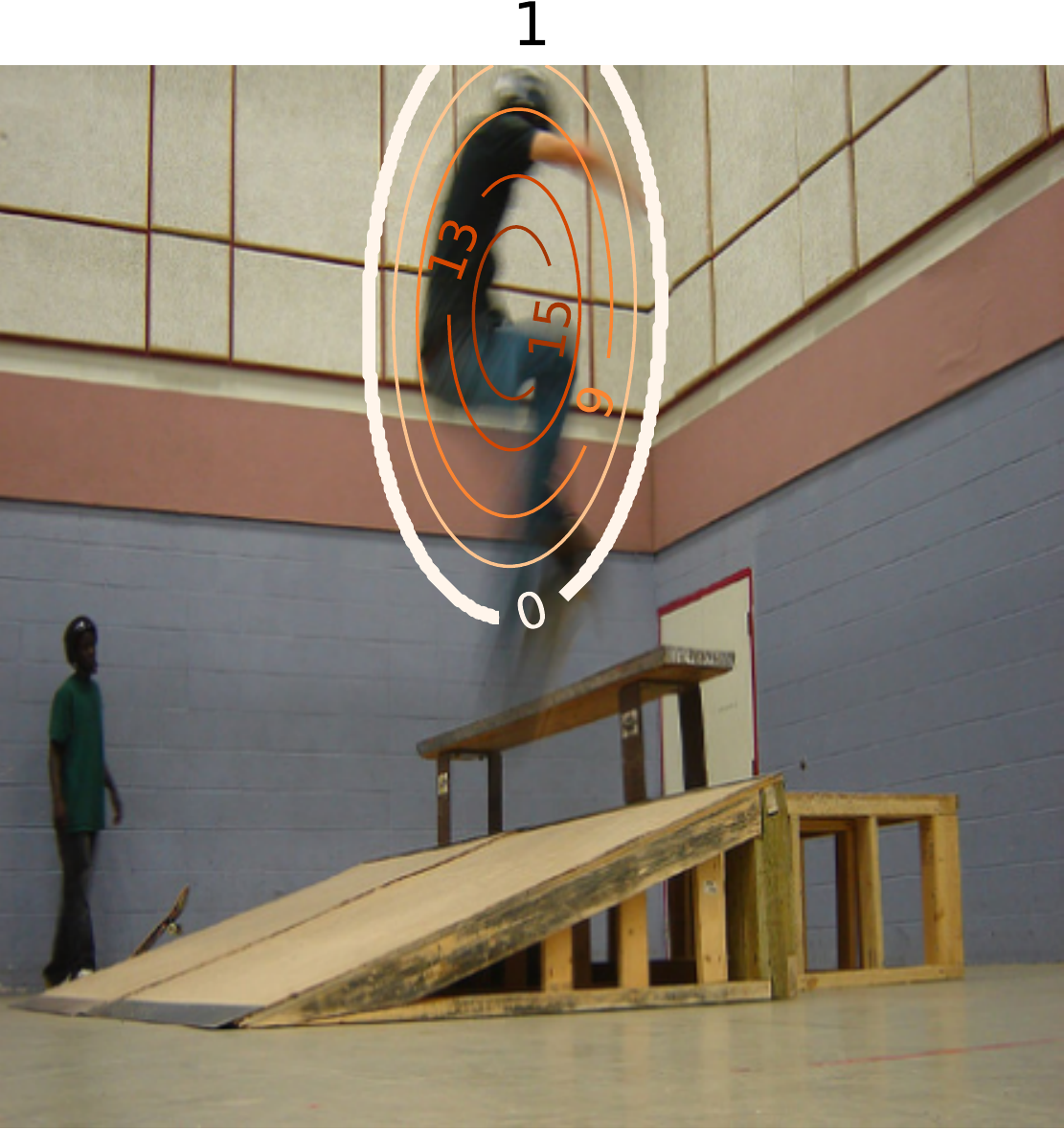}
\caption{\label{fig:examples_vqa_skate}Attention maps for an example in VQA-v2: original image, discrete attention, continuous softmax, and continuous sparsemax.}
\end{figure*}

\begin{figure*}[t]
\centering
\includegraphics[width=0.24\textwidth]{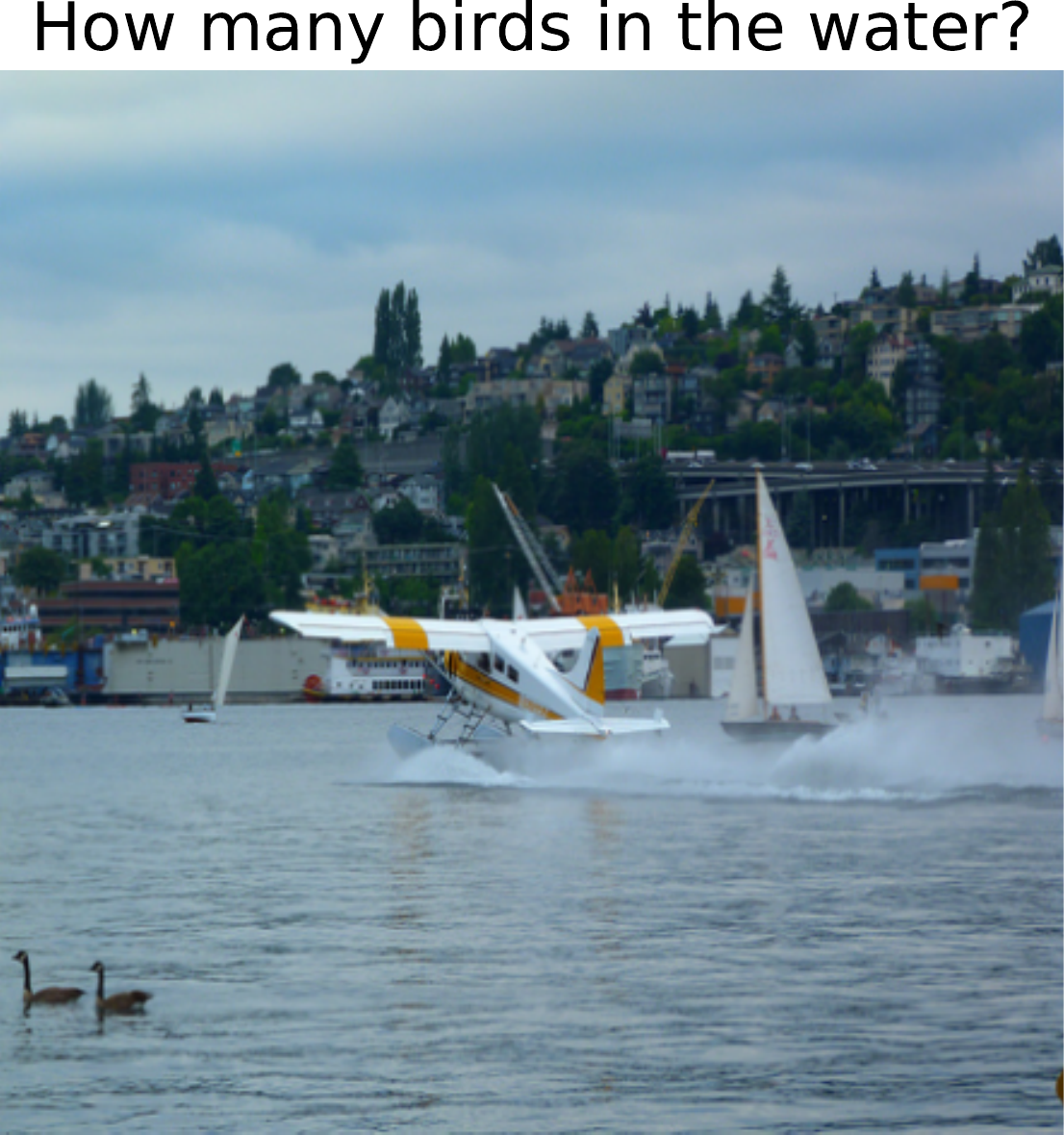}
\includegraphics[width=0.24\textwidth]{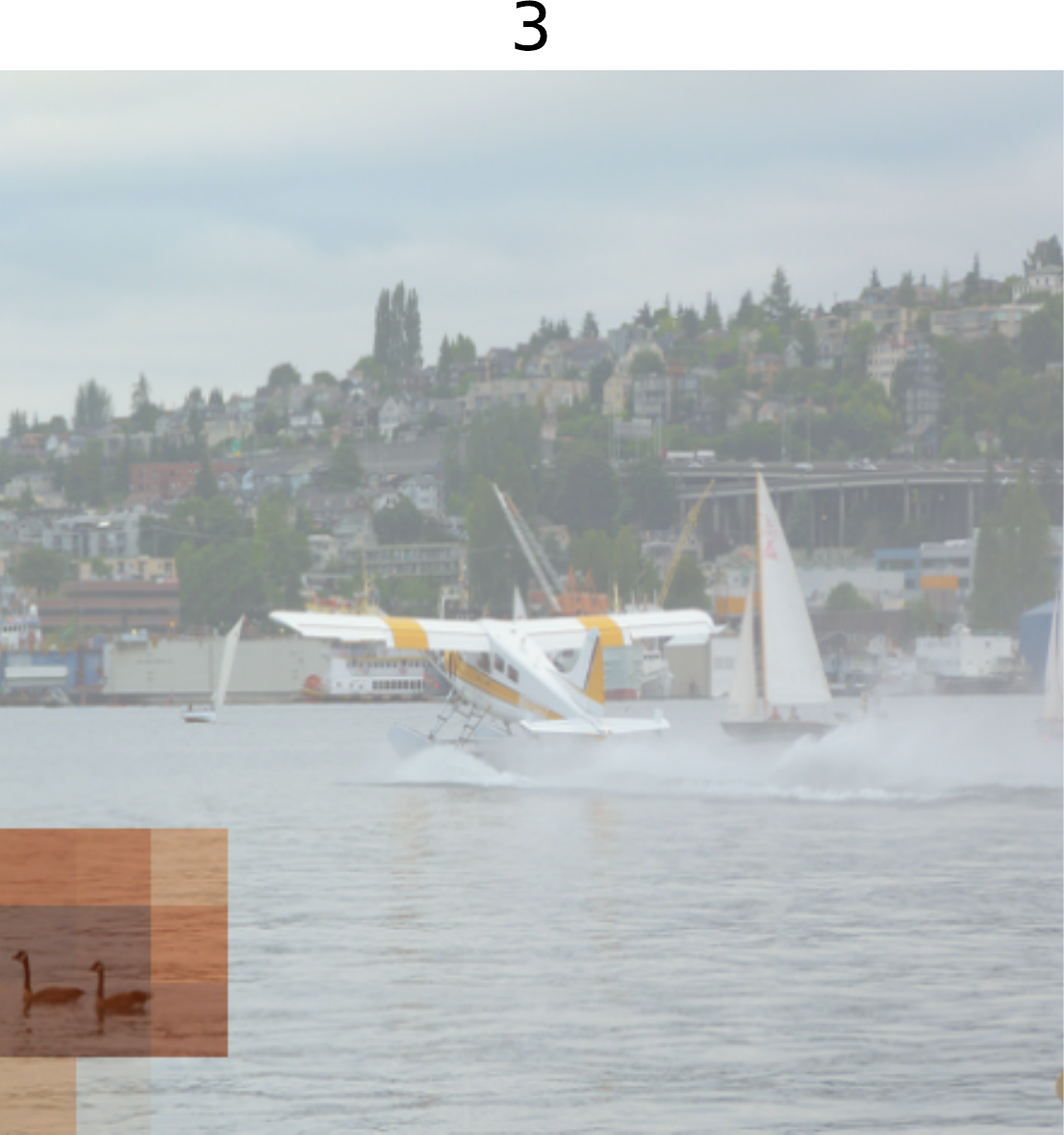}
\includegraphics[width=0.24\textwidth]{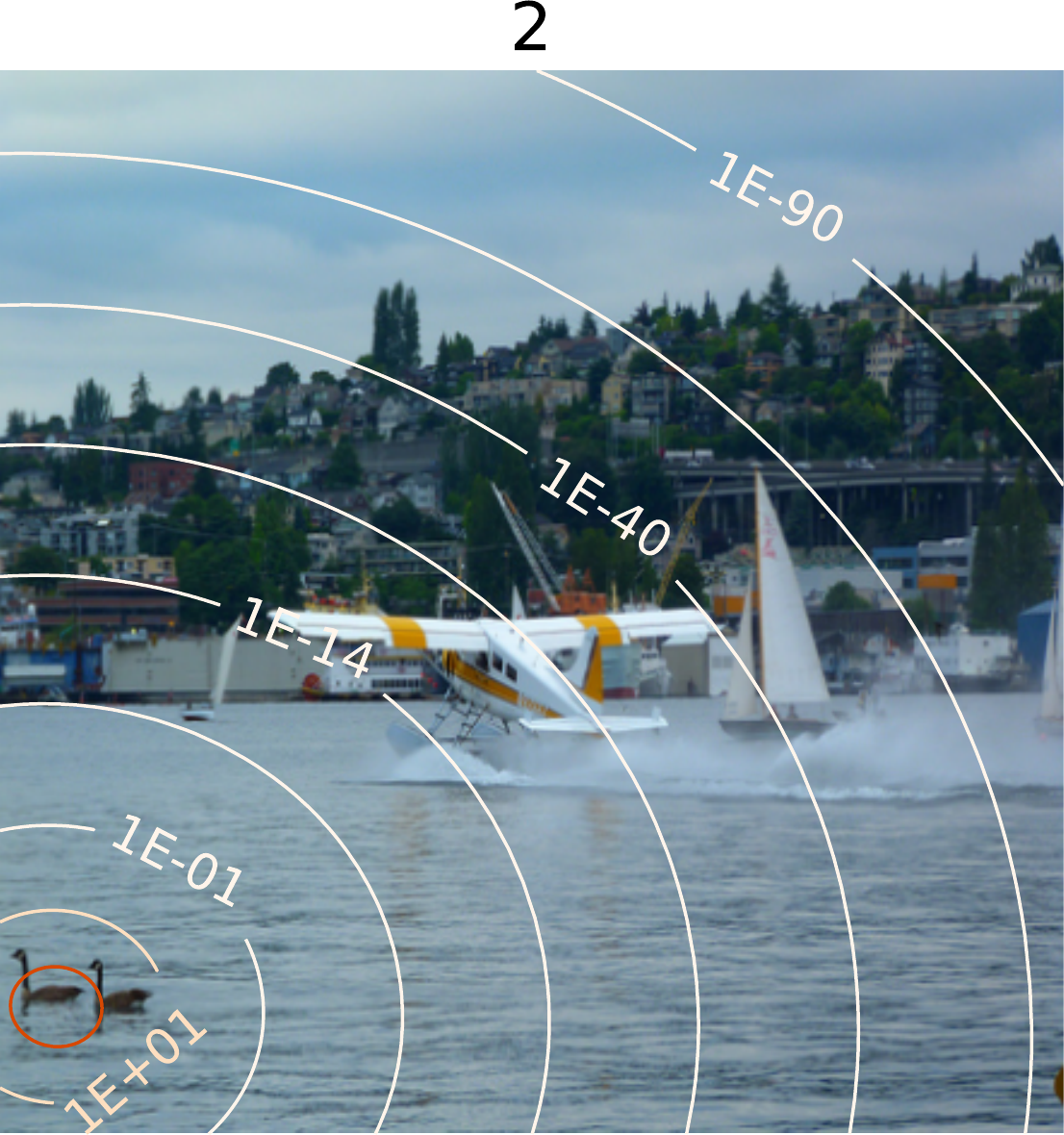}
\includegraphics[width=0.24\textwidth]{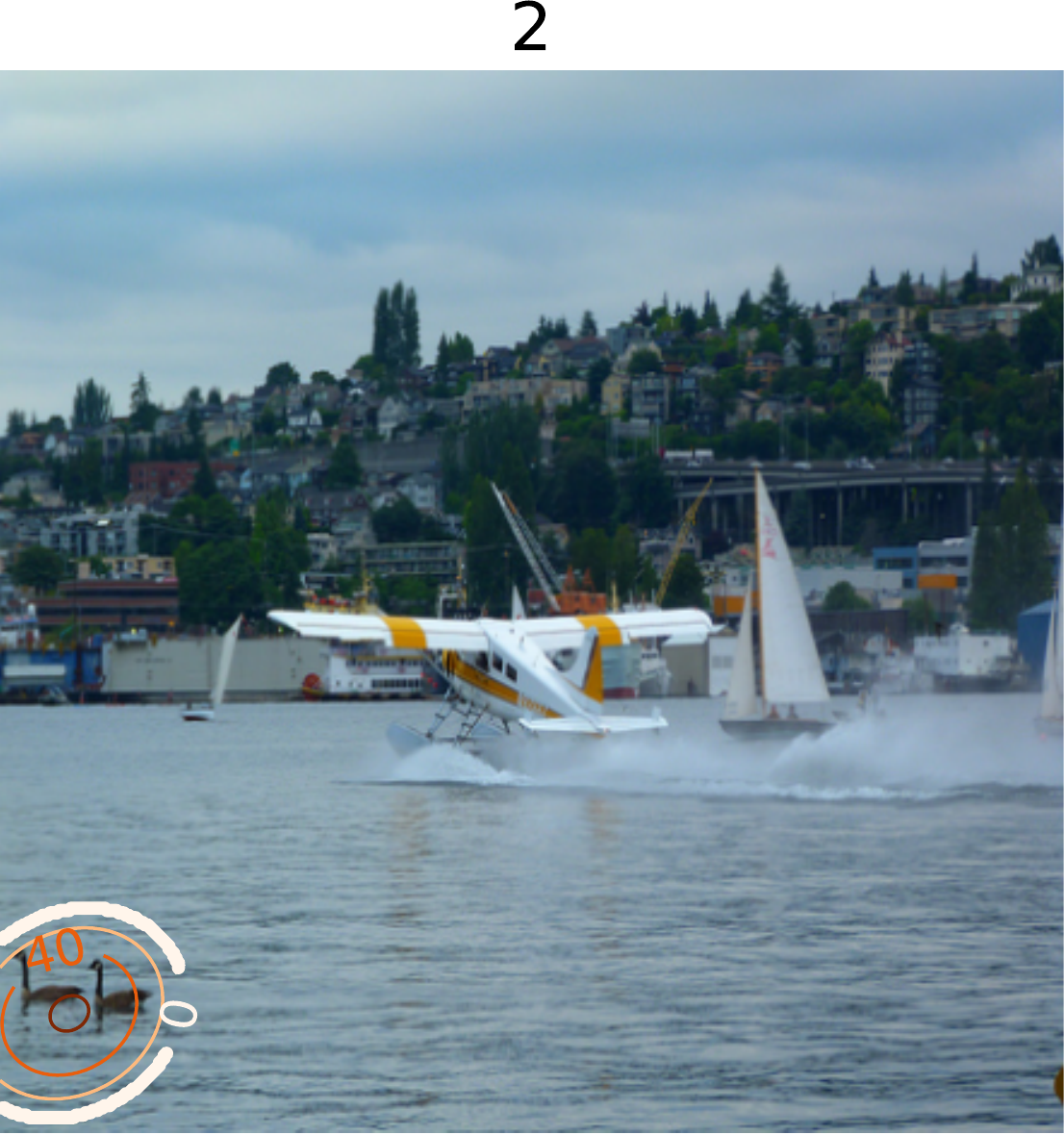}
\caption{\label{fig:examples_vqa_2birds}Attention maps for an example in VQA-v2: original image, discrete attention, continuous softmax, and continuous sparsemax.}
\end{figure*}

\begin{figure*}[t]
\centering
\includegraphics[width=0.24\textwidth]{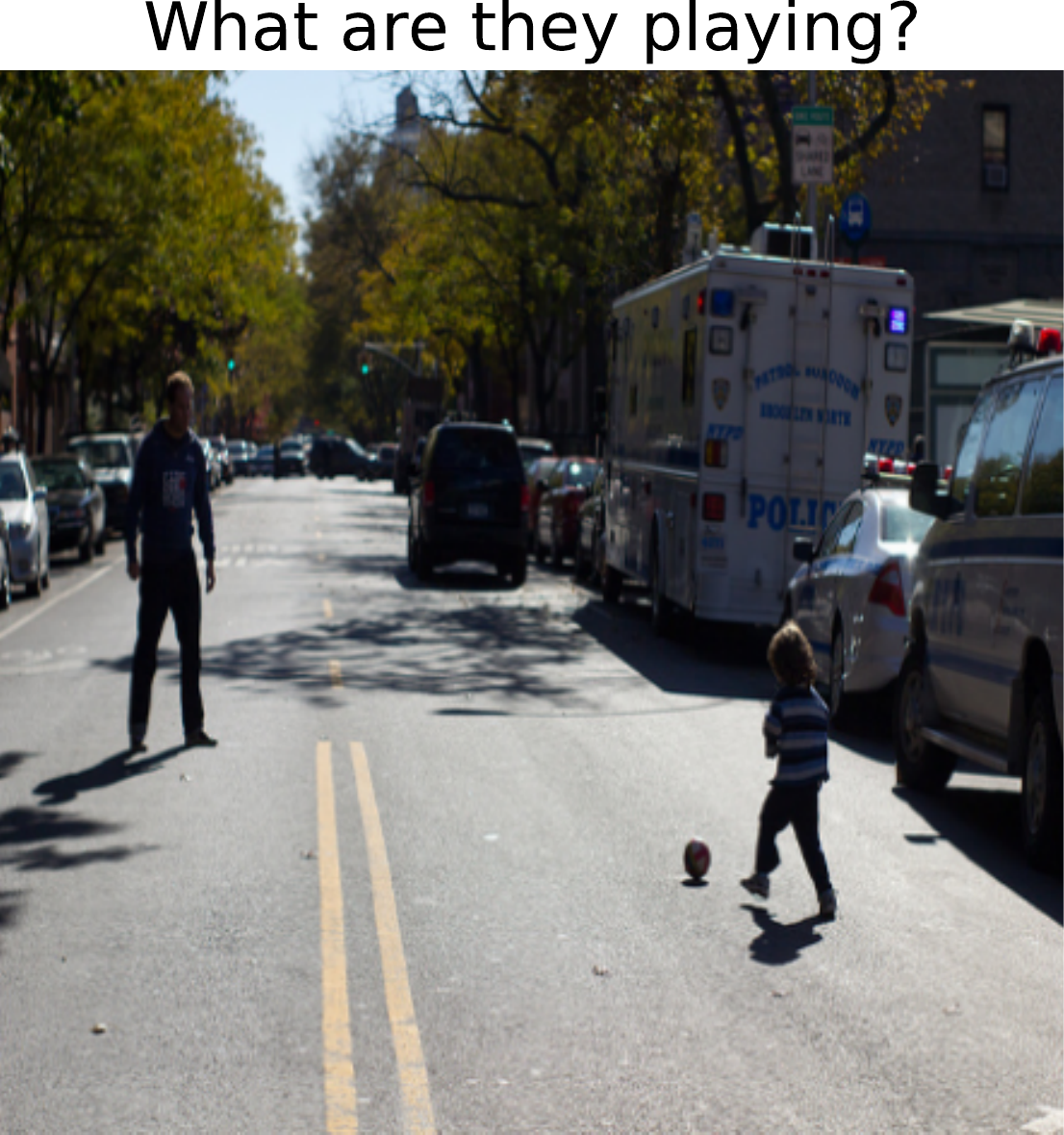}
\includegraphics[width=0.24\textwidth]{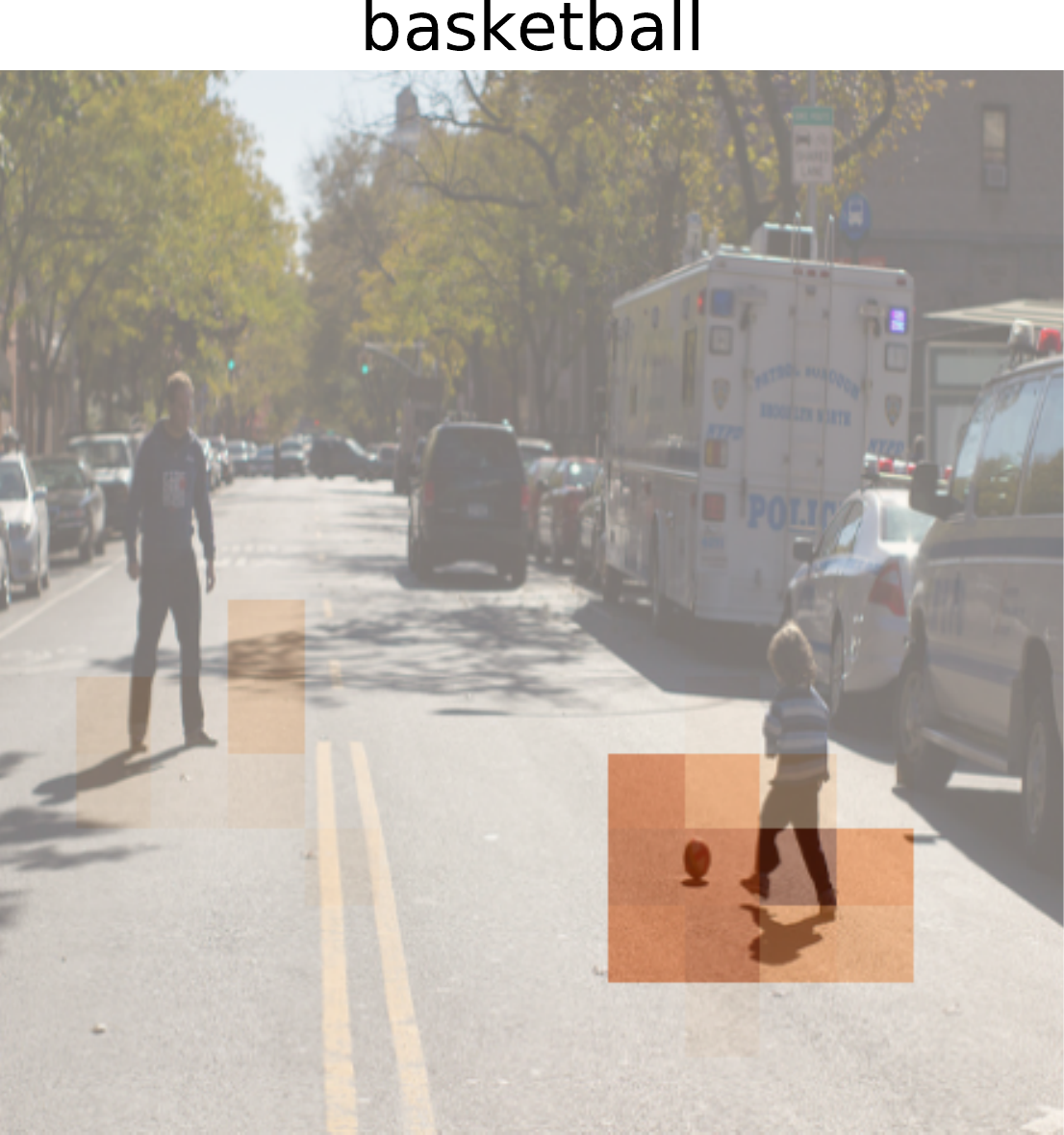}
\includegraphics[width=0.24\textwidth]{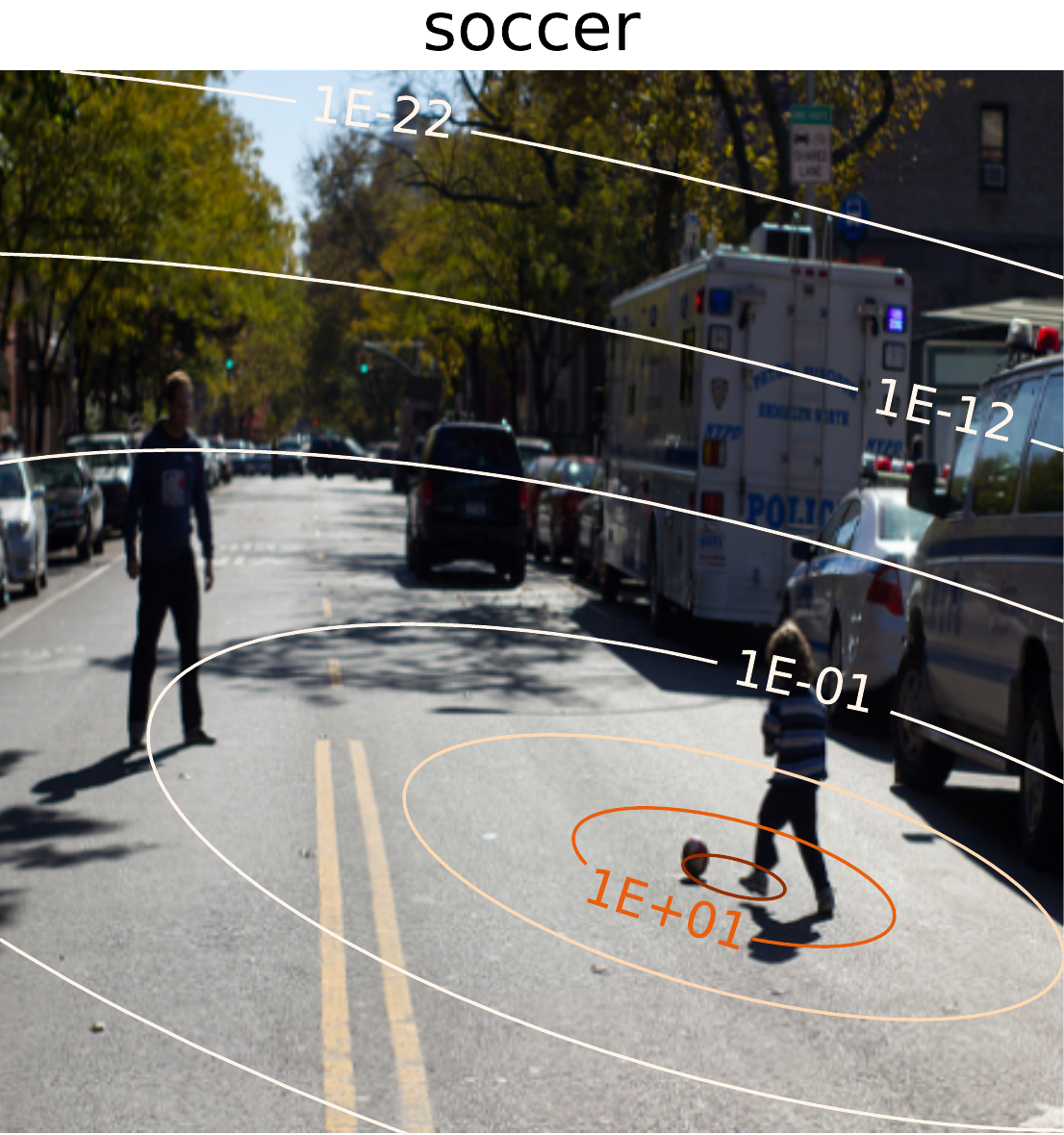}
\includegraphics[width=0.24\textwidth]{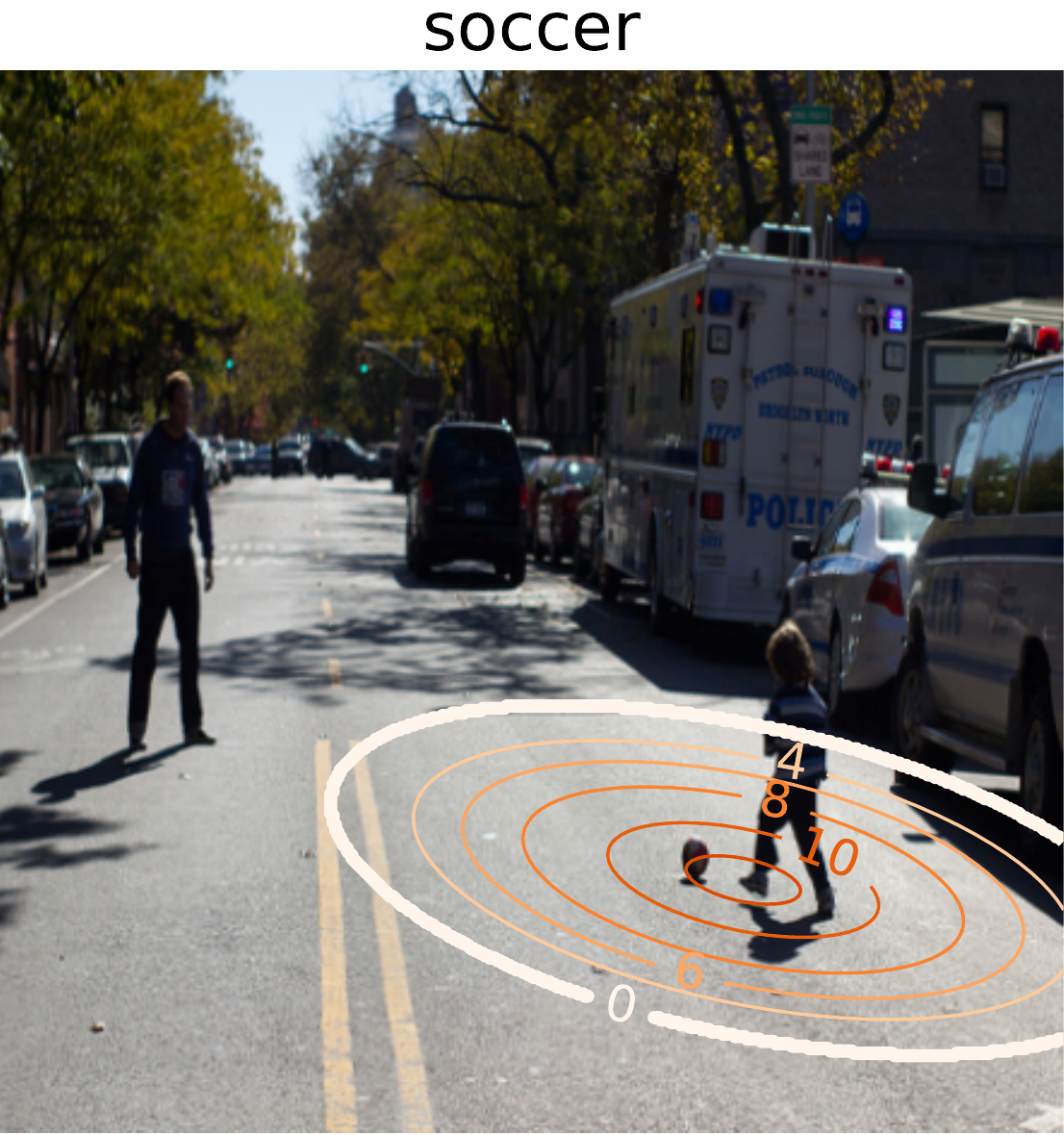}
\caption{\label{fig:examples_vqa_soccer}Attention maps for an example in VQA-v2: original image, discrete attention, continuous softmax, and continuous sparsemax.}
\end{figure*}

\bibliography{jmlr2020}

\begin{thebibliography}{99}
\providecommand{\natexlab}[1]{#1}
\providecommand{\url}[1]{\texttt{#1}}
\expandafter\ifx\csname urlstyle\endcsname\relax
  \providecommand{\doi}[1]{doi: #1}\else
  \providecommand{\doi}{doi: \begingroup \urlstyle{rm}\Url}\fi

\bibitem[Abe(2003)]{abe2003geometry}
Sumiyoshi Abe.
\newblock Geometry of escort distributions.
\newblock \emph{Physical Review E}, 68\penalty0 (3):\penalty0 031101, 2003.

\bibitem[Abe and Okamoto(2001)]{abe2001nonextensive}
Sumiyoshi Abe and Yuko Okamoto.
\newblock \emph{Nonextensive statistical mechanics and its applications},
  volume 560.
\newblock Springer Science \& Business Media, 2001.

\bibitem[Adare et~al.(2011)Adare, Afanasiev, Aidala, Ajitanand, Akiba,
  Al-Bataineh, Alexander, Aoki, Aphecetche, Armendariz,
  et~al.]{adare2011measurement}
Andrew Adare, S.~Afanasiev, C.~Aidala, N.N. Ajitanand, Yasuyuki Akiba,
  H.~Al-Bataineh, J.~Alexander, K.~Aoki, Laurent Aphecetche, R.~Armendariz,
  et~al.
\newblock Measurement of neutral mesons in p+ p collisions at s= 200 gev and
  scaling properties of hadron production.
\newblock \emph{Physical Review D}, 83\penalty0 (5):\penalty0 052004, 2011.

\bibitem[Amari(2016)]{amari2016information}
Shun-ichi Amari.
\newblock \emph{Information geometry and its applications}, volume 194.
\newblock Springer, 2016.

\bibitem[Amari and Ohara(2011)]{amari2011geometry}
Shun-ichi Amari and Atsumi Ohara.
\newblock Geometry of q-exponential family of probability distributions.
\newblock \emph{Entropy}, 13\penalty0 (6):\penalty0 1170--1185, 2011.

\bibitem[Amari et~al.(2012)Amari, Ohara, and Matsuzoe]{amari_2012}
Shun-ichi Amari, Atsumi Ohara, and Hiroshi Matsuzoe.
\newblock Geometry of deformed exponential families: Invariant, dually-flat and
  conformal geometries.
\newblock \emph{Physica A: Statistical Mechanics and its Applications},
  391\penalty0 (18):\penalty0 4308--4319, 2012.

\bibitem[Anderson et~al.(2018)Anderson, He, Buehler, Teney, Johnson, Gould, and
  Zhang]{Anderson2018}
Peter Anderson, Xiaodong He, Chris Buehler, Damien Teney, Mark Johnson, Stephen
  Gould, and Lei Zhang.
\newblock {Bottom-Up and Top-Down Attention for Image Captioning and Visual
  Question Answering}.
\newblock \emph{Proc. of CVPR}, pages 6077--6086, 2018.

\bibitem[Anevski and Soulier(2011)]{monotone}
Dragi Anevski and Philippe Soulier.
\newblock Monotone spectral density estimation.
\newblock \emph{The Annals of Statistics}, 39\penalty0 (1):\penalty0 418--438,
  2011.

\bibitem[Bahdanau et~al.(2015)Bahdanau, Cho, and Bengio]{bahdanau2014neural}
Dzmitry Bahdanau, Kyunghyun Cho, and Yoshua Bengio.
\newblock {Neural machine translation by jointly learning to align and
  translate}.
\newblock In \emph{Proc. of ICLR}, 2015.

\bibitem[Bao and Sugiyama(2021)]{bao2021fenchel}
Han Bao and Masashi Sugiyama.
\newblock {Fenchel-Young} losses with skewed entropies for class-posterior
  probability estimation.
\newblock In \emph{Proc. of AISTATS}, pages 1648--1656, 2021.

\bibitem[Barndorff-Nielsen(2014)]{barndorff2014information}
Ole Barndorff-Nielsen.
\newblock \emph{Information and Exponential Families in Statistical Theory}.
\newblock John Wiley \& Sons, 2014.

\bibitem[Bastings et~al.(2019)Bastings, Aziz, and
  Titov]{bastings2019interpretable}
Jasmijn Bastings, Wilker Aziz, and Ivan Titov.
\newblock Interpretable neural predictions with differentiable binary
  variables.
\newblock In \emph{Proceedings of the 57th Annual Meeting of the Association
  for Computational Linguistics}, pages 2963--2977, 2019.

\bibitem[Bauschke and Combettes(2011)]{Bauschke_Combettes2011}
Heinz Bauschke and Patrick Combettes.
\newblock \emph{Convex Analysis and Monotone Operator Theory in Hilbert
  Spaces}.
\newblock Springer, 2011.

\bibitem[Beck and Teboulle(2012)]{beck_2012}
Amir Beck and Marc Teboulle.
\newblock Smoothing and first order methods: A unified framework.
\newblock \emph{SIAM Journal on Optimization}, 22\penalty0 (2):\penalty0
  557--580, 2012.

\bibitem[Blondel(2019)]{blondel_2020_consistency}
Mathieu Blondel.
\newblock Structured prediction with projection oracles.
\newblock In \emph{Proc. NeurIPS}, pages 12145--12156, 2019.

\bibitem[Blondel et~al.(2020)Blondel, Martins, and
  Niculae]{blondel2020learning}
Mathieu Blondel, Andr{\'e}~F.T. Martins, and Vlad Niculae.
\newblock Learning with {Fenchel-Young} losses.
\newblock \emph{Journal of Machine Learning Research}, 21\penalty0
  (35):\penalty0 1--69, 2020.

\bibitem[Borwein and Lewis(2010)]{borwein2010convex}
Jonathan Borwein and Adrian~S Lewis.
\newblock \emph{Convex analysis and nonlinear optimization: theory and
  examples}.
\newblock Springer Science \& Business Media, 2010.

\bibitem[Bregman(1967)]{bregman1967relaxation}
Lev~M Bregman.
\newblock The relaxation method of finding the common point of convex sets and
  its application to the solution of problems in convex programming.
\newblock \emph{USSR Computational Mathematics and Mathematical Physics},
  7\penalty0 (3):\penalty0 200--217, 1967.

\bibitem[Bridle(1990)]{bridle1990probabilistic}
John~S. Bridle.
\newblock Probabilistic interpretation of feedforward classification network
  outputs, with relationships to statistical pattern recognition.
\newblock In Françoise Fogelman-Soulié and Jeanny Hérault, editors,
  \emph{Neurocomputing}, pages 227--236. Springer, 1990.

\bibitem[Brown(1986)]{brown1986fundamentals}
Lawrence~D. Brown.
\newblock \emph{Fundamentals of Statistical Exponential Families with
  Applications in Statistical Decision Theory}.
\newblock Institute of Mathematical Statistics, 1986.

\bibitem[Burlaga et~al.(2005)]{burlaga2005triangle}
L.~F. Burlaga et~al.
\newblock Triangle for the entropic index q of non-extensive statistical
  mechanics observed by voyager 1 in the distant heliosphere.
\newblock \emph{Physica A: Statistical mechanics and its applications},
  356\penalty0 (2-4):\penalty0 375--384, 2005.

\bibitem[Cambanis et~al.(1981)Cambanis, Huang, and Simons]{cambanis1981theory}
Stamatis Cambanis, Steel Huang, and Gordon Simons.
\newblock On the theory of elliptically contoured distributions.
\newblock \emph{Journal of Multivariate Analysis}, 11\penalty0 (3):\penalty0
  368--385, 1981.

\bibitem[Chen et~al.(2018)Chen, Rubanova, Bettencourt, and
  Duvenaud]{chen2018neural}
Tian~Qi Chen, Yulia Rubanova, Jesse Bettencourt, and David~K. Duvenaud.
\newblock Neural ordinary differential equations.
\newblock In \emph{Proc. of NeurIPS}, pages 6571--6583, 2018.

\bibitem[Cordonnier et~al.(2019)Cordonnier, Loukas, and
  Jaggi]{cordonnier2020relationship}
Jean-Baptiste Cordonnier, Andreas Loukas, and Martin Jaggi.
\newblock On the relationship between self-attention and convolutional layers.
\newblock In \emph{Proc. of ICLR}, 2019.

\bibitem[Correia et~al.(2020)Correia, Niculae, Aziz, and
  Martins]{correia2020efficient}
Gon{\c{c}}alo Correia, Vlad Niculae, Wilker Aziz, and Andr{\'e} Martins.
\newblock Efficient marginalization of discrete and structured latent variables
  via sparsity.
\newblock \emph{Advances in Neural Information Processing Systems},
  33:\penalty0 11789--11802, 2020.

\bibitem[Correia et~al.(2019)Correia, Niculae, and
  Martins]{correia2019adaptively}
Gon{\c{c}}alo~M Correia, Vlad Niculae, and Andr{\'e}~FT Martins.
\newblock Adaptively sparse transformers.
\newblock In \emph{Proc. of EMNLP-IJCNLP}, pages 2174--2184, 2019.

\bibitem[Cover and Thomas(2012)]{cover2012elements}
Thomas~M Cover and Joy~A Thomas.
\newblock \emph{Elements of Information Theory}.
\newblock John Wiley \& Sons, 2012.

\bibitem[Darmois(1935)]{darmois1935lois}
Georges Darmois.
\newblock Sur les lois de probabilit{\'e}a estimation exhaustive.
\newblock \emph{CR Acad. Sci. Paris}, 260\penalty0 (1265):\penalty0 85, 1935.

\bibitem[Dayan et~al.(1995)Dayan, Hinton, Neal, and Zemel]{dayan1995helmholtz}
Peter Dayan, Geoffrey~E Hinton, Radford~M Neal, and Richard~S Zemel.
\newblock The helmholtz machine.
\newblock \emph{Neural computation}, 7\penalty0 (5):\penalty0 889--904, 1995.

\bibitem[Ding and Vishwanathan(2010)]{ding2010t}
Nan Ding and S.V.N. Vishwanathan.
\newblock t-logistic regression.
\newblock In J.~D. Lafferty, C.~K.~I. Williams, J.~Shawe-Taylor, R.~S. Zemel,
  and A.~Culotta, editors, \emph{Proc. of NeurIPS}, pages 514--522. Curran
  Associates, Inc., 2010.

\bibitem[d'Onofrio(2013)]{d2013bounded}
Alberto d'Onofrio.
\newblock \emph{Bounded Noises in Physics, Biology, and Engineering}.
\newblock Springer, 2013.

\bibitem[Duchi et~al.(2018)Duchi, Khosravi, and Ruan]{duchi_2016}
John~C. Duchi, Khashayar Khosravi, and Feng Ruan.
\newblock Multiclass classification, information, divergence, and surrogate
  risk.
\newblock \emph{The Annals of Statistics}, 46\penalty0 (6B):\penalty0
  3246--3275, 2018.

\bibitem[Epanechnikov(1969)]{epanechnikov1969non}
Vassiliy~A. Epanechnikov.
\newblock Non-parametric estimation of a multivariate probability density.
\newblock \emph{Theory of Probability \& Its Applications}, 14\penalty0
  (1):\penalty0 153--158, 1969.

\bibitem[Fang et~al.(1990)Fang, Kotz, and Ng]{fang}
Kai-Tai Fang, Samuel Kotz, and Kai-Wang Ng.
\newblock \emph{Symmetric Multivariate and Related Distributions}.
\newblock Chapman and Hall, 1990.

\bibitem[Farinhas et~al.(2021)Farinhas, Martins, and
  Aguiar]{farinhas2021multimodal}
Ant{\'o}nio Farinhas, Andr{\'e} F.~T. Martins, and Pedro M.~Q. Aguiar.
\newblock Multimodal continuous visual attention mechanisms.
\newblock \emph{arXiv preprint arXiv:2104.03046}, 2021.

\bibitem[Farinhas et~al.(2022)Farinhas, Aziz, Niculae, and
  Martins]{farinhas2022sparse}
Ant{\'o}nio Farinhas, Wilker Aziz, Vlad Niculae, and Andr{\'e}~F.T. Martins.
\newblock Sparse communication via mixed distributions.
\newblock In \emph{Proc. of International Conference on Learning
  Representations}, 2022.

\bibitem[Figueiredo(2001)]{FigueiredoNIPS2001}
Mário A.~T. Figueiredo.
\newblock Adaptive sparseness using {Jeffreys} prior.
\newblock In \emph{Proc. of NeurIPS}, pages 697--704, 2001.

\bibitem[Frongillo and Reid(2014)]{frongillo_2014}
Rafael Frongillo and Mark~D Reid.
\newblock Convex foundations for generalized maxent models.
\newblock In \emph{Proc. of AIP}, 2014.

\bibitem[Funahashi and Nakamura(1993)]{funahashi1993approximation}
Ken-ichi Funahashi and Yuichi Nakamura.
\newblock Approximation of dynamical systems by continuous time recurrent
  neural networks.
\newblock \emph{Neural networks}, 6\penalty0 (6):\penalty0 801--806, 1993.

\bibitem[Gelbrich(1990)]{gelbrich1990formula}
Matthias Gelbrich.
\newblock On a formula for the l2 wasserstein metric between measures on
  euclidean and hilbert spaces.
\newblock \emph{Mathematische Nachrichten}, 147\penalty0 (1):\penalty0
  185--203, 1990.

\bibitem[Gneiting and Raftery(2007)]{gneiting_2007}
Tilmann Gneiting and Adrian~E. Raftery.
\newblock Strictly proper scoring rules, prediction, and estimation.
\newblock \emph{Journal of the American Statistical Association}, 102\penalty0
  (477):\penalty0 359--378, 2007.

\bibitem[Goyal et~al.(2019)Goyal, Khot, Agrawal, Summers-Stay, Batra, and
  Parikh]{Goyal2019}
Yash Goyal, Tejas Khot, Aishwarya Agrawal, Douglas Summers-Stay, Dhruv Batra,
  and Devi Parikh.
\newblock {Making the V in VQA Matter: Elevating the Role of Image
  Understanding in Visual Question Answering}.
\newblock \emph{International Journal of Computer Vision}, 127\penalty0
  (4):\penalty0 398--414, 2019.

\bibitem[Grasmair(2006)]{Grasmair2006}
Markus Grasmair.
\newblock The equivalence of the taut string algorithm and bv-regularization.
\newblock \emph{Journal of Mathematical Imaging and Vision}, 27:\penalty0
  59--66, 2006.

\bibitem[Gregor et~al.(2015)Gregor, Danihelka, Graves, Rezende, and
  Wierstra]{gregor2015draw}
K.~Gregor, I.~Danihelka, A.~Graves, D.~Rezende, and D.~Wierstra.
\newblock Draw: A recurrent neural network for image generation.
\newblock In \emph{Proc. of ICML}, pages 1462--1471, 2015.

\bibitem[Gr{\"u}nwald and Dawid(2004)]{grunwald_2004}
Peter~D Gr{\"u}nwald and A~Philip Dawid.
\newblock Game theory, maximum entropy, minimum discrepancy and robust bayesian
  decision theory.
\newblock \emph{Annals of Statistics}, pages 1367--1433, 2004.

\bibitem[Guerreiro and Martins(2021)]{guerreiro2021spectra}
Nuno~M Guerreiro and Andr{\'e}~FT Martins.
\newblock Spectra: Sparse structured text rationalization.
\newblock In \emph{Proceedings of the 2021 Conference on Empirical Methods in
  Natural Language Processing}, pages 6534--6550, 2021.

\bibitem[Halmos(2013)]{halmos2013measure}
Paul~R Halmos.
\newblock \emph{Measure Theory}, volume~18.
\newblock Springer, 2013.

\bibitem[Havrda and Charv{\'a}t(1967)]{havrda1967quantification}
Jan Havrda and Franti{\v{s}}ek Charv{\'a}t.
\newblock Quantification method of classification processes. concept of
  structural $ a $-entropy.
\newblock \emph{Kybernetika}, 3\penalty0 (1):\penalty0 30--35, 1967.

\bibitem[He et~al.(2016)He, Zhang, Ren, and Sun]{He2016}
Kaiming He, Xiangyu Zhang, Shaoqing Ren, and Jian Sun.
\newblock {Deep residual learning for image recognition}.
\newblock \emph{Proc. of CVPR}, pages 770--778, 2016.

\bibitem[Hinton and Zemel(1993)]{hinton1993autoencoders}
Geoffrey~E Hinton and Richard Zemel.
\newblock Autoencoders, minimum description length and helmholtz free energy.
\newblock \emph{Advances in neural information processing systems}, 6, 1993.

\bibitem[Jaynes(1957)]{jaynes1957information}
Edwin~T Jaynes.
\newblock Information theory and statistical mechanics.
\newblock \emph{Physical review}, 106\penalty0 (4):\penalty0 620, 1957.

\bibitem[Jost(2006)]{Jost2006}
Lou Jost.
\newblock Entropy and diversity.
\newblock \emph{Oikos}, 113:\penalty0 363--–375, 2006.

\bibitem[Koopman(1936)]{koopman1936distributions}
Bernard~Osgood Koopman.
\newblock On distributions admitting a sufficient statistic.
\newblock \emph{Transactions of the American Mathematical society}, 39\penalty0
  (3):\penalty0 399--409, 1936.

\bibitem[Kumar and Tsvetkov(2018)]{kumar2018mises}
Sachin Kumar and Yulia Tsvetkov.
\newblock Von mises-fisher loss for training sequence to sequence models with
  continuous outputs.
\newblock In \emph{Proc. of ICLR}, 2018.

\bibitem[LeCun et~al.(2006)LeCun, Chopra, Hadsell, Ranzato, and
  Huang]{lecun2006tutorial}
Yann LeCun, Sumit Chopra, Raia Hadsell, M~Ranzato, and F~Huang.
\newblock A tutorial on energy-based learning.
\newblock \emph{Predicting structured data}, 1\penalty0 (0), 2006.

\bibitem[Lutz(2003)]{lutz2003anomalous}
Eric Lutz.
\newblock Anomalous diffusion and tsallis statistics in an optical lattice.
\newblock \emph{Physical Review A}, 67\penalty0 (5):\penalty0 051402, 2003.

\bibitem[Martins and Astudillo(2016)]{Martins2016ICML}
Andr{\'e} F.~T. Martins and Ram{\'o}n~F. Astudillo.
\newblock {From softmax to sparsemax: A sparse model of attention and
  multi-label classification}.
\newblock In \emph{Proc. of ICML}, 2016.

\bibitem[Martins et~al.(2020)Martins, Farinhas, Treviso, Niculae, Aguiar, and
  Figueiredo]{martins2020sparse}
Andr{\'e} F.~T. Martins, Ant{\'o}nio Farinhas, Marcos Treviso, Vlad Niculae,
  Pedro M.~Q. Aguiar, and M{\'a}rio A.~T. Figueiredo.
\newblock Sparse and continuous attention mechanisms.
\newblock In \emph{Proc. of NeurIPS}, 2020.

\bibitem[Martins et~al.(2022)Martins, Marinho, and Martins]{martins2022infty}
Pedro~Henrique Martins, Zita Marinho, and Andr{\'e}~FT Martins.
\newblock $\infty$-former: Infinite memory transformer.
\newblock In \emph{Proc. of Annual Meeting of the Association for Computational
  Linguistics}, 2022.

\bibitem[Matsuzoe and Ohara(2012)]{matsuzoe2012geometry}
Hiroshi Matsuzoe and Atsumi Ohara.
\newblock Geometry for q-exponential families.
\newblock In \emph{Recent Progress in Differential Geometry and its Related
  Fields}, pages 55--71. World Scientific, 2012.

\bibitem[Mensch and Blondel(2018)]{mensch2018differentiable}
Arthur Mensch and Mathieu Blondel.
\newblock Differentiable dynamic programming for structured prediction and
  attention.
\newblock In \emph{Proc. of ICML}, 2018.

\bibitem[Mensch et~al.(2019)Mensch, Blondel, and Peyr{\'e}]{mensch_2019}
Arthur Mensch, Mathieu Blondel, and Gabriel Peyr{\'e}.
\newblock Geometric losses for distributional learning.
\newblock In \emph{Proc. ICML}, 2019.

\bibitem[Moreau(1965)]{moreau1965proximite}
Jean-Jacques Moreau.
\newblock Proximit{\'e} et dualit{\'e} dans un espace hilbertien.
\newblock \emph{Bulletin de la Soci{\'e}t{\'e} math{\'e}matique de France},
  93:\penalty0 273--299, 1965.

\bibitem[Naudts(2009)]{naudts2009q}
Jan Naudts.
\newblock The q-exponential family in statistical physics.
\newblock \emph{Central European Journal of Physics}, 7\penalty0 (3):\penalty0
  405--413, 2009.

\bibitem[Nesterov(2005)]{nesterov_smooth}
Yurii Nesterov.
\newblock Smooth minimization of non-smooth functions.
\newblock \emph{Mathematical Programming}, 103\penalty0 (1):\penalty0 127--152,
  2005.

\bibitem[Niculae and Blondel(2017)]{fusedmax}
Vlad Niculae and Mathieu Blondel.
\newblock Sparse and structured attention mechanisms.
\newblock In \emph{Proc. NeurIPS}. 2017.

\bibitem[Nock and Nielsen(2009)]{nock_2009}
Richard Nock and Frank Nielsen.
\newblock Bregman divergences and surrogates for learning.
\newblock \emph{IEEE Transactions on Pattern Analysis and Machine
  Intelligence}, 31\penalty0 (11):\penalty0 2048--2059, 2009.

\bibitem[Nowak-Vila et~al.(2020)Nowak-Vila, Bach, and
  Rudi]{nowak2020consistent}
Alex Nowak-Vila, Francis Bach, and Alessandro Rudi.
\newblock Consistent structured prediction with max-min margin markov networks.
\newblock In \emph{Proc. of ICML}, 2020.

\bibitem[Overgaard(2019)]{overgaard2019taut}
Niels~Chr Overgaard.
\newblock On the taut string interpretation and other properties of the
  rudin--osher--fatemi model in one dimension.
\newblock \emph{Journal of Mathematical Imaging and Vision}, 61\penalty0
  (9):\penalty0 1276--1300, 2019.

\bibitem[Owen and Rabinovitch(1983)]{owen1983class}
Joel Owen and Ramon Rabinovitch.
\newblock On the class of elliptical distributions and their applications to
  the theory of portfolio choice.
\newblock \emph{The Journal of Finance}, 38\penalty0 (3):\penalty0 745--752,
  1983.

\bibitem[Pennington et~al.(2014)Pennington, Socher, and
  Manning]{pennington2014glove}
Jeffrey Pennington, Richard Socher, and Christopher~D. Manning.
\newblock Glove: Global vectors for word representation.
\newblock In \emph{Proc. of EMNLP}, pages 1532--1543, 2014.

\bibitem[Peters et~al.(2019)Peters, Niculae, and Martins]{peters2019sparse}
Ben Peters, Vlad Niculae, and Andr{\'e}~F.T. Martins.
\newblock Sparse sequence-to-sequence models.
\newblock In \emph{Proc. of ACL}, 2019.

\bibitem[Peyr{\'e} and Cuturi(2019)]{cot}
Gabriel Peyr{\'e} and Marco Cuturi.
\newblock Computational optimal transport: With applications to data science.
\newblock \emph{Foundations and Trends in Machine Learning}, 11\penalty0
  (5-6):\penalty0 355--607, 2019.

\bibitem[Pickup et~al.(2009)Pickup, Cywinski, Pappas, Farago, and
  Fouquet]{pickup2009generalized}
R.M. Pickup, R.~Cywinski, C.~Pappas, B.~Farago, and P.~Fouquet.
\newblock Generalized spin-glass relaxation.
\newblock \emph{Physical review letters}, 102\penalty0 (9):\penalty0 097202,
  2009.

\bibitem[Pitman(1936)]{pitman1936sufficient}
Edwin James~George Pitman.
\newblock Sufficient statistics and intrinsic accuracy.
\newblock In \emph{Mathematical Proceedings of the Cambridge Philosophical
  Society}, volume~32, pages 567--579. Cambridge University Press, 1936.

\bibitem[Rao(1982)]{Rao1982}
R.A. Rao.
\newblock Gini-{Simpson} index of diversity: a characterization,
  generalization, and applications.
\newblock \emph{Utilitas Mathematics}, 21:\penalty0 273--282, 1982.

\bibitem[Ravanelli et~al.(2021)Ravanelli, Parcollet, Plantinga, Rouhe, Cornell,
  Lugosch, Subakan, Dawalatabad, Heba, Zhong, Chou, Yeh, Fu, Liao, Rastorgueva,
  Grondin, Aris, Na, Gao, Mori, and Bengio]{speechbrain}
Mirco Ravanelli, Titouan Parcollet, Peter Plantinga, Aku Rouhe, Samuele
  Cornell, Loren Lugosch, Cem Subakan, Nauman Dawalatabad, Abdelwahab Heba,
  Jianyuan Zhong, Ju-Chieh Chou, Sung-Lin Yeh, Szu-Wei Fu, Chien-Feng Liao,
  Elena Rastorgueva, François Grondin, William Aris, Hwidong Na, Yan Gao,
  Renato~De Mori, and Yoshua Bengio.
\newblock {SpeechBrain}: A general-purpose speech toolkit, 2021.
\newblock arXiv:2106.04624.

\bibitem[Reid and Williamson(2010)]{reid_composite_binary}
Mark~D. Reid and Robert~C. Williamson.
\newblock Composite binary losses.
\newblock \emph{Journal of Machine Learning Research}, 11:\penalty0 2387--2422,
  2010.

\bibitem[Rice(2006)]{rice}
John~A Rice.
\newblock \emph{Mathematical Statistics and Data Analysis}.
\newblock Cengage Learning, 2006.

\bibitem[Rubanova et~al.(2019)Rubanova, Chen, and Duvenaud]{rubanova2019latent}
Yulia Rubanova, Tian~Qi Chen, and David~K. Duvenaud.
\newblock Latent ordinary differential equations for irregularly-sampled time
  series.
\newblock In \emph{Proc. of NeurIPS}, pages 5321--5331, 2019.

\bibitem[Rudin et~al.(1992)Rudin, Osher, and Fatemi]{rof1992}
L.~Rudin, S.~Osher, and E.~Fatemi.
\newblock Nonlinear total variation based noise removal algorithms.
\newblock \emph{Physica D: Nonlinear Phenomena}, 60:\penalty0 259--268, 1992.

\bibitem[Russakovsky et~al.(2015)Russakovsky, Deng, Su, Krause, Satheesh, Ma,
  Huang, Karpathy, Khosla, Bernstein, Berg, and Fei-Fei]{Russakovsky2015}
Olga Russakovsky, Jia Deng, Hao Su, Jonathan Krause, Sanjeev Satheesh, Sean Ma,
  Zhiheng Huang, Andrej Karpathy, Aditya Khosla, Michael Bernstein,
  Alexander~C. Berg, and Li~Fei-Fei.
\newblock {ImageNet Large Scale Visual Recognition Challenge}.
\newblock \emph{International Journal of Computer Vision}, 115\penalty0
  (3):\penalty0 211--252, 2015.

\bibitem[Salamon et~al.(2014)Salamon, Jacoby, and
  Bello]{Salamon:UrbanSound:ACMMM:14}
J.~Salamon, C.~Jacoby, and J.~P. Bello.
\newblock A dataset and taxonomy for urban sound research.
\newblock In \emph{ACL International Conference on Multimedia}, pages
  1041--1044, 2014.

\bibitem[Sch{\"u}tt et~al.(2017)Sch{\"u}tt, Kindermans, Felix, Chmiela,
  Tkatchenko, and M{\"u}ller]{schutt2017schnet}
Kristof Sch{\"u}tt, Pieter-Jan Kindermans, Huziel Enoc~Sauceda Felix, Stefan
  Chmiela, Alexandre Tkatchenko, and Klaus-Robert M{\"u}ller.
\newblock Schnet: A continuous-filter convolutional neural network for modeling
  quantum interactions.
\newblock In \emph{Proc. of NeurIPS}, pages 991--1001, 2017.

\bibitem[Seabold and Perktold(2010)]{statsmodels}
Skipper Seabold and Josef Perktold.
\newblock statsmodels: Econometric and statistical modeling with python.
\newblock In \emph{9th Python in Science Conference}, 2010.

\bibitem[Sears(2008)]{sears2010generalized}
Timothy Sears.
\newblock \emph{Generalized Maximum Entropy, Convexity and Machine Learning}.
\newblock PhD thesis, The Australian National University, 2008.

\bibitem[Silverman(1986)]{silverman1986density}
Bernard~W Silverman.
\newblock \emph{Density estimation for statistics and data analysis},
  volume~26.
\newblock CRC press, 1986.

\bibitem[Sukhbaatar et~al.(2015)Sukhbaatar, Weston, Fergus,
  et~al.]{sukhbaatar2015end}
Sainbayar Sukhbaatar, Jason Weston, Rob Fergus, et~al.
\newblock End-to-end memory networks.
\newblock In \emph{Advances in Neural Information Processing Systems}, pages
  2440--2448, 2015.

\bibitem[Taskar et~al.(2005)Taskar, Lacoste-Julien, and
  Jordan]{taskar2005structured}
Ben Taskar, Simon Lacoste-Julien, and Michael Jordan.
\newblock Structured prediction via the extragradient method.
\newblock \emph{Advances in neural information processing systems}, 18, 2005.

\bibitem[Tibshirani et~al.(2005)Tibshirani, Saunders, Rosset, Zhu, and
  Knight]{tibshirani2005sparsity}
Robert Tibshirani, Michael Saunders, Saharon Rosset, Ji~Zhu, and Keith Knight.
\newblock Sparsity and smoothness via the fused lasso.
\newblock \emph{Journal of the Royal Statistical Society: Series B (Statistical
  Methodology)}, 67\penalty0 (1):\penalty0 91--108, 2005.

\bibitem[Tipping(2001)]{TippingJMLR2001}
M.~Tipping.
\newblock Sparse {Bayesian} learning and the relevance vector machine.
\newblock \emph{Journal of Machine Learning Research}, 1:\penalty0 211--244,
  2001.

\bibitem[Tsallis(1988)]{Tsallis1988}
Constantino Tsallis.
\newblock {Possible generalization of Boltzmann-Gibbs statistics}.
\newblock \emph{Journal of Statistical Physics}, 52:\penalty0 479--487, 1988.

\bibitem[Vaswani et~al.(2017)Vaswani, Shazeer, Parmar, Uszkoreit, Jones, Gomez,
  Kaiser, and Polosukhin]{vaswani2017attention}
Ashish Vaswani, Noam Shazeer, Niki Parmar, Jakob Uszkoreit, Llion Jones,
  Aidan~N Gomez, {\L}ukasz Kaiser, and Illia Polosukhin.
\newblock {Attention is all you need}.
\newblock In \emph{Proc. of NeurIPS}, 2017.

\bibitem[Wainwright and Jordan(2008)]{wainwright_2008}
Martin~J Wainwright and Michael~I Jordan.
\newblock {Graphical models, exponential families, and variational inference.}
\newblock \emph{Foundations and Trends in Machine Learning}, 1\penalty0
  (1--2):\penalty0 1--305, 2008.

\bibitem[Wang et~al.(2018)Wang, Suo, Ma, Pokrovsky, and Urtasun]{wang2018deep}
Shenlong Wang, Simon Suo, Wei-Chiu Ma, Andrei Pokrovsky, and Raquel Urtasun.
\newblock Deep parametric continuous convolutional neural networks.
\newblock In \emph{Proc. of CVPR}, pages 2589--2597, 2018.

\bibitem[Williamson et~al.(2016)Williamson, Vernet, and Reid]{vernet_2016}
Robert~C. Williamson, Elodie Vernet, and Mark~D. Reid.
\newblock Composite multiclass losses.
\newblock \emph{Journal of Machine Learning Research}, 2016.

\bibitem[You et~al.(2020)You, Sun, and Iyyer]{you-etal-2020-hard}
W.~You, S.~Sun, and M.~Iyyer.
\newblock Hard-coded {G}aussian attention for neural machine translation.
\newblock In \emph{Proc. of ACL}, 2020.

\bibitem[Yu(2013)]{yudecomp}
Yao-Liang Yu.
\newblock On decomposing the proximal map.
\newblock In \emph{Proc. of NeurIPS}. 2013.

\bibitem[Yu et~al.(2019)Yu, Yu, Cui, Tao, and Tian]{Yu2019}
Zhou Yu, Jun Yu, Yuhao Cui, Dacheng Tao, and Qi~Tian.
\newblock {Deep modular co-attention networks for visual question answering}.
\newblock \emph{Proc. of CVPR}, pages 6274--6283, 2019.

\end{thebibliography}

\end{document}